\documentclass[a4paper, 12pt, twoside]{report}

\usepackage{amsmath,amsfonts,bm}

\def\eqref#1{equation~\ref{#1}}

\def\1{\bm{1}}

\DeclareMathAlphabet{\mathsfit}{\encodingdefault}{\sfdefault}{m}{sl}
\SetMathAlphabet{\mathsfit}{bold}{\encodingdefault}{\sfdefault}{bx}{n}

\newcommand{\mparam}{\bm{\theta}}	
\newcommand{\vparam}{\bm{\phi}}	
\newcommand{\weight}{\mathbf{W}} 
\newcommand{\data}{\mathcal{D}} 
\newcommand{\latent}{\bm{z}} 
\newcommand{\hparam}{\bm{\varphi}} 

\newcommand{\x}{\bm{x}}

\newcommand{\y}{\mathbf{y}}
\newcommand{\z}{\bm{z}}
\newcommand{\bv}{\bm{v}}

\DeclareMathOperator*{\argmax}{arg\,max}
\DeclareMathOperator*{\argmin}{arg\,min}

\usepackage{url}
\usepackage[pdftex]{graphicx} 
\usepackage{caption}
\usepackage{subcaption}
\usepackage{enumitem}
\usepackage{soul}
\usepackage[normalem]{ulem}
\usepackage{cancel}
\usepackage{float}
\usepackage{multirow}
\usepackage{adjustbox}
\usepackage{wrapfig}
\usepackage{natbib}
\usepackage{hyperref}
\usepackage{algpseudocode}

\usepackage[a4paper,top=2.5cm,bottom=2.5cm,left=3cm,right=3cm,marginparwidth=1.75cm]{geometry}
\usepackage{setspace}
\usepackage{amsmath,amsfonts,bm,amsthm}
\usepackage{booktabs}
\usepackage{tikz}
\usepackage{parskip}
\usepackage{indentfirst}

\usepackage{amsmath,amsfonts,amssymb,amsthm}
\usepackage{pgfplots}
\usepackage{caption}
\usepackage{subcaption}
\usepackage{graphicx}
\usepackage{adjustbox}
\usepackage{multirow}
\usepackage{enumitem}
\usepackage{comment}
\usepackage{wrapfig}
\usepackage{pdflscape}  

\usepackage{cancel}

\usepackage{url}            
\usepackage{booktabs}      
\usepackage{amsfonts}       
\usepackage{nicefrac}       
\usepackage{microtype}      
\usepackage{xcolor}  

\usepackage{tikz,lipsum,lmodern}
\usepackage[most]{tcolorbox}
\usepackage{graphicx}

\makeatletter
\newcommand*{\resetMathstrut}{%
  \setbox\z@\hbox{(}%
  \ht\Mathstrutbox@\ht\z@
  \dp\Mathstrutbox@\dp\z@
}
\makeatother

\newenvironment{talign}
{\align}
{\endalign}

\hypersetup{
colorlinks=true,
allcolors=[rgb]{0.24,0,0.80}
}
\usepackage[capitalise]{cleveref}

\usepackage{url}
\usepackage{amsmath}
\usepackage{amssymb}
\usepackage{mathtools}
\usepackage{amsthm}
\usepackage{ulem}
\usepackage{caption}
\usepackage{subcaption}
\usepackage{changepage}
\usepackage[bottom]{footmisc}

\newcommand{\bK}{\mathbf{K}}

\newcommand{\bS}{\mathbf{S}}

\newcommand{\bX}{\mathbf{X}}
\newcommand{\bu}{\mathbf{u}}

\newcommand{\bZ}{\mathbf{Z}}

\newcommand{\by}{\mathbf{y}}

\theoremstyle{plain}
\newtheorem{theorem}{Theorem}[section]

\theoremstyle{definition}

\newtheorem{remark}[theorem]{Remark}
\newcommand{\ConditionallyIndependent}[3]{#1 \perp\kern-5pt \perp #2 \mid #3}

\title{Probabilistic Learning and Generation in Deep Sequence Models}
\author{Wenlong Chen}
\date{December 18, 2025}

\begin{document}
\pagenumbering{arabic}
\begin{titlepage}
    \newcommand{\HRule}{\rule{\linewidth}{0.5mm}} 
    
    
    \includegraphics[width=7cm]{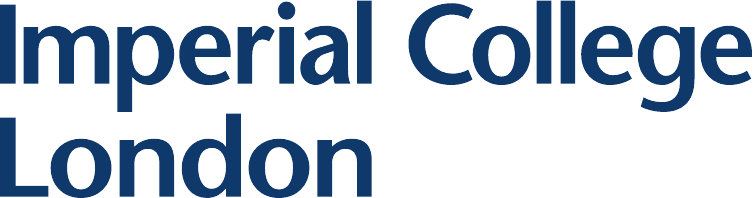}\\[2cm] 
    
    
    \center 
    
    
    \textsc{\Large Imperial College London}\\[0.2cm] 
    \textsc{\large Department of Computing}\\[0.8cm] 
    
    \makeatletter
    \vfill
    { \huge \@title}\\[1.0cm] 
    { \Large \@author} 
    \vfill
    
    
    {\large \@date}\\[2cm] 
    \makeatother
    
    This thesis is submitted for the degree of \textit{Doctor of Philosophy}
    
\end{titlepage}
\chapter*{Declaration of Originality}
    I hereby declare that the work in this thesis is my own. The work of others has been appropriately referenced. A full list of references is given in the bibliography.

\chapter*{Copyright}
    The copyright of this thesis rests with the author and its contents are made available under a Creative Commons Attribution Non-Commercial Share-Alike 4.0 International (CC BY-NC-SA 4.0) License. You may copy and redistribute the material in any medium or format. You may also remix, transform or build upon the material. In doing so, you must give appropriate credit to the author, provide a link to the license and indicate if any changes were made. If you remix, transform or build upon this material, you must redistribute your contributions under the same license. You may not use the material for commercial purposes.
    
    Please seek permission from the copyright holder for uses of this work that are not included in the license mentioned above.

\chapter*{Acknowledgements}
    First and foremost, I would like to thank my supervisor, Yingzhen Li, for her unwavering support throughout my PhD. Yingzhen, in many ways, is the best supervisor that I can ever imagine. In addition to her enthusiasm for research, she is also such a knowledgeable, patient, modest, and fun person to work with. I am grateful for the freedom and encouragement that Yingzhen gave me to pursue my own ideas and explore different research directions. Her deep insights and candid feedback always inspired me to think deeper and helped me grow as a better researcher. Thanks to Yingzhen, these four years become one of the most memorable experiences in my life.

    I am very lucky to have worked with excellent collaborators during my PhD, from whom I learned a great deal: Yegor Klochkov, Naoki Kiyohara, Harrison Zhu, Wenlin Chen, Lapo Rastrelli, Shavindra Jayasekera, Jacob Si, Filippo Valdettaro, Yohan Jung, Hyungi Lee, Thomas Möllenhoff, Mohammad Emtiyaz Khan, Bolian Li, and Ruqi Zhang. I especially thank Yegor Klochkov and Mohammad Emtiyaz Khan for offering incredible internship experiences at Bytedance AI lab and RIKEN AIP. Yegor introduced me to the world of machine learning fairness and optimization, and Emti encouraged me to think outside the box.
    
    I am grateful to everyone I shared interesting conversations with during my PhD: Carles Balsells Rodas, Zijing Ou, Xavier Sumba Toral, Jiaming Zhang, Wen Wu, Rui Xia, Jean-Francois Ton, Ruocheng Guo, Muhammad Faaiz Taufiq, Zonghao Chen, Yang Liu, Junyu Xuan, Xue Yan, Theodore Papamarkou, Andre Freitas, Danilo S. Carvalho, Hugo Monzón Maldonado, Marco Miani, Luke Ong, and many others. I also want to express my heartfelt thanks to Yingxue Yang, with whom I have shared more than twenty years of friendship, for all the jokes and memorable moments we have shared together in my leisure life.

    Many thanks to Felipe Tobar and Arno Solin for examining my thesis. I am very grateful for their time and effort in reading my thesis and providing feedback.
    
    Finally, I owe a lot to my family and loved ones. I would like to express my deepest gratitude to my parents for their unconditional support throughout this journey.

\chapter*{Abstract}
    Deep sequence models have achieved profound success across a wide range of data modalities. Despite exceptional predictive performance, the main concern of their deployment centers around the lack of uncertainty awareness. In contrast, probabilistic models quantify the uncertainty associated with unobserved variables with rules of probability. Notably, Bayesian methods leverage Bayes' rule to express our belief of unobserved variables given some observed variables in a principled way. Since exact Bayesian inference is computationally infeasible at scale, approximate inference is required in practice. Two major bottlenecks of Bayesian methods, especially when applied in deep neural networks, are prior specification and approximation quality. In Chapter \ref{cha:sgpa} and \ref{cha:hsgp}, we investigate how the architectures of deep sequence models themselves can be informative for specifying priors or choosing approximation methods in probabilistic models. We first develop an approximate Bayesian inference method tailored to the Transformer architecture based on the similarity between attention mechanism and sparse Gaussian process. Next, we exploit the long-range memory preservation capability of HiPPOs (High-order Polynomial Projection Operators) to construct an interdomain inducing point for Gaussian process, which successfully memorizes the history in online or continual learning. In addition to the progress of deep sequence models in predictive tasks, sequential generative models consisting of a sequence of latent variables (e.g., diffusion models), are popularized in the domain of deep generative models. Inspired by the explicit self-supervised signals for these latent variables in diffusion models, in Chapter \ref{cha:pseudovid}, we explore the possibility of improving other deep generative models with self-supervised signals for their latent states, and investigate desired probabilistic structures over the sequence of latent states in sequential generation. Overall, this thesis leverages inductive biases in deep sequence models to design probabilistic inference or structure, which bridges the gap between deep sequence models and probabilistic models, leading to mutually reinforced improvement.
\tableofcontents
\listoffigures
\listoftables

\chapter{Introduction}
While deep learning techniques can be traced back to the last century \citep{mcCulloch1943logical, rosenblatt1958perceptron}, their revival began in the early 2010s, with the advent of a series of breakthroughs \citep{krizhevsky2012imagenet, graves2013speech, brown2020language}, proving their effectiveness in modeling complex functions and scalability to large datasets \citep{lecun2015dl}. Given that input features from many data modalities are in the form of sequences (e.g., natural language, audio, and video) or can be represented by sequences (e.g., an image can be viewed as a sequence of patches \citep{dosovitskiy2020image}), recent progress in deep learning is largely driven by advances in deep sequence models \citep{vaswani2017attention, gu_hippo_2020, gu_s4_2022, gu_mamba_2023}, which tend to become a unifying modeling paradigm across different modalities and keep shattering our expectation of their upper limit \citep{jumper2021highly, openai2024gpt4technicalreport, zeni2024mattergengenerativemodelinorganic}.

One of the key components of deep sequence models is a few carefully designed network architectures tailored to sequence modeling. The success of Transformers \citep{vaswani2017attention} based on attention architecture has been observed in a wide range of fields including computer vision \citep{dosovitskiy2020image}, natural language processing \citep{devlin2018bert, brown2020language, openai2024gpt4technicalreport}, speech recognition \citep{openai2024gpt4ocard}, and applications in scientific domains \citep{jumper2021highly, zeni2024mattergengenerativemodelinorganic}. In addition, deep state space models (SSMs; \citep{gu_hippo_2020, gu_s4_2022, gu_mamba_2023}), initially designed to efficiently model long-term dependencies in sequences with extensive context, offer a competitive alternative to Transformers and have been deployed in a growing number of applications \citep{zhu_vision_2024, li2024videomamba, quan2024multichannel}. 

In the deep generative modeling domain, sequential generative models \citep{sonderby2016ladder, ho2020denoising} extend single latent variable models by incorporating a sequence of latent variables which improve the model expressiveness, and instead of generating data in the target domain from noise in just one shot, these models gradually transform noise into target data with multiple steps. Among them, diffusion models \citep{sohl2015deep,ho2020denoising,song2020score} and their variants \citep{kingma2021variational,nichol2021improved,song2021denoising,rissanen2022generative, bansal2023cold, hoogeboom2023blurring} have shown impressive performance in generating photorealistic images, and these models have also been adapted to other data modalities in addition to vision \citep{benita2023diffar, campbell2023trans, lou2024discrete}. 

Although deep sequence models have become foundational building blocks in modern machine learning systems, it is important to be aware that they are still not the ultimate answer to all applications. For example, these models are often unaware of the limits of their own knowledge and provide unreliable uncertainty in their predictions, which prevents their deployment in safety-critical applications \citep{xiong2024can, wen2024mitigating}. 

On the other hand, probabilistic methods, based on the law of probability and Bayesian statistics \citep{bayes1763essay}, provide a principled framework to reason about uncertainty and study the structure of random variables. In this framework, we start with a prior distribution representing our initial belief of the unknown, and the knowledge from the observed data is used to update the prior into a posterior distribution, typically with uncertainty reduction, via Bayes' rule. Furthermore, the framework is well suited for sequential learning or decision making under uncertainty. The Bayes' update can be carried out whenever new data arrives by treating the old posterior as the new prior. 

Clever integration of the complementary strength of probabilistic methods into deep sequence models may address some fundamental problems associated with them, including unreliable uncertainty estimate. In the past decade, many efforts have been made to exploit probabilistic methods to improve the reliability of deep learning, which gives rise to a vital research area, namely Bayesian deep learning \citep{wilson2022bayesian,arbel2023primer,papamarkou2024positionbayesiandeeplearning, chen2025bayesian}. However, two main bottlenecks of Bayesian statistics are amplified when applied in deep learning. First, meaningful prior and model specification is challenging, especially for Bayesian neural networks. Since deep neural networks are highly nonlinear and overparameterized, the effect of prior specified for their weights on the functions induced is not straightforward \citep{sun2019functional, ma2021functional}. Second, in practice, approximate Bayesian inference \citep{jordan1999vi, brooks2011mcmc, li2018approx} is indispensable since the exact posterior requires intractable integration. This raises the question about choices and structure of the approximate posterior and an ideal approximation needs to be both accurate and scalable.

This thesis aims to advance Bayesian deep learning specifically for deep sequence models, which now become a unifying modeling tool for widespread applications. Instead of thinking of solutions purely from probabilistic perspective, such as handcrafting meaningful priors on countless network weights, we propose to build upon the enormous progress of deep sequence models by exploiting the inductive biases within them, such as the network architectures. We leverage these inductive biases to construct probabilistic models and approximation techniques tailored to these models. This not only enables reinforced improvements by combining both deep learning and probabilistic modeling, but also ensures that the methods developed here leave the characteristic mechanisms of deep sequence models intact.

\section{Thesis Outline and Contributions}
The remaining chapters in this thesis are organized as follows.

\begin{itemize}
\item Chapter~\ref{cha:back} introduces the technical background for the thesis. We review key concepts and ideas in deep sequence models, probabilistic modeling, and Bayesian deep learning, which form the basis of the methods proposed in this thesis. The chapter is concluded with two core applications of deep probabilistic modeling, namely uncertainty quantification and deep generative models.

\item Chapter~\ref{cha:sgpa} aims at calibrating the uncertainty of Transformer models \citep{vaswani2017attention} by identifying connection between dot-product attention and the posterior mean of sparse Gaussian process (SGP; \citep{snelson2005sgp}). Based on this key connection, we propose Sparse Gaussian Process attention (SGPA), which replaces the scaled dot-product operation in Transformer with a valid symmetric kernel and uses sparse Gaussian processes techniques to approximate the posterior processes in the output space of multi-head attention blocks (MHAs) directly. Empirically, on a suite of prediction tasks on text, images and graphs, SGPA-based Transformers achieve competitive predictive accuracy, while noticeably improving both in-distribution uncertainty calibration and out-of-distribution robustness and detection. This chapter is based on \citet{chen2023calibrating}

\item Chapter~\ref{cha:hsgp} proposes a novel online Gaussian process (GP) model \citep{bui_streaming_2017} to capture long-term memory in sequentially observed data in an online learning setting. Our model, Online HiPPO Sparse Variational Gaussian Process (OHSVGP), leverages the long-range memory modeling capabilities of HiPPO framework \citep{gu_hippo_2020} by interpreting the HiPPO time-varying orthogonal projections as inducing variables capable of memorizing the process history. We show that the HiPPO framework fits naturally into the interdomain GP framework \citep{van_der_wilk_framework_2020} and demonstrate that the kernel matrices can also be updated online in a recurrence form based on the ODE evolution of HiPPO. We evaluate OHSVGP with online prediction for 1D time series, continual learning in discriminative GP model for data with multidimensional inputs, and deep generative modeling with sparse Gaussian process variational autoencoder \citep{jazbec_scalable_2021}, showing that it outperforms existing online GP methods in terms of predictive performance, long-term memory preservation, and computational efficiency. This chapter is based on \citet{chen2025recurrent}

\item Chapter~\ref{cha:pseudovid} is motivated by the superior performance of diffusion models \citep{sohl2015deep,ho2020denoising,song2020score} against other sequential generative models, such as hierarchical variational autoencoders (HVAEs; \citep{sonderby2016ladder,maaloe2019biva,vahdat2020nvae}). We hypothesize that this can be partly attributed to the additional self-supervision information for their intermediate latent states provided by corrupted images, which along with the original image form a pseudo video. Based on this hypothesis, we explore the possibility of improving other types of generative models with such pseudo videos by first extending an image generative model to its video generative model counterpart, and then train the video generative model on pseudo videos constructed by applying data augmentation to the original images. Furthermore, we analyze the potential issues of first-order Markov data augmentation methods, as typically used in diffusion models, in sequential generation from a probabilistic perspective, and propose to use data augmentation with more expressive probabilistic structure to construct more useful information in pseudo videos. We empirically verify the effectiveness of additional self-supervised information from pseudo videos with experiments on the CIFAR10 and CelebA datasets. This chapter is based on \citet{chen2025your}

\item Chapter~\ref{cha:conclusion} provides an outlook of our key findings and contributions made in this thesis. Moreover, we identify several open challenges and plausible avenues for future work.
\end{itemize}

\section{List of Publications}
This section provides a full list of publications that I co-authored during my PhD. Titles are boldfaced for papers whose content is included in this thesis and I also give a brief description of my contribution to each of these works. Some of the material, including ideas, concepts, texts, figures, tables presented in this thesis have previously appeared in these works. The asterisk superscript ($\ast$) indicates co-first authorship with equal contribution.

\paragraph{Peer-reviewed Conference Publications.}
\begin{enumerate}
    \item \citep{chen2025recurrent} \underline{Wenlong Chen}$^\ast$, Naoki Kiyohara$^\ast$, Harrison Bo Hua Zhu$^\ast$, Jacob Curran-Sebastian, Samir Bhatt, and Yingzhen Li. \textbf{Recurrent Memory for Online Interdomain Gaussian Processes}. In \textit{Advances in Neural Processing Systems (NeurIPS)}, 2025.\\
    My contribution: the main idea was developed by me while the last author provided useful suggestions. The three co-first authors ($^\ast$) all contributed to the code implementation, experimentation, and paper writing under the supervision of the last author. Moreover, I performed all the derivations and helped orchestrate other authors’ contributions.
    
    \item \citep{jayasekera2025variational} I. Shavindra Jayasekera$^\ast$, Jacob Si$^\ast$, Filippo Valdettaro, \underline{Wenlong Chen}, A. Aldo Faisal, and Yingzhen Li. Variational Uncertainty Decomposition for In-Context Learning. In \textit{Advances in Neural Processing Systems (NeurIPS)}, 2025.
    
    \item \citep{jung2025compact2} Yohan Jung, Hyungi Lee, \underline{Wenlong Chen}, Thomas Möllenhoff, Yingzhen Li, Juho Lee, Mohammad Emtiyaz Khan. Compact Memory for Continual Logistic Regression. In \textit{Advances in Neural Processing Systems (NeurIPS)}, 2025.

    \item \citep{chen2024post} \underline{Wenlong Chen}$^\ast$, Yegor Klochkov$^\ast$, and Yang Liu. Post-hoc Bias Scoring is Optimal for Fair Classification. In \textit{International Conference on Learning Representations (ICLR)}, 2024. Awarded spotlight presentation.
    
    \item \citep{chen2023calibrating} \underline{Wenlong Chen} and Yingzhen Li. \textbf{Calibrating Transformers via Sparse Gaussian Processes}. In \textit{International Conference on Learning Representations (ICLR)}, 2023.\\
    My contribution: the main idea was developed by me while the last author provided useful suggestions. The manuscript was mainly written by me while the last author helped polish it. Moreover, I wrote all the code and performed all the experiments for this work.
\end{enumerate}

\paragraph{Peer-reviewed Workshop Publications.}
\begin{enumerate}
    \item \citep{chen2025your} \underline{Wenlong Chen}$^\ast$, Wenlin Chen$^\ast$, Lapo Rastrelli, and Yingzhen Li. \textbf{Your Image is Secretly the Last Frame of a Pseudo Video}. In \textit{Deep Generative Model in Machine Learning: Theory, Principle and Efficacy (DeLTa) Workshop at ICLR}, 2025.\\
    My contribution: the main idea was developed by me while the last author provided useful suggestions. Both co-first authors ($\ast$) contributed to code implementation, experimentation, and paper writing under the supervision of the last author. The third author helped with the experiments related to video diffusion models. I also helped orchestrate other authors’ contributions.

    \item \citep{jung2025compact} Yohan Jung$^\ast$, Hyungi Lee$^\ast$, \underline{Wenlong Chen}$^\ast$, Thomas Möllenhoff, Yingzhen Li, Juho Lee, and Mohammad Emtiyaz Khan. Compact Memory for K-prior Based Continual Learning. In \textit{Symposium on Advances in Approximate Bayesian Inference (AABI)}, 2025.
    
\end{enumerate}

\paragraph{Preprints.}
\begin{enumerate}
     \item \citep{chen2025bayesian} \underline{Wenlong Chen}$^\ast$, Bolian Li$^\ast$, Ruqi Zhang, and Yingzhen Li. Bayesian Computation in Deep Learning. In \textit{arXiv preprint arXiv:2502.18300}, 2025 (to appear as a chapter in Handbook of Markov Chain Monte Carlo - 2nd Edition).
    
\end{enumerate}

\chapter{Background}
\label{cha:back}
As two rapidly evolving research areas, both deep sequence models and probabilistic modeling have a vast amount of literature. This chapter aims to provide a compact introduction to relevant concepts and methods in these two fields, based on which the research in this thesis is developed. We begin by reviewing two popular network architectures in deep sequence modeling, namely Transformers (Section~\ref{sec:transformer}), and High-order Polynomial Projection Operators (HiPPO) which establishes the theoretical foundation of deep state space models (SSMs) (Section~\ref{sec:hippo}). Next, we introduce key concepts in probabilistic machine learning and variational inference for approximate posterior (Section~\ref{sec:intro_vi}). We review in detail a classical class of probabilistic models, Gaussian processes, and common techniques of building deep models with them (Section~\ref{sec:gp_back}). In addition to predictive models, we last review the applications of probabilistic machine learning in the domain of deep generative models, with a particular focus on sequential generative models (Section~\ref{sec:dsgm}).

\section{Deep Sequence Model Architectures}
Sequence modeling is at the core of many machine learning applications nowadays, and the tremendous success of deep learning in this field can be largely attributed to several powerful network architectures. In this section, we review two popular deep sequence model architectures, Transformer and High-order Polynomial Projection Operators (HiPPO). 

\subsection{Transformer}

\label{sec:transformer}

A generic Transformer is constructed by multi-layer perceptron (MLP) and multi-head self-attention (MHSA), and we review these two base architectures below.

\paragraph{Multi-layer Perceptron.} Multi-layer perceptron (MLP; \citep{rosenblatt1958perceptron}) is an ubiquitous base component in deep neural networks. Given an input $\bm{x} \in \mathbb{R}^{d_{in}}$, an $L$-layer MLP $f_{\mparam}(\bm{x}): \mathbb{R}^{d_{in}} \rightarrow \mathbb{R}^{d_{out}}$ applies a series of transformations to the input as follows:
\begin{equation}
f_{\mparam}(\bm{x})=\weight^L g(\weight^{L-1}g(\cdot\cdot\cdot g(\weight^1 \bm{x}+\bm{b}^1))+\bm{b}^{L-1})+\bm{b}^L,
\label{eq:feed_forward_dnn}
\end{equation}
where we use $\mparam=\{\weight^l, \bm{b}^l\}_{l=1}^L$ to denote the collection of learnable weight matrices and bias vectors,
\begin{equation*}
\begin{split}
&\weight^1 \in \bm{R}^{d_h \times d_{in}}, \bm{b}^1 \in \bm{R}^{d_{in}}, \weight^l \in \bm{R}^{d_h \times d_h}, \bm{b}^l \in \bm{R}^{d_h}, l = 2, ..., L-1,\\
&\weight^L \in \bm{R}^{d_{out} \times d_h}, \bm{b}^L \in \bm{R}^{d_{out}}. 
\end{split}
\end{equation*}
Here, the hidden dimension (or width) and the output dimension of this MLP are $d_{h}$ and $d_{out}$, respectively. $g(\cdot)$ is an element-wise nonlinear activation, which makes
the MLP a non-linear function w.r.t. $\bm{x}$. Among them, Rectified Linear Unit (ReLU; \citep{hinton2010relu}) and its variants \citep{klambauer2017self} are the most commonly used activation functions. ReLU is a piecewise function that maps positive inputs to itself and negative inputs to 0:
\begin{equation}
    \text{ReLU}(s) := \max(0, s).
\end{equation}
Although most modern neural network architectures are more complex than pure MLPs, they are still often used as basic modules within modern architectures, for instance, as expressive neural feature extractors.

\paragraph{Multi-head Attention.} Attention mechanism, first introduced in \citet{graves2013hybrid}, has become the characteristic building block for Transformer models. The most widely used attention module in modern Transformers are dot-product based attention \citep{vaswani2017attention, dosovitskiy2020image}. Given a sequence of $T_q$ queries $\bm{q}\in\mathbb{R}^{T_q\times d_q}$, a sequence of $T_k$ (often $T_q=T_k:=T$ as in self-attention discussed below) keys $\bm{k}\in\mathbb{R}^{T_k\times d_k}$ ($d_k=d_q$), and the values, $\bm{v}\in\mathbb{R}^{T_k\times d_v}$, associated with the keys, dot-product attention \citep{vaswani2017attention} is computed as follows:
\begin{equation}
\bm{F}=\omega(\bm{q}\bm{k}^{\top})\bm{v},
\label{eq:attn}
\end{equation}
where $\omega$ is a nonlinear activation function. The dot-product attention computes the weighted sum of the values, where the weights can be viewed as similarity scores between the queries and keys measured by their dot products, $\bm{q}\bm{k}^{\top}$. For self-attention, the keys are simply set to be equal to the queries, i.e., $\bm{k} = \bm{q}$. Transformers employ multi-head self-attention (MHSA) constructed by multiple heads of dot-product attention whose results are aggregated to produce the final output. The queries, keys, and values for each head are all obtained from the same input sequence. Specifically, MHSA modifies dot-product self-attention as follows. Assume $H$ attention heads are in use, then given a sequence of $T$ inputs $\bm{s}\in \mathbb{R}^{T\times d_s}$ to the MHSA block, we project them to the queries for each head $h$ with a projection matrix $\bm{W}_q^{h}\in \mathbb{R}^{d_s\times d_q}$: $\bm{q}^{h}=\bm{s}\bm{W}_q^{h}$. We obtain the keys $\bm{k}^{h}$ and values $\bm{v}^{h}$ accordingly by projections using matrices $\bm{W}_k^{h}\in \mathbb{R}^{d_s\times d_k}$ and $\bm{W}_v^{h}\in \mathbb{R}^{d_s\times d_v}$ respectively. Typically, we use the same $d_q = d_k = d_v$ for all the heads. Then the head's output $\bm{F}_h$ is obtained by plugging $\bm{q}^{h}$, $\bm{k}^{h}$ and $\bm{v}^{h}$ to Eq.~\ref{eq:attn}. Lastly, the attention outputs from each head is combined with the output projection matrix $\bm{W}_F\in \mathbb{R}^{(H d_v) \times (H d_v)}$ as follows:
\begin{equation}
    \bm{F} = \text{concat}(\bm{F}_1, \cdot\cdot\cdot, \bm{F}_H)\bm{W}_F.
\label{eq:combine_heads}
\end{equation}
In Transformers, multiple layers of MHSA may be stacked together, where the output of the $(l-1)$-th MHSA layer is further processed by a nonlinear function $G_{\vparam^l}$ -- parameterized by an MLP with common deep learning tricks (e.g., residual connection \citep{https://doi.org/10.48550/arxiv.1512.03385} and layer normalization \citep{ba2016layer}) -- to obtain the input to the $l$-th MHSA layer, i.e., $\bm{s}^l = G_{\vparam^l}(\bm{F}^{l-1})$. Figure~\ref{fig:attn} illustrates an MHSA block in a Transformer model (excluding the combination projection step of Eq.~\ref{eq:combine_heads}).
\begin{figure}[t]
\centering
   \includegraphics[width=0.5\linewidth]{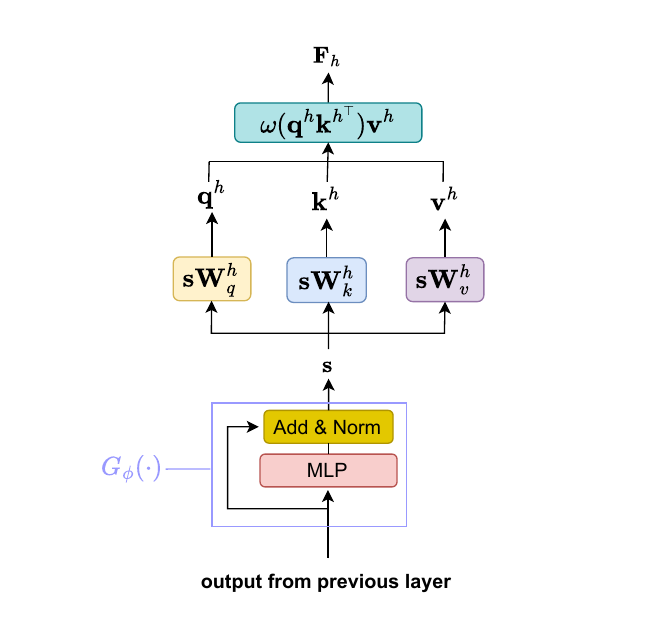}
\caption{Illustration of one head ($h$) of multi-head self attention in one layer of Transformer.}
\label{fig:attn}
\end{figure}

\newpage
\subsection{HiPPO: Recurrent Memory with Optimal Polynomial Projections}
\label{sec:hippo}
\begin{figure}
    \centering
    \includegraphics[width=0.8\linewidth]{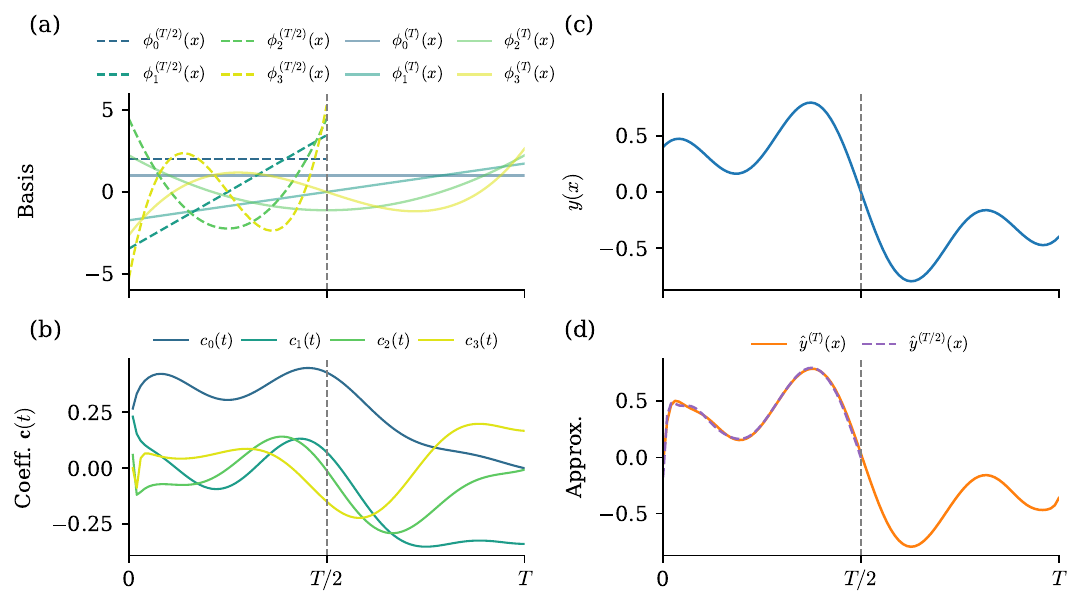}
    \caption{An illustration of HiPPO-LegS approximation of a toy time series (c) $y(x)$. Here $x$ is used to denote arbitrary time index. (a) Time-dependent basis functions with end time index $t = T/2$ and $t = T$ ($T=1$ here). (b) Evolution of the memory state $\bm{c}(t)$. (d) Finite basis approximation $\hat{y}^{t}(x)$ with the end time index $t=T/2$ and $T$.}
    \label{fig:hippo_deterministic}
\end{figure}
The HiPPO framework \citep{gu_hippo_2020} is proposed to build a compact representation that captures long-term dependencies in sequences, which is a major challenge in long sequence modeling. HiPPO provides mathematical foundations for compressing continuous-time signals into finite-dimensional memory states through optimal polynomial projections. Given a time series \( y(t) \), HiPPO maintains a memory state \( \bm{c}(t) \in \mathbb{R}^M \) that optimally approximates the historical signal \( \{y(x)\}_{x \leq t} \). The framework consists of a time-dependent measure \( \omega^{(t)}(x) \) over \( (-\infty, t] \) that defines input importance, along with adaptive polynomial basis functions \( \{g_m^{(t)}(x)\}_{m=0}^{M-1} \) obtained from adapting the domain of the basis functions to cover the new time region. For example, the standard Legendre polynomials $P_m(x)$ considered in HiPPO-LegS variant have domain $[-1,1]$, and when the last historical time instance is $t$ (assuming to be positive), we define the adaptive basis 
\begin{equation}
g_m^{(t)}(x):=(2m+1)^{1/2}P_m(\frac{2x}{t}-1),
\end{equation}
which has the domain $[0, t]$. Moreover, these adaptive bases are orthonormal under certain \( \omega^{(t)}(x) \) (e.g., for adaptive Legendre polynomials above, the corresponding  $\omega^{(t)}(x)$ is uniform over $[0, t]$): 
\begin{equation}
\int_{-\infty}^t g_m^{(t)}(x)g_l^{(t)}(x)\omega^{(t)}(x)\mathrm{d}x = \delta_{ml}.
\label{eq:hippo_orthonormal}
\end{equation}
The historical signal is encoded through projection coefficients given by
\begin{equation}
    c_m(t) = \int_{-\infty}^t y(x)g_m^{(t)}(x)\omega^{(t)}(x)\mathrm{d}x,
    \label{eq:hippo_coef}
\end{equation}
which yields the approximation 
\begin{equation}
    y(x) \approx \hat{y}^{(t)}(x)=\sum_{m=0}^{M-1}c_m(t)g_m^{(t)}(x),
\end{equation}
for $x \in (-\infty,t]$. The compression is usually lossy unless $y(x)$ lives in the span of the finite bases. However, $\{c_m(t)\}_{m=0}^{M-1}$ obtained in Eq.~\ref{eq:hippo_coef} is optimal in the sense that it minimizes the following $L^2$-error: 
%
%
\begin{equation}
   \{c_n\}_{n=0}^{M-1} = \argmin_{\{c'_n\}_{n=0}^{M-1}}  \int_{-\infty}^t \|y(x)-\sum_{n=0}^{M-1} c'_n(t)g_n(x)\|^2 \omega^{(t)}(x)\mathrm{d}x.
\end{equation}
Hence, the $M$-dimensional memory state $\bm{c}(t) := [c_0(t), \ldots, c_{M-1}(t)]^\top$ forms a compact representation that optimally captures the historical information in $y(x)$ up to $x=t$.

This memory state $\bm{c}(t)$ is a function of the end point of the history $t$. Every time when observation at new time instance, $t+\Delta t$ appear, we adapt the basis to $t+\Delta t$, and compute $\bm{c}(t+\Delta t)$. A neat property of the memory state is that it enables online updates: we do not need to recompute the integrals from scratch to obtain the memory state each time we observe new data, instead we can obtain $\bm{c}(t+\Delta t)$ using an online update based on $\bm{c}(t)$. The time derivative of the adaptive polynomial basis functions considered in HiPPO, $\frac{\partial}{\partial t}g_m^{(t)}(x)$, can be expressed as a linear combination of lower-order polynomial bases $\{g_l^{(t)}(x)\}_{l=0}^{m-1}$. Consequently, the time derivative of $\bm{c}(t)$ w.r.t. $t$ induces a linear ordinary differential equation (ODE)
\begin{equation}
    \frac{\mathrm{d}}{\mathrm{d}t}\bm{c}(t) = \bm{A}(t)\bm{c}(t) + \bm{B}(t)y(t),
    \label{eq:hippo_ode}.
\end{equation}
Discretizing the linear ODE yields the recurrence of the form
\begin{equation}
    \bm{c}_k = \bm{A}_k \bm{c}_{k-1} + \bm{B}_k y_k,
\end{equation}
which allows online updates without having to revisit the past signals. Here, $\bm{A}_k$ and $\bm{B}_k$ are determined by $\bm{A}(t)$, $\bm{B}(t)$ and the discretization step size.

HiPPO supports various measure-basis configurations which can enable online updates of the corresponding memory state \citep{gu_hippo_2020,gu_httyh_2023}. Here, we provide some details for a canonical instantiation, HiPPO-LegS. For a given end time index $t$, HiPPO-LegS uses the uniform measure $\omega^{(t)}(x) = \frac{1}{t}\bm{1}_{[0,t]}(x)$ and scaled Legendre polynomials with input domain adapted to $[0,t]$, \(g_m^{(t)}(x) = (2 m+1)^{1 / 2} P_m\left(\frac{2 x}{t}-1\right)\), as basis functions, where $P_m(\cdot)$ is the canonical $m$-th Legendre polynomial with input domain $[-1, 1]$, and $\bm{1}_{[0,t]}(x)$ is the indicator function on the interval $[0,t]$. The uniform measure over the past provides the inductive bias that encourages HiPPO-LegS to keep the whole past in memory. The memory state of HiPPO-LegS evolves according to a linear ODE as in Eq.~\ref{eq:hippo_ode}. To see this, we define $\phi_m^{(t)}(x) := g_m^{(t)}(x)\omega^{(t)}(x)$. The time derivative of $\phi_m^{(t)}(x)$ can be shown to be:
\begin{equation}
\begin{split}
\frac{\partial}{\partial t}\phi_{m}^{(t)}(x)
&= \omega_{m}^{(t)}(x)\frac{\partial}{\partial t}g_{m}^{(t)}(x) + g_{m}^{(t)}(x)\frac{\partial}{\partial t}\omega_{m}^{(t)}(x)\\
&=-\frac{\sqrt{2m+1}}{t}
[
\frac{m+1}{\sqrt{2m+1}}\phi_{m}^{(t)}(x)
+
\sqrt{2m-1}\phi_{m-1}^{(t)}(x)\\&\qquad\qquad\qquad + \sqrt{2m-3}\phi_{m-2}^{(t)}(x) + \cdots
]
+
\frac{1}{t}\,\delta_{t}(x),
\end{split}
\end{equation}
where $\delta_{t}(x)$ is the Dirac delta at $x=t$ and it arises from the time derivative of the uniform measure $\omega^{(t)}(x)$ (see Appendix~D.3 in \citet{gu_hippo_2020} for details). Plugging this expression into 
\begin{equation}
\begin{split}
\frac{\mathrm{d}}{\mathrm{d}t}\bm{c}(t)&=[\frac{\mathrm{d}}{\mathrm{d}t}c_0(t), \cdots, \frac{\mathrm{d}}{\mathrm{d}t}c_{M-1}(t)]^\top\\
&=[\int y(x)\frac{\partial}{\partial t}\phi_0^{(t)}(x)dx, \cdots, \int y(x)\frac{\partial}{\partial t}\phi_{M-1}^{(t)}(x)dx]^\top
\end{split}
\end{equation}
allows us to identify the matrices $\bm{A}(t) \in \mathbb{R}^{M \times M}$ and $\bm{B}(t) \in \mathbb{R}^{M \times 1}$ in the ODE update of the memory state (Eq.~\ref{eq:hippo_ode}). For HiPPO-LegS, their explicit formulas are given by:
\begin{equation}
\label{eq:hippo_matrix_a}
\bm{A}(t) = \frac{1}{t}\bm{A}, \quad
[\bm{A}]_{mk} = \begin{cases}
    -\sqrt{(2m+1)(2k+1)} & \text{if } m > k \\
    -m-1 & \text{if } m = k \\
    0 & \text{if } m < k
\end{cases}
\end{equation}
and
\begin{equation}
\label{eq:hippo_matrix_b}
[\bm{B}(t)]_m =\frac{[\bm{B}]_m}{t} =\frac{\sqrt{2m+1}}{t},
\end{equation}
where the factor $1/t$ reflects the time-dependent scaling of the basis functions to the adaptive interval $[0,t]$. Figure~\ref{fig:hippo_deterministic} shows evolution of the memory state for a toy time series in HiPPO-LegS, and the approximation of the time series based on the memory state and adaptive Legendre polynomials for two different end time indices of the history.

We include details of other HiPPO variants which use different combinations of function basis and measures in Appendix~\ref{appendix:hippo_measure_basis}. Compared with HiPPO-LegS, they use measures allocating more importance over the more recent history so that they are less suitable for tasks requiring long-term memory in sequences.

A few recent works extend the HiPPO recurrences for more efficient long-range memory modeling and stack them along with MLPs to build deep SSMs, which is now a class of architectures competitive to Transformers. The structured state space sequential (S4) model \citep{gu_s4_2022} extends HiPPO with trainable recurrence parameters and accelerates the computation with convolutional kernels, while Mamba \citep{gu_mamba_2023,dao_mamba2_2024} further introduces hardware-aware selective state mechanisms.

\section{Scalable Probabilistic Learning}
Probabilistic machine learning, based on rules of probability and Bayesian statistics, offers a principled framework for systematically expressing and updating our knowledge about unknowns under uncertainty. In particular, Bayesian inference leverages Bayes' rule \citep{bayes1763essay} to update our belief about the unobserved variables after observing the data. 

\subsection{Bayesian Inference}
\label{sec:back_bayes}
Given a machine learning model specified by some parameters $\mparam$ (e.g., weights of a deep neural network), we turn the model into a probabilistic model by treating the parameters as unobserved random variables and placing a prior distribution, $p(\mparam)$, over them, which represents our initial belief about the model parameters before observing any data. For instance, it can be used to rule out extreme parameter values that can lead to erratic predictions. The data $\data=\{(\bm{x}_i,\bm{y}_i)\}_{n=1}^N$ are treated as observed variables and the information for $\mparam$ from the data, also known as likelihood $p(\data|\mparam)$, is used to refine the prior into the posterior:
\begin{equation}
    p(\mparam|\data)=\frac{p(\mparam, \data)}{p(\data)}=\frac{p(\mparam)p(\data|\mparam)}{\int p(\mparam)p(\data|\mparam)d\mparam}.
\label{eq:bayespost}
\end{equation}
The model uncertainty in $\mparam$ from the posterior is then translated to predictive uncertainty associated with the prediction at new input $\bm{x}^\ast$, via the posterior predictive distribution defined by the following marginalization:
\begin{equation}
    p(\bm{y}^\ast|\bm{x}^\ast,\data) = \int p(\bm{y}^\ast|\bm{x}^\ast,\mparam) p(\mparam|\data) d\mparam.
\label{eq:bayespred}
\end{equation}
In addition to the uncertainty quantification capability, Bayesian inference is naturally suitable for online learning \citep{cesabianchi2006prediction,ritter2018online} or continual learning \citep{kirkpatrick2017overcoming,nguyen2018variational}, where data points are observed sequentially in batches. When new data arrive, the previous posterior obtained based on old data, $p(\mparam|\data_{old})$, becomes the new prior and we combine the information from the current data, $\data_{new}$, to form the new posterior again via Bayes' rule:
\begin{equation}
    p(\mparam|\data_{new}, \data_{old}) =  \frac{p(\mparam|\data_{old})p(\data_{new}|\mparam, \data_{old})}{\int p(\mparam|\data_{old})p(\data_{new}|\mparam, \data_{old}) d\mparam},
\end{equation}
This procedure can be repeated whenever new data arrive, enabling online or continual learning under uncertainty.

Unfortunately, exact Bayesian inference as in Eq.~\ref{eq:bayespost} is typically intractable since it requires solving a complex integral for the marginal likelihood in the denominator, $p(\data)=\int p(\mparam)p(\data|\mparam)d\mparam$, which has no analytic solution unless the prior is conjugate to the likelihood. Hence, in practice, approximate inference techniques \citep{li2018approx}, which approximate the target posterior with an approximate posterior $q(\mparam) \approx p(\mparam | \data)$, become indispensable. In such case, the posterior predictive distribution in Eq.~\ref{eq:bayespred} can also be approximated based on $q(\mparam)$ as follows:
\begin{equation}
    p(\bm{y}^* | \bm{x}^*, \data) \approx \int p(\bm{y}^\ast|\bm{x}^\ast,\mparam) q(\mparam) d\mparam \approx \frac{1}{M} \sum_{m=1}^M p(\bm{y}^\ast|\bm{x}^\ast,\mparam_m), \quad \mparam_m \sim q(\mparam),
\label{eq:approximate_predictive_distribution}
\end{equation}
where a further Monte Carlo estimation is typically applied for marginalization.

In this thesis, we use variational inference \citep{jordan1999vi,beal:vi2003,li2018approx,zhang2018advances}, an optimization-based approximate inference technique, which is much more scalable than simulation-based techniques, such as Markov chain Monte Carlo \citep{metropolis1953equation,gelfand2000gibbs,ma2015recipe}. This makes variational inference a popular method for approximate inference in deep learning applications, where scalability is crucial.

\subsection{Variational Inference}
\label{sec:intro_vi}

Variational inference (VI) \citep{beal:vi2003,jordan1999vi} turns the inference problem into an optimization procedure. It first specifies a parametric distribution family, $\mathcal{Q} := \{q_{\vparam}(\mparam) \}$, with tunable parameters $\vparam$, such as factorized Gaussians where $\vparam$ are the mean and variance parameters. Then we optimize over $\mathcal{Q}$ to find the best approximate posterior (also called optimal variational distribution), $q_{\vparam^*}(\mparam)$, which is defined as the one closest to the target posterior $p(\mparam|\mathcal{D})$ within the parametric family, measured by the Kullback-Leibler (KL) divergence. Specifically, the procedure requires the minimization of KL-divergence w.r.t. the parameters $\vparam$ of the chosen parametric distribution family to obtain the optimal variational distribution:
\begin{equation}
\begin{split}
    &\vparam^* = \arg\min_{q_{\vparam} \in \mathcal{Q}} \text{KL}(q_{\vparam}(\mparam)\|p(\mparam|\data)),\\ &\text{KL}(q_{\vparam}(\mparam)\|p(\mparam|\data)) = \int q_{\vparam}(\mparam) \log \frac{q_{\vparam}(\mparam)}{p(\mparam|\data)} d\mparam.
\end{split}
\label{eq:kl}
\end{equation}
We can then substitute in $q_{\vparam^*}(\mparam)$ for $q(\mparam)$ in Eq.~\ref{eq:approximate_predictive_distribution} for approximate posterior predictive distribution.

However, the KL-divergence requires the evaluation of $p(\mparam | \mathcal{D})$, and its dependence on the intractable marginal likelihood (also known as model evidence) $p(\mathcal{D})$ is the motivation for approximate inference in the first place. Therefore, direct minimization of KL-divergence is infeasible, and an alternative objective is needed for its minimization, which prompts the following evidence lower bound (ELBO) as objective. The ELBO is a lower bound of the log marginal likelihood and it is obtained by subtracting the KL-divergence from the log marginal likelihood, which is a constant w.r.t. $\vparam$. Hence, minimization of the KL-divergence can be achieved indirectly via the maximization of the (ELBO).
\begin{equation}
    \begin{split}
        \mathcal{L}_{ELBO}(\vparam) &:=\log p(\data) -\text{KL}(q_{\vparam}(\mparam)\|p(\mparam|\data))\\
        &=\log p(\mathcal{D}) - \int q_{\vparam}(\mparam) \log \frac{q_{\vparam}(\mparam)p(\mathcal{D})}{p(\mparam)p(\mathcal{D}|\mparam)} d\mparam\\
        &=\int q_{\vparam}(\mparam) \log \frac{p(\mparam)p(\mathcal{D}|\mparam)}{q_{\vparam}(\mparam)} d\mparam\\
        &=\underbrace{\mathbb{E}_{q_{\vparam}(\mparam)}[\log p(\data|\mparam)]}_{\text{ELL}} - \underbrace{\text{KL}(q_{\vparam}(\mparam)\|p(\mparam))}_{\text{KL regularizer}}
    \end{split}
\label{eq:elbo_back}
\end{equation}
The above ELBO is decomposed into two terms. The first term is commonly called the expectation of log-likelihood (ELL), and the second term is a KL regularizer between the approximate posterior $q_{\vparam}(\mparam)$ and prior $p(\mparam)$. The ELL term encourages $q_{\vparam}(\mparam)$ to place more importance over the model configurations that explain the data, while
the second term regularizes $q_{\vparam}(\mparam)$ toward the prior. The combination of the two terms balance the data fit and the regularization.

In practice, unless with special conditions (e.g., Gaussianity) for the specification of the prior, likelihood, and the variational distribution, we may not be able to analytically compute one or both terms in ELBO. Hence, typically a Monte Carlo estimation is further employed:
\begin{equation}
    \mathcal{L}_{ELBO} = \frac{1}{M} \sum_{m=1}^M \left[ \log p(\mathcal{D}|\mparam_m) - \log q_{\vparam}(\mparam_m) + \log p(\mparam_m) \right], \qquad \mparam_m \stackrel{\text{iid}}{\sim} q_{\vparam}(\mparam).
\end{equation}
If the chosen variational distribution family, $\mathcal{Q}$, contains the true posterior, VI will return it, if the optimization is solved perfectly. In practice, due to the computational budget and tractability of distribution families, a simple yet tractable $\mathcal{Q}$ is often considered to enable fast computation and reduced memory complexity.

To exploit gradient based optimization of the ELBO, the \textit{reparameterization trick} \citep{welling2014auto} is proposed for a variety of variational distribution families including Gaussians. This approach makes the sampling operation $\mparam \sim q_{\vparam}(\mparam)$ differentiable w.r.t. $\vparam$ by viewing it as passing an auxiliary noise variable $\bm{\epsilon} \sim p_{base}(\bm{\epsilon})$ through a function $T_{\vparam}(\bm{\epsilon})$ that is differentiable w.r.t.~$\vparam$:
\begin{equation}
   \mparam \sim q_{\vparam}(\mparam) \quad \Leftrightarrow \quad \mparam = T_{\vparam}(\bm{\epsilon}), \ \bm{\epsilon} \sim p_{base}(\bm{\epsilon}).
\end{equation}
For instance, for $D$-dimensional factorized Gaussians with mean $\bm{\mu}=[\mu_1, \cdots, \mu_d]^\top$ and variance $\bm{\sigma}^2=[\sigma_1^2, \cdots, \sigma_d^2]^\top$, $q_{\vparam}(\mparam) = \prod_{d=1}^D \mathcal{N}(\mparam_d; \mu_d, \sigma_d)$, this sampling operation can be written as $\mparam = T_{\vparam}(\bm{\epsilon}) := \bm{\mu} + \bm{\sigma} \odot \bm{\epsilon}$, with $\bm{\epsilon}_{d} \sim \mathcal{N}(0, 1)$, and $\odot$ denotes element-wise multiplication. The Monte Carlo estimate of the ELL term in Eq.~\ref{eq:elbo_back} can therefore be rewritten using the change-of-variable rule as follows:
\begin{equation}
    \mathbb{E}_{q_{\vparam}(\mparam)}[\log p(\data|\mparam)] \approx \frac{1}{M} \sum_{m=1}^M \log p(\data|\bm{\mu} + \bm{\sigma} \odot \bm{\epsilon}_m), \quad \bm{\epsilon}_m \stackrel{\text{iid}}{\sim} \mathcal{N}(\bm{0}, \bm{I}).
\end{equation}
Moreover, with likelihood that factorizes w.r.t. the data points, stochastic optimization methods based on mini-batch estimation of the above ELL term, such as stochastic gradient descent (SGD), can be used to scale VI to large datasets.

VI has been used extensively for scalable inference in probabilistic models, such as Gaussian processes and Bayesian neural networks, and we will review its applications in Gaussian process models in detail next. I would be remiss to not review some excellent works about weight space VI in the context of Bayesian neural networks since they are not directly related to the research in this thesis. We refer the interested readers to \citet{arbel2023primer} for an overview of the topic.
\section{Gaussian Processes}
\label{sec:gp_back}
Gaussian processes (GPs; \citep{rasmussen2006gp}) is a class of powerful Bayesian non-parametric models that are widely used to infer unknown functions under uncertainty. Instead of assuming a parametric form of the model and performing posterior inference over parameters, GPs directly specify distributions over function values. Informally, GPs can be viewed as infinite-dimensional Gaussian distributions for the function values, evaluated over an index set $\mathcal{X}$ (domain of $f$), which can be infinite, e.g., $\mathcal{X}=\mathbb{R}^{D}$.

A GP is fully specified by a mean function $m_{\hparam}(\cdot): \mathcal{X}\rightarrow \mathbb{R}$ and a symmetric positive-definite covariance covariance function, typically parameterized by a kernel function $k_{\hparam}(\cdot, \cdot):\mathcal{X}\times\mathcal{X}\rightarrow \mathbb{R}$, $p(f;\hparam) = \mathcal{GP}(f;m_{\hparam}(\cdot), k_{\hparam}(\cdot, \cdot))$. Here, $\hparam$ is used to to denote the hyperparameters associated with the mean function and the kernel function.

GP provides a powerful Bayesian modeling paradigm in the function space by first placing a GP prior over the unknown function $f$. Typically, an uninformative prior mean function is in use, e.g., zero function $m(\cdot):=0$, and this convention will be followed throughout this thesis.
\begin{equation}
    p(f;\hparam) = \mathcal{GP}(f;0, k_{\hparam}(\cdot, \cdot)).
\end{equation}
Notably, as stochastic processes, a core property of GPs is marginal consistency. For example, the marginal distribution according to the above GP prior for function values $\bm{f_X}:=[f(\bm{x}_1),\cdot\cdot\cdot,f(\bm{x}_N)]^\top$ over a finite subset of input locations, $\bm{X} := [\bm{x}_1,\cdot\cdot\cdot,\bm{x}_N]^\top\in \mathcal{X}$, is a multivariate Gaussian as follows:
\begin{equation}
    p(\bm{f_X};\hparam) = \mathcal{N}(\bm{f_X};\bm{0}, \bm{K}_{\bm{XX}}),
\end{equation}
where the covariance matrix is obtained via the kernel matrix, $\bm{K}_{\bm{XX}}$, whose $ij$-th element is the kernel value between $\bm{x}_i$ and $\bm{x}_j$, $[\bm{K}_{\bm{XX}}]_{ij}=k_{\hparam}(\bm{x}_i,\bm{x}_j)$. Notice that $\bm{K}_{\bm{XX}}$ depends on hyperparameters $\hparam$, but we do not include $\hparam$ in its notation for consice presentation.

\subsection{Exact Inference for Regression}
In regression problems with observed data $\data=(\bm{X}, \bm{y}):=\{\bm{x}_n, y_n\}_{n=1}^N$, typically an i.i.d Gaussian likelihood with noise variance $\sigma^2$ is assumed,
\begin{equation}
    p(\bm{y}|\bm{f_X}; \sigma^2)=\mathcal{N}(\bm{y};\bm{f_X}, \sigma^2\bm{I}),
\end{equation}
the posterior process is also a GP since the Gaussian prior and likelihood are conjugate, and the following closed-form posterior predictive distribution can be obtained:
\begin{equation}
\begin{split}
    p(\bm{y}^\ast|\bm{X}^\ast,\data; \hparam)=\mathcal{N}(&\bm{y}^\ast;\bm{K}_{\bm{X}^\ast\bm{X}}(\bm{K}_{\bm{XX}}+\sigma^2\bm{I})^{-1}\bm{y},\\
    &\bm{K}_{\bm{X^\ast X^\ast}} + \sigma^2\bm{I} \underbrace{-\bm{K}_{\bm{X}^\ast\bm{X}}(\bm{K}_{\bm{XX}}+\sigma^2\bm{I})^{-1}\bm{K}_{\bm{X}\bm{X}^\ast})}_{\text{uncertainty reduction}},
\end{split}
\label{eq:posterior_predictive_fgp}
\end{equation}
where $\bm{X}^\ast$ are test inputs. The last term in the posterior predictive covariance is commonly referred to as uncertainty reduction, and its magnitude is determined by the similarity between test inputs $\bm{X}^\ast$ and observed inputs $\bm{X}$, measured by the kernel. For $\bm{X}^\ast$ closer to $\bm{X}$, $\bm{K}_{\bm{X^\ast X}}$ will be large, resulting in a larger uncertainty reduction. For instance, if kernels based on Euclidean distance, such as RBF kernel, are used, the predictive uncertainty will be Euclidean distance-aware. In general, the kernel can be specified to encode some prior assumptions for the underlying function (e.g., smoothness, periodicity).

The kernel hyperparameters $\hparam$ (e.g., the length-scale in RBF kernel) along with the observation noise variance $\sigma^2$ can influence the predictions of GPs noticeably, and a common practice to select them is via the maximization of log marginal likelihood (also known as type-II maximum likelihood estimation). For regression problems, the log marginal likelihood also admits a closed form as follows:
\begin{equation}
\begin{split}
    \log p(\bm{y}|\bm{X}; \hparam) &= \log \int p(\bm{y}|\bm{f_X}; \hparam)p(\bm{f_X};\hparam)d\bm{f_X}\\
    &=\log \mathcal{N}(\bm{y};\bm{0}, \bm{K_{XX}}+\sigma^2 \bm{I})\\
    &=-\underbrace{\frac{1}{2}{\bm{y}^\top(\bm{K}_{\bm{XX}}}+\sigma^2\bm{I})^{-1}\bm{y}}_{\text{data fit}}-\underbrace{\frac{1}{2}{\log \det(\bm{K}_{\bm{XX}}+\sigma^2\bm{I})}}_{\text{model complexity}} + \text{const.}
\end{split}
\end{equation}
It consists of two terms that balance between the data fit and the model complexity penalty, which ultimately selects the configuration of hyperparameters corresponding to the simplest model that explains the data well. 

Interestingly, the posterior inference of GP can be reformulated through the lens of Bayesian linear regression with kernel feature map \citep[Chapter 2]{rasmussen2006gp}. Consider the linear regression problem as follows.
\begin{equation}
    y(x) = \bm{w}^\top\Phi(\bm{x}) + \epsilon, \quad \epsilon \sim \mathcal{N}(0, \sigma^2),
\end{equation}
where $\Phi(\cdot): \mathcal{X}\rightarrow\mathbb{R}^{d_{\Phi}}$ ($d_{\Phi}$ can be $\infty$) is the feature map associated with a kernel $k(\cdot, \cdot)$ such that $k(\bm{x}, \bm{x}')=\Phi(\bm{x})^\top\Phi(\bm{x}')$. Mercer's Theorem \citep{mercer1909functions} guarantees the existence of such feature maps for kernel functions meeting Mercer's condition. A standard Gaussian prior over the weight, $p(\bm{w})=\mathcal{N}(\bm{0}, \bm{I})$, induces a GP prior in the function space: given data $\data=(\bm{X}, \bm{y})$, we know that the distribution of function values evaluated over $\bm{X}$ is a Gaussian due to the linearity of the model, and its mean and covariance can be shown to be $\bm{0}$ and $\bm{K_{XX}}+\sigma^2 \bm{I}$, respectively, which exactly matches the prior for $\bm{y_X}$ in GP model. The posterior of weight is a Gaussian 
\begin{equation}
\begin{split}
    p(\bm{w}| \data) &= \mathcal{N}(\bm{w}; \bm{m_w}, \bm{S_w}),\\
    \bm{m_w}&=\frac{1}{\sigma^2}\bm{S_w}\bm{\Phi_X}^\top\bm{y},\\
    \bm{S_w}&=(\bm{I}+\frac{1}{\sigma^2}\bm{\Phi_X}^\top\bm{\Phi_X})^{-1}\\&=\bm{I} - \bm{\Phi_X}^\top(\sigma^2 \bm{I}+ \bm{\Phi_X}\bm{\Phi_X}^\top)^{-1}\bm{\Phi_X} \qquad \text{(Woodbury identity)} \\
    &=\bm{I} - \bm{\Phi_X}^\top(\sigma^2 \bm{I}+ \bm{K_{XX}})^{-1}\bm{\Phi_X},
\end{split}
\end{equation}
where $\bm{\Phi_X} \in \mathbb{R}^{N \times d_{\phi}}$ is the design matrix, and notice that $\bm{\Phi_X}\bm{\Phi_X}^\top = \bm{K_{XX}}$.
The posterior predictive for new input locations $\bm{X}^\ast$ can be derived based on the weight posterior and it is a Gaussian as follows:
\begin{equation}
    \begin{split}
        p(\bm{y}^\ast|\bm{X}^\ast,\data)&=\int p(\bm{y}^\ast| \bm{w}, \bm{X}^\ast)p(\bm{w}|\data)d\bm{w}\\
        &=\mathcal{N}(\bm{y}^\ast; \bm{\Phi}_{\bm{X}^\ast}\bm{m_w}, \sigma^2\bm{I}+\bm{\Phi_w}^\top\bm{S_w}\bm{\Phi_w}).
    \end{split}
\end{equation}
It is easy to see that the posterior predictive covariance is exactly the same as the posterior GP in Eq.~\ref{eq:posterior_predictive_fgp}. The posterior predictive mean can be further simplified below:
\begin{equation}
\begin{split}
    \bm{\Phi}_{\bm{X}^\ast}\bm{m_w} &=  \frac{1}{\sigma^2} \bm{\Phi}_{\bm{X}^\ast} \bm{S_w}\bm{\Phi_X}^\top\bm{y}\\
    &=\bm{K}_{\bm{X}^\ast \bm{X}} \left\{ \frac{1}{\sigma^2}\left[\bm{I}-(\sigma^2\bm{I}+\bm{K_{XX}})^{-1}\bm{K_{XX}}\right] \right\}\bm{y}\\
    &=\bm{K}_{\bm{X}^\ast\bm{X}}(\bm{K}_{\bm{XX}}+\sigma^2\bm{I})^{-1}\bm{y}.
\end{split}
\end{equation}
For the last line, simply notice that
\[
 (\bm{K}_{\bm{XX}}+\sigma^2\bm{I})\left\{\frac{1}{\sigma^2}\left[\bm{I}-(\sigma^2\bm{I}+\bm{K_{XX}})^{-1}\bm{K_{XX}}\right]\right\}=\frac{1}{\sigma^2} \bm{K}_{\bm{XX}}+ \bm{I} -\frac{1}{\sigma^2}\bm{K_{XX}}=\bm{I}.
\]
Hence, a GP model is equivalent to a Bayesian linear regression model based on kernel feature map.

\subsection{Sparse Variational Gaussian Process}
\label{sec:svgp_back}
Although GP exhibits exceptional performance and provides reliable predictive uncertainty in the small data regime, its posterior inference requires matrix inversion with a cubic computational cost w.r.t. the number of observed data points ($O(N^3)$), which makes it computationally intractable for large dataset. One way to extend GP to big data applications is working with an approximate posterior GP with low rank structure in the framework of VI. One popular such approach is sparse variational Gaussian process (SVGP; \citep{snelson2005sgp}), which approximates the posterior process with a GP, $Q$, constructed based on a small set of augmented inducing points (pseudo input-function value pairs) of size $M$ (typically $M \ll N$), $\{(\bm{z}_m, u_m)\}_{m=1}^M$ \citep{hensman2013gpbig, hensman2015sgpclass}. SVGP reduces the $O(N^3)$ computational cost of full GP to $O(M^3+NM^2)$, since, by construction, $Q$ is fully determined by the approximate posterior of the small set of inducing points, which is set to be a Gaussian with learnable variational mean, $\bm{m_u}$, and variational covariance, $\bm{S_u}$, $q(\bm{u})=\mathcal{N}(\bm{m_u},\bm{S_u})$.

According to the standard GP prior, $p(f)=\mathcal{GP}(0(\cdot), k_{\hparam}(\cdot, \cdot))$, the prior marginal distribution of the augmented set consisting of $\bm{u}$ and $\bm{f}^\ast$, evaluated at any arbitrary set of input locations $\bm{X}^\ast$, is
\begin{equation}
    p(\bm{f}^\ast, \bm{u}|\bm{X}^\ast,\bm{Z})= \mathcal{N}(\bm{0}, \begin{pmatrix}
     \bm{K}_{\bm{X^*X^*}} & \bm{K}_{\bm{X^*Z}} \\
     \bm{K}_{\bm{Z}\bm{X}^*} & \bm{K}_{\bm{ZZ}}.
    \end{pmatrix}).
\end{equation}
The aforementioned computational gain is achieved by imposing the following low rank structure for the approximate posterior GP. Given the approximate posterior of the inducing points, $q(\bm{u})$, the approximate posterior conditional distribution of $\bm{f}^\ast$ given $\bm{u}$ is assumed to be the same as the prior conditional distribution.
\begin{figure}[t]
\centering
   \includegraphics[width=0.5\linewidth]{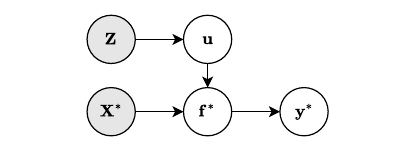}
\caption{The graphical model of prediction of sparse variational Gaussian process (SVGP).}
\label{fig:svgp_graphical}
\end{figure}
\begin{equation}
    q(\bm{f}^\ast|\bm{u,Z,X^\ast})=p(\bm{f}^\ast|\bm{u,Z,X^\ast}).
    \label{eq:prior_conditional_matching}
\end{equation}
Hence, the induced approximate posterior for $\bm{f}^\ast$ is a Gaussian with a low rank structure defined by the inducing points:
\begin{equation}
\begin{split}
    q(\bm{f}^\ast&|\bm{X}^\ast, \bm{Z})=\int p(\bm{f}^\ast|\bm{u,Z,X^\ast})q(\bm{u}) d \bm{u}\\
    &=\mathcal{N}(\bm{f}^{\ast};\bm{K_{X^\ast Z}K_{ZZ}^{-1}m_u},\bm{K_{X^\ast X^\ast}+K_{X^\ast Z}K_{ZZ}^{-1}(S_u-K_{ZZ})K_{ZZ}^{-1}K_{ZX^\ast}}).
\end{split}
\label{eq:q_predictive_svgp}
\end{equation}
Note that now the matrix inversion is applied to $\bm{K_{ZZ}}$, with a $O(M^3)$ computational cost, which is significantly lower than the $O(N^3)$ inversion cost in full GP. We show the graphical model of prediction of SVGP in Figure~\ref{fig:svgp_graphical}. Furthermore, from the perspective of weight space inference, the prediction of SVGP is equivalent to an approximate Bayesian linear model based on kernel feature map, $f(x) = \bm{w}^\top\Phi(\bm{x})$, where the following structured low-rank approximate posterior, parameterized by variational parameters $\bm{m_u}$ and $\bm{S_u}$, for $\bm{w}$ is considered.
\begin{equation}
    q(\bm{w}) = \mathcal{N}(\bm{w};\bm{\Phi_Z}^\top\bm{K_{ZZ}}^{-1}\bm{m_u},\bm{I}+ \bm{\Phi_Z}^\top\bm{K_{ZZ}}^{-1}(\bm{S_u}-\bm{K_{ZZ}})\bm{K_{ZZ}}^{-1}\bm{\Phi_Z}).
\end{equation}
Given observed data $\data=(\bm{X}, \bm{y})$, the inducing input locations $\bm{Z}$ are treated as model hyperparameters and are tuned jointly with the kernel hyperparameters and the variational parameters by ELBO maximization.
\begin{equation}
    \begin{split}
        \mathcal{L}_{ELBO}&=\mathbb{E}_{q(\bm{f},\bm{u},|\bm{Z}, \bm{X})}[\log\frac{p(\bm{y}|\bm{f})p(\bm{f}|\bm{u,Z,X})p(\bm{u}|\bm{Z})}{q(\bm{f},\bm{u}|\bm{Z},\bm{X})}]\\
        &=\int\int p(\bm{f}|\bm{u,Z,X})q(\bm{u}) \log\frac{p(\bm{y}|\bm{f})p(\bm{u}|\bm{Z})}{q(\bm{u})} d\bm{u} d\bm{f}\\
       &=\underbrace{\mathbb{E}_{q(\bm{f}|\bm{X},\bm{Z})}[\log p(\bm{y|f})]}_{\text{ELL}} -\underbrace{\text{KL}(q(\bm{u})||p(\bm{u}|\bm{Z}))}_{\text{KL regularizer}}
    \end{split}
    \label{eq:elbo_svgp}
\end{equation} 
The full derivation of the above ELBO can be found in Appendix~\ref{a4:svgp_elbo}. The VI framework also permits the use of non-Gaussian observation likelihood, which broadens the applications of GPs beyond regression tasks. For non-Gaussian likelihoods, the ELL term has no closed form, and we typically use Monte Carlo estimation, where the main cost ($O(M^3+NM^2)$) comes from computing $q(\bm{f}|\bm{X}, \bm{Z})$. The KL regularizer can be evaluated analytically, since $q(\bm{u})$ and $p(\bm{u}|\bm{Z})$ are both Gaussian, and its main computational cost comes from the evaluation of the determinant $\det\bm{K_{ZZ}}$, and the matrix inversion $\bm{K_{ZZ}}^{-1}$ (both are $O(M^3)$). Combined, evaluating $\mathcal{L}_{ELBO}$ for SVGP incurs a cost of $O(M^3+NM^2)$.

The posterior predictive distribution of $\bm{y}^\ast$ at test input locations, $\bm{X}^\ast$, is 
\begin{equation}
\begin{split}
    p(\bm{y}^\ast|\bm{X}^\ast, \bm{Z}, \bm{u})&=\int\int p(\bm{y}^\ast|\bm{f}^\ast)p(\bm{f}^\ast|\bm{u,Z, X^\ast})q(\bm{u}) d \bm{u} d\bm{f}^\ast\\&= \mathbb{E}_{q(\bm{f}^\ast|\bm{X}^\ast, \bm{Z})}[p(\bm{y^\ast|f^\ast})].
\end{split}
\end{equation}
Again, we can resort to Monte Carlo to estimate it in tasks requiring non-Gaussian observation likelihoods.

For regression based on Gaussian observation likelihood with noise variance $\sigma^2$, the ELL term can also be computed analytically. Thus, the optimal variational distribution $q_{\ast}(\bm{u})$ can be derived analytically, and it is a Gaussian,
\begin{equation}
\begin{split}
    q_{\ast}(\bm{u})=\mathcal{N}(\bm{u}; 
     \frac{1}{\sigma^2}\bm{K_{ZZ}}\bm{M}^{-1}\bm{K_{ZX}}\bm{y},
    \bm{K_{ZZ}}\bm{M}^{-1}\bm{K_{ZZ}}),
\end{split}
\label{eq:optimal_qu_svgp}
\end{equation}
where $\bm{M}:=\bm{K_{ZZ}}+\frac{1}{\sigma^2}\bm{K_{ZX}k_{XZ}}$. We commonly call the SVGP in this case sparse Gaussian process regression (SGPR; \Citep{svgp2009titsias}). Substituting $q_{\ast}(\bm{u})$ back into ELBO, we can obtain a so-called \textit{collapsed variational bound}, which becomes an objective to further tune the kernel hyperparameters $\hparam$ and inducing input locations $\bm{Z}$,
\begin{equation}
\begin{split}
    \mathcal{L}_{ELBO}(q_{\ast})=\log \mathcal{N}(\bm{y};\bm{0}, &\sigma^2\bm{I}+ \bm{K_{ZX}K_{ZZ}^{-1}K_{ZX}})\\&-\frac{1}{2\sigma^2}\text{Tr}(\bm{K_{XX}}-\bm{K_{ZX}K_{ZZ}^{-1}K_{ZX}}).
\end{split}
\end{equation}
Naturally, the corresponding posterior predictive distribution is also analytically tractable with a Gaussian form.
\begin{equation}
\begin{split}
    p(\bm{y}^\ast|\bm{X}^\ast,\data, \bm{Z})=\mathcal{N}(&\bm{y}^\ast;\frac{1}{\sigma^2}\bm{K_{X^\ast Z}\bm{M}^{-1}K_{ZX}}\bm{y},\\
    &\bm{K}_{\bm{X^\ast X^\ast}} + \sigma^2\bm{I} \underbrace{-\bm{K}_{\bm{X}^\ast\bm{Z}}(\bm{K}^{-1}_{\bm{ZZ}}-\bm{M}^{-1})^{-1}\bm{K}_{\bm{Z}\bm{X}^\ast})}_{\text{uncertainty reduction}}.
\end{split}
\end{equation}
For SVGP, the inducing points can be viewed as a compact summary of the training set, and the uncertainty reduction in the posterior predictive depends on the kernel similarity between $\bm{X}^\ast$ and $\bm{Z}$ (which tends to cover the whole region of the training inputs).

\subsection{Deep Kernel Learning and Deep Gaussian Processes}
It is not surprising that GPs may underperform in complex predictive tasks, compared with deep neural networks. After all, they are just shallow Bayesian linear models based on kernel feature as discussed previously. This section introduces two common approaches to enhance GPs with deep structures.

\paragraph{Deep kernel learning.}
While linear models rely on hand-crafted feature maps, deep neural networks can learn expressive feature representations for input data if we view the network right before the last layer as an expressive nonlinear feature extractor. Therefore, a natural way to improve GP is to combine kernel feature map and deep feature extractor. Deep kernel learning \citep{agw2015dkl} introduces deep kernel by parameterizing the kernel with a deep neural network. The network weights $\mparam$ then become hyperparameters of the deep kernel. Specifically, a deep kernel is obtained by compositing a regular base kernel $k_{base}(\cdot,\cdot)$ (e.g., RBF kernel) and a deep feature extractor parameterized by a deep neural network ($h_{\mparam}(\cdot)$): $k_{deep}(\cdot,\cdot)=k_{base}(h_{\mparam}(\cdot),h_{\mparam}(\cdot))$. 

\paragraph{Deep Gaussian processes.} In addition to incorporating deep structure in the kernel function, multiple GP layers, on their own, can also be stacked together to construct the so-called deep Gaussian processes (deep GPs) \citep{damia2013dgp}. Unlike standard GPs, deep GPs are not liner models, and therefore the posterior inference becomes intractable. Again, we resort to the variational inducing points method discussed previously to approximate the posterior process for each GP layer \citep{salimbeni2017doubly}. Consider an $L$-layer deep GP consisting of $L$ latent functions $\{f_l\}_{l=1}^L$. The input for each layer, $f_l$, is the function value output by the previous layer: $\bm{f}_l=f_{l}({\bm{f}_{l-1}})$. In each layer, a GP prior is used for its output: $p(f_l)=\mathcal{GP}(0(\cdot), k^{(l)}(\cdot,\cdot))$. Moreover, we also incorporate a set of inducing points $\{\bm{z}_l^{(i)}, u_l^{(i)}\}_{i=1}^{M_l}$ for each layer. The prior model is then
\begin{equation}
\begin{split}
    p(\bm{y},\{\bm{f}_l\}_{l=1}^L, & \{\bm{u}\}_{l=1}^L|\bm{X},\{\bm{Z}_l\}_{l=1}^L)=\\&
    p(\bm{y}|\bm{f}_L)\left[\prod_{l=2}^L p(\bm{f}_l|\bm{u}_l,\bm{Z}_l, \bm{f}_{l-1})p(\bm{u}_l|\bm{Z}_l)\right]p(\bm{f}_1|\bm{u}_1,\bm{Z}_1,\bm{X})p(\bm{u}_1|\bm{Z}_1).
\end{split}
\end{equation}
The approximate posterior of $\{\bm{u}_l\}_{l=1^L}$ is assumed to factorize between layers and the approximate posterior conditional distribution of $\bm{f}_l$ given $\bm{u}_l$, in each layer, is again assumed to be the same as the prior conditional distribution. In summary, the joint approximate posterior factorizes as follows:
\begin{equation}
    q(\{\bm{f}_l\}_{l=1}^L, \{\bm{u}_l\}_{l=1}^L|\bm{X},\{\bm{Z}_l\}_{l=1}^L)=\left[\prod_{l=2}^Lp(\bm{f}_l|\bm{u}_l,\bm{Z}_l, \bm{f}_{l-1})q(\bm{u}_l)\right]p(\bm{f}_1|\bm{u}_1,\bm{Z}_1,\bm{X})q(\bm{u}_1).
\end{equation}
The structured approximate posterior above allows us to obtain the following ELBO as training objective.
\begin{equation}
    \mathcal{L}_{ELBO}=\mathbb{E}_{q(\bm{f}_L|\bm{X},\{\bm{Z}_l\}_{l=1}^L)}[\log p(\bm{y}|\bm{f}_L)]-\sum_{l=1}^L \text{KL}(q(\bm{u}_l)||p(\bm{u}_l|\bm{Z}_l)),
\end{equation}
where 
\begin{equation}
q(\bm{f}_L|\bm{X},\{\bm{Z}_l\}_{l=1}^L)=\int\cdots\int q(\bm{f}_1|\bm{X},\bm{Z}_1)\prod_{l=2}^L q(\bm{f}_l|\bm{f}_{l-1}, \bm{Z}_l)  d\bm{f}_{1:L-1}.
\end{equation}
\begin{figure}[t]
\centering
   \includegraphics[width=0.6\linewidth]{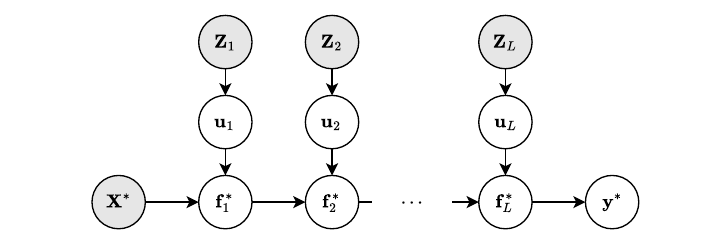}
\caption{The graphical model of prediction of an $L$-layer deep Gaussian process with sparse GP approximations.}
\label{fig:dgp_graphical}
\end{figure}

The approximate posterior for the final output defined by the above integral is not analytically tractable, and in practice we can only work with samples from it. Usually, the final output samples from $q(\bm{f}_L|\bm{x},\{\bm{Z}_l\}_{l=1}^L)$ is obtained by passing samples from $q(\bm{f}_l|\bm{f}_{l-1}, \bm{Z}_l)$ iteratively through each layer, which can be generated using the reparameterization trick since $q(\bm{f}_l|\bm{f}_{l-1}, \bm{Z}_l)$ is a Gaussian,
\begin{equation}
\begin{split}
q(\bm{f}_l|\bm{f}_{l-1}, \bm{Z}_l)&=\int p(\bm{f}_l|\bm{u}_l,\bm{Z}_l,\bm{f}_{l-1})q(\bm{u}_l) d \bm{u}_l\\
&=\mathcal{N}(\bm{K_{f_{l-1} Z_l}^{(l)}K^{(l)-1}_{Z_lZ_l}m_{u_l}},\\
&\quad\quad \bm{K^{(l)}_{f_{l-1} f_{l-1}}+K^{(l)}_{f_{l-1} Z_l}K_{Z_l Z_l}^{(l)-1}(S_{u_l}-K^{(l)}_{Z_lZ_l})K_{Z_lZ_l}^{(l)-1}K^{(l)}_{Z_lf_{l-1}}}).
\end{split}
\end{equation}
We show the graphical model of prediction of deep GP in Figure~\ref{fig:dgp_graphical}.
\section{Deep Sequential Generative Models}
\label{sec:dsgm}
In addition to predictive models, we also leverage probabilistic machine learning and deep neural networks to build deep generative models, which now form a rapidly evolving research field due to the tremendous advances from both topics. In short, deep generative models aim to estimate the true data distribution $p(\bm{x})$ with a parametric model $p_{\mparam}(\bm{x})$ learned from a training set consisting of samples from $p(\bm{x})$, $\data=\{\bm{x}_n\}_{n=1}^N$, where $p_{\mparam}(\bm{x})$ is parameterized with deep neural networks with parameters $\mparam$. After learning, new samples can be generated from $p_{\mparam}(\bm{x})$, which can be treated as approximate samples from the true data-generating process. This section introduces the framework of deep sequential generative models, and we discuss in detail two popular models in this framework, hierarchical variational autoencoder and diffusion model.

\subsection{Deep Latent Variable Models}
We first review deep latent variable models (DLVMs; \citep{welling2014auto,rezende:vae2014}) since many deep sequential generative models can be viewed as extensions of DLVMs. DLVMs assume that each observation $\bm{x}$ is generated by transforming a latent variable $\latent$, through a deep neural network $f_{\mparam}(\cdot)$ with parameters $\mparam$. Typically, the distribution of the latent variable, $p_{\mparam}(\z)$, is chosen to be a simple distribution (e.g., a standard Gaussian) a priori so that one can easily sample from it. We demonstrate the graphical model of a DLVM in Figure~\ref{fig:dlvm}. Mathematically, the generative model can be written as follows.
\begin{figure}[t]
    \centering
    \includegraphics[width=0.4\textwidth]{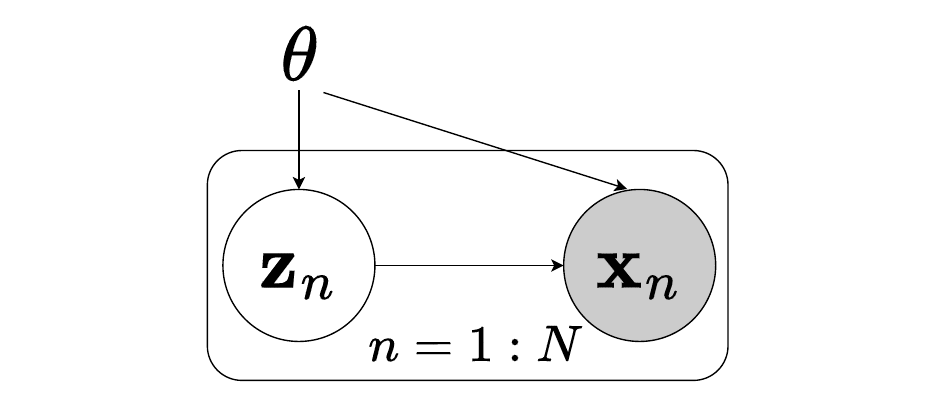}
    \caption{Graphical model of a deep latent variable model.}
    \label{fig:dlvm}
\end{figure}
\begin{equation}
\latent \sim p_{\mparam}(\z) := \mathcal{N}(\bm{0}, \bm{I}), \quad \bm{x} \sim p_{\mparam}(\bm{x} | \latent) := \mathcal{N}(f_{\mparam}(\z), \sigma^2 \bm{I}).
\label{eq:dlvm_model}
\end{equation}
Given a training set $\data=\{\bm{x}_n\}_{n=1}^N$, ideally, the model parameters $\mparam$ can be fitted by maximum likelihood estimation (MLE):
 \begin{equation}
     \mparam^* = \arg\max_{\mparam} \sum_{n=1}^N \log p_{\mparam}(\bm{x}_n), \quad p_{\mparam}(\bm{x}_n) = \int p_{\mparam}(\bm{x} | \latent) p_{\mparam}(\latent) d\latent.
     \label{eq:dlvm_mle}
 \end{equation}
However, the MLE objective (Eq.~\ref{eq:dlvm_mle}) above requires solving complex integrals w.r.t. the latent variables $\latent_n$ for the log marginal likelihoods $\log p_{\mparam}(\bm{x}_n)$, which are intractable since $f_{\mparam}(\z)$ is a nonlinear deep neural network. In practice, we again resort to approximate inference and many DLVMs are learned via VI-based methods \citep{welling2014auto,rezende:vae2014}. We discuss a class of representative VI-based DLVMs, variational autoencoders, in detail.

\paragraph{Variational Autoencoders.}
Variational Autoencoders (VAEs) \citep{welling2014auto,rezende:vae2014} are a class of DLVMs that employ VI for approximate maximum likelihood maximization. Notably, to reduce computational costs, they further consider amortized inference \citep{gershman2014amortized} to construct the approximate distributions in DLVMs: instead of constructing individual approximate distribution to the posterior $p_{\mparam}(\latent_n|\bm{x}_n)$ for each $\bm{x}_n$, VAEs use another deep neural network (commonly referred to as \textit{inference network} or \textit{encoder network}) to map each $\bm{x}_n$ to the variational parameters of its corresponding variational distribution. For example, when a factorized Gaussian distribution family, $q_{\vparam}(\latent|\bm{x})=\mathcal{N}(\bm{\mu}_{\vparam}(\bm{x}), \bm{\sigma}^2_{\vparam}(\bm{x}))$, is considered, the variational mean and variance associated with an observation $\bm{x}$ are then obtained by an inference network $g_{\vparam}(\bm{x}) := [\bm{\mu}_{\vparam}(\bm{x}), \log \bm{\sigma}_{\vparam}(\bm{x})]$ with amortized variational parameters $\vparam$.
For each observation $\bm{x}_n$, an ELBO can be constructed as a lower bound for the corresponding log marginal likelihood, and we can sum these ELBOs to form the final objective for joint training of the model parameters $\mparam$ and the variational parameters $\vparam$: 
\begin{equation}
\begin{aligned}
\mparam^*, \vparam^* &= \arg\max_{\mparam, \vparam} \sum_{n=1}^N \mathcal{L}_{ELBO}(\bm{x}_n, \mparam, q_{\vparam}), \\
\log p_{\mparam}(\bm{x}) \geq \mathcal{L}_{ELBO}(\bm{x}, \mparam, q_{\vparam}) &= \underbrace{\mathbb{E}_{q_{\vparam}(\latent|\bm{x})}[\log p_{\mparam}(\bm{x}|\latent)]}_{\text{ELL}} - \underbrace{\text{KL}(q_{\vparam}(\latent|\bm{x})\|p_{\mparam}(\latent))}_{\text{KL regularizer}}.
\end{aligned}
\label{eq:vae_objective}
\end{equation}
The intractable expectation in the ELL term can again be approximated with Monte Carlo estimation and usually one sample is sufficient for decent empirical performance. To enable backpropagation through the inference network, the reparameterization trick \citep{welling2014auto,rezende:vae2014} is again applied to sample from $q_{\vparam}(\bm{z}|\bm{x})$:
\begin{equation}
    \mathcal{L}_{ELBO}(\bm{x}, \mparam, q_{\vparam}) \approx \log p_{\mparam}(\bm{x}|\bm{\mu}_{\vparam}(\bm{x})+ \bm{\sigma}_{\vparam}(\bm{x}) \odot \bm{\epsilon}) - \text{KL}(q_{\vparam}\|p_{\mparam}), \quad \bm{\epsilon} \sim \mathcal{N}(\bm{0}, \bm{I}).
\label{eq:vae_objective_mc_reparam}
\end{equation}
The inference network $q_{\vparam}(\latent | \bm{x})$ and the DLVM's conditional distribution $p_{\mparam}(\bm{x} | \latent)$ can be viewed as a stochastic encoder and decoder, respectively. This explains why we call the method an ``autoencoder''. 


\subsection{Hierarchical Variational Autoencoders}
\label{sec:hvae_back}
\begin{figure}[t]
    \centering
    \includegraphics[width=0.5\textwidth]{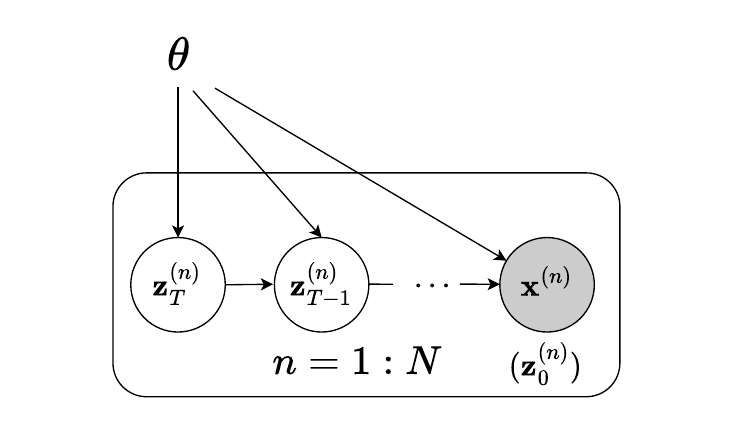}
    \caption{Graphical model of a deep sequential generative variable model.}
    \label{fig:hvae_graphical}
\end{figure}
It can be challenging for DLVMs with only one latent variable to generate realistic observations, such as photo-realistic images, since they try to transform noise to target distribution in just one shot. Deep sequential generative models (DSGMs) extend DLVMs by introducing a sequence of latent variables and transforming noise into the target distribution using multiple intermediate steps. Specifically, they employ a sequence of latent variables $\z_1,\cdots,\z_T$ to capture representations of the data $\z_0\coloneqq \x$ at different fidelity \citep{salimans2016structured}:
\begin{equation}
    \x=\z_0 \sim p_{\mparam}(\z_0)=\int p(\z_T) \prod_{t=1}^T p_{\mparam}(\z_{t-1}|\z_t) d\z_{1:T},\label{eq:hvae-diffusion}
\end{equation}
where the prior distribution over the last latent variable $\z_T$ is often set to a simple distribution, such as standard Gaussian $p(\z_T):=\mathcal{N}(\z_T;\bm{0},\bm{I})$, and the latent conditional distributions are parameterized by deep neural networks (we often only parameterize the conditional means):
\begin{equation}
    p_{\mparam}(\z_{t-1}|\z_t):=\mathcal{N}(\z_{t-1};\bm{\mu}_{\mparam}(\z_t,t),\sigma_t^2 \bm{I}).\label{eq:hvae-denoising}
\end{equation}
We show the graphical model of a DSGM in Figure~\ref{fig:hvae_graphical}.

Similar to DLVMs, direct MLE estimation for model parameters $\mparam$ of DSGMs is intractable, and again VI provides a principled framework for constructing a lower bound objective for log likelihood. A prominent VI-based DSGM is hierarchical variational autoencoders (HVAEs; \citep{sonderby2016ladder,maaloe2019biva,vahdat2020nvae}). Specifically, HVAEs approximate the intractable posterior
\begin{equation}
   p_{\mparam}(\z_{1:T}|\z_0)= \frac{p(\z_T) \prod_{t=1}^T p_{\mparam}(\z_{t-1}|\z_t)}{p_{\mparam}(\z_0)}
\end{equation}
with an amortized variational distribution (or inference model) $q_{\vparam}(\z_{1:T}|\z_0)$. Different design choices for the factorization of the inference model have been proposed, including ``bottom-up'' factorization \citep{burda2015importance}:
\begin{equation}
    q_{\vparam}(\z_{1:T}|\z_0)=\prod_{t=1}^T q_{\vparam}(\z_t|\z_{t-1}),
\end{equation} 
and ``top-down'' factorization \citep{sonderby2016ladder}:
\begin{equation}
    q_{\vparam}(\z_{1:T}|\z_0)=q_{\vparam}(\z_T|\z_0) \prod_{t=1}^T q_{\vparam}(\z_{t-1}|\z_t,\z_0),
\end{equation}
where the factors $q_{\vparam}(\z_t|\z_{t-1})$ and $q_{\vparam}(\z_{t-1}|\z_t,\z_0)$ are factorized Gaussian distributions with mean and diagnoal variance parameterized by neural networks. Given a set of training observations $\data=\{\bm{x}^{(n)}\}_{n=1}^N$ ($\z_0^{(n)}:=\x^{(n)}$), we again jointly train both the generative model and the inference model by maximizing the sum of ELBOs:
\begin{equation}
\begin{aligned}
\mparam^*, \vparam^* &= \arg\max_{\mparam, \vparam} \sum_{n=1}^N \mathcal{L}_{ELBO}(\bm{x}^{(n)}, \mparam, q_{\vparam}), \\
    \log p_{\mparam}(\bm{x}^{(n)}) \geq \mathcal{L}_{ELBO}(\bm{x}^{(n)},\mparam,& q_{\vparam})=\mathbb{E}_{q_{\vparam}(\z^{(n)}_{1:T}|\z^{(n)}_0)}\left[\log \frac{p(\z^{(n)}_T) \prod_{t=1}^T p_{\mparam}(\z^{(n)}_{t-1}|\z^{(n)}_t)}{q_{\vparam}(\z^{(n)}_{1:T}|\z^{(n)}_0)}\right].
\end{aligned}
\label{eq: have_loss_back}
\end{equation}
Monte Carlo estimation and reparameterization trick can again be employed to further simplify the above objective and enable backpropagation.
%
%
%
%
\subsection{Diffusion Models}
\label{sec:diffusion_back}
Diffusion models \citep{sohl2015deep,ho2020denoising,song2020score} are arguably the most successful DSGMs in recent years and they achieve photo-realistic image generation quality. It assumes a Markovian data-generating process (or denoising process) based on a sequence of latent variables as in Eq.~\ref{eq:hvae-diffusion} and the parameterization of the conditional distributions is similar to that in Eq.~\ref{eq:hvae-denoising}. However, unlike HVAEs, diffusion models define a fixed ``bottom up'' inference model (or diffusion process):
\begin{equation}
    q(\z_{1:T}|\z_0)=\prod_{t=1}^T q(\z_t|\z_{t-1})=\prod_{t=1}^T \mathcal{N}(\z_t;\sqrt{\alpha_t} \z_{t-1}, (1-\alpha_t)\bm{I}), \quad \z_0:=\x
    \label{eq:diffusion_inference}
\end{equation}
where $q(\z_t|\z_{t-1})$'s are predefined Gaussian convolution kernels and no dimensionality reduction for latent variables is performed (i.e., $d=d_x$, where $d_x=\text{dim}(\x)$). Diffusion models are also trained by maximizing the ELBO but only with respect to the parameters $\mparam$ of the generation model, since the inference model is fixed and therefore does not include any trainable parameters. Since the inference model is simply defined to be a sequence of fixed Gaussian convolutions, one can analytically derive the corresponding ``top-down'' version of this inference model, which is of the form:
\begin{equation}
    q(\z_{t-1}|\z_t,\z_0)=\mathcal{N}(x_{t-1};\tilde{\mu}(\z_t,\z_0),\tilde{\beta}_t^2),
\end{equation}
where $\tilde{\mu}$ and $\tilde{\beta}_t$ have analytical solutions with closed-form expressions. As a result of the simple and fixed inference model, the ELBO of diffusion models for each training observation $\x$ ($\z_0$) admits the following simplified form:
\begin{equation}
\begin{split}
    \mathcal{L}_{ELBO}(\x, \mparam)
    &=\mathbb{E}_{q(\z_{1:T}|\z_0)}\left[\log p_{\mparam}(\z_0|\z_1) - \sum_{t=1}^T  \text{KL}(q(\z_{t-1}|\z_t,\z_0)||p_{\mparam}(\z_t|\z_{t-1})) \right]\\
    &=\mathbb{E}_{q(\z_{1:T}|\z_0)}\left[\log p_{\mparam}(\z_0|\z_1) -\sum_{t=1}^T  \underbrace{\frac{\lVert \mu_{\mparam}(\z_t, t) - \tilde{\mu}(\z_t,\z_0) \rVert^2}{2\sigma_t^2}}_{\text{mean matching}} \right].
\end{split}
\label{eq: diffusion_loss_back}
\end{equation}
The fixed inference model imposes strong inductive bias for the DSGM as the DSGM now is optimized towards a denoising process for a fixed diffusion process. The inductive bias can also be seen from the direct supervision signal for each intermediate latent variable in the ``mean matching" term of the above ELBO. In practice, during each training iteration, we typically just sample one term from the above ELBO (consisting of $T+1$ terms) for each training observation to form a even simpler objective.

Based on the structure of $\tilde{\mu}(\z_t, \z_0)$, one can also reparameterize $\mu_{\mparam}(\z_t, t)$ accordingly to rewrite the ``mean matching" term in other forms. For example, using the reparameterization trick, the noise parameterization of diffusion models \citep{ho2020denoising, song2021denoising} rewrites $\z_0=\frac{1}{\sqrt{\bar{\alpha}_t}} \z_t- \sqrt{\frac{1-\bar{\alpha}_t}{\bar{\alpha_t}}} \bm{\epsilon}$, $\bm{\epsilon} \sim \mathcal{N}(\bm{0},\bm{I})$, where $\sqrt{\bar{\alpha}_t}$'s are fixed scalars determined by the fixed inference model. With it we can further reparameterize $\mu_{\mparam}(\z_t, t) = \tilde{\mu}(\z_t, \frac{1}{\sqrt{\bar{\alpha}_t}} \z_t- \sqrt{\frac{1-\bar{\alpha}_t}{\bar{\alpha_t}}} \epsilon_{\mparam}(\z_t,t))$, which allows us to rewrite the ``mean matching" term into a ``noise matching" term:
\begin{equation}
    \lambda_t \lVert \epsilon_{\mparam}(\z_t, t) - \bm{\epsilon} \rVert^2 = \bar{\lambda}_t \lVert \epsilon_{\mparam}(\sqrt{\bar{\alpha}_t} \z_0+ \sqrt{1-\bar{\alpha}_t} \bm{\epsilon}, t) - \bm{\epsilon} \rVert^2, \quad \bm{\epsilon} \sim \mathcal{N}(\bm{0}, \bm{I}),
\end{equation}
where $\lambda_t$'s are positive scalars determined by analytical form of $\tilde{\mu}$. This objective encourages the predicted noise $\epsilon_{\mparam}(\z_t, t)$ to match the noise $\bm{\epsilon}$ used to corrupt the observation $\x$ ($\z_0$) to its noisy version $\z_t$. Alternatively, by exploiting the connection between the denoising mean and the score of the noisy model, one can also obtain score-matching based objective to train diffusion models \citep{song2020score}.

\chapter{Calibrating Transformers via Sparse Gaussian Processes}
\label{cha:sgpa}

\begin{tcolorbox}           [enhanced,colback=gray!5!white,colframe=gray!75!black,colbacktitle=red!80!black,fonttitle=\bfseries]
  This chapter is based on \citet{chen2023calibrating}:
  \begin{itemize}
  \item  \underline{Wenlong Chen} and Yingzhen Li. \textbf{Calibrating Transformers via Sparse Gaussian Processes}. In \textit{International Conference on Learning Representations (ICLR)}, 2023.
  \end{itemize}
  The main idea was developed by me while the last author provided useful suggestions. The manuscript was mainly written by me while the last author helped polish it. Moreover, I wrote all the code and performed all the experiments for this work.
\end{tcolorbox}

In this chapter, we propose an approximate Bayesian inference method, sparse Gaussian process attention, to improve the uncertainty quantification in Transformer models. Generic weight-space inference techniques treat all deep neural networks as some parametric models and ignore the inductive bias within their architectures. Furthermore, because of the complex interactions between weights during the forward pass, there is no good reason to favor any specific prior or approximation techniques in these weight-space methods. In contrast, sparse Gaussian process attention is inspired by the identification of the connection between attention architecture in Transformers and the posterior mean of sparse variational Gaussian processes. This connection naturally informs us to use GP priors over the attention outputs and construct approximate posterior for them based on sparse GPs.

\section{Introduction}
Significant improvements have been made for deep sequence models in predictive tasks, and as reviewed in Section~\ref{sec:transformer}, Transformer \citep{vaswani2017attention}, a deep sequence model architecture based on multi-head attention (MHA), has gained popularity in recent years. 
With Transformers being deployed in many downstream applications \citep{vaswani2017attention, dosovitskiy2020image, brown2020language}, it is crucial to prevent poor robustness which often comes from erratic outputs with high confidence from these models \citep{guo2017calibration, mukhoti2020calibrating}.
This requires calibrated uncertainty quantification for Transformers which is much less well-studied at the time of this work, and it raises concerns about using Transformers for safety-critical tasks which require rational and risk-averse decision making under uncertainty.

Regarding uncertainty quantification, Bayesian inference is a powerful and principled framework to build probabilistic models for rational prediction and decision-making under uncertainty \citep{gal2016thesis}. Significant progress is observed for applying (approximate) Bayesian inference methods to quantify uncertainty in fully-connected, convolutional and recurrent neural networks \citep{blundell2015bbp,gal2016mcdrop,zhang2019cyclical,ritter2021sparse}. Initial efforts have been made on extending these techniques to Transformers but with mixed results \citep{tran2019bl,xue2021bayestran}.
On the other hand, Gaussian processes (GPs), reviewed in Section~\ref{sec:gp_back}, are gold standard methods for tasks requiring reliable function-space uncertainty estimates \citep{rasmussen2006gp, wilsonl2020eff}. Researchers have proposed to integrate deep learning ideas to GP model design, including deep kernel learning \citep{agw2015dkl} and deep GPs \citep{damia2013dgp,salimbeni2017doubly}. Still, these models have yet to be scaled to modern deep learning tasks such as large-scale image classification and language modeling.

In this work, we propose sparse Gaussian process attention (SGPA), a novel uncertainty quantification technique tailored to attention-based models (e.g., Transformers), by leveraging techniques from sparse variational Gaussian processes (SVGP) \citep{snelson2005sgp, hensman2013gpbig} for improved uncertainty estimates. Our work presents the following insights and contributions:
\begin{itemize}
    \item (Section~\ref{sec:generic_sgpa}) Our key observation is that kernel-based attention \citep{tsai2019TransformerDissection} is equivalent to the posterior mean of an SVGP. This inspires us to specify GP priors for the attention outputs and leverage SVGP techniques for approximate posterior inference in Transformers. The resulting Transformer based on our SGPA approach can be viewed as a sparse deep GP \citep{salimbeni2017doubly} with deep kernel in use for each GP layers.
    \item (Section~\ref{sec:decoupled_sgpa} \&~\ref{sec:sgpa_transformer}) We address the computational inefficiency issues of a naive extension of SVGP to multi-head self-attention with decoupled inducing points techniques \citep{cheng2017variational,sali2018ortho}, making SPGA scalable to deep learning tasks that Transformers are applied to.
    
    \item (Section~\ref{sec:sgpa_experiments}) Empirically, on a variety of vision, NLP and graph prediction tasks and compared with baselines, SGPA-based Transformers improve considerably over in-distribution calibration, out-of-distribution (OOD) robustness, and OOD detection, while achieving competitive accuracy against Transformers with standard \citep{vaswani2017attention} or kernel attention \citep{tsai2019TransformerDissection}.
\end{itemize}
\section{Sparse Gaussian Process Attention}
\label{sec:generic_sgpa}
We propose Sparse Gaussian Process Attention (SGPA) to perform approximate Bayesian inference for Transformer-based models. The key idea is to replace the softmax operation in scaled dot-product attention \citep{vaswani2017attention} with a kernel \citep{tsai2019TransformerDissection}, and connect the resulting attention to the mean of an SVGP. This insight allows us to apply SVGP equations for uncertainty estimation, and we further introduce decoupled inducing points to improve computational efficiency.

%
\subsection{Attention as the Mean of a Sparse Variational Gaussian Process}
%
Standard Transformers use attention blocks based on scaled dot-product \citep{vaswani2017attention}. Given queries $\bm{q}\in\mathbb{R}^{T\times d_q}$, keys $\bm{k}\in\mathbb{R}^{T\times d_k}$ and value $\bm{v}\in\mathbb{R}^{T\times d_v}$, the scaled dot-product (SDP) attention is given as follows:
%
%
\begin{equation}
    \text{SDP-Attention:} \quad \bm{F} = \text{softmax}(\frac{\bm{qk}^\top}{\sqrt{d_k}})\bm{v},
\label{eq:sdp}
\end{equation}
%
%
where $d_k$ ($=d_q$) is the dimension of the keys. Since attention involves measuring the similarity between $\bm{q}$ and $\bm{k}$, \cite{tsai2019TransformerDissection} generalized SDP-Attention by replacing $\text{softmax}(\frac{\bm{qk}^\top}{\sqrt{d_k}})$ in Eq.~\ref{eq:sdp} with a kernel gram matrix $\bm{K}_{\bm{qk}}$ ($[\bm{K}_{\bm{qk}}]_{i,j}=k(\bm{q}_i,\bm{k}_j)$) computed using a valid symmetric kernel $k(\cdot, \cdot)$, for which we refer to it as kernel attention or $K$-Attention for short:
%
%
\begin{equation}
    K\text{-Attention:} \quad \bm{F}=\bm{K}_{\bm{qk}}\bm{v}.
\label{eq:kattn}
\end{equation}
%
%
Recall that the posterior mean of SVGP in Eq.~\ref{eq:q_predictive_svgp} reviewed in the Section~\ref{sec:svgp_back} is $\bm{m}=\bm{K}_{\bm{XZ}}\bm{K}_{\bm{ZZ}}^{-1}\bm{m}_{\bm{u}}$, when evaluated on inputs $\bm{X}$. Now we reparameterize the variational mean parameter of SVGP as $[\bm{v}]_{:,d} := \bm{K}_{\bm{ZZ}}^{-1}\bm{m}_{\bm{u}}$ for each dimension ($d$) of $\bm{v}$, and define the queries and keys as the input locations and inducing point locations: $\bm{q}:=\bm{X}, \bm{k}:=\bm{Z}$. This allows us to rewrite the posterior mean of SVGP as
\begin{equation}
    \text{SVGP-mean}\quad \bm{m}=\bm{K}_{\bm{qk}}[\bm{v}]_{:,d}
\end{equation}
By doing so, equivalence can be identified between the posterior mean of an SVGP and each dimension of the output of a kernel attention block. This allows us to naturally extend the toolbox of Gaussian processes and their scalable approximations for quantifying uncertainty in Transformers in the following sections.

\begin{figure}[t]
\centering
\begin{subfigure}{0.33\textwidth}
\centering
   \includegraphics[width=0.999\linewidth]{C3/figures/self_attn.pdf}
\caption{Vanilla Transformer}
\label{fig:attn_ch3}
\end{subfigure}
\hfill
\begin{subfigure}{0.320\textwidth}
\centering
   \includegraphics[width=0.999\linewidth]{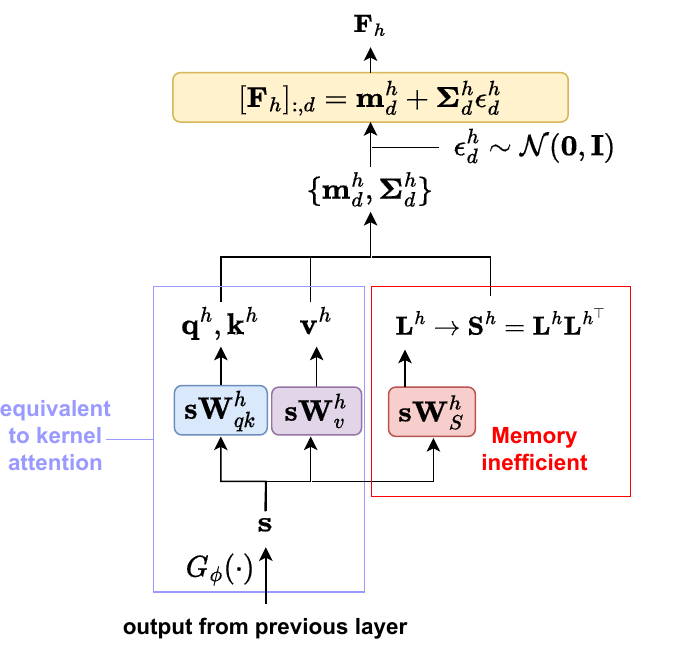}
\caption{Standard SGPA (ours)}
\label{fig:sgpa_v1}
\end{subfigure}
\hfill
\begin{subfigure}{0.333\textwidth}
\centering
   \includegraphics[width=0.999\linewidth]{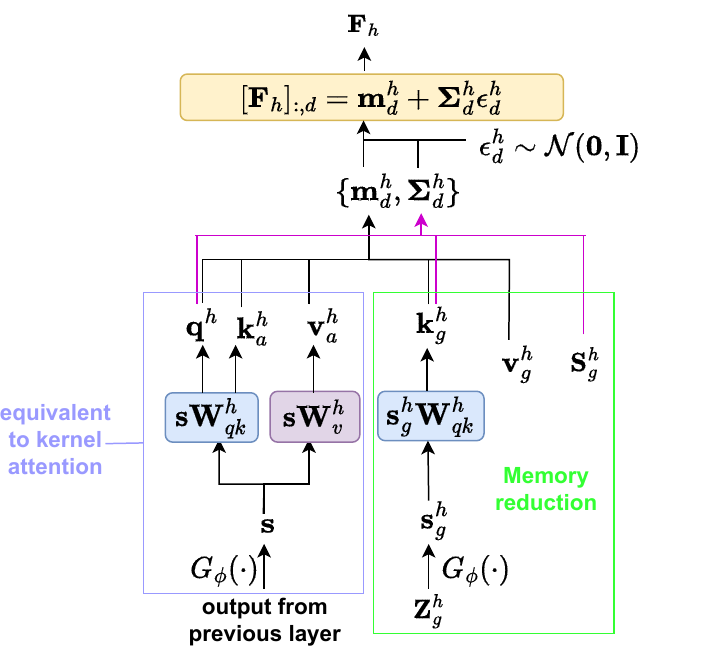}
\caption{Decoupled SGPA (ours)}
\label{fig:sgpa}
\end{subfigure}
\caption{Illustration of one head (h) of multi-head self attention in one layer of (a) vanilla Transformer, (b) Transformer based on standard SGPA and (c) Transformer based on decoupled SGPA.}
\end{figure}
\subsection{Standard SGPA \& its Inefficiency for Self-attention}
%
Observing the equivalence between $K$-Attention and SVGP mean, a natural idea for uncertainty estimation is to apply SVGP techniques for approximate posterior variance computations. In detail, we introduce a set of variational covariance parameters $\bm{S}\in \mathbb{R}^{T\times T\times d_v}$ (with $T$ as the number of keys/inducing inputs), and optimize them using the ELBO (Eq.~\ref{eq:elbo_svgp}) as discussed in Section~\ref{sec:svgp_back} (see Appendix~\ref{a4:svgp_elbo} for full derivation). This procedure returns the mean and covariance for each dimension ($d$) of the posterior attention output as:
%
%
\begin{equation}
    \bm{m}_{d}=\bm{K}_{\bm{q} \bm{k}}[\bm{v}]_{:,d},\quad
    \bm{\Sigma}_{d}= \bm{K}_{\bm{q} \bm{q}} + \bm{K}_{\bm{q} \bm{k}}(\bm{K_{kk}^{-1}[S]_{:,:,d}K_{kk}^{-1}-K^{-1}_{kk}})\bm{K}_{\bm{k}\bm{q}}.
\label{eq:fpost}
\end{equation}
%
%
In this way, we fit an SVGP for each dimension of attention outputs independently: for each dimension $d$, an SVGP given in Eq.~\ref{eq:fpost} is fitted using the same kernel, but with different variational mean ($[\bm{v}]_{:,d}$) 
and covariance ($[S]_{:,:,d}$) parameters. We name this approach as \textit{standard SGPA} and provide a visualization of the operations in Figure~\ref{fig:sgpa_v1}.

%
%
Unfortunately, standard SGPA becomes computationally inefficient when applied to Transformers based on multi-head self-attention. In each attention layer the keys for each head $h$, $\bm{k}^h$, are obtained by passing the output from previous layer through a neural network. Moreover, the projection matrices (defined in Section~\ref{sec:transformer}) for queries and keys are tied as in \citet{tsai2019TransformerDissection} (i.e., $\bm{W}^{h}_{q}=\bm{W}^{h}_{k}:=\bm{W}^{h}_{qk}\in\mathbb{R}^{d_s \times d_k}$) to obtain a valid symmetric kernel.
As a result, the queries and keys in a self-attention layer are the same, more importantly they are input-dependent, i.e., $\bm{k}^h = \bm{s} \bm{W}_k^h$ (where $\bm{s} \in \mathbb{R}^{T\times d_s}$ is the input to the MHSA block) and they vary as the input sequence to the Transformer changes. Therefore, to extend standard SVGP framework (Eq.~\ref{eq:fpost}) to self-attention, the covariance parameters $\bm{S}^h$ need to be input-dependent as well to accommodate the varying inducing inputs $\bm{k}^h$. A naive idea would parameterize $\bm{S}^h$ by linear projection, e.g., $vec(\bm{L}^h) = \bm{s} \bm{W}_s^h$ for one head where $\bm{L}^h$ is the Cholesky factor of $\bm{S}^h$ (see Figure~\ref{fig:sgpa_v1}). This will incur a memory cost of $O(T^2)$ per head even if we tie the variational covariances across output dimensions, and a run-time cost of $O(T^3)$ per head per input sequence for inverting $\bm{K}_{\bm{k}^h\bm{k}^h}$ as $\bm{k}^h$ is input-dependent. Therefore standard SGPA is both memory and run-time inefficient especially for long input sequences. 

\begin{remark}
    Unlike the case where variational parameters are freely-optimized variables, note that in self-attention settings, different parameterizations of the posterior mean in general define different posterior predictive distributions. To see this, suppose that the value for one head in a MHSA block $\bm{v}^h$ is obtained by projection $\bm{v}^h=\bm{s}^h\bm{W}_v^h$ ($\bm{W}_v^h \in \mathbb{R}^{d_s \times d_v}$) and the reparameterized SVGP posterior mean yields $\bm{m}^h=\bm{K}_{q^hk^h}\bm{v}^h=\bm{K}_{q^hk^h}\bm{s}^h\bm{W}_v^h$. Now given the canonical posterior mean parameterization and the corresponding values obtained via $\tilde{\bm{v}}^{h}=\bm{s}^h\tilde{\bm{W}}_v^h$, in order to obtain the same posterior mean as the reparameterized one for all $N$ sequences in the dataset, we need to solve for $\tilde{\bm{W}}_v^h$ with $N$ linear systems of the form
    \begin{equation}
    \begin{split}
        \bm{K}_{q^hk^h}\bm{K}_{k^hk^h}^{-1}\bm{s}^h\tilde{\bm{W}}_v^h=\bm{K}_{q^hk^h}\bm{s}^h\bm{W}_v^h \iff \left(\bm{K}_{k^hk^h}^{-1}\bm{s}^h\right)\tilde{\bm{W}}_v^h=\bm{s}^h\bm{W}_v^h.
    \end{split}
    \end{equation}
    Hence, we need to solve a linear system with $N\times T\times d_v$ equations in total, and $d_s \times d_v$ variables. In general, since $N\times T \gg d_s$, it does not have a solution.
\end{remark}

\section{Improving Time \& Memory Efficiencies via Decoupled SGPA}
\label{sec:decoupled_sgpa}

We propose to address the aforementioned inefficiency issues by extending the decoupled sparse Gaussian process approximation \citep{cheng2017variational,sali2018ortho} to self-attention SGPA. 
\subsection{Preliminaries of Decoupled SVGPs}
Decoupled SVGPs \citep{cheng2017variational,sali2018ortho} can be interpreted as SVGPs but with structured variational distributions for the inducing points. Suppose we split the inducing points into two sets: $\{(\bm{z}_{g}^{(m)}, u_{g}^{(m)})\}_{m=1}^{M_{g}}$ and $\{(\bm{z}_{a}^{(m)}, u_{a}^{(m)})\}_{m=1}^{M_{a}}$. The original decoupled SVGP (DSVGP; \citep{cheng2017variational}) considers a structured variational distribution over $\bm{u} := \bm{u}_{g\cup a}=\bm{u}_{g}\cup \bm{u}_{a}$:
\begin{equation}
    q(\bm{u})=q(\bm{u_{g}})q(\bm{u}_a|\bm{u}_g)
    =\mathcal{N}(\bm{u}_g; \bm{0}, \bm{S}_g)\mathcal{N}(\bm{u}_{a}; \bm{C}\bm{u}_g+\bm{m}_a,, \bm{B}),
\end{equation}
where $\bm{B}=\bm{K_{Z_a Z_a}-K_{Z_a Z_g}K_{Z_g Z_g}^{-1}K_{Z_g Z_a}}$ and $\bm{C}=\bm{K_{Z_a Z_g}K_{Z_g Z_g}^{-1}}$. The corresponding variational mean and covariance are:
\begin{equation}
    \bm{m}_{\bm{u}}=\begin{pmatrix}
     \bm{0} \\ 
     \bm{m}_{a}
    \end{pmatrix},\quad
    \bm{S}_{\bm{u}}=\begin{pmatrix}
     \bm{S}_g & \bm{S}_g\bm{C}^\top \\
     \bm{C} \bm{S}_g & \bm{B}+\bm{C} \bm{S}_g  \bm{C}^\top
    \end{pmatrix}.
\end{equation}
Some terms used in the posterior predictive covariance of SVGP can then be canceled out:
\begin{equation}
    \bm{K}^{-1}_{\bm{ZZ}} \bm{S_u} = \begin{bmatrix}
    (\bm{K}^{-1}_{\bm{Z}_g \bm{Z}_g} \bm{S}_{g} - \bm{I}) \bm{K}^{-1}_{\bm{Z}_g \bm{Z}_g} \bm{K}_{\bm{Z}_a \bm{Z}_g} & \bm{K}^{-1}_{\bm{Z}_a \bm{Z}_a} \bm{S}_{g}\\
    \bm{I} & \bm{0} 
    \end{bmatrix},
    \label{eq:cancel1}
\end{equation}
\begin{equation}
    \bm{K}^{-1}_{\bm{ZZ}} \bm{S_u} \bm{K}^{-1}_{\bm{ZZ}} - \bm{K}^{-1}_{\bm{ZZ}} = \begin{bmatrix}
    \bm{K}^{-1}_{\bm{Z_g Z_g}} \bm{S}_{g} \bm{K}^{-1}_{\bm{Z_g Z_g}} - \bm{K}^{-1}_{\bm{Z_g Z_g}} & \bm{0} \\
    \bm{0} & \bm{0}
    \end{bmatrix}.
\label{eq:cancel2}
\end{equation}
Plugging this structured variational mean and covariance parameters into the posterior predictive distribution of SVGP (Eq.~\ref{eq:q_predictive_svgp}) allows us to obtain the posterior distribution $q(\bm{f} |\bm{X}, \bm{Z})$ used in DSVGP, which is a Gaussian with the following mean $\bm{m}_{\bm{f}}$ and covariance $\bm{\Sigma}_{\bm{f}\bm{f}}$:
\begin{equation}
    \begin{split}
         \bm{m}_{\bm{f}}&=\bm{K_{XZ_a}}\bm{K}_{\bm{Z_{a}Z_{a}}}^{-1}\bm{m}_{a}, \\
         \bm{\Sigma}_{\bm{f} \bm{f}}&= K_{\bm{X} \bm{X}} + \bm{K}_{\bm{X} \bm{Z}_{g}}\bm{K}_{\bm{Z}_{g}\bm{Z}_{g}}^{-1}(\bm{S}_{g}-\bm{K}_{\bm{Z}_{g}\bm{Z}_{g}})\bm{K}^{-1}_{\bm{Z}_{g}\bm{Z}_{g}}\bm{K}_{\bm{Z}_{g}\bm{X}}.
    \end{split}
    \label{eq:dsvgp_predcitive}
\end{equation}
Clearly, the inducing points used for this posterior predictive mean and covariance are now completely decoupled. Moreover, since in the posterior predictive distribution $q(\bm{f}|\bm{X}, \bm{Z})$, the variational mean parameter $\bm{m}_{a}$ always appears together with a left matrix multiplication of $\bm{K}_{\bm{Z_a Z_a}}^{-1}$, we can reparameterize the variational mean as $\bm{v}_{a}:=\bm{K_{Z_a Z_a}}^{-1}\bm{m}_{a}$, which is then treated as a final variational parameter to optimize. As a result, the posterior predictive mean can be rewritten as $\bm{m}_{\bm{f}}=\bm{K_{XZ_a}}\bm{v}_{a}$.

Orthogonally decoupled SVGP (ODSVGP; \citep{sali2018ortho}) considers another structured Gaussian variational distribution for the two sets of inducing points:
\begin{equation}
\begin{split}
    q(\bm{u})&=q(\bm{u_g})q(\bm{u}_a|\bm{u}_g)\\
    &=\mathcal{N}(\bm{u}_g; \bm{m}_g, \bm{S}_g)\mathcal{N}(\bm{u}_a; \bm{C}\bm{u}_g+\bm{B}\bm{K_{Z_a Z_a}}^{-1}\bm{m}_a, \bm{B}),
\end{split}
\end{equation}
The corresponding variational mean and covariance become:
\begin{equation}
    \bm{m}_{\bm{u}}=\begin{pmatrix}
    \bm{m}_{g} \\ \bm{C} \bm{m}_{g} + \bm{B}\bm{K_{Z_a Z_a}}^{-1}\bm{m}_{a}
    \end{pmatrix},\quad
    \bm{S}_{\bm{u}}=\begin{pmatrix}
     \bm{S}_g & \bm{S}_g\bm{C}^\top \\
     \bm{C} \bm{S}_g & \bm{B}+\bm{C} \bm{S}_g  \bm{C}^\top
    \end{pmatrix}.
\label{eqa:odsvgp}
\end{equation}
Again, plugging the above $\bm{m}_{\bm{u}}$ and $\bm{S}_{\bm{u}}$ in Eq.~\ref{eq:q_predictive_svgp}, we can obtain the posterior predictive distribution of ODSVGP for $\bm{f}$ after canceling some terms (Eq.~\ref{eq:cancel1} and ~\ref{eq:cancel2}). It is a Gaussian with mean $\bm{m}_{\bm{f}}$ and covariance $\bm{\Sigma}_{\bm{f}\bm{f}}$ given as:
\begin{equation}
    \begin{split}
         \bm{m}_{\bm{f}}&=\bm{K_{XZ_a}}\bm{K_{Z_aZ_a}}^{-1}\bm{m}_{a}-\bm{K}_{\bm{X} \bm{Z}_g} \bm{K}_{\bm{Z}_{g}\bm{Z}_{g}}^{-1} \bm{K}_{\bm{Z}_g \bm{Z}_a} \bm{K}_{\bm{Z}_a \bm{Z}_a}^{-1}\bm{m}_{a}+\bm{K}_{\bm{X} \bm{Z}_g}\bm{K}_{\bm{Z}_{g}\bm{Z}_{g}}^{-1}\bm{m}_g, \\
         \bm{\Sigma}_{\bm{f} \bm{f}}&= K_{\bm{X} \bm{X}} + \bm{K}_{\bm{X} \bm{Z}_{g}}\bm{K}_{\bm{Z}_{g}\bm{Z}_{g}}^{-1}(\bm{S}_{g}-\bm{K}_{\bm{Z}_{g}\bm{Z}_{g}})\bm{K}^{-1}_{\bm{Z}_{g}\bm{Z}_{g}}\bm{K}_{\bm{Z}_{g}\bm{X}}.
    \end{split}
\end{equation}
By again reparameterizing the variational mean parameters, $\bm{v}_{a}:=\bm{K}_{\bm{Z}_{a}\bm{Z}_{a}}^{-1} \bm{m}_{a}$ and $\bm{v}_{g}:=\bm{K}_{\bm{Z}_{g}\bm{Z}_{g}}^{-1}\bm{m}_g$, we arrive at the final expressions of the posterior predictive mean and covariance of ODSVGP:
\begin{equation}
    \begin{split}
         \bm{m}_{\bm{f}}&=\bm{K}_{\bm{X} \bm{Z}_a} \bm{v}_{a} -\bm{K}_{\bm{X} \bm{Z}_g} \bm{K}_{\bm{Z}_{g}\bm{Z}_{g}}^{-1} \bm{K}_{\bm{Z}_g \bm{Z}_a} \bm{v}_{a}
         +\bm{K}_{\bm{X} \bm{Z}_g}\bm{v}_{g}, \\
         \bm{\Sigma}_{\bm{f} \bm{f}}&= K_{\bm{X} \bm{X}} + \bm{K}_{\bm{X} \bm{Z}_{g}}\bm{K}_{\bm{Z}_{g}\bm{Z}_{g}}^{-1}(\bm{S}_{g}-\bm{K}_{\bm{Z}_{g}\bm{Z}_{g}})\bm{K}^{-1}_{\bm{Z}_{g}\bm{Z}_{g}}\bm{K}_{\bm{Z}_{g}\bm{X}}.
    \end{split}
    \label{eq:dsgp_predictive}
\end{equation}
Notice that unlike DSVGP, the posterior predictive mean of ODSVGP are computed based on both sets of inducing points. Through our preliminary experiments at the stage of hyperparameter tuning, we find ODSVGP works better than DSVGP when applied in the SGPA framework (see Appendix~\ref{a2:ssgpa}), so we focus on integrating ODSVGP in SGPA and call the resulting method decoupled SGPA from now on.

\subsection{Decoupled SGPA}
When applying ODSVGP in the framewok of SGPA, in addition to input-dependent (or ``amortized'') keys/inducing inputs $\bm{k}^h$, which we will call $\bm{k}_{a}^h$ from now on, for each head $h$, we also incorporate another $M_g$ number of ``global'' keys/inducing inputs $\bm{k}_g^h$ that are shared across all input sequences. The main idea is to compute the variance of sparse GP using the global keys only, so that the variational parameters for the $\bm{S}^h$ matrix become independent to the input sequences. Indeed, following Eq.~\ref{eq:dsgp_predictive} presented in the above paragraph, we can compute the mean and covariance for each output dimension ($d$) of each head as (we drop the superscript $h$ here for more concise notation):
%
%
\begin{equation}
\begin{split}
\bm{m}_{d}&=\bm{K}_{\bm{q} \bm{k}_a}[\bm{v}_{a}]_{:,d}-\bm{K}_{\bm{q} \bm{k}_g}\bm{K}_{\bm{k}_g \bm{k}_g}^{-1} \bm{K}_{\bm{k}_g \bm{k}_a}[\bm{v}_{a}]_{:,d}+\bm{K}_{\bm{q} \bm{k}_g}[\bm{v}_g]_{:,d}, \\
\bm{\Sigma}_{d}&= \bm{K}_{\bm{q} \bm{q}} + \bm{K}_{\bm{q} \bm{k}_{g}}\bm{K}_{\bm{k}_{g}\bm{k}_{g}}^{-1}([\bm{S}_{g}]_{:,:,d}-\bm{K}_{\bm{k}_{g}\bm{k}_{g}})\bm{K}^{-1}_{\bm{k}_{g}\bm{k}_{g}}\bm{K}_{\bm{k}_{g}\bm{q}},
\end{split}
\label{eq:odsvgp}
\end{equation}
%
%
where $\bm{v}_{g}\in \mathbb{R}^{M_{g}\times d_v}, \bm{S}_{g}\in \mathbb{R}^{M_{g}\times M_g \times d_v}$ are the variational parameters associated with the global keys $\bm{k}_{g}$, and $\bm{v}_{a} \in \mathbb{R}^{T\times d_v}$ is computed via projection $\bm{v}_a = \bm{s} \bm{W}_v$. We name this approach as \textit{decoupled SPGA} which is illustrated in Figure~\ref{fig:sgpa}.

%
%
Compared to standard SGPA (Eq.~\ref{eq:fpost}), where $\bm{k}_a^h$ in decoupled SGPA is the same as $\bm{k}^h$ in standard SGPA), we see that the posterior mean of decoupled SGPA also involves two extra terms to take into account the effect of global inducing points. But more importantly, the posterior variance of the two SGPA methods differ only in the keys/inducing inputs in use (input-dependent keys $\bm{k}^h$ versus global keys $\bm{k}_g^h$), and this brings in the key advantange of decoupled SGPA. As the posterior covariance in Eq.~\ref{eq:odsvgp} only involves the global inducing points, the variational covariance no longer needs to be input-dependent, and (the Cholesky factor of) $\bm{S}_g^h$ can be parameterized freely. Now the number of parameters for the covariance part is of order of $O(M_{g}^2)$ (vs $O(T^2)$ in standard SPGA), and the computation of matrix inversion pays a one-off cost of $O(M_g^3)$ (vs $O(T^3)$ for every input sequence). Notice that we are free to choose the number of global inducing points $M_{g}$, and in practice we find $M_{g}=O(\frac{T_{avg}}{H})$ is usually sufficient for experiments considered in this work, where $T_{avg}$ is the average length of training input sequences. In Table~\ref{table:cost}, we summarize time complexity (with batch size $B$) and the additional memory (number of parameters) required for SGPA in one head of a Transformer. We also include maximum likelihood estimation (MLE) for reference (note that memory complexity for MLE does not depend on input sequence length $T$). 
\begin{table}[t]
\centering
\caption{Complexity comparison for standard and decoupled SGPA.}
 \begin{tabular}{  c  c c    } 
\hline
 Model & Time & Additional Memory \\ 
 \hline
 MLE & $O(BT^2)$ & - \\
   Standard SGPA & $O(BT^3)$ & $O(T^2)$ \\
  Decoupled SGPA & $O(BT^2M_g+M_g^3)$ & $O(M_g^2)$ \\
\hline
\end{tabular}
\label{table:cost}
\end{table}
As the time and memory savings are significant, we mainly evaluate decoupled SGPA in our experiments, and in the rest of the main text we will refer to decoupled SGPA as SGPA for short. 
\section{Transformer Based on Decoupled SGPA}
%
\label{sec:sgpa_transformer}
So far we have presented SGPA methods for uncertainty quantification in a multi-head self-attention module. When applied to Transformer models, multiple layers of attention blocks are in use, and in the following we describe the construction of a Transformer model based on decoupled SGPA. Note that, as SGPA is equivalent to a sparse GP, the Transformer model presented below can be viewed as a sparse approximation to a deep GP \citep{damia2013dgp, salimbeni2017doubly} with deep kernel in each layer.

%
%
Our Transformer architecture mostly follows the one in \cite{vaswani2017attention}. The input to the $l$-th SGPA layer is the output from previous SGPA layer $\bm{F}^{l-1}\in \mathbb{R}^{T\times d^{l-1}}$. We first process the input with a non-linear mapping $G_{\phi^l}: \mathbb{R}^{d^{l-1}}\rightarrow \mathbb{R}^{d^{l}}$, and then perform projections to obtain the queries, amortized \& global keys and values. Specifically, we have for each head $h$:
%
%
\begin{equation}
        \bm{q}^{l,h} = \bm{k}_{a}^{l,h} = G_{\phi^l}(\bm{F}^{l-1})\bm{W}^{l,h}_{qk},\quad
        \bm{k}_{g}^{l,h} = G_{\phi^l}(\bm{Z}_{g} ^{l,h})\bm{W}^{l,h}_{qk},\quad
        \bm{v}_{a}^{l,h} = G_{\phi^l}(\bm{F}^{l-1})\bm{W}^{l,h}_v,
\label{eq:fhl}
\end{equation}
%
%
where $\bm{Z}_{g}^{l,h}\in \mathbb{R}^{M_{g}\times d^{l-1}}$ are global inducing locations of the $l$-th layer defined on the same space as $\bm{F}^{l-1}$. Then we apply a base kernel $K_{base}(\cdot, \cdot)$ to compute the kernel matrices. This is equivalent to using a deep kernel defined on the space of $\bm{F}^{l-1}$, and the parameters of $G_{\phi^l}$ are viewed as the hyperparameters of the deep kernel.
Lastly, with variational parameters ($\bm{v}_{g}^{l,h}, \bm{S}_{g}^{l,h}$) associated with the global inducing locations $\bm{Z}_{g}^{l,h}$, we can obtain $\bm{m}^{l,h}_{d}$ and $\bm{\Sigma}_{d}^{l,h}$ using Eq.~\ref{eq:odsvgp}. We then propagate uncertainty to the next layer by generating samples of output for each head using the reparameterization trick as in \cite{salimbeni2017doubly}:
%
%
\begin{equation}
    [\bm{F}^l_h]_{:,d}=\bm{m}^{l,h}_{d}+\bm{L}_{\bm{\Sigma}_d}^{l,h}\bm{\epsilon}_d^{l,h},\quad \bm{\epsilon}_d^{l,h} \sim \mathcal{N}(\bm{0},\bm{I}),
\label{eq:reparam_dgp}
\end{equation}
where $\bm{L}_{\bm{\Sigma}_d}^{l,h}$ is the Cholesky factor of $\bm{\Sigma}_d^{l,h}$.
%
%
The final output $\bm{F}^l\in\mathbb{R}^{T\times d^l}$ of this SGPA layer is obtained by linear combination in the same way as in standard Transformers (see Eq.~\ref{eq:combine_heads}).

%
%
The ELBO objective for training the variational \& kernel parameters is derived following deep GP and additive GP \citep{add2011david} approaches. The key idea is that as each head in MHSA with SGPA is a (sparse) GP, the final output $\bm{F}^l$ can also be viewed as a weighted summation of (sparse) GPs, which is again a GP \citep{add2011david}. This allows us to perform variational approximations on each of the heads before the final combination instead of using a direct approximation on the $\bm{F}^l$ process \citep{hkd2021sun}. Assuming the approximate posterior $q$ for $\{\bm{F}_h^l\}^{H}_{h=1}$ factorizes over $h$, the corresponding ELBO with input sequence $\bm{F}^0 := \bm{X}$ is (derivations in Appendix~\ref{a4:sgpa_obj}):
%
%
\begin{equation}
\begin{split}
     \mathcal{L}_{ELBO}&=\mathbb{E}_{q(\bm{F}^L|\bm{F}^0,\{\bm{k}^{l,h}_{g}\}_{l=1, h=1}^{L,H})}[\log p(\bm{Y}|\bm{F}^L)]\\ &-\sum_{l=1}^L \sum_{h=1}^H \mathbb{E}_{q(\bm{F}^{l}|\bm{F}^0,\{\bm{k}^{j,h}_{g}\}_{j=1, h=1}^{l,H}))}[\text{KL}(q(\bm{u}_{a \cup g}^{l,h}|\bm{k}^{l,h}_{g},\bm{F}^{l-1})||p(\bm{u}_{a \cup g}^{l,h}|\bm{k}^{l,h}_{g},\bm{F}^{l-1}))].
\end{split}
\label{eq:elbo}
\end{equation}
%
%
In practice, we resort to Monte-Carlo to estimate $\mathcal{L}_{ELBO}$ with samples of function values generated iteratively while passing through each layer using the reparameterization trick (Eq.~\ref{eq:reparam_dgp}).

\section{Experiments}
\label{sec:sgpa_experiments}
We evaluate SGPA on prediction tasks across modalities, with the following experimental set-up.
\begin{itemize}
    \item Datasets: CIFAR10 \& CIFAR100 (image classification \citep{cifar}, CV tasks); CoLA (linguistic acceptability prediction \citep{cola}, NLP task) and IMDB (sentiment analysis, \citep{imdb}, NLP task).
    \item Network architectures: We use Vision Transformers (ViT \citep{dosovitskiy2020image}) for CV tasks. For kernel attention we use the exponential kernel \citep{tsai2019TransformerDissection} and the ARD-RBF kernel \citep{rasmussen2006gp} for NLP and CV tasks respectively. Scaled dot-product (SDP) attention based Transformers are also evaluated. As in \citet{tsai2019TransformerDissection}, we find kernel attention tends to outperform SDP attention in most tasks considered, thus we do not include the results of SDP attention in the main text. These results can be found in the tables in Appendix~\ref{a5:table}.
    \item Baselines: We compare our approach with the following ``single-model'' methods: maximum likelihood estimation (MLE), Bayesian inference methods including mean-field variational inference (MFVI, \citep{blundell2015bbp}), Monte-Carlo Dropout (MCD, \citep{gal2016mcdrop}), Kronecker-factored last layer Laplace approximation (KFLLLA) \citep{kris2020bayes}, and Spectral-normalized Neural Gaussian Process (SNGP) \citep{liu2020sngp}. For tasks where a validation set is used, we also consider temperature scaling (TS) \citep{guo2017calibration} and use the validation set as the calibration set. For CV tasks, we also consider ensemble methods: we compare SGPA ensemble (SGPAE) with deep ensemble (DE) \citep{bala2017ensemble}. We don't consider ensemble models in NLP tasks since we use different train-(valid)-test splits in different runs for them.
    \item Evaluations \& metrics: We consider three evaluation set-ups: in-distribution performance, out-of-distribution (OOD) robustness and OOD detection. The metrics on test set include predictive accuracy metrics for each task, uncertainty calibration metrics such as negative predictive log-likelihood (NLL), expected calibration error (ECE) and maximum calibration error (MCE) \citep{Pakdaman_Naeini_Cooper_Hauskrecht_2015,guo2017calibration}. A brief introduction for these uncertainty calibration metrics is provided in Appendix~\ref{appendix:uq_metric}. We report the mean$\pm$two standard errors for each metric obtained from 5 independent runs. For OOD detection tasks we consider the area under the ROC \& precision-recall curves (AUROC \& AUPR, respectively), and we report the average ranks in terms of AUROC and AUPR over all of the 6 OOD detection tasks for each method. 
\end{itemize}
For fair comparisons, within each task, all the models are trained using the same architecture and optimization setting. All the models are trained from scratch without pre-training. We include the experimental details in Appendix~\ref{a1:hyperparam} and the comparison of wall-clock running time in Appendix~\ref{appendix:sgpa_time}. Results in tables are also presented in Appendix~\ref{a5:table}.

\subsection{In-distribution Calibration}
We report the evaluation results for in-distribution test data on image classification (CIFAR10 \& CIFAR100, without data augmentation), sentiment analysis (IMDB), and linguistic acceptability (CoLA) tasks in the first, second, third and fourth row of Figure~\ref{cifar} respectively. Here for the CoLA dataset, predictive accuracy is measured by Matthew correlation coefficient (MCC) \citep{mcc} instead of accuracy, as in \citet{cola}.

\begin{figure}[t]
\centering
\begin{subfigure}{0.99\textwidth}
    \includegraphics[width=0.99\linewidth]{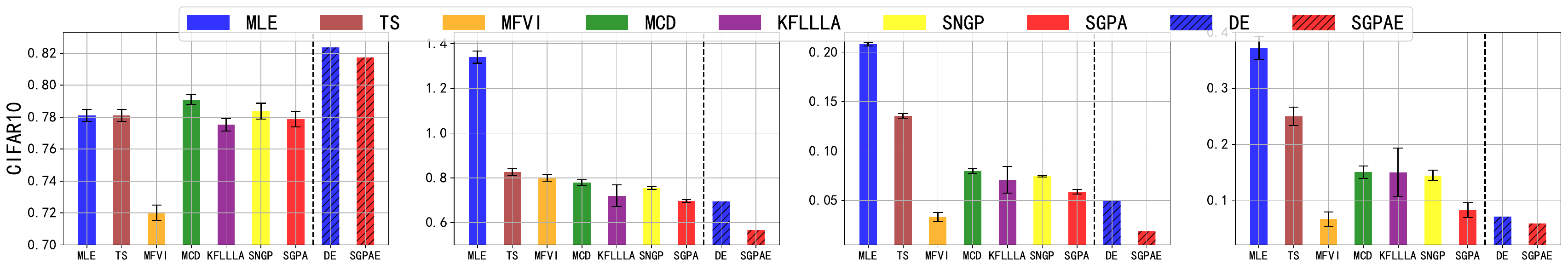}
\end{subfigure}
\begin{subfigure}{0.99\textwidth}
    \includegraphics[width=0.99\linewidth]{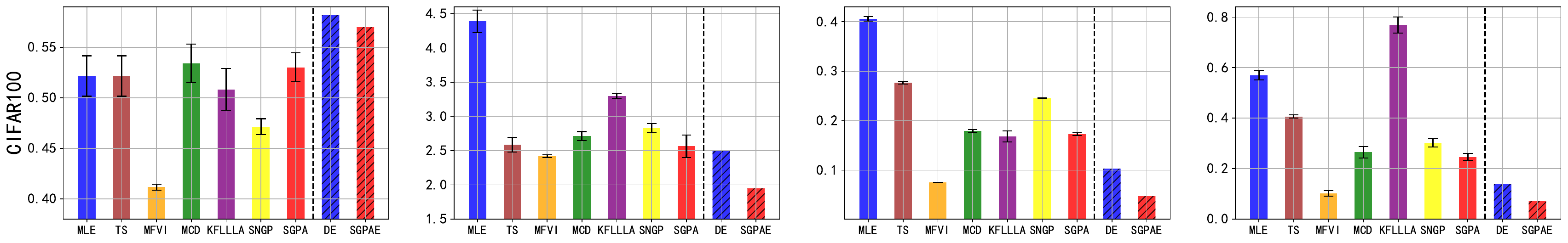}
\end{subfigure}
\begin{subfigure}{0.99\textwidth}
    \includegraphics[width=0.99\linewidth]{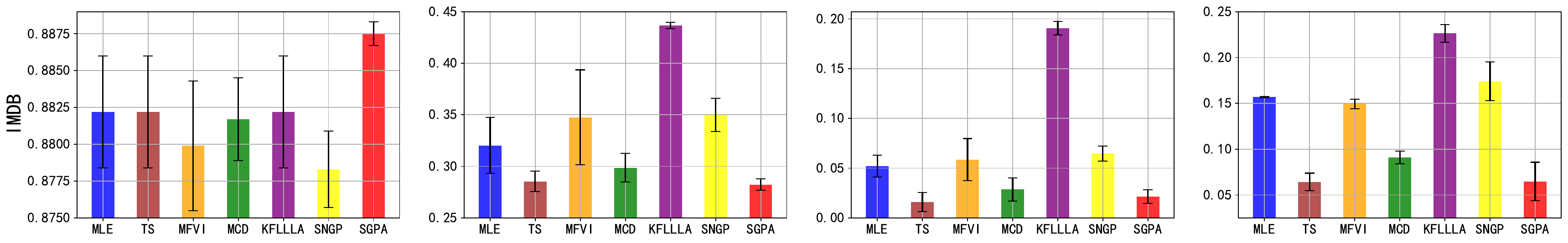}
\end{subfigure}
\begin{subfigure}{0.99\textwidth}
    \includegraphics[width=0.99\linewidth]{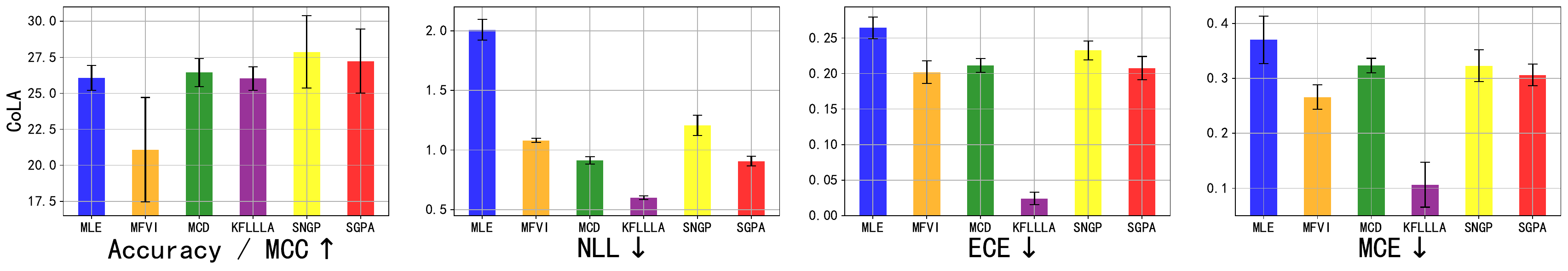}
\end{subfigure}
\caption{Test set accuracy (or MCC for CoLA) \& calibration metrics of Transformers or ViTs trained on CIFAR10 (1st row), CIFAR100 (2nd row), IMDB (3rd row) and CoLA (4th row).}
\label{cifar}
\end{figure}
%
%
\begin{figure}[b]
\centering
\begin{subfigure}{0.99\textwidth}
    \includegraphics[width=0.99\linewidth]{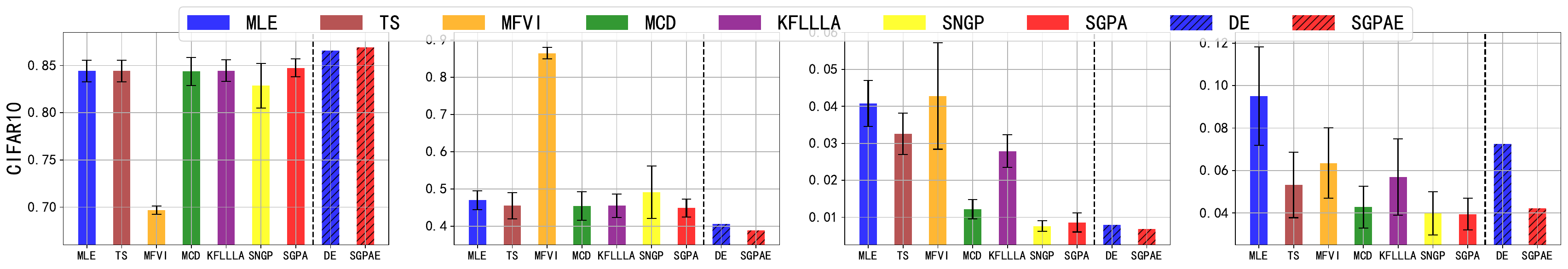}
\end{subfigure}
\begin{subfigure}{0.99\textwidth}
    \includegraphics[width=0.99\linewidth]{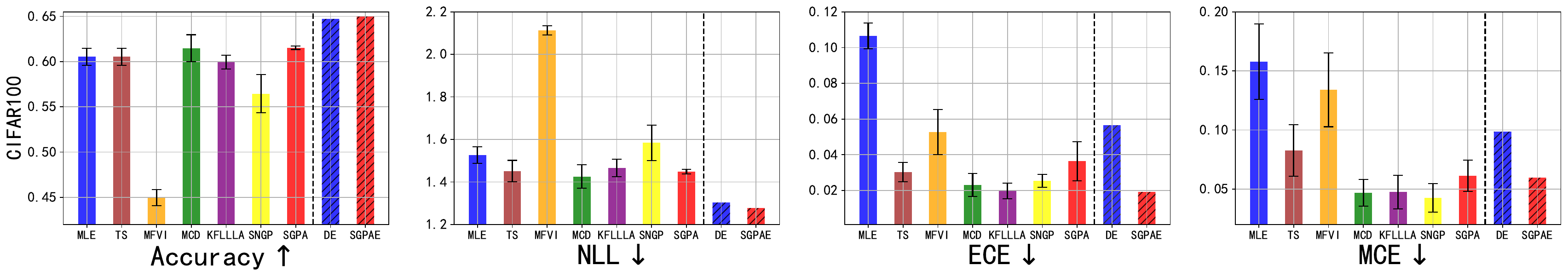}
\end{subfigure}
\caption{Test set accuracy \& calibration metrics of ViTs trained on CIFAR10 (top row) and CIFAR100 (bottom row) with data
augmentation.}
\label{cifar_aug}
\end{figure}

All ``single-model'' calibration methods considered tend to improve the calibration, except for sentiment analysis, where KFLLLA fails in the sense that it achieves worse calibration even than MLE (although KFLLLA achieves best calibration for linguistic acceptability (CoLA), its performance is unstable across tasks). Although MFVI tends to achieve the lowest calibration errors, it severely underfits the data in all the experiments. This is undesirable, as improvement in calibration should not come at a price of noticeable drop in predictive correctness. As a counter example, one can achieve perfect calibration in terms of zero ECE by predicting marginal class probability, but this prediction is useless in practice. For image classification on CIFAR100, KFLLLA achieves competitive ECE compared with SGPA, however, it achieves worse NLL and the worst MCE among all the methods. Overall, SGPA achieves the best performance when compared with the other ``single-model'' baselines: it consistently achieves better calibration across all tasks while maintaining competitive (or even better, on IMDB) predictive accuracy. Compared with ``single-model'' methods, both ensemble methods, DE and SGPAE, achieve much better predictive accuracy. SGPAE noticeably outperforms DE in terms of calibration while maintaining competitive predictive accuracy.

For CIFAR10 and CIFAR100, we also consider training ViTs with data augmentation and report the results of in-distribution calibration in Figure~\ref{cifar_aug}. Although some ``single-model'' methods can achieve lower ECE or MCE than DE and SGPAE in some cases, DE and SGPAE consistently outperform them in terms of accuracy and NLL. SGPAE again achieves the best overall performance. Among ``single-model''methods, MFVI still underfits the data, but for the other methods, data augmentation improves the model performance with SNGP achieving relatively low accuracy. The difference between SGPA and other ``single-model'' baselines becomes smaller, perhaps due to the strong regularization from data augmentation. Still, SGPA performs more robustly as it generally returns smaller error bars when compared to TS, MCD, and KFLLLA.

\subsection{Robust Prediction on Out-of-distribution Data}
%
%
\begin{figure}[t]
\includegraphics[width=0.99\linewidth]{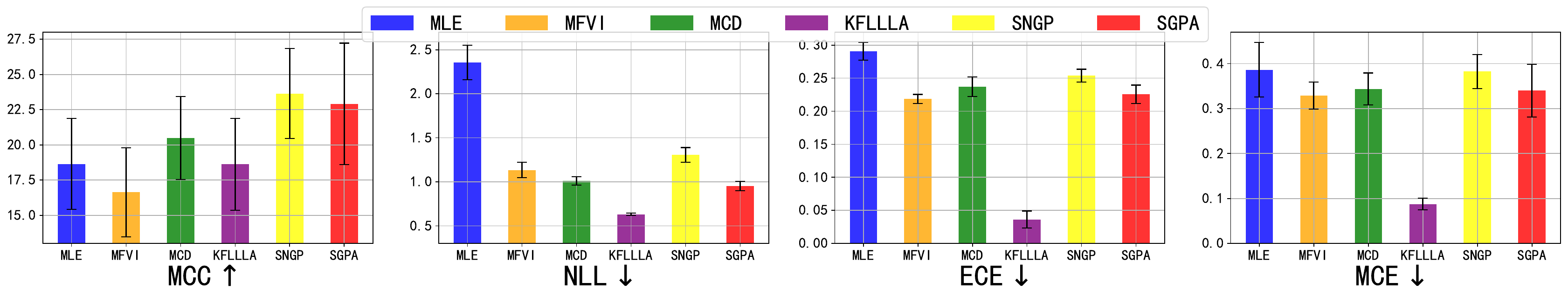}
\centering
\caption{Test set MCC \& calibration metrics for OOD test set of Transformers trained on COLA.}
\label{cola}
\end{figure}
\begin{figure}[t]
\centering
\begin{subfigure}{0.99\textwidth}
    \includegraphics[width=0.999\linewidth]{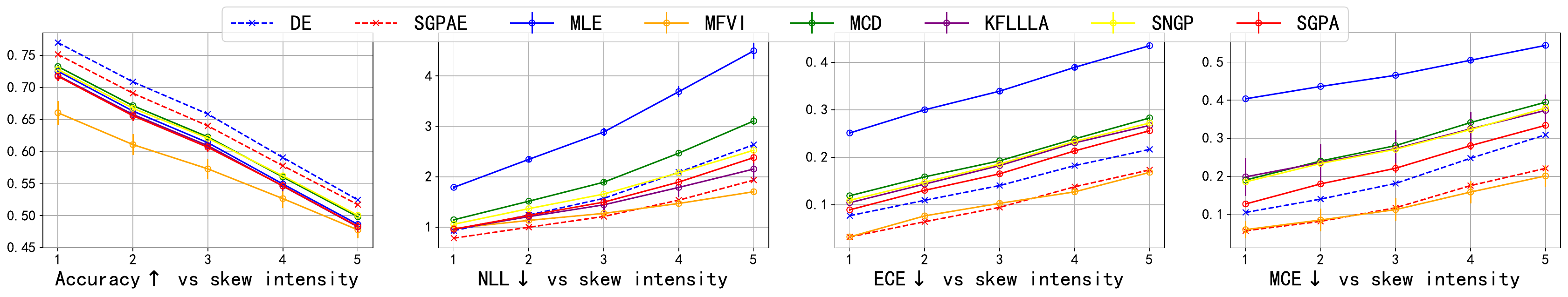}
\end{subfigure}
\begin{subfigure}{0.99\textwidth}
    \includegraphics[width=0.999\linewidth]{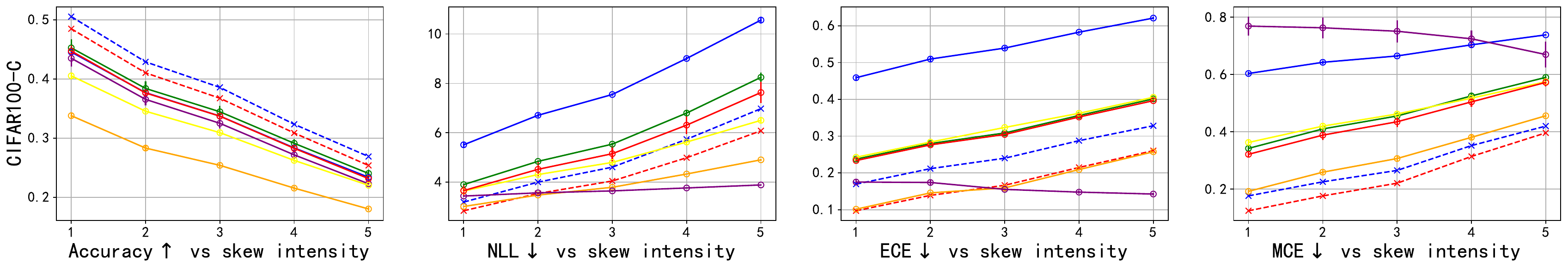}
\end{subfigure}
\caption{Test set accuracy \& calibration metrics on CIFAR10-C (top row) and CIFAR100-C (bottom row) against skew intensity of corruption for ViTs trained on corresponding clean data without data augmentation.}
\label{rob_cifar}
\end{figure}
\begin{figure}[t]
\centering
\begin{subfigure}{0.99\textwidth}
    \includegraphics[width=0.999\linewidth]{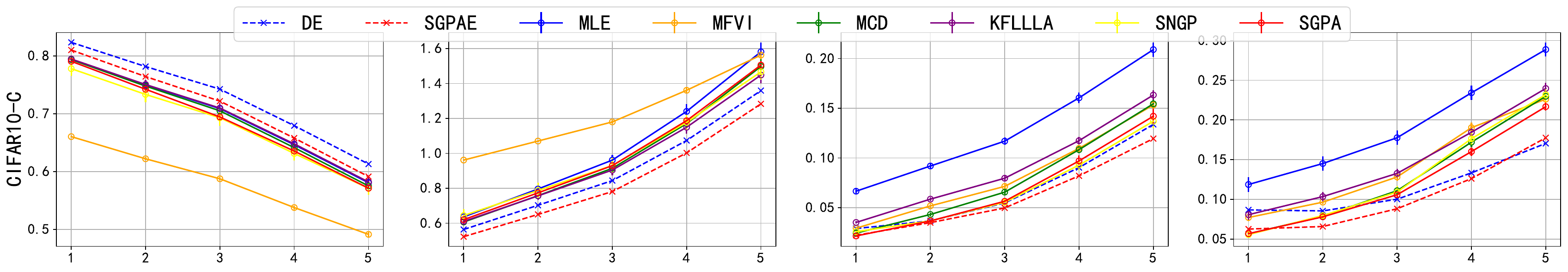}
\end{subfigure}
\begin{subfigure}{0.99\textwidth}
    \includegraphics[width=0.999\linewidth]{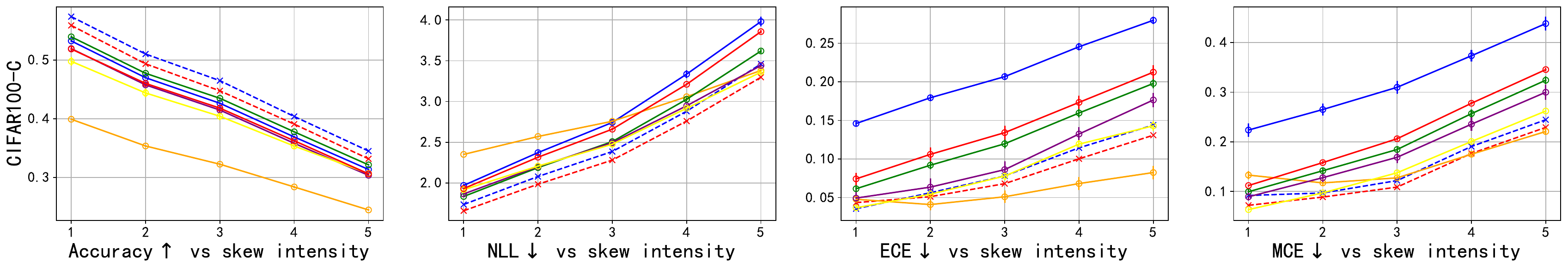}
\end{subfigure}
\caption{Test set accuracy \& calibration metrics on CIFAR10-C (top row) and CIFAR100-C (bottom row) against skew intensity of corruption for ViTs trained on corresponding clean data with data augmentation.}
\label{rob_cifar_aug}
\end{figure}
Next we evaluate the performance of SGPA under distribution shift for both the linguistic acceptability task (CoLA) and the image classification tasks (CIFAR10 \& CIFAR100). The OOD data for CoLA is introduced by the same authors of \citet{cola}, while for the CIFAR datasets, we use corrupted CIFAR datasets (CIFAR10-C and CIFAR100-C) \citep{cifarc} as the OOD data, which contains noisy CIFAR images with different types of distortions introduced to their clean counterparts at different skew intensities. Note that we don't consider TS for OOD tasks in this and the next section since as a Frequentist method, it is proposed to calibrate the uncertainty on in-distribution data only.

%
%
We report in Figure~\ref{cola} the MCC and calibration metrics for the OOD test on CoLA. The observations are similar with the in-distribution test: SGPA outperforms MLE, MCD, and SNGP in terms of NLL and calibration errors while achieving improved accuracy. MFVI and KFLLLA achieve lower calibration errors but they achieve worse predictive accuracy than SGPA. In particular, MFVI again underfits the data.

%
%
%
%
\begin{table}[b]
\centering
 \caption{Average ranks of different methods in terms of AUROC and AUPR over 6 OOD detection tasks}
\begin{adjustbox}{width=\textwidth}
 \begin{tabular}{  c  c c   c c } 
\hline
 \multicolumn{1}{c}{} & \multicolumn{2}{c}{Without data augmentation}                                                 & \multicolumn{2}{c}{With data augmentation}  \\
 \multicolumn{1}{c}{Model} & \multicolumn{1}{c}{rank (AUROC) $\downarrow$} & \multicolumn{1}{c}{rank (AUPR) $\downarrow$} & \multicolumn{1}{c}{rank (AUROC) $\downarrow$} & \multicolumn{1}{c}{rank (AUPR) $\downarrow$}\\ 
 \hline
MLE & 6.1667 & 5.5000 & 6.1667 & 6.0000\\
MFVI & 7.0000 & 7.1667 & 8.0000 & 8.0000\\
MCD & 5.6667 & 5.6667 & 4.1667 & 4.5000\\
KFLLLA & 4.3333 & 4.5000 & 4.3333 & 4.0000\\
SNGP & 5.3333 & 5.3333 & 6.3333 & 6.3333\\
SGPA & 4.5000 & 4.6667 & 4.0000 & 4.1667\\
\hline
DE & 1.6667 & 2.0000 & 1.8333 & 1.8333\\
SGPAE & \textbf{1.3333}& \textbf{1.1667} & \textbf{1.1667} & \textbf{1.1667}\\
\hline
\end{tabular}
\end{adjustbox}
\label{table:ood_rank}
\end{table}
For OOD robustness test on image classification, we compute metrics against skew intensity on the corrupted CIFAR datasets, and report the results without data augmentation in Figure~\ref{rob_cifar} and the results with data augmentation in Figure~\ref{rob_cifar_aug}. Without data augmentation, among ``single-model'' methods, SGPA outperforms MLE, MCD and SNGP in terms of calibration without hurting accuracy. MFVI achieves lower calibration errors than SGPA but it pays the price of underfitting especially when the skew intensity is small. The performance of KFLLLA seems to be not as stable as SGPA. For CIFAR10-C, SGPA achieves better calibration than KFLLLA. For CIFAR100-C, KFLLA achieves the best NLL and ECE but the worst MCE. When the data augmentation is applied, MFVI still suffers from underfitting, but the performance gap between SGPA and other ``single-model'' baselines again becomes smaller. Regardless of using data augmentation or not, ensemble methods achieve better accuracy than ``single-model'' methods and SGPAE still outperforms DE in terms of calibration while achieving similar accuracy.

\subsection{Out-of-distribution Detection}
\begin{figure}[t]
\centering
\includegraphics[width=0.99\linewidth]{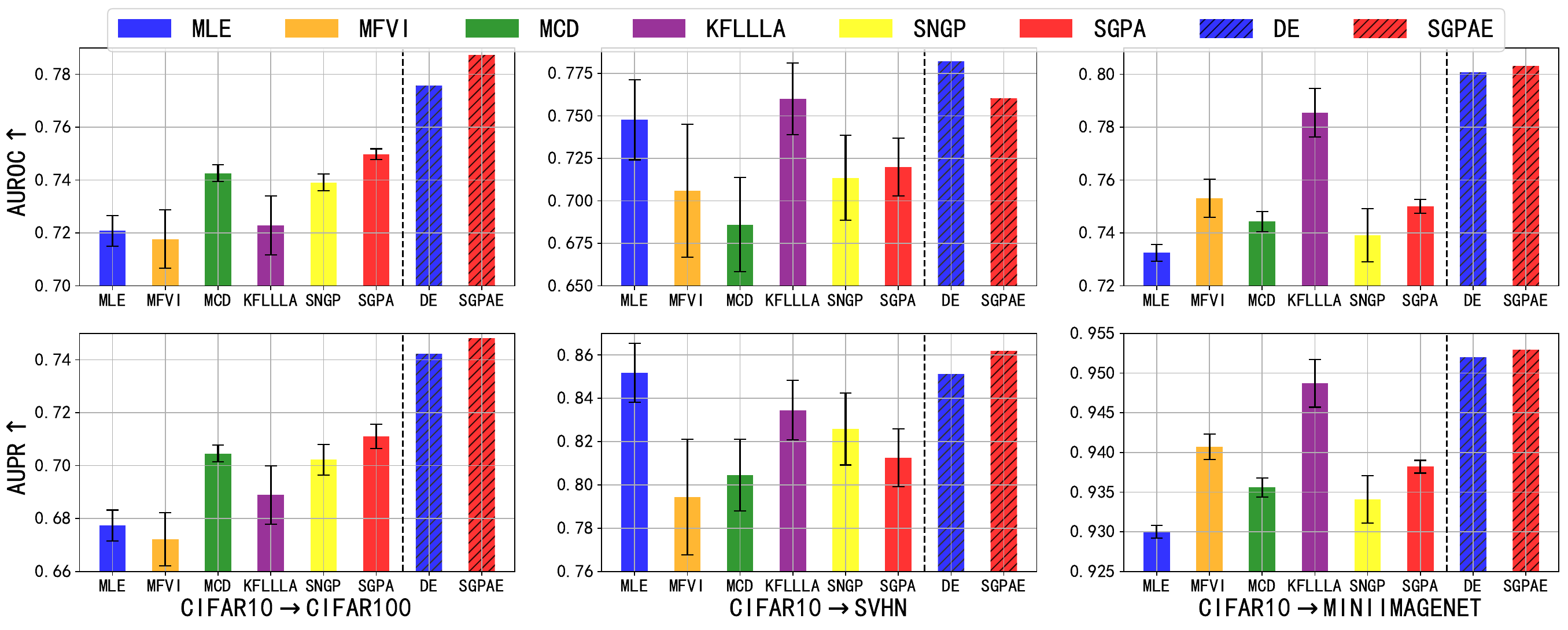}
\caption{AUROC (top) and AUPR (bottom) metrics for OOD detection using ViTs trained on CIFAR10 without data augmentation.}
\label{ood_cifar10}
\end{figure}
\begin{figure}[b]
\centering
\includegraphics[width=0.99\linewidth]{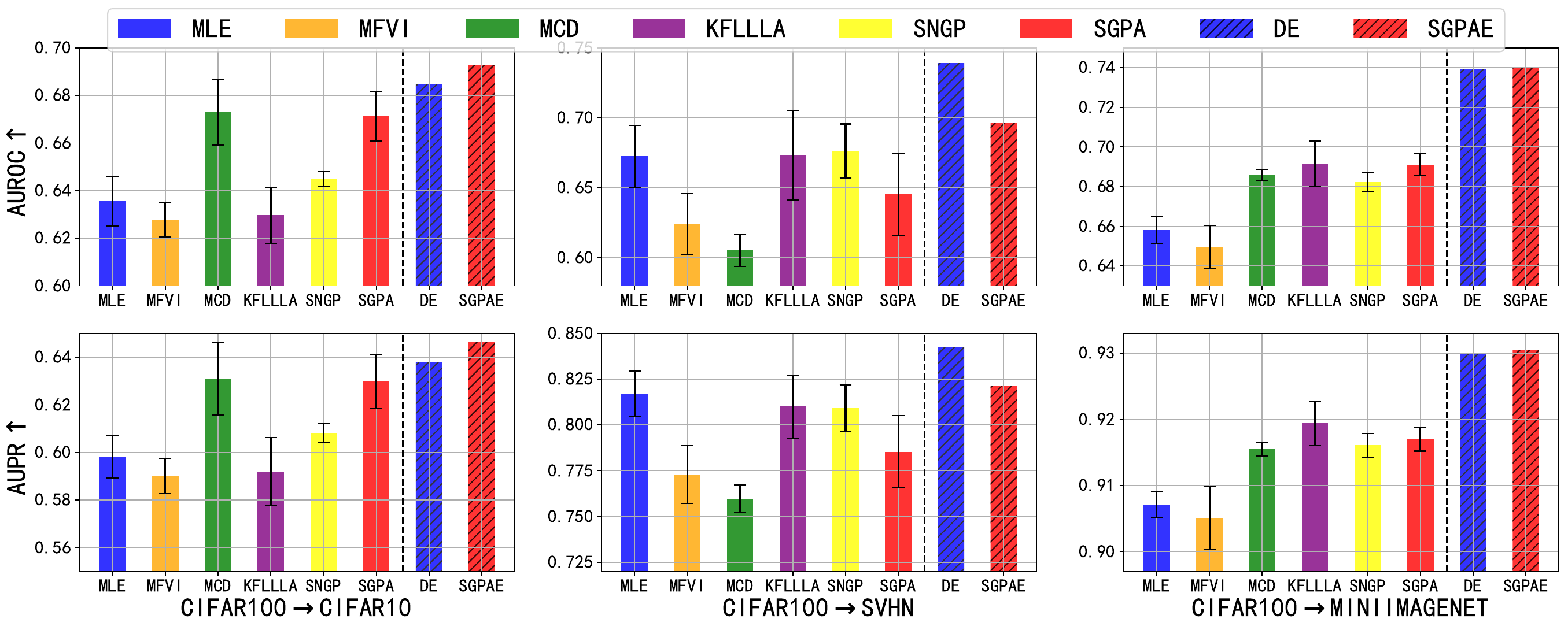}
\caption{AUROC (top) and AUPR (bottom) metrics for OOD detection using ViTs trained on CIFAR100 without data augmentation.}
\label{ood_cifar100}
\end{figure}
\label{a:ood_fig}
\begin{figure}[t]
\centering
\includegraphics[width=0.99\linewidth]{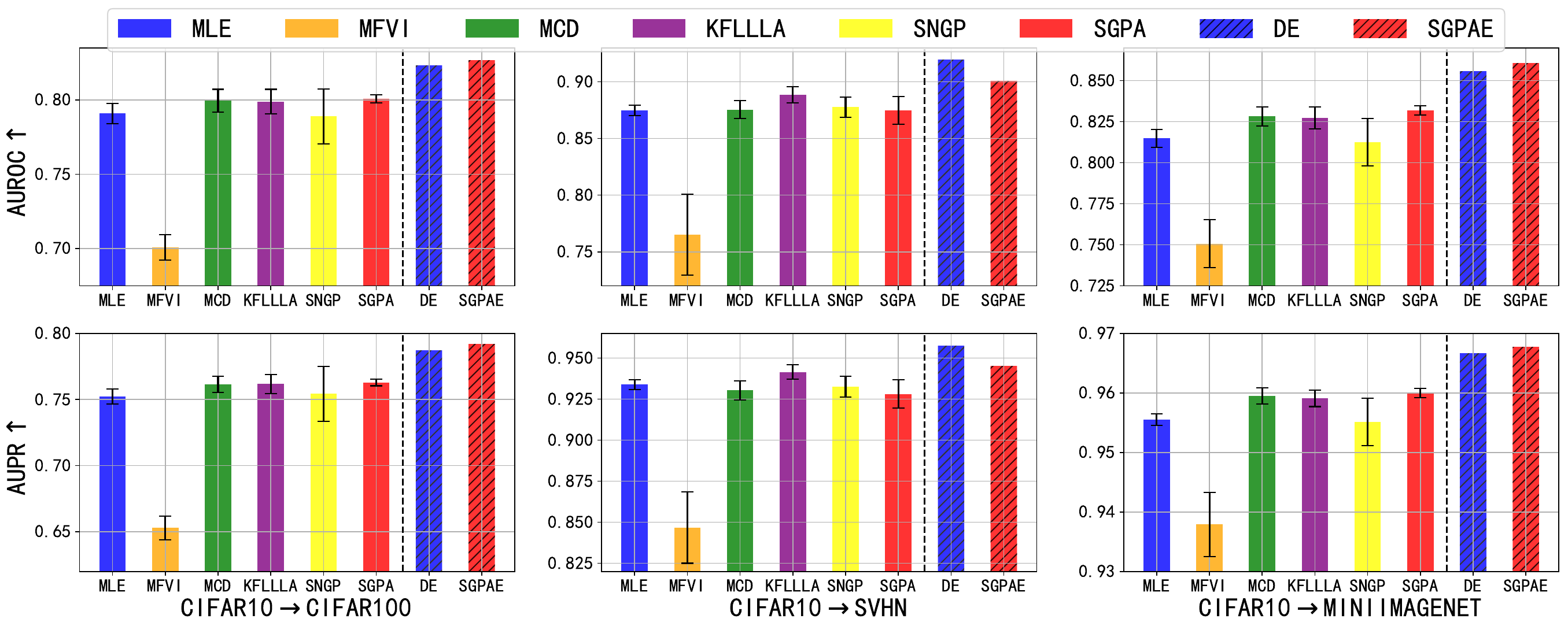}
\caption{AUROC (top) and AUPR (bottom) metrics for OOD detection using ViTs trained on CIFAR10 with data augmentation.}
\label{ood_cifar10_aug}
\end{figure}
\begin{figure}[t]
\centering
\includegraphics[width=0.99\linewidth]{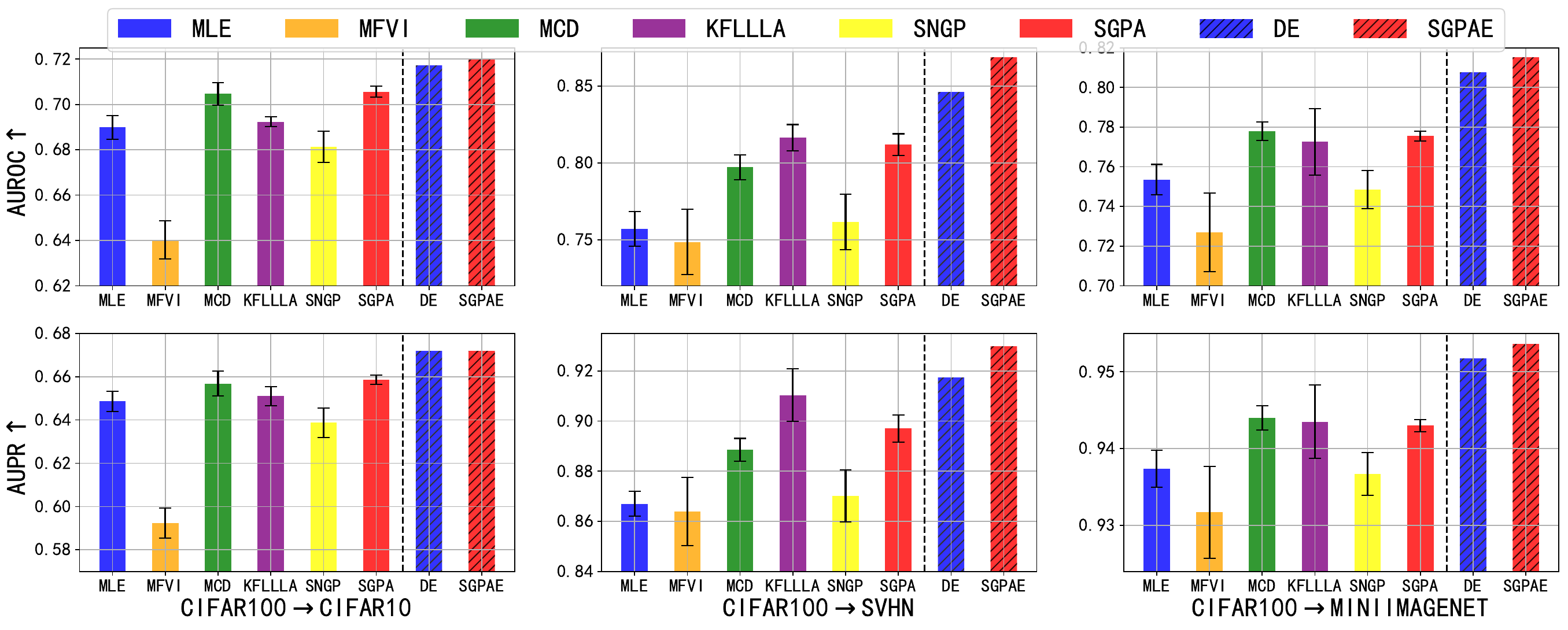}
\caption{AUROC (top) and AUPR (bottom) metrics for OOD detection using ViTs trained on CIFAR100 with data augmentation.}
\label{ood_cifar100_aug}
\end{figure}
Lastly we consider OOD detection tasks on Transformer models trained for image classification, to further evaluate the quality of uncertainty estimates. Here we use predictive entropy to score each input (from either in-distribution or OOD data) and make decisions of ``in/out'' if the entropy is smaller/greater than a specified threshold. Using different thresholds for this detection task allows us to compute both the receiver operator characteristic (ROC) and the precision-recall (PR) curves, and we use the area under the curves, i.e., AUROC and AUPR, for performance evaluations.

For each of the two CIFAR datasets, we consider the other CIFAR dataset, SVHN and mini-ImageNet as OOD datasets so that we construct 6 OOD detection tasks in total. For each method trained with or without data augmentation, we report its average ranks in terms of AUROC and AUPR over 6 tasks in Table~\ref{table:ood_rank}. Ensemble methods outperform ``single-model'' methods with SGPAE achieving the best performance. Among ``single-model'' methods, apart from KFLLLA which achieves the best performance, SGPA outperforms all the other baselines. For comparison within each task, we plot the values of AUROC and AUPR achieved by all methods trained without data augmentation for each task in Figure~\ref{ood_cifar10} and~\ref{ood_cifar100}, and in Figure~\ref{ood_cifar10_aug} and~\ref{ood_cifar100_aug}, we plot the values of AUROC and AUPR achieved by all methods trained with data augmentation within each task.

\subsection{Summary of Additional Results}
%
Additional experiments are reported in Appendix~\ref{a2} where we summarize the main results as follows.
\begin{itemize}
    \item In Appendix~\ref{a2:ssgpa}, we find that in addition to the parameter inefficiency problem, Transformers based on standard SGPA also suffer from underfitting. Compared with decoupled SGPA, standard SGPA achieves significantly worse accuracy on CIFAR10 classification task. 
    
    \item In Appendix~\ref{a2:zinc}, we report results of an additional experiment, graph property regression with ZINC dataset \citep{dwivedi2020benchmarkgnns}.
    For this task, SGPAE achieves the best performance and SGPA outperforms the other ``single-model'' baselines. 
\end{itemize}
\section{Related Work}
\textbf{Bayesian Transformers.} \citet{tran2019bl} and \citet{xue2021bayestran} propose to perform approximate posterior inference using MFVI in weight space for a subset of layers in Transformers. However, in our experiments we find this type of approaches underfits the data. This pathology of underfitting is theoretically confirmed for weight space MFVI \citep{foong2020expressive,coker_aistats2022}. Another line of research proposes to perform VI over the attention matrices directly \citep{bam,cinquin2021bayestrans}. However, \cite{bam} only considers finetuning with variational attention, and \citep{cinquin2021bayestrans} only considers experiments on synthetic or simple datasets with shallow networks and the variational distribution fitted over the attention weights are shared across data, which might be too restrictive for complex problems. Moreover, they find that a data-dependent variational distribution over attention weights can even hurt the performance of their approaches. 
\citet{liu2020sngp} consider performing Bayesian inference directly over the Transformer output by fitting a GP over the last layer output \citep{bradshaw2017adversarial}. This approach can be viewed as using a GP model with a deep kernel defined by the Transformer. Instead, SGPA fits a deep GP so that uncertainty is propagated through each attention layer of the Transformer. In addition, \citet{liu2020sngp} propose to preserve the distance awareness property for the deep kernel. Note that this distance-preserving trick is orthogonal to ours and can also be easily integrated into SGPA.

%
%
\textbf{Related GP methods.} 
The ELBO of SGPA is similar to that in \citet{hkd2021sun} which also propose to independently approximate the posterior for each additive component in an additive GP \citep{add2011david}. The difference is in the kernel design: \citet{hkd2021sun} aim to decompose a given kernel function into orthogonal ``kernel basis'', while in SGPA we consider the same type of kernel for each attention head but with different kernel hyperparameters. 
Our approach is also related to sparse within sparse Gaussian process (SWSGP) \citep{tran2021swsgp, daniel2022inputde} which allows adaptive inducing points for each data point (similar to the input-dependent keys $\bm{k}_a$ in SGPA). This connection between SGPA and SWSGP is further discussed in Appendix~\ref{a3:swsgp}.

\section{Conclusion and Future Work}
We have proposed SGPA to directly perform approximate Bayesian inference over the output of attention blocks in Transformers. Unlike other Bayesian neural network approaches, SGPA takes advantage of the inductive bias embedded in the attention architecture to build Bayesian Transformers. Moreover, we showed Transformers based on SGPA achieve better balance between predictive accuracy and calibration compared with other baselines. Furthermore, the improved quality of uncertainty estimation provided by SGPA has been proved useful in maintaining robustness under distribution shift and in out of distribution detection.

The following directions are interesting to explore as future work. First, masked pre-training \citep{devlin2018bert}, which has been proved crucial for downstream tasks for standard Transformers, may also improve the performance of Transformers based on SGPA. In this work, we are not able to consider pretraining due to the high computational cost, and since SGPA replaces scaled dot-product with a valid kernel, there is no existing pre-trained backbone that can be directly used for the downstream fine-tuning tasks. Second, many modern tasks using Transformers require next token prediction in an autoregressive manner, which involve decoder-based architectures \citep{vaswani2017attention,brown2020language,openai2024gpt4technicalreport}. While we only consider encoder-based Transformers in this work, extending SGPA to decoder-based Transformers will further broaden the impact of this method, and We give one potential solution in Chapter~\ref{cha:conclusion}. 

\chapter{Recurrent Memory for Online Interdomain
Gaussian Processes}
\label{cha:hsgp}

\begin{tcolorbox}           [enhanced,colback=gray!5!white,colframe=gray!75!black,colbacktitle=red!80!black,fonttitle=\bfseries]
  This chapter is based on \citet{chen2025recurrent}:
  \begin{itemize}
  \item  \underline{Wenlong Chen}$^\ast$, Naoki Kiyohara$^\ast$, Harrison Bo Hua Zhu$^\ast$, Jacob Curran-Sebastian, Samir Bhatt, and Yingzhen Li. \textbf{Recurrent Memory for Online Interdomain Gaussian Processes}. In \textit{Advances in Neural Processing Systems (NeurIPS)}, 2025.
  \end{itemize}
  The main idea was developed by me while the last author provided useful suggestions. The three co-first authors ($^\ast$) all contributed to the code implementation, experimentation, and paper writing under the supervision of the last author. Moreover, I performed all the derivations and helped orchestrate other authors’ contributions.
\end{tcolorbox}

As we briefly mentioned in Section~\ref{sec:back_bayes}, Bayesian inference provides a natural framework for online or continual learning. In this chapter, we focus on improving approximate online Bayesian inference in sparse GP models, again, by exploiting inductive bias in deep sequence models. Specifically, we propose HiPPO-SVGP which leverages the long-range memory preservation capability of HiPPO, a powerful RNN architecture reviewed in Section~\ref{sec:hippo}, to construct a better sparse approximation that mitigates the catastrophic forgetting issue in online sparse GP methods.


\section{Introduction}
Gaussian processes (GPs) are popular choices for modeling time series due to their functional expressiveness and uncertainty quantification abilities \citep{roberts_gaussian_2013, fortuin_gpvae_2020}. However, GPs are computationally expensive and memory intensive, with cubic and quadratic complexities, respectively. In online regression settings, such as weather modeling, the number of time steps can be very large, quickly making GPs infeasible. Although variational approximations, such as utilizing sparse inducing points (SGPR \citep{svgp2009titsias}; SVGP \citep{hensman2013gpbig,hensman2015sgpclass}) and Markovian GPs \citep{sarkka2019applied,wilkinson_sparse_2021}, have been proposed to address the computational complexity, it would still be prohibitive to re-fit the GP model from scratch every time new data arrives. \citet{bui_streaming_2017} proposed an online sparse variational GP (OSVGP) learning method that sequentially updates the GP posterior distribution only based on the newly arrived data. However, as indicated in their paper, their models may not maintain the memory of the previous data, as the inducing locations will inevitably shift as new data arrive. This is a major drawback, as their models may not model long-term memory unless using a growing number of inducing points. 

In deep learning, as an alternative to Transformers \citep{vaswani2017attention}, significant works on state space models (SSMs) have been proposed to model long-term memory in sequential data. Originally proposed to instill long-term memory in recurrent neural networks, the HiPPO (High-order Polynomial Projection Operators) framework \citep{gu_hippo_2020} provides mathematical foundations for compressing continuous-time signals into memory states through orthogonal polynomial projections. HiPPO is computationally efficient and exhibits strong performance in long-range memory tasks, and forms the basis for the state-of-the-art SSMs, e.g., structured state space sequential (S4) model \citep{gu_s4_2022} and Mamba \citep{gu_mamba_2023,dao_mamba2_2024}.

Inspired by HiPPO, we propose Online HiPPO SVGP (OHSVGP), by applying the HiPPO framework to SVGP in order to leverage the long-range memory modeling capabilities. Our method interprets the HiPPO time-varying orthogonal projections as inducing variables of an interdomain SVGP \citep{lazaro-gredilla_inter-domain_2009,leibfried_tutorial_2020,van_der_wilk_framework_2020}, where the basis functions are time-dependent orthogonal polynomials. We show that we are able to significantly resolve the memory-loss issue in OSVGP, thereby opening up the possibility of applying GPs to long-term online learning tasks. In summary, our contributions include: 
\begin{itemize}
\item (Section~\ref{sec:method}) We demonstrate that HiPPO integrates into the interdomain GPs by interpreting the HiPPO projections as inducing variables with time-dependent orthogonal polynomial basis functions. This allows the inducing variables to compress historical data, capturing long-term information. 

\item (Section~\ref{sec:covariance_evol} \&~\ref{sec:wall_clock_run}) We show that the kernel matrices can leverage the efficient ODE evolution of the HiPPO framework, bringing an extra layer of computational efficiency to OHSVGP.  

\item (Section~\ref{sec:experiments}) We demonstrate OHSVGP on a variety of online/continual learning tasks including time series prediction, continual learning on UCI benchmarks, and continual learning in Gaussian process variational autoencoder, showing that it outperforms other online sparse GP baselines in terms of predictive performance, long-term memory preservation, and computational efficiency.
\end{itemize}

\section{Preliminaries}
In this section, we provide a brief overview of interdomain inducing point methods, online learning with sparse GPs, and Gaussian process variational autoencoders.
%
\subsection{Interdomain Gaussian Processes}\label{sec:vi_and_interdomain_GP}

Interdomain GPs are almost the same as the standard SVGP, and the only difference is that instead of setting the inducing variables as the function values evaluated at some input locations $\bm{Z}$, interdomain GPs \citep{lazaro-gredilla_inter-domain_2009,van_der_wilk_framework_2020} propose to generalize the inducing variables to 
\begin{equation}
u_m := \int f(x)\phi_m(x) \mathrm{d}x,
\end{equation}
where $\phi_m(x)$ are basis functions, to allow for further flexibility. Since integration is a linear operator and $f$ is a GP, $u_m$ follows a Gaussian distribution\footnote{In fact, interdomain $u$ can be defined by applying any linear operator $L$ to the prior GP $f(\cdot)$ such that $u:=Lf(\cdot)$. In this work, we focus on the case where $L$ is an integral operator.}. The prior cross-covariance between the function values $f_x:=f(x)$ and the inducing variables, and the prior inducing covariance between the inducing variables themselves are computed as integrals rather than direct kernel evaluations: 
\begin{equation}
\begin{split}
[\bK_{f \bu}]_m &:= \text{COV}(f_x,u_m)=\int k(x, x^\prime)\phi_m(x^\prime)\mathrm{d}x^\prime,\\ 
[\bK_{\bu\bu}]_{nm} &:= \text{COV}(u_n, u_m,)= \iint k(x, x^\prime)\phi_n(x)\phi_m(x^\prime)\mathrm{d}x\mathrm{d}x^\prime.
\end{split}
\end{equation}
After replacing the standard prior covariance matrices $\bm{K_{XZ}}$ and $\bm{K_{ZZ}}$ with $\bm{K_{fu}}$ and $\bm{K_{uu}}$ defined above, the posterior predictive distribution and the ELBO training objective for interdomain GPs remain exactly the same as the standard SVGP shown in Section~\ref{sec:svgp_back}. 

We see that unlike standard SVGP based on the inducing locations $\bZ\in\mathbb{R}^{M}$, which can shift locations according to the training data, the interdomain GP bypasses the selection of the inducing locations, and reformulates it with the selection of the basis functions $\phi_i$. The basis functions dictate the structure of the two covariance/kernel matrices above, which in turn modulate the function space of the GP approximation. For example, some interdomian GPs utilize properties of some basis functions to sparsify $\bm{K_{uu}}$ and therefore reduce the computational cost for its inversion \citep{hensman2018variational, dutordoir2020sparse}.

\subsection{Online Sparse Gaussian Processes}\label{sec:online_gp}
Consider online learning where data arrives sequentially in batches $\data_{t_1}:=(\bX_{t_1},\by_{t_1})$, $\data_{t_2}:=(\bX_{t_2},\by_{t_2}),\ldots$ etc. For example, in the time series prediction setting, the data arrives in intervals of $(0, t_1), (t_1, t_2), \ldots$ etc. Sequential Bayesian inference provides a principled framework for online learning where the old posterior for the model can be used as the new prior for the new task. Combining with the likelihood based on the new data, we can then obtain the new posterior. Again, approximation is required when the true sequential posterior update is intractable. 

In this work, we focus on online learning with sparse GPs (reviwed in Section~\ref{sec:svgp_back}), which sequentially update the sparse GP posterior distribution as new data arrive. Suppose that we have already obtained an SVGP (with inducing points $(\bZ_{t_1}, \bu_{t_1})$), $q_{t_1}(f)=\int p_{t_1}(f|\bu_{t_1})q_{t_1}(\bu_{t_1})d\bu_{t_1}$, trained on the first batch of data, $\data_{t_1}$, online SVGP (OSVGP; \citep{bui_streaming_2017}) utilizes VI based approximation to update the SVGP when observing the second task. The new SVGP, based on an updated set of inducing points $(\bZ_{t_2}, \bu_{t_2})$, is denoted as $q_{t_2}(f)=\int p_{t_2}(f|\bu_{t_2})q_{t_2}(\bu_{t_2})d\bu_{t_2}$. Notice that the prior for the new SVGP can be updated as well (i.e., the kernel hyperparameters can be updated), so $p_{t_2}(f)$ can be different from $p_{t_1}(f)$. Ultimately, we would like to minimize the full-batch KL-divergence $\text{KL}\left(q_{t_2}(f)|| p_{t_2}(f|\data_{t_1}, \data_{t_2})\right)$, which is equivalent to maximizing the following ELBO: 
\begin{equation}
\begin{split}
    \mathcal{L}_{ELBO}&= \log \frac{p_{t_2}(\data_{t_1}, \data_{t_2})}{p_{t_1}(\data_{t_1})} - \text{KL}\left(q_{t_2}(f)|| p_{t_2}(f|\data_{t_1}, \data_{t_2})\right)\\
    &= \int q_{t_2}(f) \log \frac{p_{t_2}(\by_{t_2}, \by_{t_1})}{p_{t_1}(\y_{t_1})q_{t_2}(f)} p_{t_2}(f|\by_{t_1}, \by_{t_2})df\\
    &=\int q_{t_2}(f) \log \frac{p_{t_2}(f)p_{t_2}(\by_{t_2}|f)p_{t_2}(\by_{t_1}|f)}{p_{t_1}(\y_{t_1})q_{t_2}(f)} df.
\end{split}
\label{eq:original_online_elbo}
\end{equation}
Notice that $p_{t_1}(\mathcal{D}_{t_1})$ in the denominator of the first term in the first line is a constant since it only depends on the old kernel hyperparameters. Unfortunately, in an online setting, $p_{t_2}(\by_{t_1}|f)$ is intractable since we are not allowed to revisit the old data with the new model. To bypass this intractability, \citet{bui_streaming_2017} proposed to approximate this term with
\begin{equation}
    p_{t_2}(\by_{t_1}|f) \approx p_{t_1}(\by_{t_1}|f) = \frac{p_{t_1}(f|\by_{t_1})p_{t_1}(\by_{t_1})}{p_{t_1}(f)}\approx \frac{q_{t_1}(f)p_{t_1}(\by_{t_1})}{p_{t_1}(f)}.
\end{equation}
Plugging this approximation back to Eq.~\ref{eq:original_online_elbo} yields the ``online ELBO":
\begin{equation}
    \mathcal{L}_{ELBO}^{o} = \int q_{t_2}(f) \log \frac{p_{t_2}(f)p_{t_2}(\by_{t_2}|f)q_{t_1}(f)}{q_{t_2}(f)p_{t_1}(f)} df.
\end{equation}
Since $q_{t_1}(f)$ and $q_{t_2}(f)$ are completely determined based on finite sets of inducing points ($\bu_{t_1}$ and $\bu_{t_2}$, respectively) and we assume that the likelihood factorizes across data points, the objective can be further simplified to be:
\begin{equation}
\begin{aligned}
    \mathcal{L}_{ELBO}^{o} =&\sum_{i=1}^{n_{t_2}} \mathbb{E}_{q_{t_2}(f_i)}
    \left[ \log p_{t_2}(y_i \mid f_i) \right] 
    + \text{KL} \left( \tilde{q}_{t_2}(\bu_{t_1}) \,\middle\|\, p_{t_1}(\bu_{t_1}) \right) \\
    & \quad - \text{KL} \left( \tilde{q}_{t_2}(\bu_{t_1}) \,\middle\|\, q_{t_1}(\bu_{t_1}) \right) 
    - \text{KL} \left( q_{t_2}(\bu_{t_2}) \,\middle\|\, p_{t_2}(\bu_{t_2}) \right),
\end{aligned}
\label{eq:online_elbo}
\end{equation}
where $y_i\in\by_{t_2}$ for $i=1,\ldots,n_{t_2}$ and $\tilde{q}_{t_2}(\bu_{t_1}) := \int p_{t_2}(\bu_{t_1}|\bu_{t_2})q_{t_2}(\bu_{t_2}) \mathrm{d}\bu_{t_2}$. 

Unfortunately, with more and more tasks, OSVGP may not capture the long-term memory in the data since as new data arrives, it is not guaranteed that the inducing locations after optimization can sufficiently cover all the previous tasks' input domains.

\subsection{Gaussian Processes Variational Autoencoders}
In addition to predictive tasks, we also consider continual learning in Gaussian processes variational autoencoder (GPVAE; \citep{casale_gaussian_2018, fortuin_gpvae_2020, ashman_sparse_2020, jazbec_scalable_2021, zhu_markovian_2023}), a GP based generative model, and here we give a brief introduction. GPVAEs embed Gaussian processes within a variational autoencoder (VAE; \citep{welling2014auto}) framework. For sparse GPs with inducing variables, \citet{jazbec_scalable_2021} introduced the SVGPVAE, which combines the sparse variational GP (SVGP) with the VAE formulation. The likelihood $p(y \mid \varphi_\theta(f))$
is parameterized by a decoder network $\varphi_\theta$, which takes GP latent draws $f$ as input, together with the variational inducing posterior $q_\theta(\bm{u}\mid y)$. This posterior, $q_\theta(\bm{u}\mid \phi(y))$, is parameterized by the encoder network $\phi$. Finally, the latent GP $f$ is typically modeled as a multi-output GP with independent components. GPVAEs have been shown to successfully model high-dimensional time series such as weather data and videos \citep{zhu_markovian_2023,fortuin_gpvae_2020}. In this work, we consider the SVGPVAE model defined in \citet{jazbec_scalable_2021} for one set of our experiments, and the detailed specification of the model and training objective can be found in Appendix~\ref{appendix:gpvae}.

\section{Interdomain
 Inducing Point Gaussian Processes with HiPPO}
\label{sec:method}

We bridge the HiPPO framework with interdomain Gaussian processes by interpreting HiPPO's state vector defined by time-varying orthogonal projections as interdomain inducing points. This enables adaptive compression of the history of a GP while preserving long-term memory.

\subsection{HiPPO as Interdomain Inducing Variables}
\label{sec:hippo_inter}

Recall that in an interdomain setting in Section~\ref{sec:vi_and_interdomain_GP}, inducing variables are defined through an integral transform against a set of basis functions. Let $f\sim\mathcal{GP}(0,k)$, and consider time-dependent basis functions 
\begin{equation}
\phi_{m}^{(t)}(x) = g_{m}^{(t)}(x)\omega^{(t)}(x),
\end{equation}
where \(g_{m}^{(t)}\) are the orthogonal functions of HiPPO and \(\omega^{(t)}\) is the associated measure. We define the corresponding interdomain inducing variables as 
\begin{equation}
u_{m}^{(t)} = \int f(x)\phi_{m}^{(t)}(x)\mathrm{d}x,
\end{equation}
which is not a one-dimensional random variable as in Section~\ref{sec:vi_and_interdomain_GP}. Rather, it is a random functions (i.e. stochastic processes) over time ($u_m^{(t)}:= u_m(t)$) due to time-dependent basis functions. These inducing variables adapt in time, capturing long-range historical information in a compact form via HiPPO's principled polynomial projections.

\subsection{Adapting the Kernel Matrices over Time}
\label{sec:covariance_evol}
When new observations arrive at later times in a streaming scenario, we must adapt both the prior cross-covariance \(\bm{K}_{\bm{fu}}\) and the prior covariance of the inducing variables \(\bm{K}_{\bm{uu}}\). In particular, the basis functions in our HiPPO construction evolve with time, so the corresponding kernel quantities also require updates. Below, we describe how to compute and update these matrices at a new time \(t_2\) given their values at time \(t_1\). For clarity, we first discuss \(\bm{K}_{\bm{fu}}\), then \(\bm{K}_{\bm{uu}}\).
%
\paragraph{Prior cross-covariance \(\bm{K}_{\bm{fu}}^{(t)}\).}
Recall that for a single input \(x_{n}\), the prior cross-covariance with the \(m\)-th inducing variable is 
\begin{equation}
\left[\bm{K}_{\bm{fu}}^{(t)}\right]_{nm} = \int k\left(x_{n}, x\right)\phi_{m}^{(t)}(x)\mathrm{d}x,
\end{equation}
which is of the same form as a projection coefficient (Eq.~\ref{eq:hippo_coef}) in the HiPPO framework. Hence, We can compute the temporal evolution of \(\bm{K}_{\bm{fu}}^{(t)}\) in a manner consistent with the HiPPO approach, leveraging the same parameters \(\bm{A}(t)\) and \(\bm{B}(t)\). Specifically,  
\begin{talign}
\label{eq:kfu_ode}
\frac{\mathrm{d}}{\mathrm{d}t}\left[\bm{K}_{\bm{fu}}^{(t)}\right]_{n,:}
=
\bm{A}(t)\left[\bm{K}_{\bm{fu}}^{(t)}\right]_{n,:}
+
\bm{B}(t)k\left(x_{n},t\right),
\end{talign}
where $\left[\bm{K}_{\bm{fu}}^{(t)}\right]_{n,:}$ is the $n$-th row of $\bm{K}_{\bm{fu}}^{(t)}$. The matrices $\bm{A}(t)$ and $\bm{B}(t)$ depend on the specific choice of the HiPPO measure and basis functions. In our experiments, we employ HiPPO-LegS, whose explicit matrix forms are provided in Section~\ref{sec:hippo}. One then discretizes in \(t\) (e.g. using an Euler method or a bilinear transform) to obtain a recurrence update rule.
%
\paragraph{Prior inducing covariance \(\bm{K}_{\bm{uu}}^{(t)}\).}
\label{sec:kuu_ode_rff} 
The $\ell m$-th element of the prior covariance matrix for the inducing variables is given by 
\begin{equation}
\left[\bm{K}_{\bm{uu}}^{(t)}\right]_{\ell m}
=
\iint
k\left(x,x^\prime\right)\phi_{\ell}^{(t)}(x)\phi_{m}^{(t)}(x^\prime)\mathrm{d}x\mathrm{d}x^\prime.
\label{eq:double_integral_kuu}
\end{equation} 
Since $k(x, x^\prime)$ depends on both $x$ and
$x^\prime$, a recurrence update rule based on the original HiPPO formulation, which is designed for single integral, can not be obtained directly for \(\bm{K}_{\bm{uu}}^{(t)}\). Fortunately, for stationary kernels, Bochner Theorem \citep{rudin_fourier_1994} can be applied to factorize the double integrals into two separate single integrals, which gives rise to Random Fourier Features (RFF) approximation \citep{rahimi_random_2007}: for a stationary kernel \(k(x,x^\prime) = k(|x-x^\prime|)\), RFF approximates it as follows:
\begin{equation}
\begin{split}
k(x,x^\prime) &= \mathbb{E}_{p(w)}\Bigl[\cos(wx)\,\cos(wx^{\prime}) +\sin(wx)\,\sin(wx^{\prime})\Bigr]\\
&\approx \frac{1}{N}\sum_{n=1}^N\left[
\cos\left(w_nx\right)\cos\left(w_nx^\prime\right) + \sin\left(w_nx\right)\sin\left(w_nx^\prime\right)\right],
\end{split}
\end{equation}
where $w_n\sim p(w)$ is the spectral density of the kernel\footnote{For non-stationary kernels, more advanced Fourier feature approximation techniques (e.g., \citep{ton2018spatial}) can be applied.}. Substituting this into the double integral (Eq.~\ref{eq:double_integral_kuu}) factorizes the dependency on \(x\) and \(x^\prime\), reducing \([\bm{K}_{\bm{uu}}^{(t)}]_{\ell m}\) to addition of products of one-dimensional integrals. Each integral based on a Monte Carlo sample $w$ has the form of either 
\begin{equation}
 Z_{w,\ell}^{(t)} = \int \cos(wx)\phi_{\ell}^{(t)}(x)\,\mathrm{d}x  \quad \text{or} \quad Z_{w,\ell}^{\prime(t)} = \int \sin(wx)\phi_{\ell}^{(t)}(x)\,\mathrm{d}x,
 \end{equation}
which corresponds to a standard projection coefficient in the HiPPO framework (Eq.~\ref{eq:hippo_coef}).
We further stack these integrals based on $M$ basis functions and define
\begin{equation}
    \bm{Z}_{w}^{(t)} = \left[Z_{w,1}^{(t)},  \cdots, Z_{w,M}^{(t)}\right]^\top, \quad \bm{Z}_{w}^{\prime(t)} = \left[Z_{w,1}^{\prime(t)}, \cdots, Z_{w,M}^{\prime(t)}\right]^\top. 
\end{equation}
Collecting $N$ Monte Carlo samples $\{w_n\}_{n=1}^N$, we form the feature matrix
\begin{equation}
    \bm{Z}^{(t)} = \begin{bmatrix} \bm{Z}_{w_1}^{(t)} & 
    \bm{Z}_{w_2}^{(t)} & \cdots & 
    \bm{Z}_{w_N}^{(t)} & 
    \bm{Z}_{w_1}^{\prime(t)} & 
    \bm{Z}_{w_2}^{\prime(t)} & \cdots & 
    \bm{Z}_{w_N}^{\prime(t)} 
    \end{bmatrix},
\end{equation}
and the RFF approximation of the covariance is
\begin{equation}
    \bm{K}_{\bm{uu}}^{(t)} \approx \frac{1}{N}\,\bm{Z}^{(t)}\left(\bm{Z}^{(t)}\right)^\top.
\end{equation}
Since $\bm{Z}_{w_n}^{(t)}$ and $\bm{Z}_{w_n}^{\prime(t)}$ are standard HiPPO projection coefficient, their computation is governed by the HiPPO ODE evolution as before
\begin{equation}
    \frac{\mathrm{d}}{\mathrm{d}t}\bm{Z}_{w_n}^{(t)} = \bm{A}(t)\,\bm{Z}_{w_n}^{(t)} + \bm{B}(t)\,h_n(t), \quad \frac{\mathrm{d}}{\mathrm{d}t}\bm{Z}_{w_n}^{\prime(t)} = \bm{A}(t)\,\bm{Z}_{w_n}^{\prime(t)} + \bm{B}(t)\,h^{\prime}_n(t),
\end{equation}
with $ h_n(t) = \cos(w_nt) $ and $ {h}^{\prime}_n(t) = \sin(w_nt) $ and these ODEs can be solved in parallel across different Monte Carlo samples. In summary, the procedure involves sampling multiple random features, updating them recurrently to time $t$, and averaging across samples to obtain RFF approximation of $\bm{K_{uu}}^{(t)}$.

Alternatively, one may differentiate \(\bm{K}_{\bm{uu}}^{(t)}\) directly w.r.t. \(t\). This yields a matrix ODE of the form different from the original HiPPO formulation. For details, see Appendix~\ref{appendix:direct-ode-hippo-legs}. Empirically, a vanilla implementation of this approach shows numerical unstability (see Appendix~\ref{appendix:direct-ode-hippo-legs-unstability}). Hence, we conduct our experiments based on the RFF approximation described above.
%
\paragraph{Sequential variational updates.}
%
\begin{figure}[t]
    \begin{subfigure}[t]{0.49\textwidth}
      \centering
      \includegraphics[width=\textwidth]{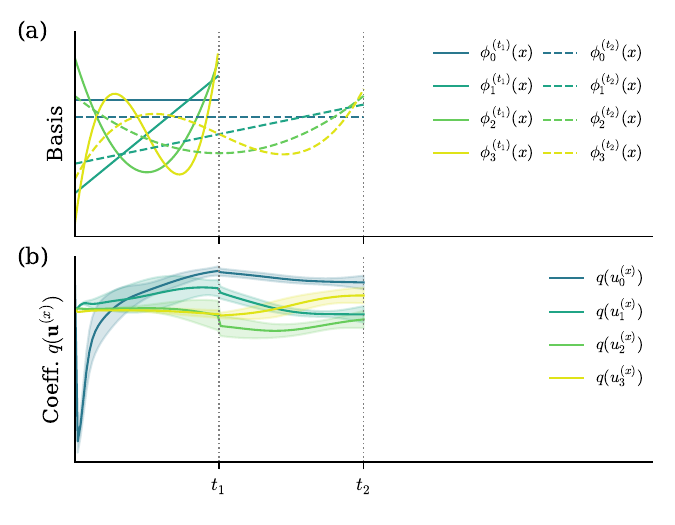}
    \end{subfigure}
    \hfill
    \begin{subfigure}[t]{0.49\textwidth}
      \centering
      \includegraphics[width=\textwidth]{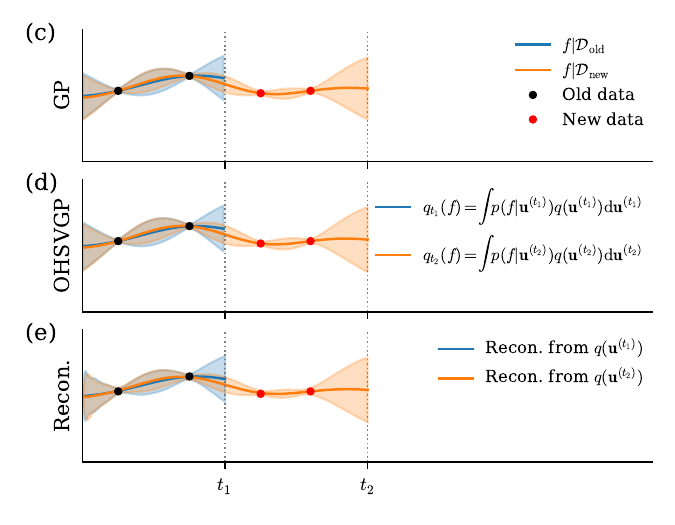}
    \end{subfigure}
    \caption{Online HiPPO Sparse Variational Gaussian Process (OHSVGP) on a toy time series with 2 tasks. Here $x$ is used to denote arbitrary time index. \textbf{(a)} Time-dependent basis functions with end time index $x=t_1$ and $x=t_2$. \textbf{(b)} Evolution of optimal approximate posterior of inducing variables (mean $\pm$2 marginal standard deviation). \textbf{(c)}, \textbf{(d)}, \textbf{(e)} illustrate predictive mean $\pm$2 standard deviation of posterior online GP, OHSVGP and finite basis reconstruction of posterior OHSVGP, respectively.}
    \label{fig:figure1}
\end{figure}
Having obtained \(\bm{K}_{\bm{fu}}^{(t_2)}, \bm{K}_{\bm{uu}}^{(t_2)}\) at a new time \(t_2>t_1\), we perform variational updates following the online GP framework described in Section~\ref{sec:online_gp}. This ensures the posterior at time \(t_2\) remains consistent with both the new data and the previous posterior at time \(t_1\), based on \(\bm{K}_{\bm{fu}}^{(t_1)}, \bm{K}_{\bm{uu}}^{(t_1)}\). Overall, this procedure endows interdomain HiPPO-based GPs with the ability to capture long-term memory online. By viewing the induced kernel transforms as ODEs in time, we efficiently preserve the memory of past observations while adapting our variational posterior in an online fashion. Figure~\ref{fig:figure1}b illustrates the evolution of the optimal posterior $q(\bm{u}^{(x)})$ as time $x$ increases on a toy online time series regression problem with two tasks, where $x$ determines the end of the recurrent update for the prior cross and inducing covariance matrices (evolved up to $\bm{K}_{\bm{fu}}^{(x)}$ and $\bm{K}_{\bm{uu}}^{(x)}$, respectively). Furthermore, when $x>t_1$, we will update $q(\bm{u}^{(x)})$ online with the two data points from the second task by optimizing the online ELBO (Eq.~\ref{eq:online_elbo}), which gives the discrete jump at $x=t_1$. Figure~\ref{fig:figure1}d shows the posterior OHSVGP compared with the fit of the gold-standard online GP in Figure~\ref{fig:figure1}c. Notably, if $f \sim q_{t}(f)$, then $\int f(x) \phi_{m}^{(t)}(x) \mathrm{d}x \sim q(u_m^{(t)})$ (detailed derivation in Appendix~\ref{appendix:recon}). Therefore, our framework also provides a finite basis approximation of the posterior OHSVGP as a byproduct: $f = \sum_{m=1}^M u_m^{(t)}g_{m}^{(t)}(x)$, $u_m^{(t)}\sim q(u_m^{(t)})$. Figure~\ref{fig:figure1}e plots the finite basis approximation/reconstruction and it is close to the posterior OHSVGP for this simple example. 

\section{Extending OHSVGP to Multidimensional Input}
\label{sec:multi_dim}
\begin{figure}[htbp]
  \begin{subfigure}[t]{0.32\textwidth}
      \centering
      \includegraphics[width=\textwidth]{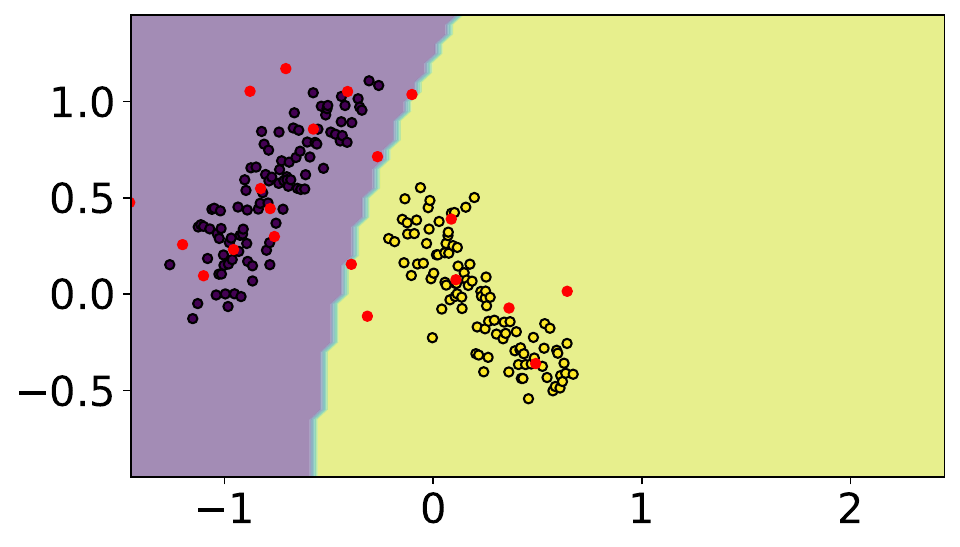}
      \caption{OSVGP, after Task 1}
  \end{subfigure}
  \hfill
  \begin{subfigure}[t]{0.32\textwidth}
      \centering
      \includegraphics[width=\textwidth]{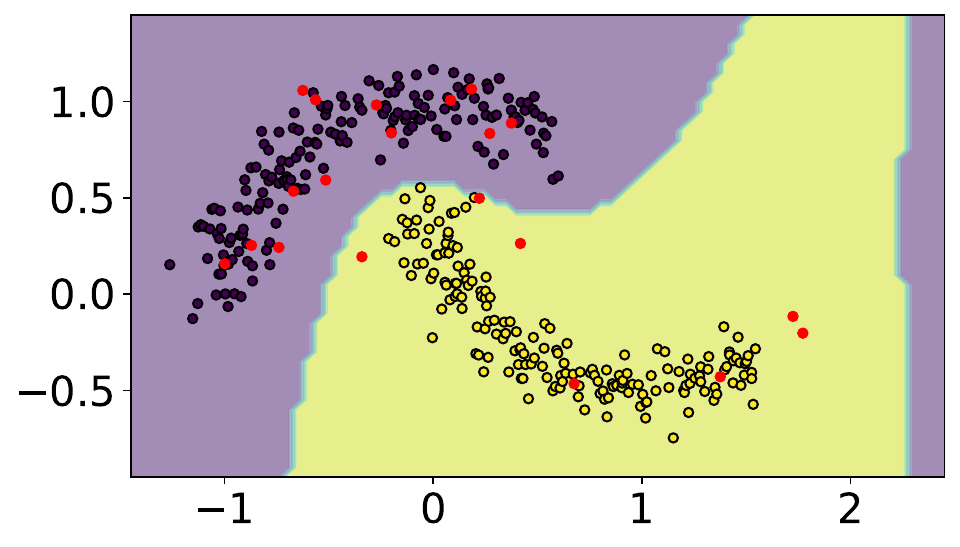}
      \caption{OSVGP, after Task 2}
  \end{subfigure}
  \hfill
  \begin{subfigure}[t]{0.32\textwidth}
      \centering
      \includegraphics[width=\textwidth]{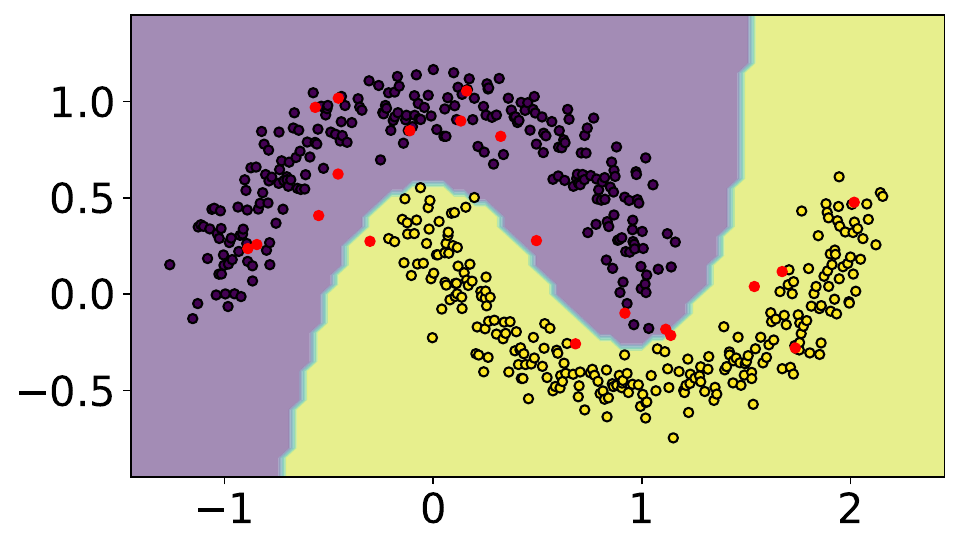}
      \caption{OSVGP, after Task 3}
  \end{subfigure}

  \begin{subfigure}[t]{0.32\textwidth}
      \centering
      \includegraphics[width=\textwidth]{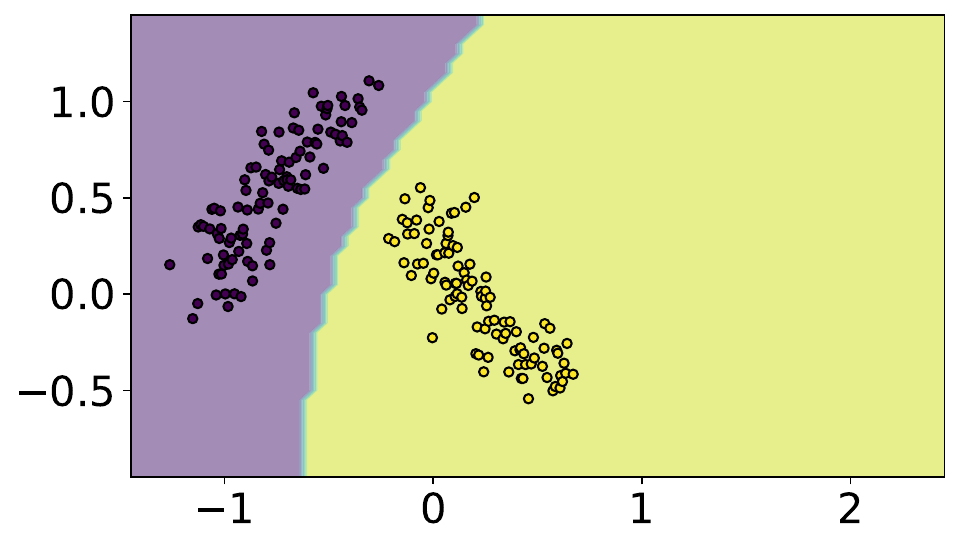}
      \caption{OHSVGP-rand, after Task 1}
  \end{subfigure}
  \hfill
  \begin{subfigure}[t]{0.32\textwidth}
      \centering
      \includegraphics[width=\textwidth]{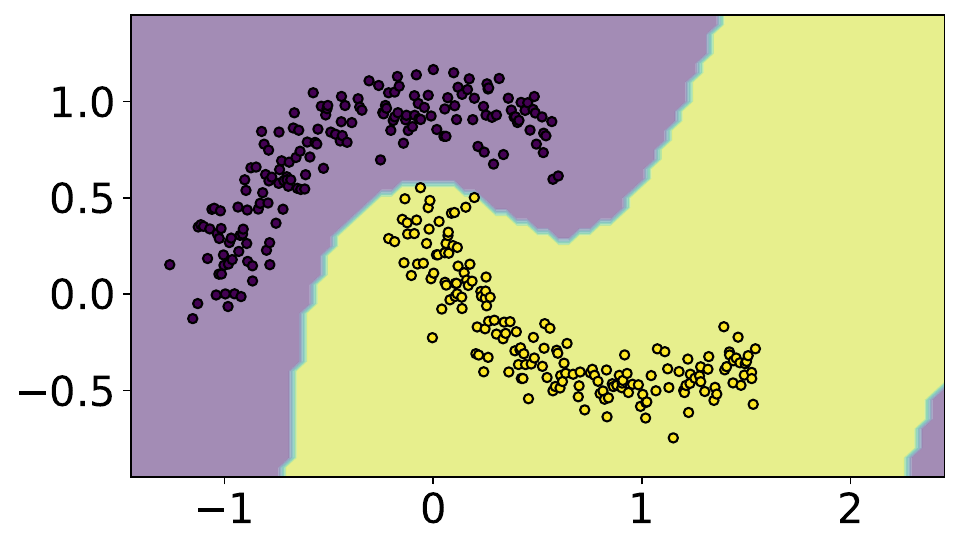}
      \caption{OHSVGP-rand, after Task 2}
  \end{subfigure}
  \hfill
  \begin{subfigure}[t]{0.32\textwidth}
      \centering
      \includegraphics[width=\textwidth]{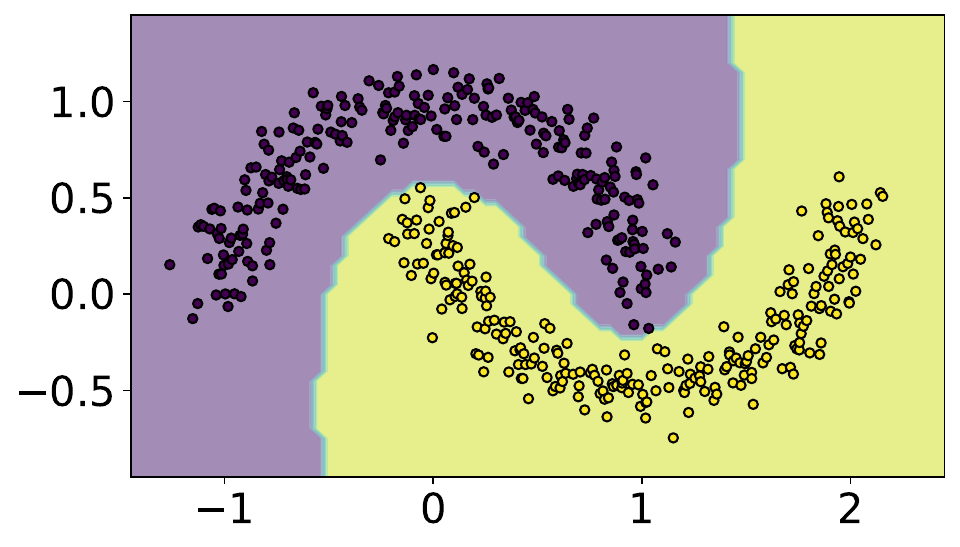}
      \caption{OHSVGP-rand, after Task 3}
  \end{subfigure}

  \begin{subfigure}[t]{0.32\textwidth}
      \centering
      \includegraphics[width=\textwidth]{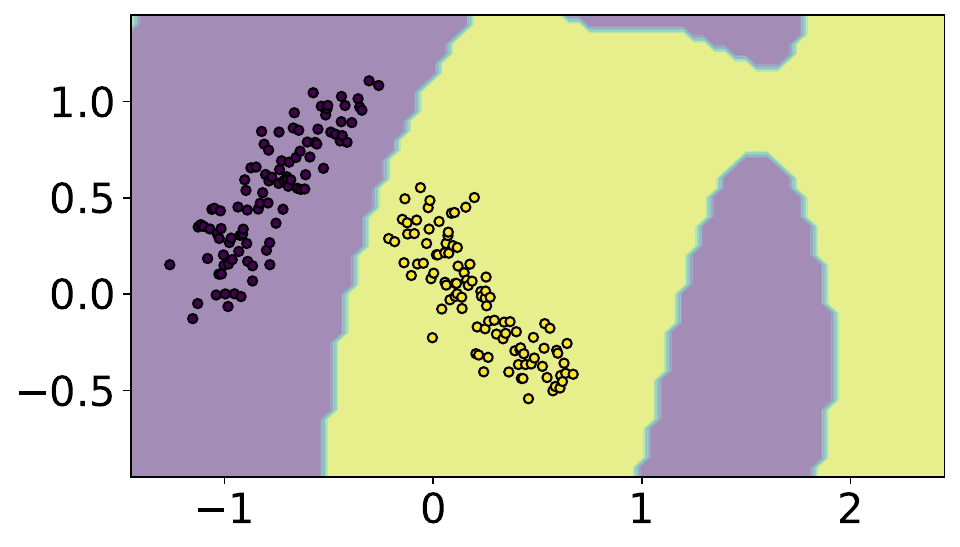}
      \caption{OHSVGP-k-max, after Task 1}
  \end{subfigure}
  \hfill
  \begin{subfigure}[t]{0.32\textwidth}
      \centering
      \includegraphics[width=\textwidth]{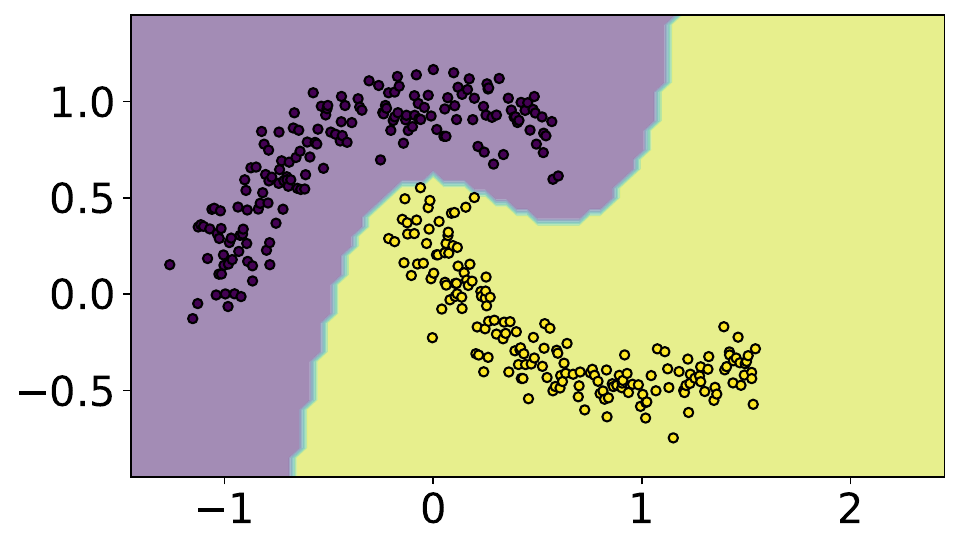}
      \caption{OHSVGP-k-max, after Task 2}
  \end{subfigure}
  \hfill
  \begin{subfigure}[t]{0.32\textwidth}
      \centering
      \includegraphics[width=\textwidth]{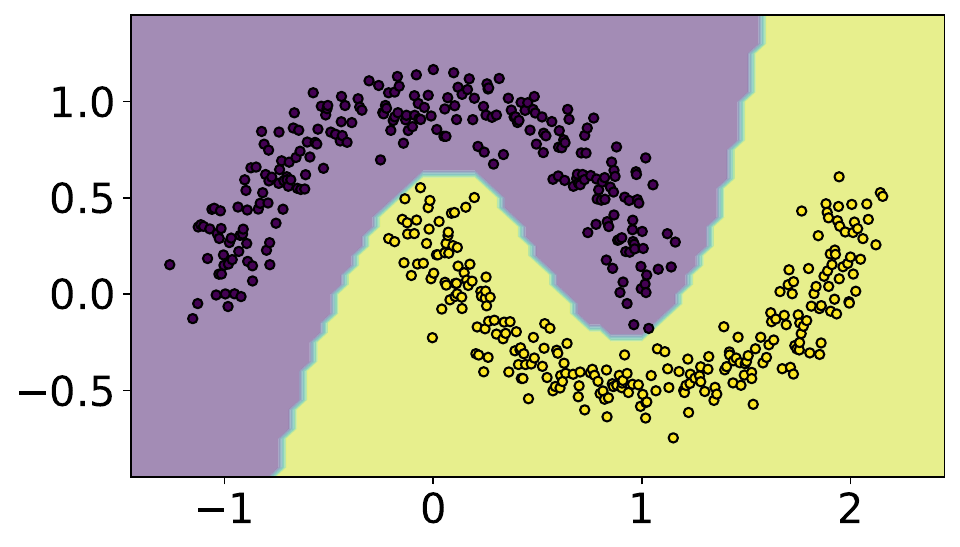}
      \caption{OHSVGP-k-max, after Task 3}
  \end{subfigure}

  \begin{subfigure}[t]{0.32\textwidth}
      \centering
      \includegraphics[width=\textwidth]{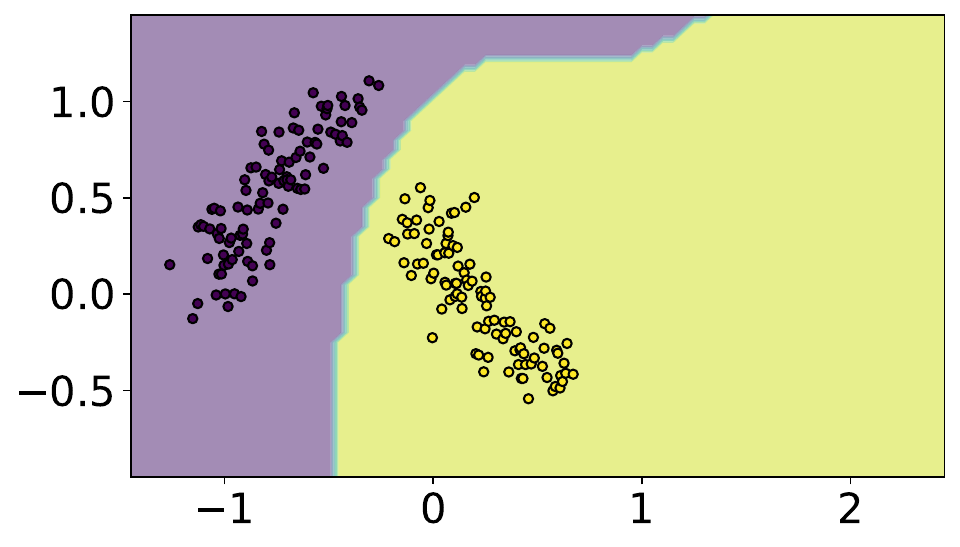}
      \caption{OHSVGP-k-min, after Task 1}
  \end{subfigure}
  \hfill
  \begin{subfigure}[t]{0.32\textwidth}
      \centering
      \includegraphics[width=\textwidth]{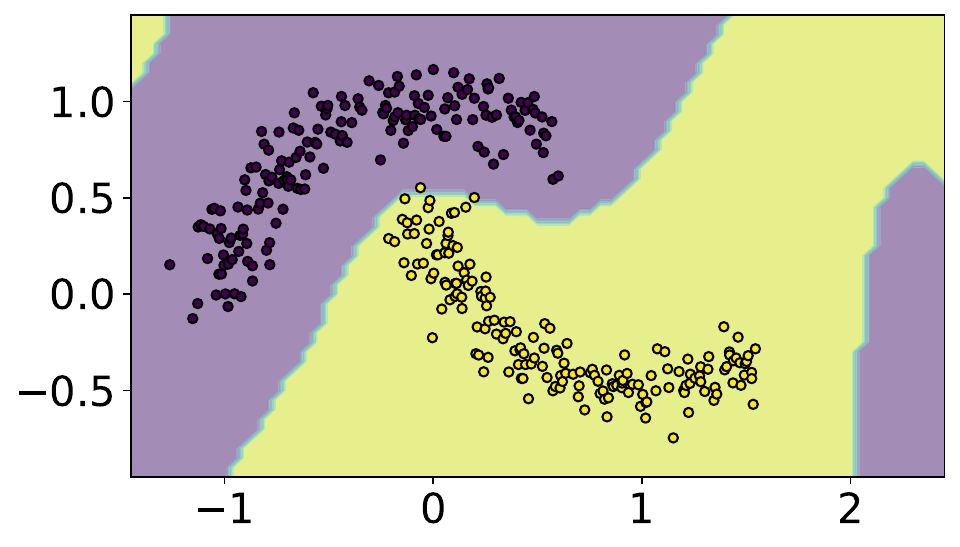}
      \caption{OHSVGP-k-min, after Task 2}
  \end{subfigure}
  \hfill
  \begin{subfigure}[t]{0.32\textwidth}
      \centering
      \includegraphics[width=\textwidth]{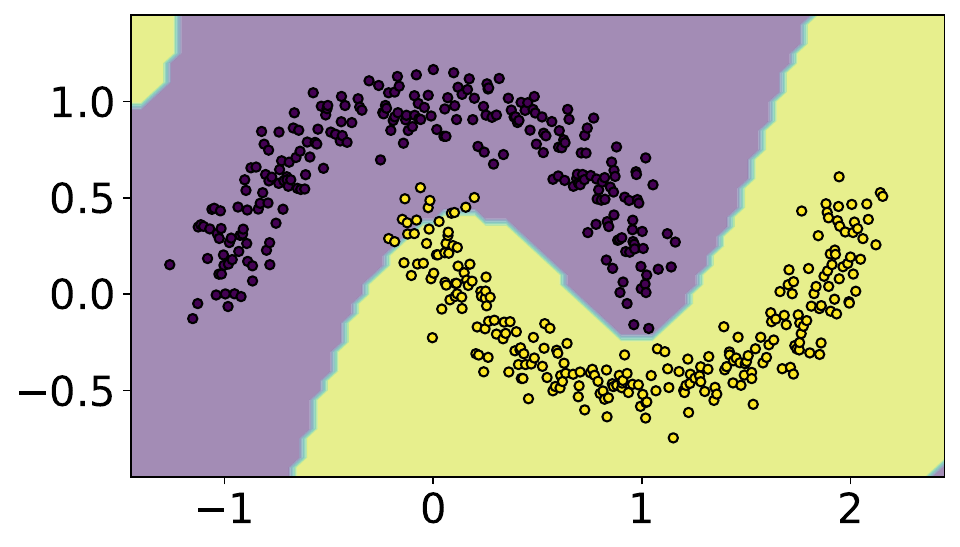}
      \caption{OHSVGP-k-min, after Task 3}
  \end{subfigure}
 
  \caption{Decision boundaries of OSVGP, and OHSVGP models with different sorting criteria after each task (3 in total) on the Two-moon dataset. For OSVGP, we visualize the inducing locations with red color.}
  \label{fig:decision_boundary_moon}
\end{figure}
For multidimensional input data, suppose there is a time order for the first batch of training points with inputs $\{\bm{x}^{(1)}_n\}_{n=1}^{N_1}$, such that $\bm{x}^{(1)}_i$ appears after $\bm{x}^{(1)}_j$ if $i>j$, and we further assume $\bm{x}^{(1)}_i$ appears at time index $i\Delta t$ (i.e., $\bm{x}(i \Delta t) = \bm{x}^{(1)}_i$), where $\Delta t$ is a user-specified constant step size. In this case, we can again obtain interdomain prior covariance matrices via HiPPO recurrence. For example, a forward Euler method applied to the ODE in Eq.~\ref{eq:kfu_ode} for $\bm{K}_{\bm{fu}}^{t}$ yields
\begin{equation}
    [\bm{K}_{\bm{fu}}^{((i+1)\Delta t)}]_{n,:} =[\bm{I}+\Delta t\bm{A}(i\Delta t)][\bm{K}_{\bm{fu}}^{(i\Delta t)}]_{n,:} + \Delta t\bm{B}(i \Delta t)k \left(\bm{x}^{(1)}_n, \bm{x}^{(1)}_i \right).
    \label{eq:multi_dim_hippo_kfu}
\end{equation}
The equation above can be viewed as a discretization (with step size $\Delta t$) of an ODE solving path integrals of the form $\int_0^{N_1\Delta t} k \left(\bm{x}^{(1)}_n, \bm{x}(s) \right)\phi_m^{(t)}(s)ds$. The $i$-th training input $\bm{x}^{(1)}_i$ is assumed to be $\bm{x}^{(1)}_i:=\bm{x}(i\Delta t)$ and thus the path integral is approximately solved with discretized recurrence based on the training inputs corresponding to $\{\bm{x}(i\Delta t)\}_{i=1}^N$. We continue the recurrence for the second task with ordered training inputs $\{\bm{x}^{(2)}_n\}_{n=1}^{N_2}$ by assigning time index $(N_1+i)\Delta t$ to its $i$-th instance. and keep the recurrence until we learn all the tasks continually. In practice, one may use a multiple of $\Delta t$ as the step size to accelerate the recurrence, e.g., instead of using all the training inputs, one can compute the recurrence based on $\{\bm{x}_1, \bm{x}_3, \bm{x}_5, \cdots\}$ only by using step size $2 \Delta t$. 

When there is no natural time order for training instances in each task, such as in standard continual learning applications, we need to sort the instances with some criterion to create pseudo time order to fit OHSVGP, similar to the practice of applying SSMs to non-sequence data modalities, e.g., SSMs, when applied to vision tasks, assign order to patches in an image for recurrence update of the memory \citep{zhu_vision_2024}. In our experiments, we show that the performance of OHSVGP, when applied to continual learning, depends on the sorting criterion used. As an illustrative example, we visualize how different sorting methods impact OHSVGP's performance in continual learning with a 2D continual classification problem. We consider fitting OHSVGPs with 20 inducing variables for a continual binary classification problem on the Two-moon dataset \citep{ganin2016domain}. The data is splitted into three tasks and we use a Bernoulli likelihood to model binary labels. We consider three different sorting criteria:
\begin{itemize}
    \item \textbf{Random,} denoted as OHSVGP-rand. The order of data points in each task is obtained via random permutation.
    \item \textbf{Kernel similarity maximization,} denoted as OHSVGP-k-max. We select the $i$-th point in task $j$ to be $\bm{x}_i^{(j)} = \argmax_{\bm{x} \in \bm{X}^{(j)}}k(\bm{x}, \bm{x}_{i-1}^{(j)})$ for $i>1$, and the first point in first task is set to be $\bm{x}_1^{(1)} = \argmax_{\bm{x} \in \bm{X}^{(1)}}k(\bm{x}, \bm{0})$. The intuition is that the signals to memorize, when computing the prior covariance matrices, tend to be more smooth if the consecutive $\bm{x}$'s are close to each other.
    \item \textbf{Kernel similarity minimization,} denoted as OHSVGP-k-min. We select the $i$-th point in task $j$ to be $\bm{x}_i^{(j)} = \argmin_{\bm{x} \in \bm{X}^{(j)}}k(\bm{x}, \bm{x}_{i-1}^{(j)})$ for $i>1$, and the first point in first task is set to be $\bm{x}_1^{(1)} = \argmin_{\bm{x} \in \bm{X}^{(1)}}k(\bm{x}, \bm{0})$. In this case, we deliberately make it difficult to memorize the signals in the recurrent computation for the prior covariance matrices.
\end{itemize}

Figure~\ref{fig:decision_boundary_moon} show the decision boundaries after each task for different OHSVGPs based on different sorting criteria. We also include the decision boundaries of an OSVGP model for reference. Both OHSVGP-k-max and OHSVGP-rand return decision boundaries achieving 100\% accuracy, while OHSVGP-k-min show catastrophic forgetting after Task 3, which suggests OHSVGP requires a sensible sorting criterion to perform well in continual learning tasks.

\section{Related Work}
\paragraph{Online sparse GPs.}
%
Previous works mainly focus on reducing the sparse approximation error with different approximate inference techniques, such as variational inference \citep{bui_streaming_2017, maddox_conditioning_2021}, expectation propagation \citep{csato2002sparse, bui_streaming_2017}, Laplace approximation \citep{maddox_conditioning_2021}, and approximation enhanced with replay buffer \citep{chang_memory_2023}. The orthogonal research problem of online update of inducing points remains relatively underexplored, and pivoted-Cholesky \citep{burt_rates_2019} as deployed in \citet{maddox_conditioning_2021, chang_memory_2023} is one of the most effective approaches for online update of inducing points up to date. We tackle this problem by taking advantage of the long-term memory capability of HiPPO to design an interdomain inducing variable based method and the associated recurrence based online update rules. Notably, our HiPPO inducing variables in principle are compatible with all the aforementioned approximate inference frameworks since only the way of computing prior covariance matrices will be different from standard online sparse GPs.
%
\paragraph{Interdomain GPs.} 
To our knowledge, OHSVGP is the first interdomain GP method in the context of online learning. Previous interdomain GPs typically construct inducing variables via integration based on a predefined measure (e.g., a uniform measure over a fixed interval \citep{hensman2018variational} or a fixed Gaussian measure \citep{lazaro-gredilla_inter-domain_2009}) to prevent diverging covariances, and this predefined measure may not cover all regions where the time indices from future tasks are, making them unsuitable for online learning. In contrast, OHSVGP bypasses this limitation by utilizing adaptive basis functions constructed based on time-dependent measure which keeps extending to the new time region as more tasks arrive.
%
\paragraph{Markovian GPs.}
Markovian GPs \citep{sarkka2019applied, wilkinson_sparse_2021} have similar recurrence structure during inference and training due to their state space SDE representation. However, Markovian GPs are tailored to one-dimensional input tasks such as time series modeling, while we extend OHSVGP to multidimensional input tasks.

\section{Experiments}
\label{sec:experiments}

\paragraph{Applications \& datasets.} We evaluate OHSVGP against baselines in the following tasks.
\begin{itemize}
\item \textbf{Time series prediction.} We consider regression benchmarks, Solar Irradiance \citep{lean2004solar}, and Audio Signal \citep{bui_tree_2014} produced from the TIMIT database \citep{Garofolo1993timit}. We preprocess the two datasets following similar procedures described in \citet{gal_improving_2015} and \citet{bui_streaming_2017}, respectively (the train-test split is different due to random splitting). In addition, we consider a  daily death‐count time series from Santa Catarina State, Southern Brazil spanning the March 2020 to February 2021 COVID‐19 pandemic, obtained from \citet{hawryluk2021gaussian}. We construct online learning tasks by splitting each dataset into 10 (5 for COVID) sequential partitions with an equal number of training instances.

\item \textbf{Continual learning.} We consider continual learning on two UCI datasets with multi-dim inputs, Skillcraft \citep{skillcraft1_master_table_dataset_272} and Powerplant \citep{combined_cycle_power_plant_294}, using the same data preprocessing procedure as in \citet{stanton_kernel_2021}. We construct two types of continual learning problems by first sorting the data points based on either the values in their first dimension or their L2 distance from the origin, and then splitting the sorted datasets into 10 sequential tasks with an equal number of training instances.

\item \textbf{High dimensional time series prediction.} We evaluate GPVAEs on hourly climate data from ERA5 \citep{cds_era5_single_levels_2023}, comprising 17 variables across randomly scattered locations around the UK from January 2020 onward. The dataset is split into 10 sequential tasks of 186 hourly time steps each.
\end{itemize}
%
\paragraph{Baseline.}
%
We compare OHSVGP with OSVGP \citep{bui_streaming_2017} and OVC (Online Variational Conditioning; \citep{maddox_conditioning_2021}). At the beginning of each task, OSVGP initialize the inducing locations by sampling from the old inducing locations and the new data inputs, while OVC initializes them via pivoted-Cholesky \citep{burt_rates_2019} and we consider both fixing the initialized inducing locations as in \citet{chang_memory_2023} (OVC) or keep training them as in \citet{maddox_conditioning_2021} (OVC-optZ). For time series regression with Gaussian likelihood, we consider OHSGPR and OSGPR (OHSVGP and OSVGP based on closed form ELBO), and we further consider OVFF (OSGPR based on variational Fourier feature (VFF), an interdomain inducing point approach from \citet{hensman2018variational}). 
%
\paragraph{Hyperparameters.}
%
Within each set of experiments, all the models are trained using Adam \citep{kingma2015adam} with the same learning rate and number of iterations. For OHSVGP, we construct inducing variables based on HiPPO-LegS \citep{gu_hippo_2020} (see Appendix~\ref{appendix:basis_measure_variant} for visualizations of using other HiPPO variants) and use 1000 RFF samples. We use ARD-RBF kernel, except for OVFF, tailored specifically to Mat\'ern kernels, where we use Mat\'ern-$\frac{5}{2}$ kernel instead. Similar to \citet{maddox_conditioning_2021}, we do not observe performance gain by keeping updating kernel hyperparameters online, and we include results with trainable kernel hyperparameters in Appendix~\ref{appendix:trainable_kernel} for time series regression, but the performance becomes unstable when number of tasks is large. Thus, we either only train the kernel hyperparameters during the initial task and keep them fixed thereafter (Section~\ref{sec:gpvae}) or obtain them from a full GP model trained over the initial task. It is also worth noting that OVFF requires computing covariances as integrals over a predefined interval covering the whole range of the time indices from all tasks (including unobserved ones), which is impractical in real online learning scenarios. For our experiments, we set the two edges of this interval to be the minimum and maximum time index among the data points from all the tasks, respectively.
%
\paragraph{Evaluations \& metrics.}
%
We report results in test NLL (reviewed in Appendix~\ref{appendix:uq_metric}), or commonly referred to as Negative Log Predictive Density (NLPD) in GP literature, and Root Mean Squared Error (RMSE) (Expected Calibration Error (ECE; \citep{Pakdaman_Naeini_Cooper_Hauskrecht_2015,guo2017calibration}) for COVID data instead). 
\begin{align}
\mathrm{NLPD} = -\frac{1}{N}\sum_{i=1}^N \log \hat{p}\bigl(y_i \mid x_i\bigr)\,, \quad
\mathrm{RMSE} = \sqrt{\frac{1}{N}\sum_{i=1}^N\bigl(y_i - \hat y_i\bigr)^2}. 
\end{align}
The Expected Calibration Error (\citep{Pakdaman_Naeini_Cooper_Hauskrecht_2015,guo2017calibration}; ECE) is defined as the mismatch between the confidence and the coverage:
\begin{equation}
\mathrm{ECE}
= \frac{1}{K}
  \sum_{k=1}^{K}
    \left\lvert
      \frac{1}{N}
      \sum_{i=1}^{N}
        \bm{1}\!\bigl(
          y_i \in \bigl[
            \hat{q}_{\,\tfrac{1-c_k}{2}}(x_i)\;,\;
            \hat{q}_{\,\tfrac{1+c_k}{2}}(x_i)
          \bigr]
        \bigr)
      \;-\;c_k,
    \right\rvert
\end{equation} where $K = 10, \quad c_k \in \{0.05,\,0.15,\,\dots,\,0.95\}$, $N \text{ is the number of test points }(x_i, y_i)$, and $\bm{1}(\cdot)$ is the indicator function. \[\hat{q}_{p}(x_i)
= \text{the empirical $p$–quantile of the $S$ predictive samples}
\;\{\,\hat{y}_i^{(s)}\}_{s=1}^S, \quad S=100.
\]
We report the mean and the associated 95\% confidence interval obtained from 5 (3 for experiments on ERA5) independent runs. 

\subsection{Online Time Series Prediction}
%
%
\begin{figure}[t]
  \begin{subfigure}[t]{0.29\textwidth}
      \centering
      \includegraphics[width=\textwidth]{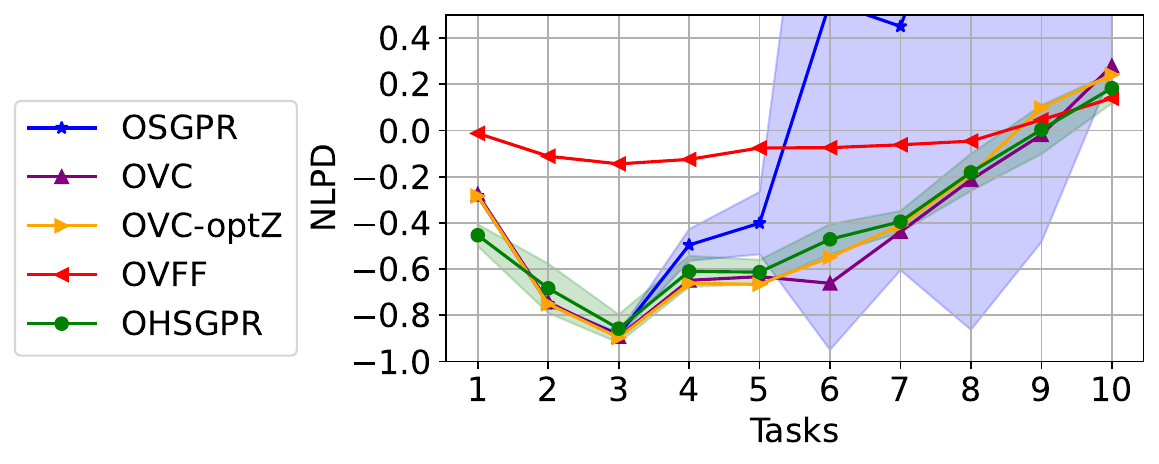}
      \caption{Solar, M=50}
      \label{fig:solar_50}
  \end{subfigure}
  \hfill
  \begin{subfigure}[t]{0.224\textwidth}
      \centering
      \includegraphics[width=\textwidth]{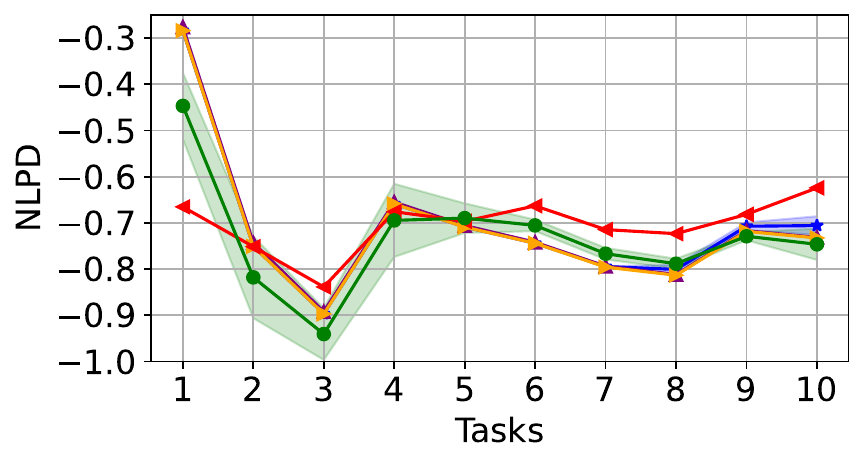}
      \caption{Solar, M=150}
      \label{fig:solar_150}
  \end{subfigure}
  \hfill
  \begin{subfigure}[t]{0.224\textwidth}
      \centering
      \includegraphics[width=\textwidth]{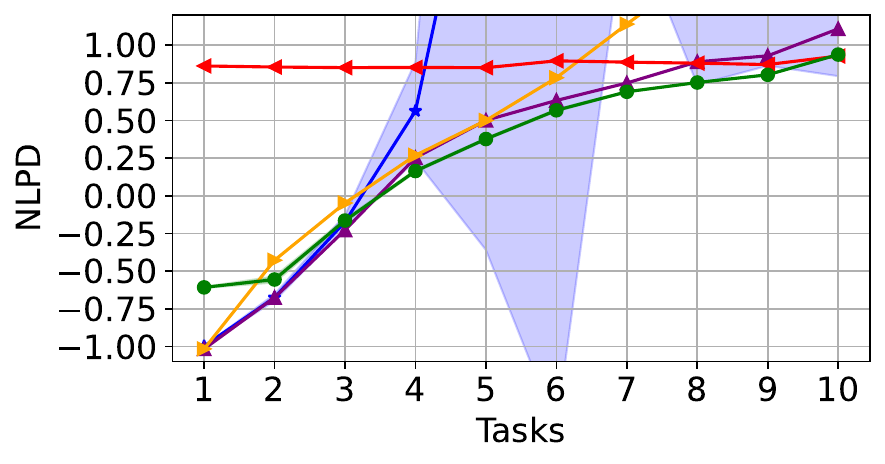}
    \caption{Audio, M=100}
    \label{fig:baseband_100}
  \end{subfigure}
  \begin{subfigure}[t]{0.224\textwidth}
      \centering
      \includegraphics[width=\textwidth]{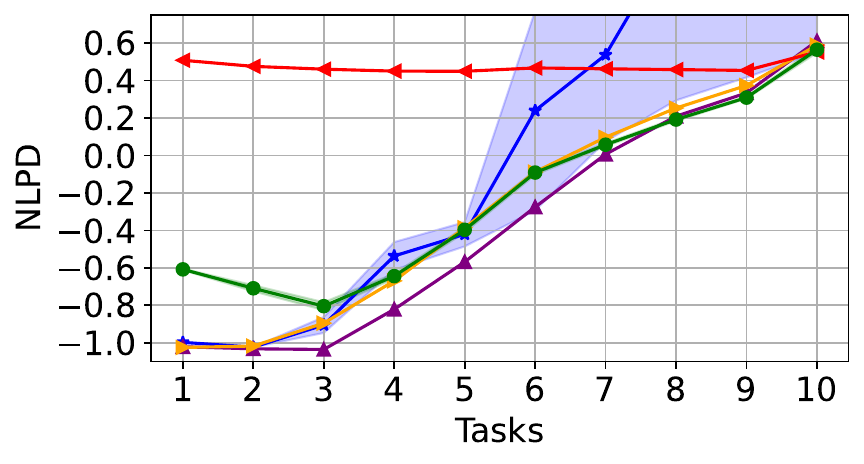}
    \caption{Audio, M=200}
     \label{fig:baseband_200}
  \end{subfigure}
  \caption{Test set NLPD over the learned tasks vs. number of learned tasks for Solar Irradiance and Audio signal prediction dataset.}
  \label{fig:time_series_regression}
\end{figure}

\begin{figure}[t]
  \begin{subfigure}[t]{0.29\textwidth}
      \centering
      \includegraphics[width=\textwidth]{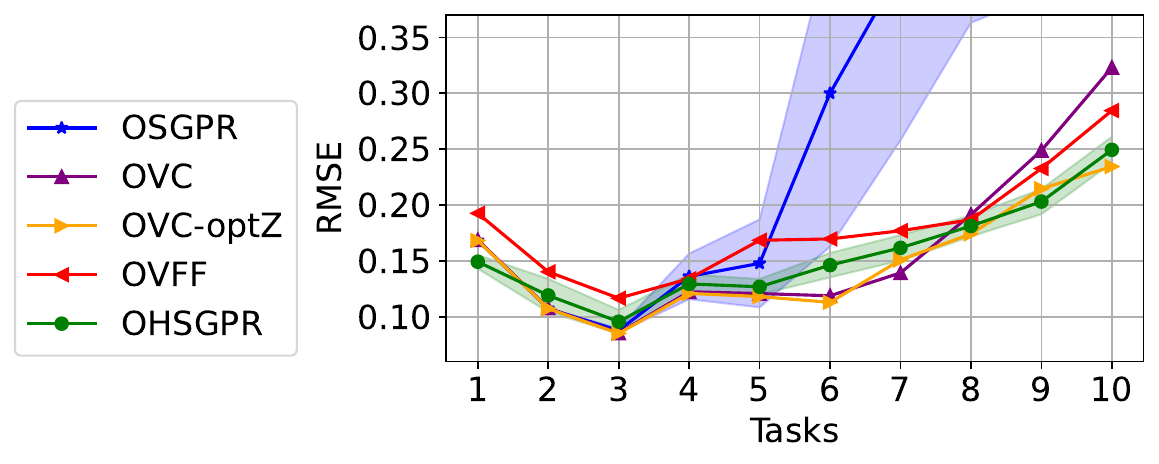}
      \caption{Solar, M=50}
  \end{subfigure}
  \hfill
  \begin{subfigure}[t]{0.224\textwidth}
      \centering
      \includegraphics[width=\textwidth]{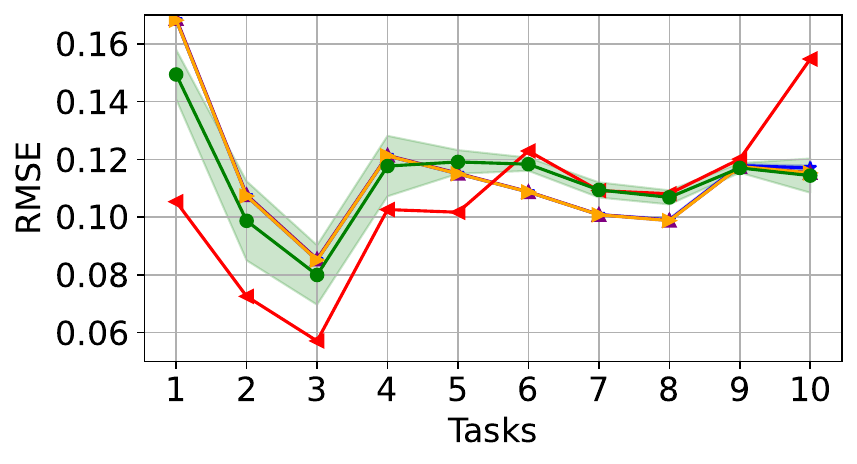}
      \caption{Solar, M=150}
  \end{subfigure}
  \hfill
  \begin{subfigure}[t]{0.224\textwidth}
      \centering
      \includegraphics[width=\textwidth]{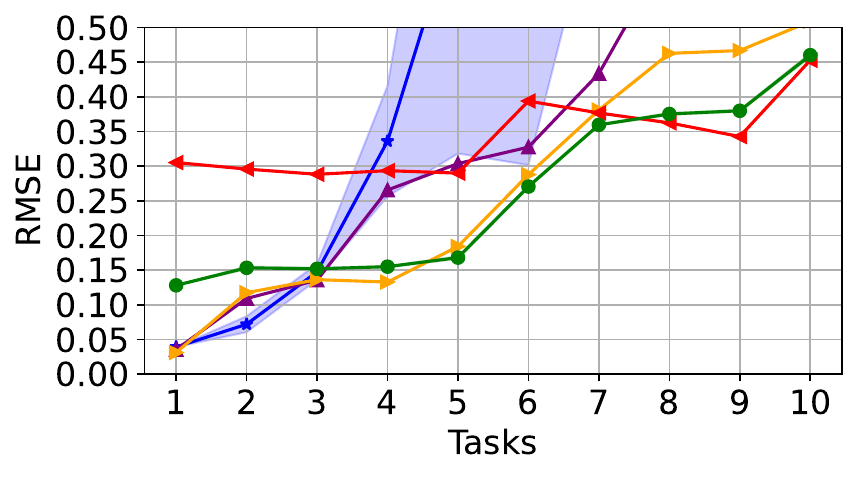}
    \caption{Audio, M=100}
  \end{subfigure}
  \begin{subfigure}[t]{0.224\textwidth}
      \centering
      \includegraphics[width=\textwidth]{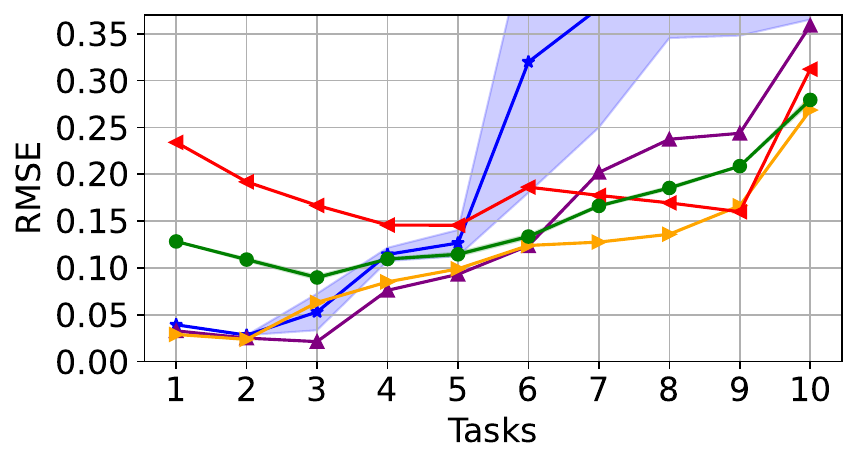}
    \caption{Audio, M=200}
  \end{subfigure}
  \caption{Test set RMSE over the learned tasks vs. number of learned tasks for Solar Irradiance and Audio signal prediction dataset.}
  \label{fig:time_series_reg_rmse}
\end{figure}

\begin{figure}[b]
  \begin{subfigure}[t]{0.304\textwidth}
      \centering
      \includegraphics[width=\textwidth]{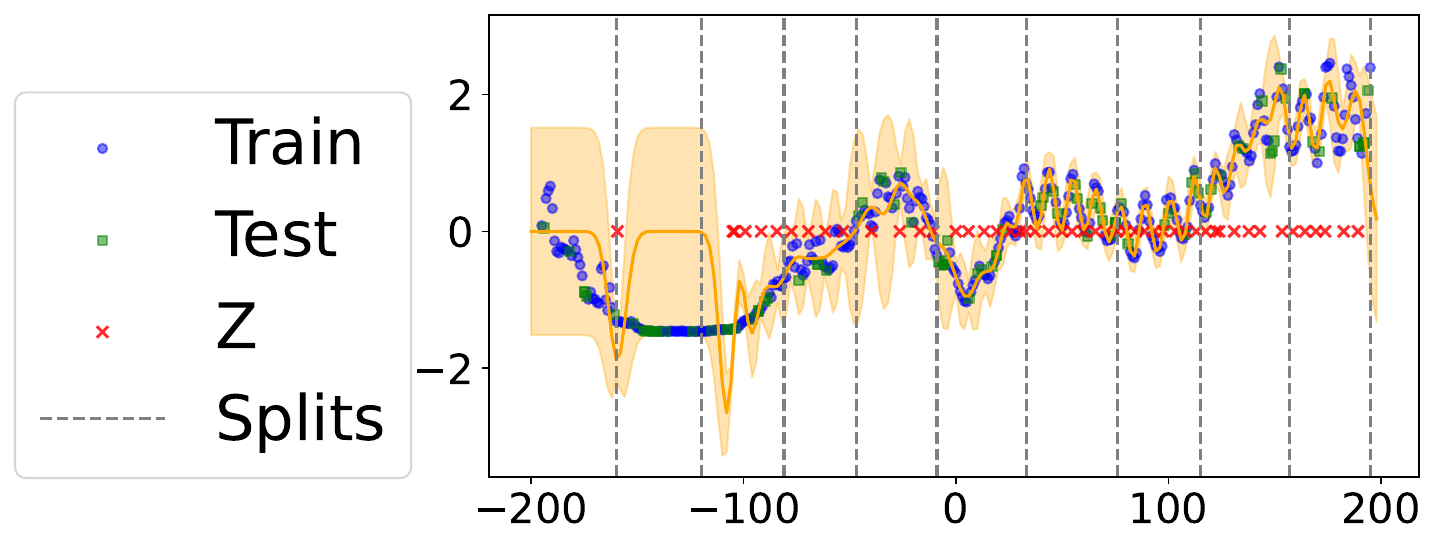}
      \caption{OSGPR}
  \end{subfigure}
  \hfill
  \begin{subfigure}[t]{0.222\textwidth}
      \centering
      \includegraphics[width=\textwidth]{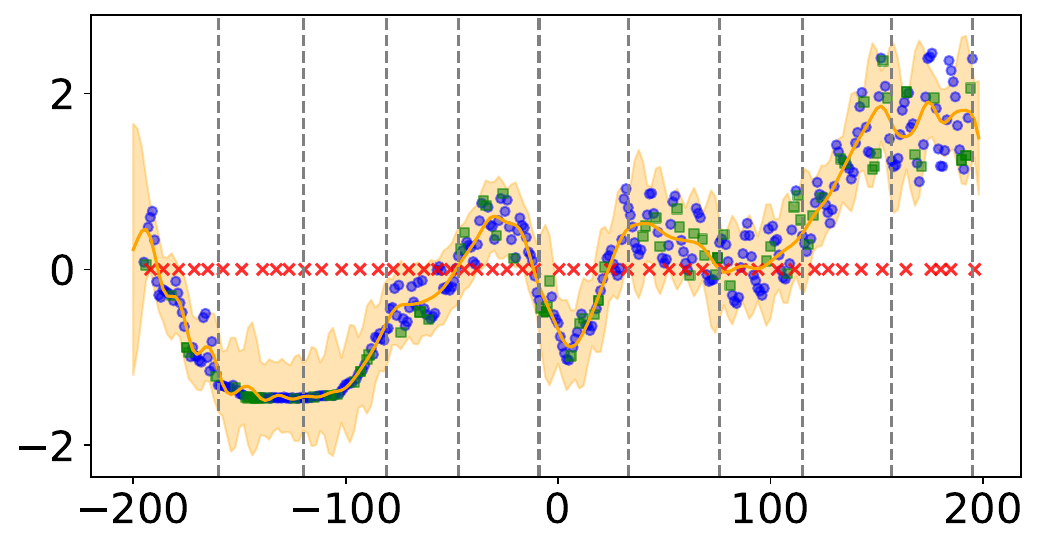}
      \caption{OVC}
  \end{subfigure}
  \hfill
  \begin{subfigure}[t]{0.222\textwidth}
      \centering
      \includegraphics[width=\textwidth]{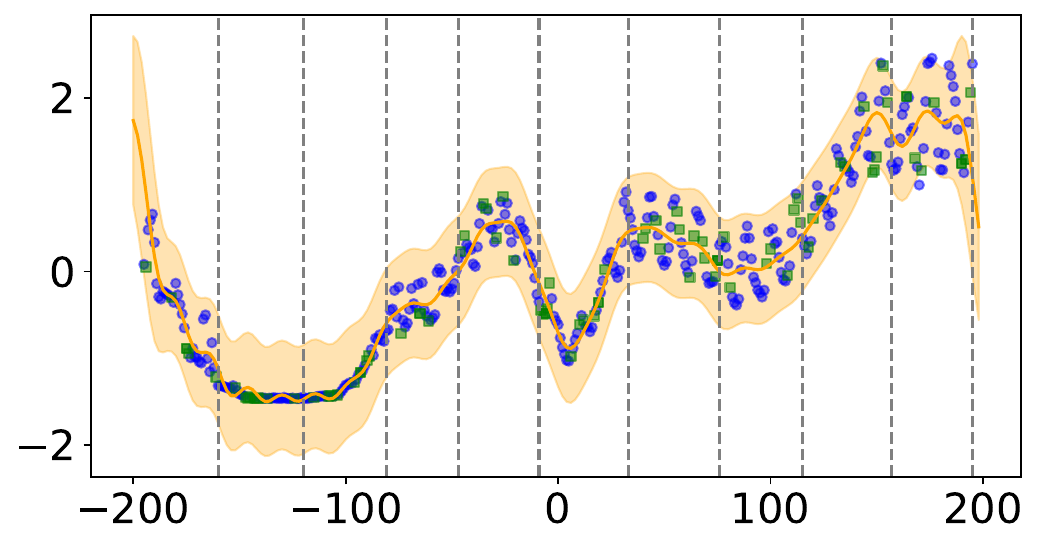}
    \caption{OVFF}
  \end{subfigure}
  \hfill
  \begin{subfigure}[t]{0.222\textwidth}
      \centering
      \includegraphics[width=\textwidth]{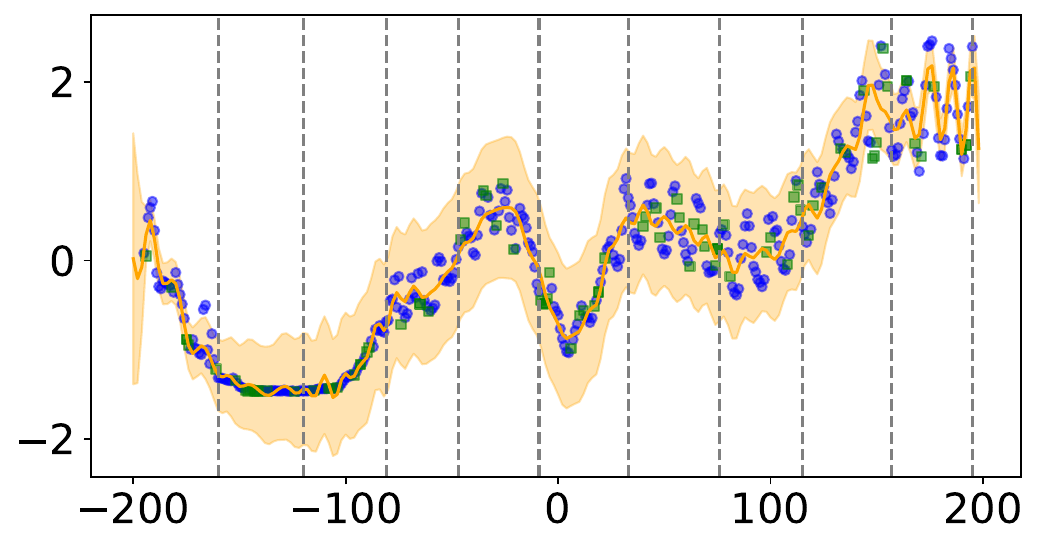}
    \caption{OHSGPR}
  \end{subfigure}
  \caption{Predictive mean $\pm$2 standard deviation of OSGPR, OVC, OVFF, and OHSGPR after task 10 of the Solar dataset. M = 50 inducing variables are used.}
  \label{fig:illustration_forgetting}
\end{figure}
\paragraph{Time series regression.}

Figure~\ref{fig:time_series_regression} and Figure~\ref{fig:time_series_reg_rmse} show NLPD and RMSE (over the past tasks) of different methods during online learning through the 10 tasks for Solar Irradiance and Audio dataset, respectively. 
Overall, OHSGPR consistently achieves the best performance with OVC performing competitively, especially as we learn more and more tasks, suggesting OHSGPR effectively preserves long-term memory through its HiPPO-based memory mechanism. 
OSGPR shows catastrophic forgetting starting around task 5, especially when the number of inducing points $M$ is small. Although OVC-optZ also initializes inducing locations with pivoted-Cholesky as OVC, with further optimization, its performance starts to degrade starting from task 6 for the audio dataset when $M=100$, which suggests the online ELBO objective cannot guarantee optimal online update of inducing locations that preserve memory. OVFF tends to perform well at the later stage. However, during the first few tasks, it underfits the data significantly compared with other methods since its inducing variables are computed via integration over a predefined interval capturing the region of all the tasks, which is unnecessarily long and suboptimal for learning at the early stage. 

In Figure~\ref{fig:illustration_forgetting}, we compare the final predictive distributions for different methods after finishing online learning all 10 tasks of Solar Irradiance. The inducing locations $\bm{Z}$ for OSGPR tend to move to the regions where the later tasks live after online training, and the prediction of OSGPR in the initial regions without sufficient inducing points becomes close to the uninformative prior GP. In contrast, OHSGPR maintains consistent performance across both early and recent time periods.

\paragraph{Infectious disease modeling}
\label{sec:exp_covid_main}
\begin{figure}[t]
  \begin{subfigure}[t]{\columnwidth}
      \centering
      \includegraphics[width=\columnwidth]{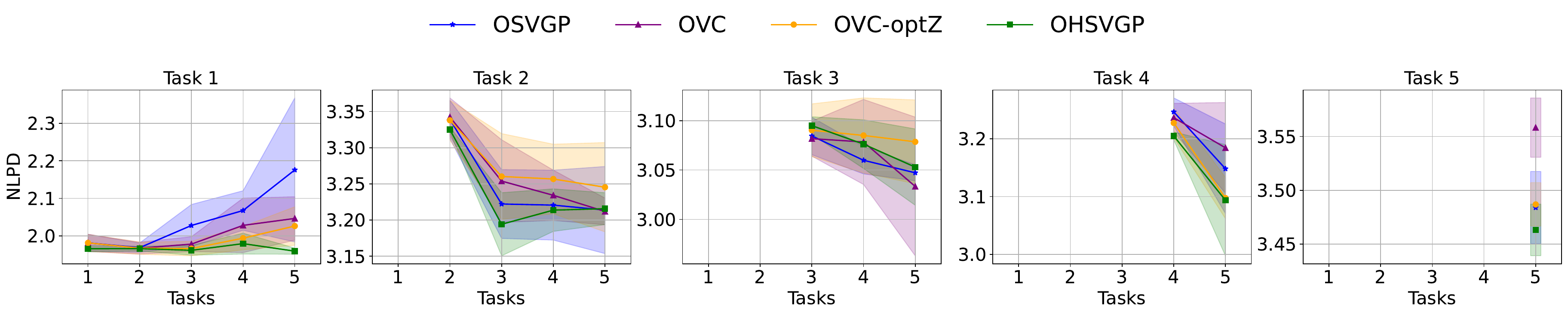}
      \caption{M=15}
  \end{subfigure}
  \hfill
  \begin{subfigure}[t]{\columnwidth}
      \centering
      \includegraphics[width=\columnwidth]{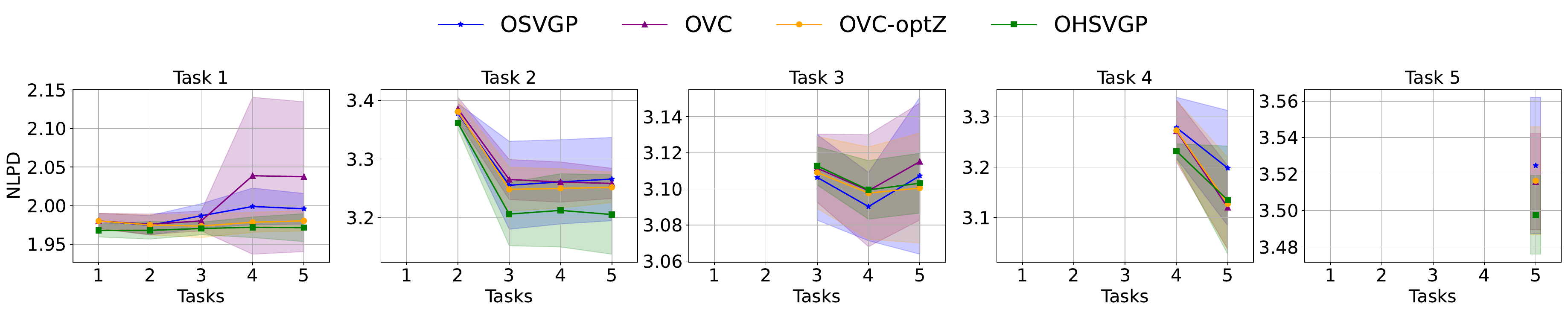}
      \caption{M=30}
  \end{subfigure}
  \caption{Test set NLPD per task after continually learning each task for all the 5 tasks on COVID dataset.}
  \label{fig:covid_nlpd_full}
\end{figure}

\begin{figure}[t]
  \begin{subfigure}[t]{\columnwidth}
      \centering
      \includegraphics[width=\columnwidth]{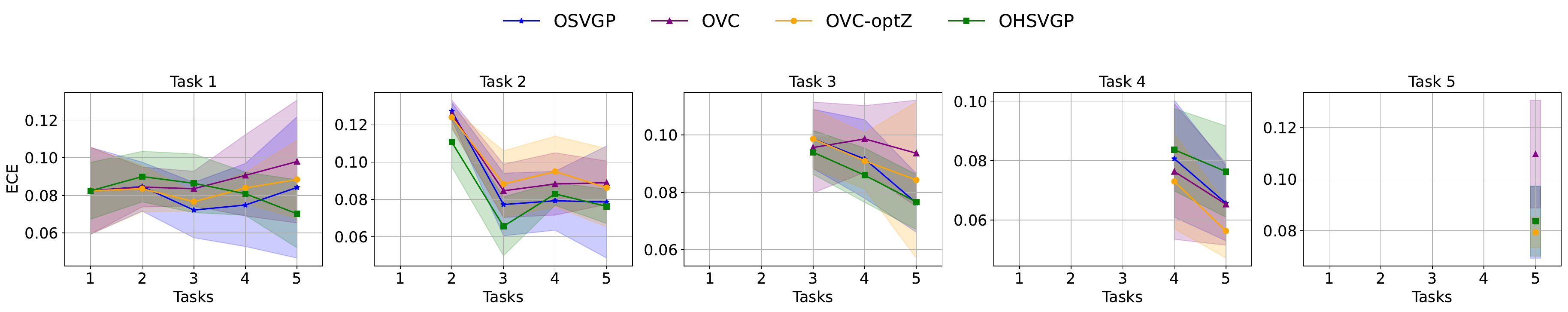}
      \caption{M=15}
  \end{subfigure}
  \hfill
  \begin{subfigure}[t]{\columnwidth}
      \centering
      \includegraphics[width=\columnwidth]{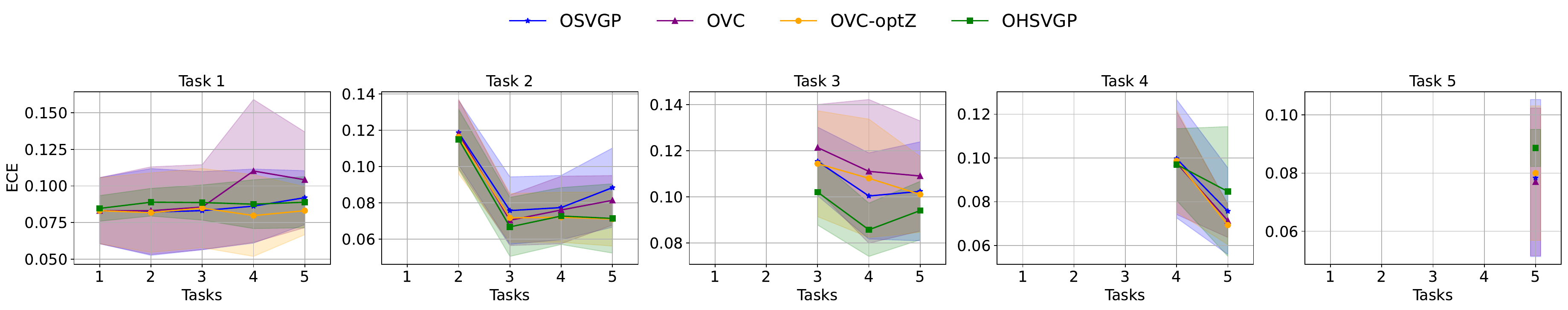}
      \caption{M=30}
  \end{subfigure}
  \caption{Test set ECE per task after continually learning each task for all the 5 tasks on COVID dataset.}
  \label{fig:covid_ece_full}
\end{figure}
We replace the Gaussian likelihood with a Negative Binomial likelihood (non‑conjugate) to capture the over‑dispersion in COVID‑19 death counts. All methods use $M\in\{15,30\}$ inducing points and are trained for 5000 iterations per task with a learning rate of 0.01. Figure~\ref{fig:covid_nlpd_full} and Figure~\ref{fig:covid_ece_full} show the change of NLPD and ECE through online learning for all five tasks, respectively (i.e., evaluation of Task $i$ after learning tasks $j=i,i+1,\cdots, 5$ for all $i$). The wide metric variance reflects the noisy nature of death‑count data as it is difficult to accurately track down COVID-19 death counts. OHSVGP achieves the best performance overall while OSVGP tend to forget Task 1 with small $M$.
\paragraph{Runtime comparison.}
\label{sec:wall_clock_run}
\begin{table}[t]
\centering
  \caption{Wall-clock accumulated runtime for learning all the tasks on a single NVIDIA RTX3090 GPU in seconds, of different models for time series prediction experiments.}
  \label{table:wall_time_time_series_prediction}
  \begin{tabular}{lccccccc}
  \toprule
   & \multicolumn{2}{c}{\textbf{Solar Irradiance}} &  \multicolumn{2}{c}{\textbf{Audio Data}} & \multicolumn{2}{c}{\textbf{COVID}}\\
      \midrule
      \multirow{2}{*}{\textbf{Method}} 
      & \multicolumn{2}{c}{ \( M \)} & \multicolumn{2}{c}{\( M \)} & \multicolumn{2}{c}{\( M \)} \\
      & \( 50 \)  & \( 150 \) 
      & \( 100 \) & \( 200 \) & \( 15 \) & \( 30 \) \\
      \midrule
      OSGPR/OSVGP  & 140 & 149  & 144 & 199 & 525 & 530\\
      OVC  & 0.450 & 0.620 & 0.558 & 0.863 & 345 & 360\\
      OVFF  & 0.327 & 0.354  & 0.295 & 0.356 & - & -\\
      OHSGPR/OHSVGP  & 0.297 & 0.394  &  0.392&  0.655 & 370 & 380\\
      \bottomrule
  \end{tabular}
\end{table}
Table~\ref{table:wall_time_time_series_prediction} shows the accumulated wall-clock runtime for different methods to  learn all the tasks. 
Unlike OSVGP and OVC-optZ, which must iteratively optimize inducing locations (for which we train e.g., 1000 iterations for time-series regression tasks), OHSVGP, OVFF (both based on interdomain inducing points), and OVC (based on one-time pivoted-Cholesky update of inducing locations for each task) bypass this cumbersome optimization. In particular, OHSVGP recurrently evolves $\bm{K}_{\bm{fu}}$ and $\bm{K}_{\bm{uu}}$ for each new task. 
For regression problems where closed-form posterior can be obtained, OHSGPR requires no training at all when considering fixed kernel hyperparameters. As a result, OHSGPR, OVC and OVFF run significantly faster, adapting to all tasks within a couple of seconds for Solar Irradiance and Audio data. For COVID data, even when free-form variational parameters of inducing variables are learned using uncollapsed ELBO, OHSVGP and OVC are still significantly faster than OSVGP since no gradient computation is required for the inducing locations.

\subsection{Continual Learning on UCI Datasets}
\label{sec:uci}
\begin{figure}[t]
  \begin{subfigure}[t]{0.95\linewidth}
      \centering
      \includegraphics[width=0.99\linewidth]{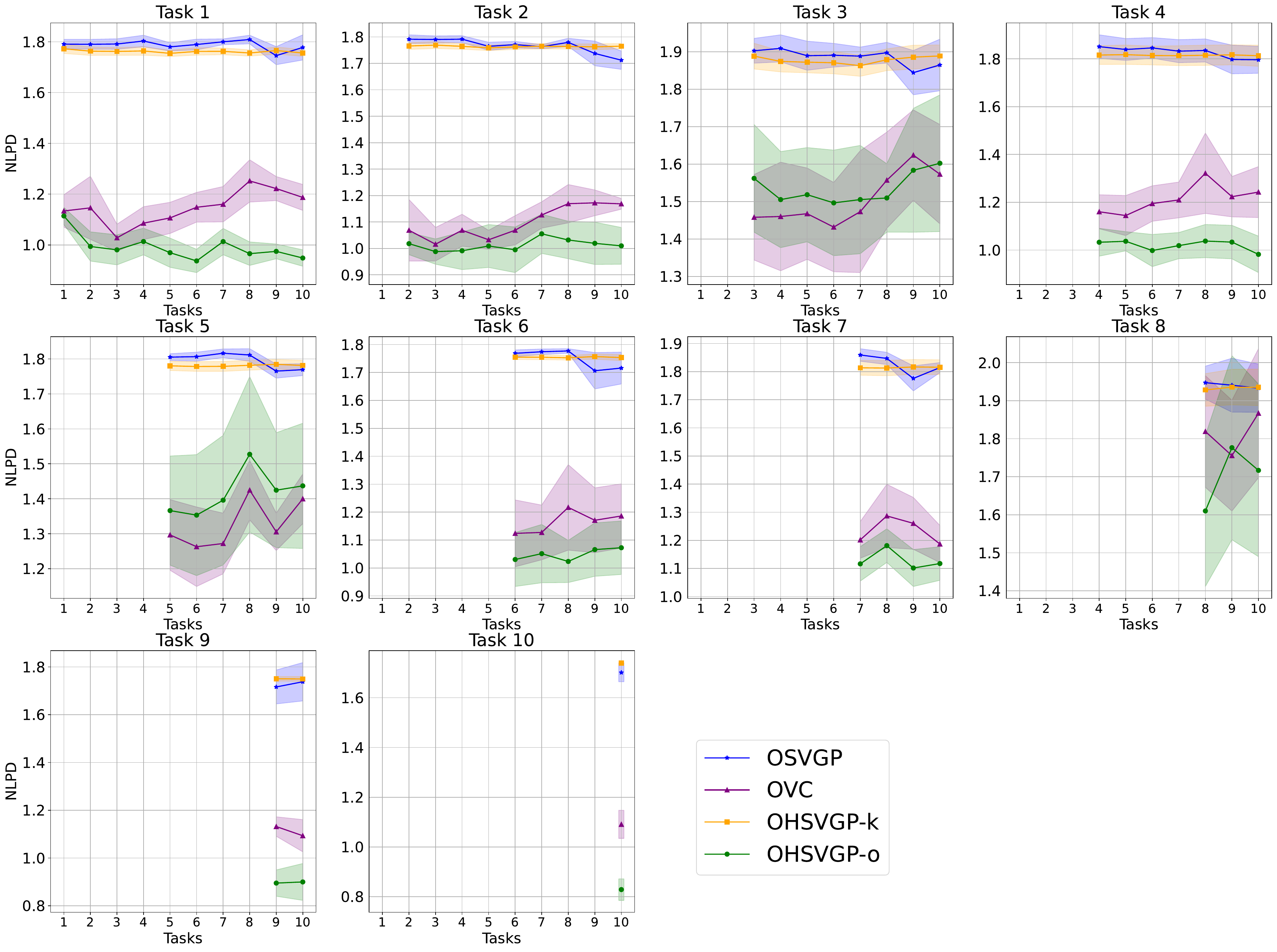}
      \caption{Skillcraft (1st dimension)}
  \end{subfigure}
  \begin{subfigure}[t]{0.95\linewidth}
      \centering
      \includegraphics[width=0.99\linewidth]{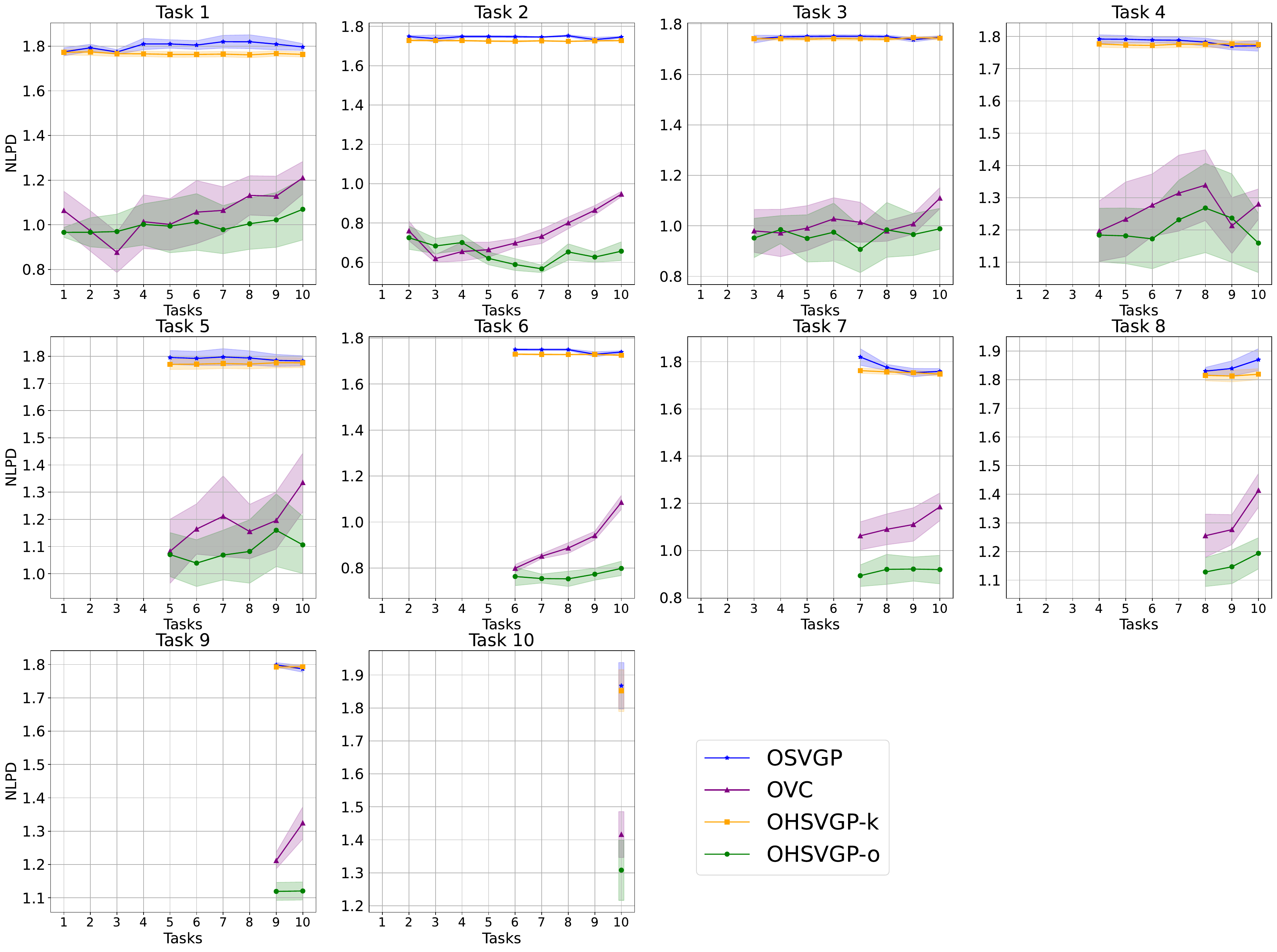}
      \caption{Skillcraft (L2)}
  \end{subfigure}
  \caption{Test set NLPD per task after continually learning each task for all the 10 tasks on UCI Skillcraft dataset.}
  \label{fig:uci_nlpd_fullres_skillcraft}
\end{figure}

\begin{figure}[t]
\centering
  \begin{subfigure}[t]{0.95\linewidth}
      \centering
      \includegraphics[width=0.99\linewidth]{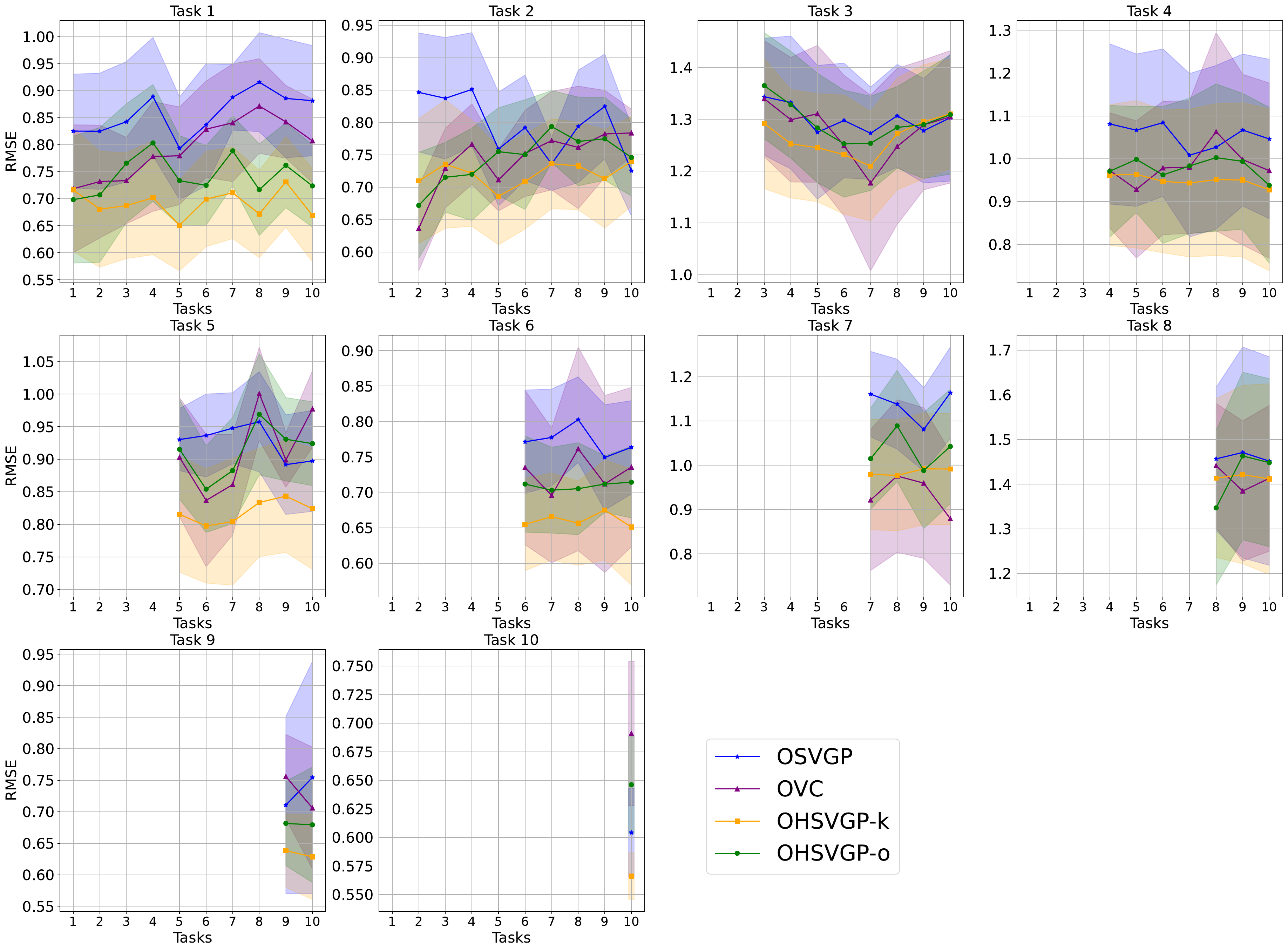}
      \caption{Skillcraft (1st dimension)}
  \end{subfigure}
  \hfill
  \begin{subfigure}[t]{0.95\linewidth}
      \centering
      \includegraphics[width=0.99\linewidth]{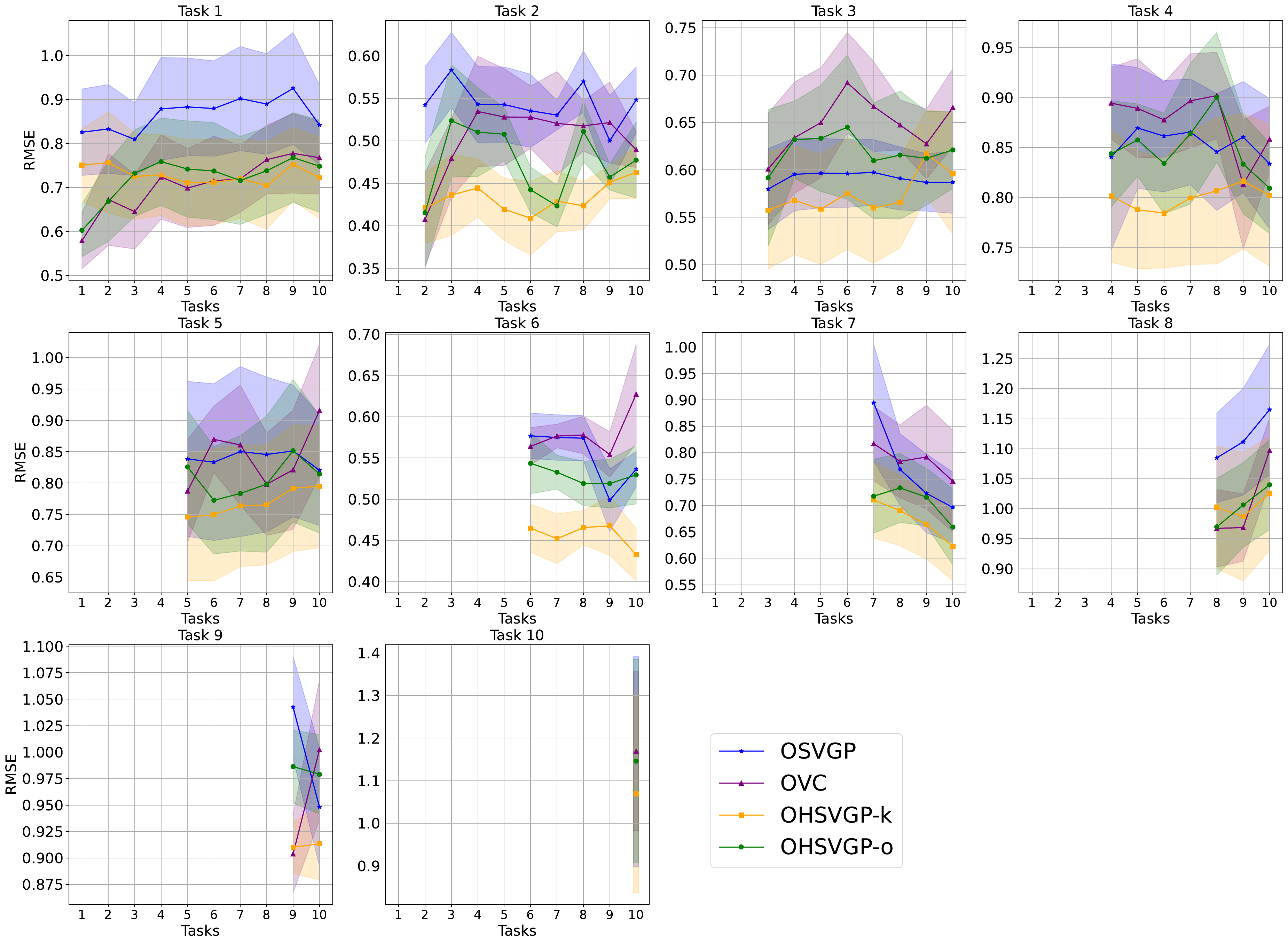}
      \caption{Skillcraft (L2)}
  \end{subfigure}
   \caption{Test set RMSE per task after continually learning each task for all the 10 tasks on UCI Skillcraft dataset.}
  \label{fig:uci_rmse_fullres_skillcraft}
\end{figure}

\begin{figure}[t]
  \begin{subfigure}[t]{0.95\linewidth}
      \centering
      \includegraphics[width=0.99\linewidth]{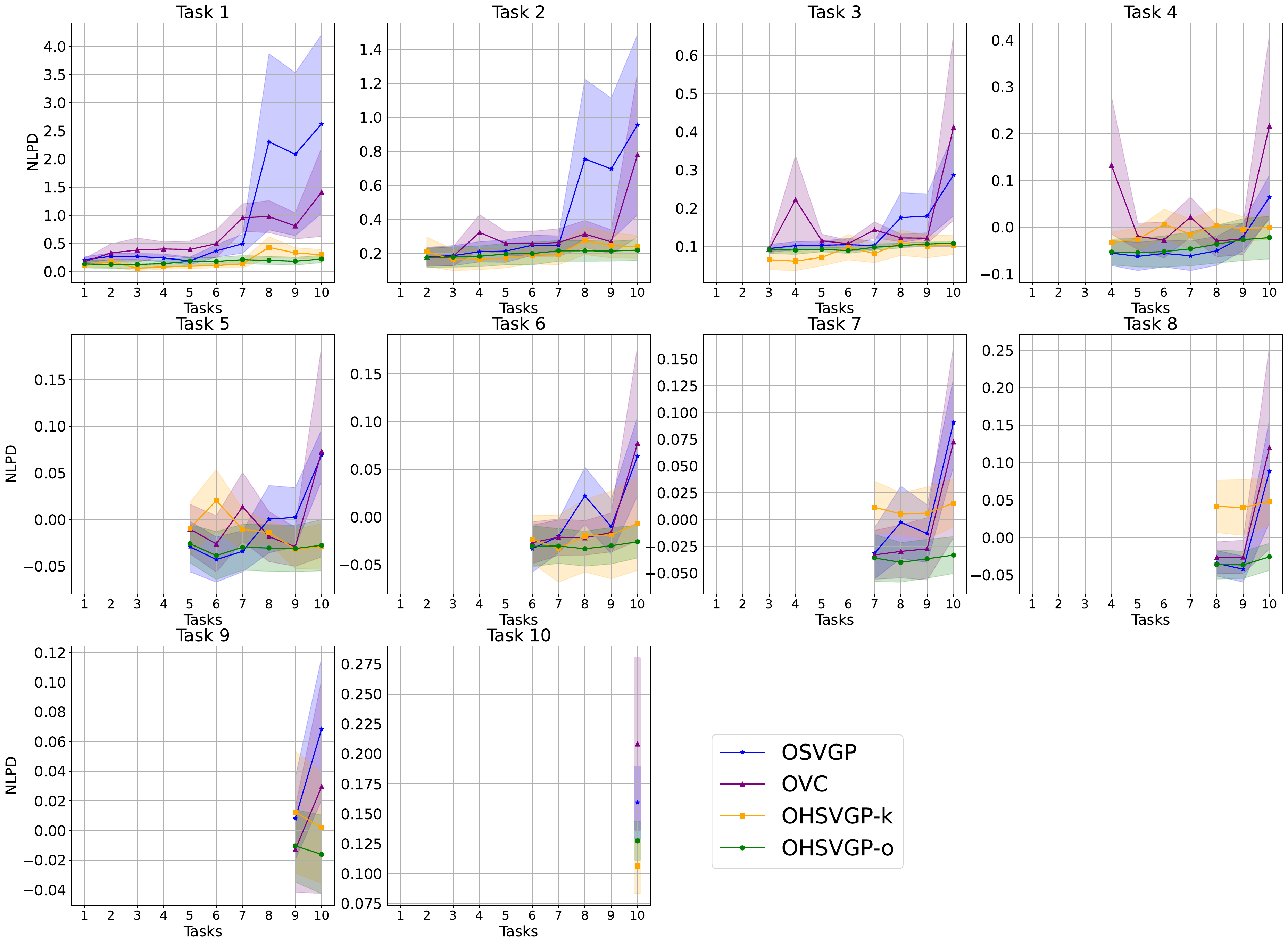}
      \caption{Powerplant (1st dimension)}
      \label{fig:powerplant_dim0_nlpd}
  \end{subfigure}
  \begin{subfigure}[t]{0.95\linewidth}
      \centering
      \includegraphics[width=0.99\linewidth]{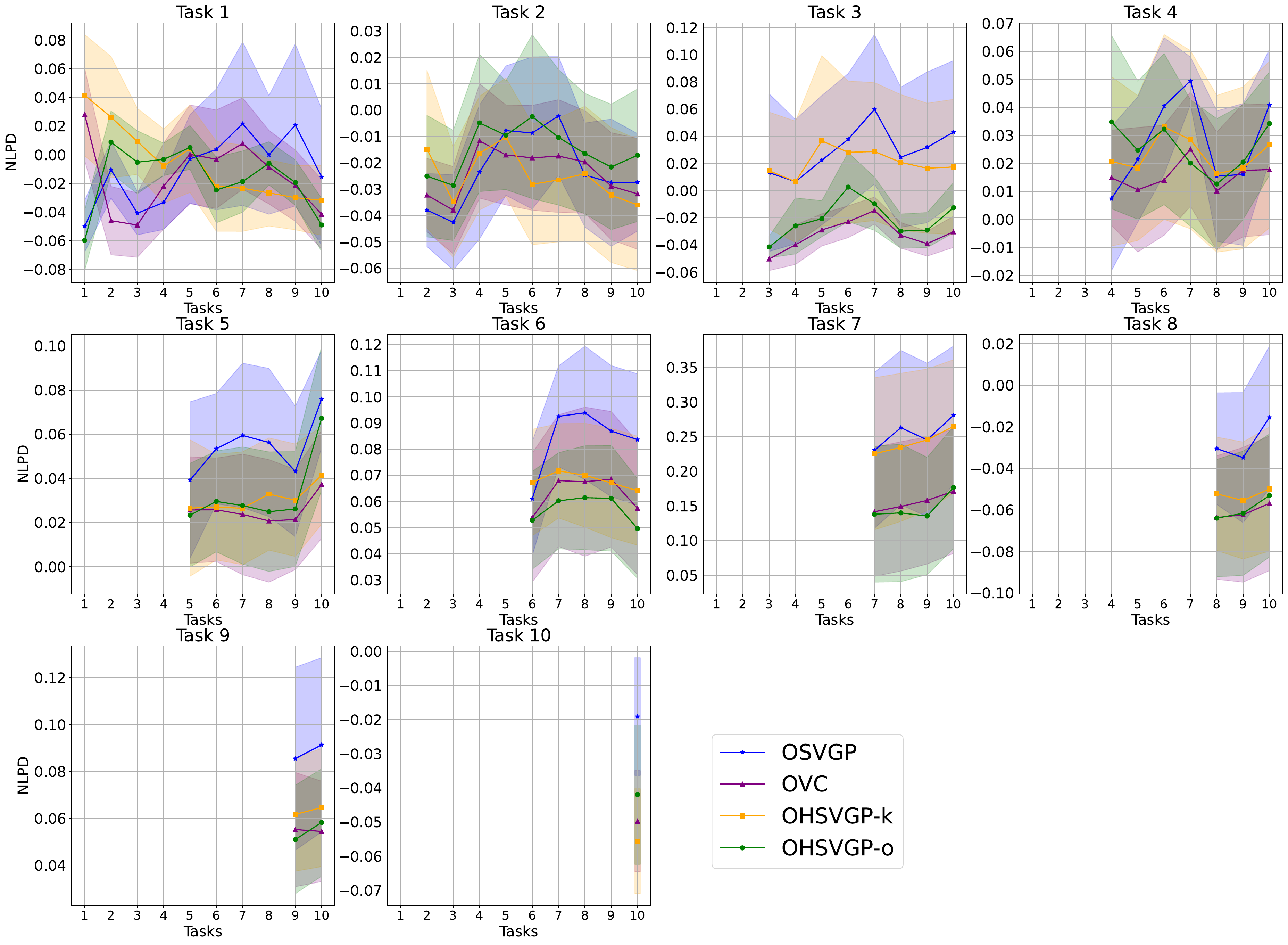}
      \caption{Powerplant (L2)}
  \end{subfigure}
  \caption{Test set NLPD per task after continually learning each task for all the 10 tasks on UCI Powerplant dataset.}
  \label{fig:uci_nlpd_fullres_powerplant}
\end{figure}

\begin{figure}[t]
\centering
  \begin{subfigure}[t]{0.95\linewidth}
      \centering
      \includegraphics[width=0.99\linewidth]{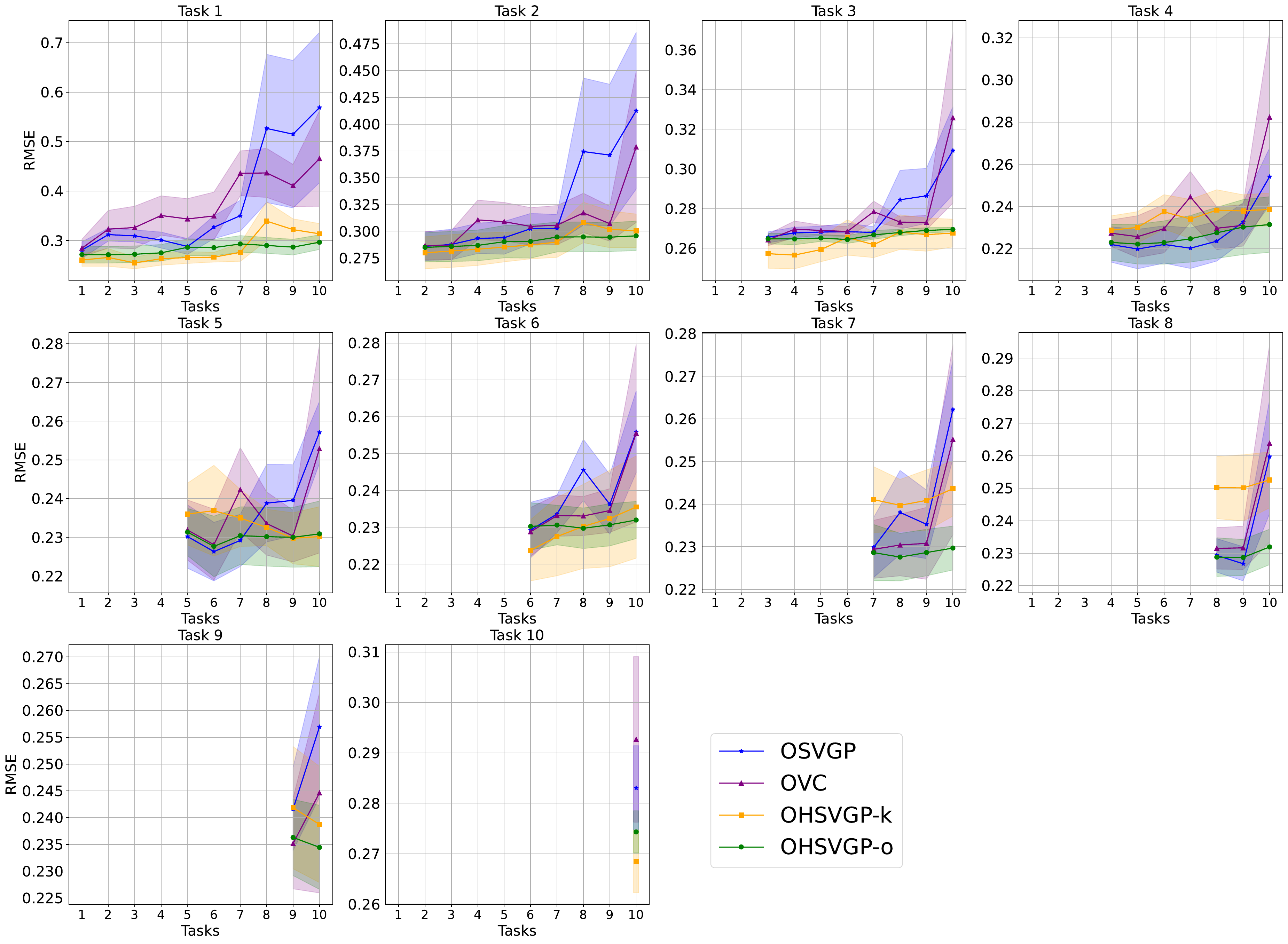}
      \caption{Powerplant (1st dimension)}
  \end{subfigure}
  \hfill
  \begin{subfigure}[t]{0.95\linewidth}
      \centering
      \includegraphics[width=0.99\linewidth]{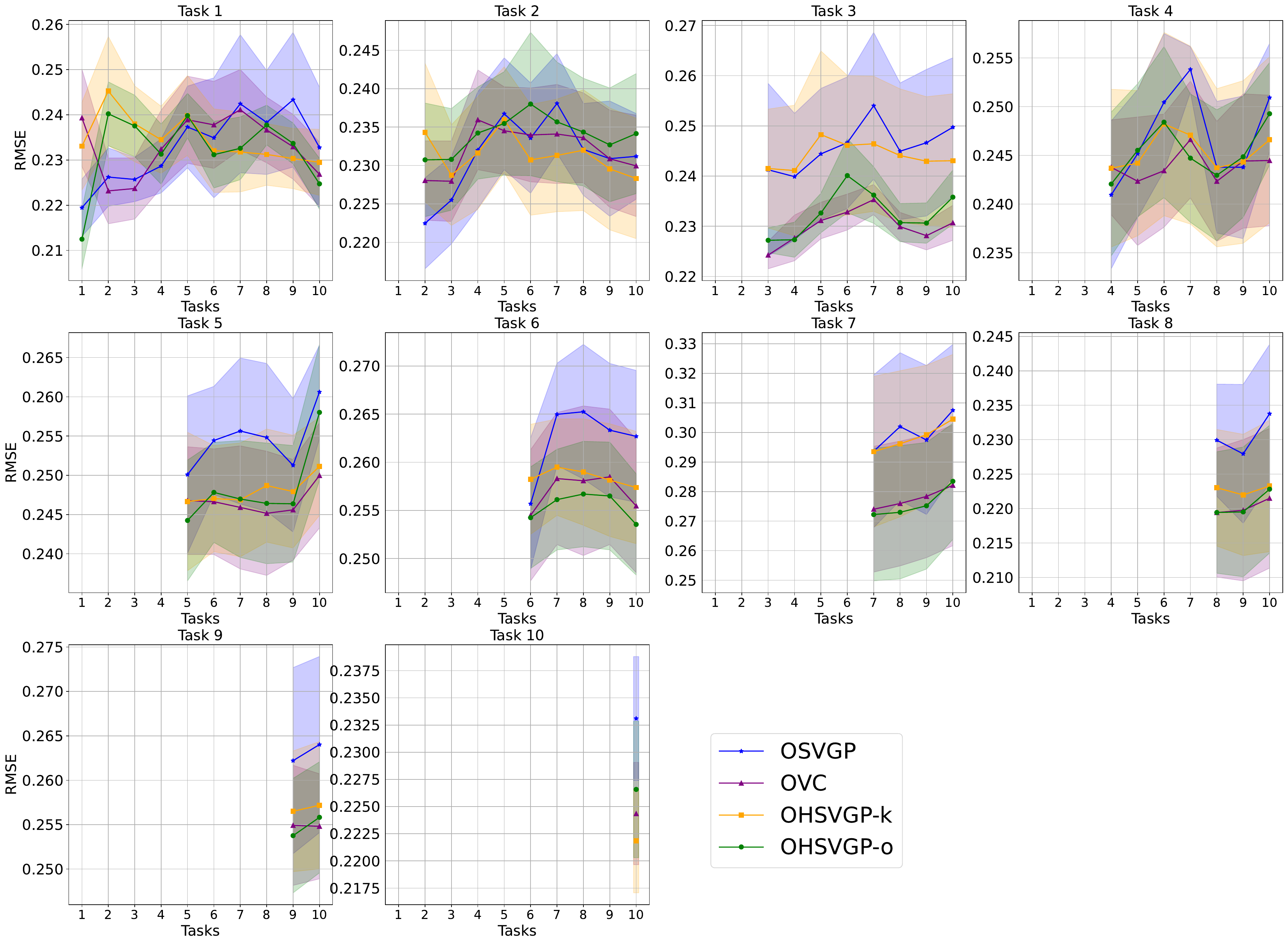}
      \caption{Powerplant (L2)}
  \end{subfigure}
  \caption{Test set RMSE per task after continually learning each task for all the 10 tasks on UCI Powerplant dataset.}
  \label{fig:uci_rmse_fullres_powerplant}
\end{figure}

We use 256 inducing variables for all methods, and for each task, we train each method for 2000 iterations with a learning rate of 0.005. We only consider OVC here since initial trials show OVC-optZ give worse results on these two datasets. As described in Section~\ref{sec:multi_dim}, within each task, OHSVGP requires sorting the data points to compute prior covariance matrices via recurrence. We consider two sorting criteria. The first one, which we call OHSVGP-o, uses the oracle order compatible with how the tasks are created (e.g., sort with L2-distance to the origin if the tasks are initially splitted based on it). In real-world problems, we typically do not have the information on how the distribution shifts from task to task. Hence, we also consider OHSVGP-k (same as OHSVGP-k-max in Section~\ref{sec:multi_dim}), which uses a heuristic sorting method based on kernel similarity: we select the $i$-th point in task $j$ to be $\bm{x}_i^{(j)} = \argmax_{\bm{x} \in \bm{X}^{(j)}}k(\bm{x}, \bm{x}_{i-1}^{(j)})$ for $i>1$, and the first point in first task is set to be $\bm{x}_1^{(1)} = \argmax_{\bm{x} \in \bm{X}^{(1)}}k(\bm{x}, \bm{0})$. Figure~\ref{fig:uci_nlpd_fullres_skillcraft} to Figure~\ref{fig:uci_rmse_fullres_powerplant} compares the two variants of OHSVGP with OSVGP and OVC. We report evaluation (in NLPD and RMSE) of Task $i$ after learning tasks $j=i,i+1,\cdots, 10$ for all $i$. Overall, OSVGP achieves the worst performance and is again prone to forgetting the older tasks, especially in Figure~\ref{fig:powerplant_dim0_nlpd}. OVC performs decently for Skillcraft but it also demonstrates catastrophic forgetting in Figure~\ref{fig:powerplant_dim0_nlpd}. While OHSVGP-k achieves similar performance as OSVGP on Skillcraft, OHSVGP-o consistently outperforms the other methods across all 4 scenarios, suggesting the importance of a sensible sorting method when applying OHSVGP for continual learning.

\clearpage
\subsection{Continual Learning for High Dimensional Time Series Prediction}
\label{sec:gpvae}

\begin{figure}[t]
  \begin{subfigure}[t]{0.95\linewidth}
      \centering
      \includegraphics[width=0.99\linewidth]{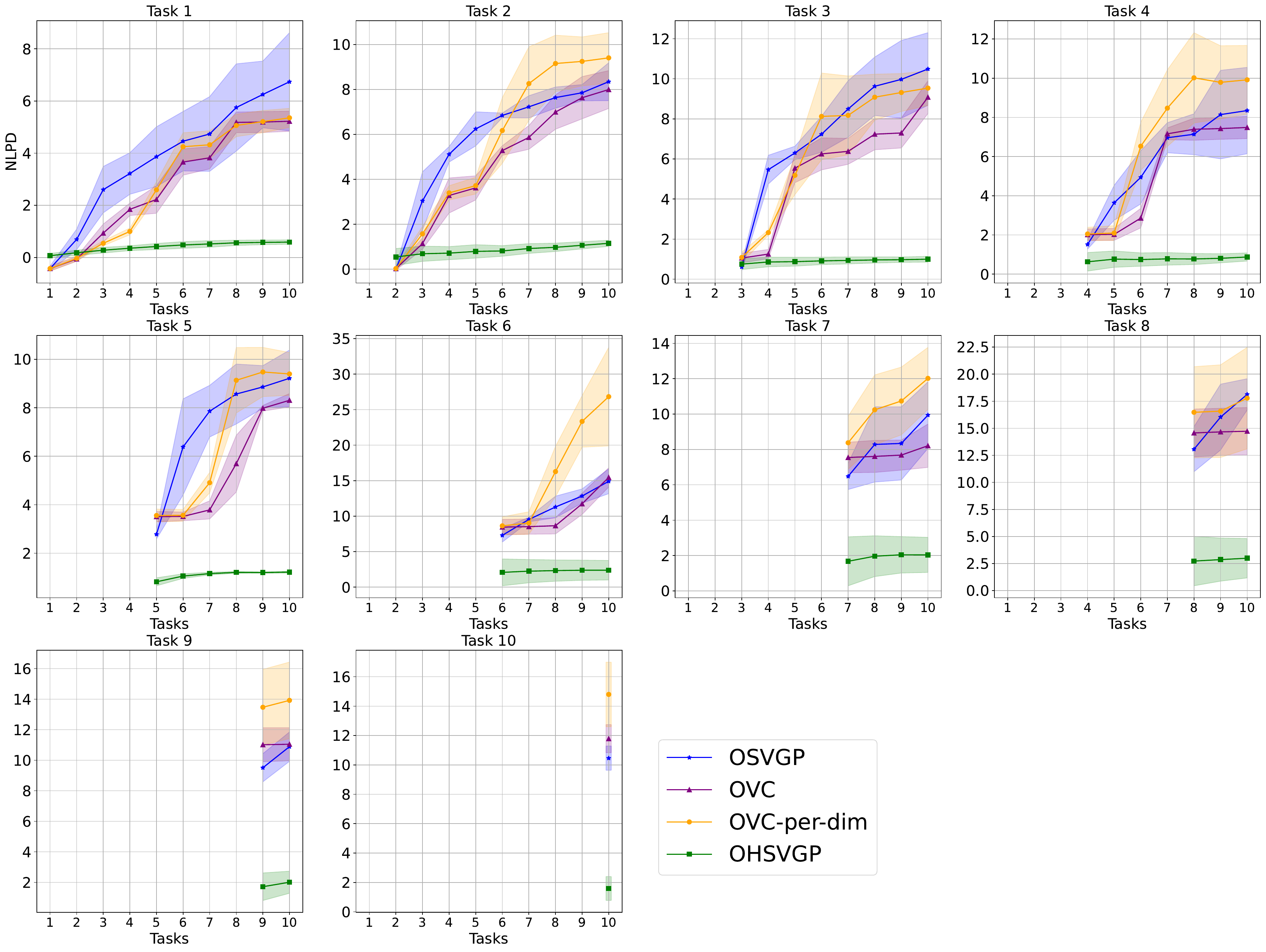}
      \caption{M=50}
  \end{subfigure}
  \hfill
  \begin{subfigure}[t]{0.95\linewidth}
      \centering
      \includegraphics[width=0.99\linewidth]{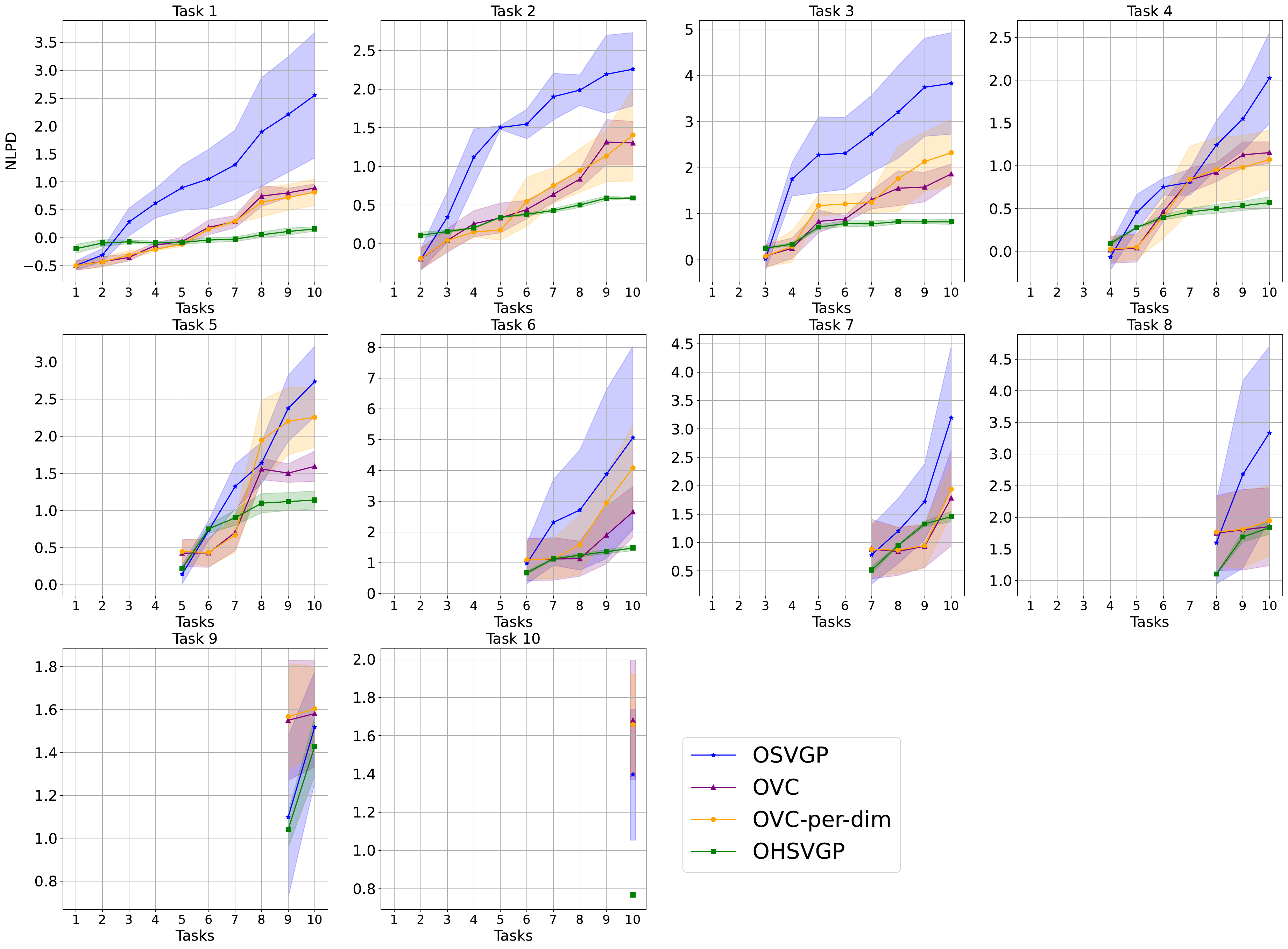}
      \caption{M=100}
  \end{subfigure}
  \caption{Test set NLPD per task after continually learning each task for all the 10 tasks on ERA5 dataset.}
   \label{fig:era_nlpd_fullres}
\end{figure}

\begin{figure}[t]
  \begin{subfigure}[t]{0.95\linewidth}
      \centering
      \includegraphics[width=0.99\linewidth]{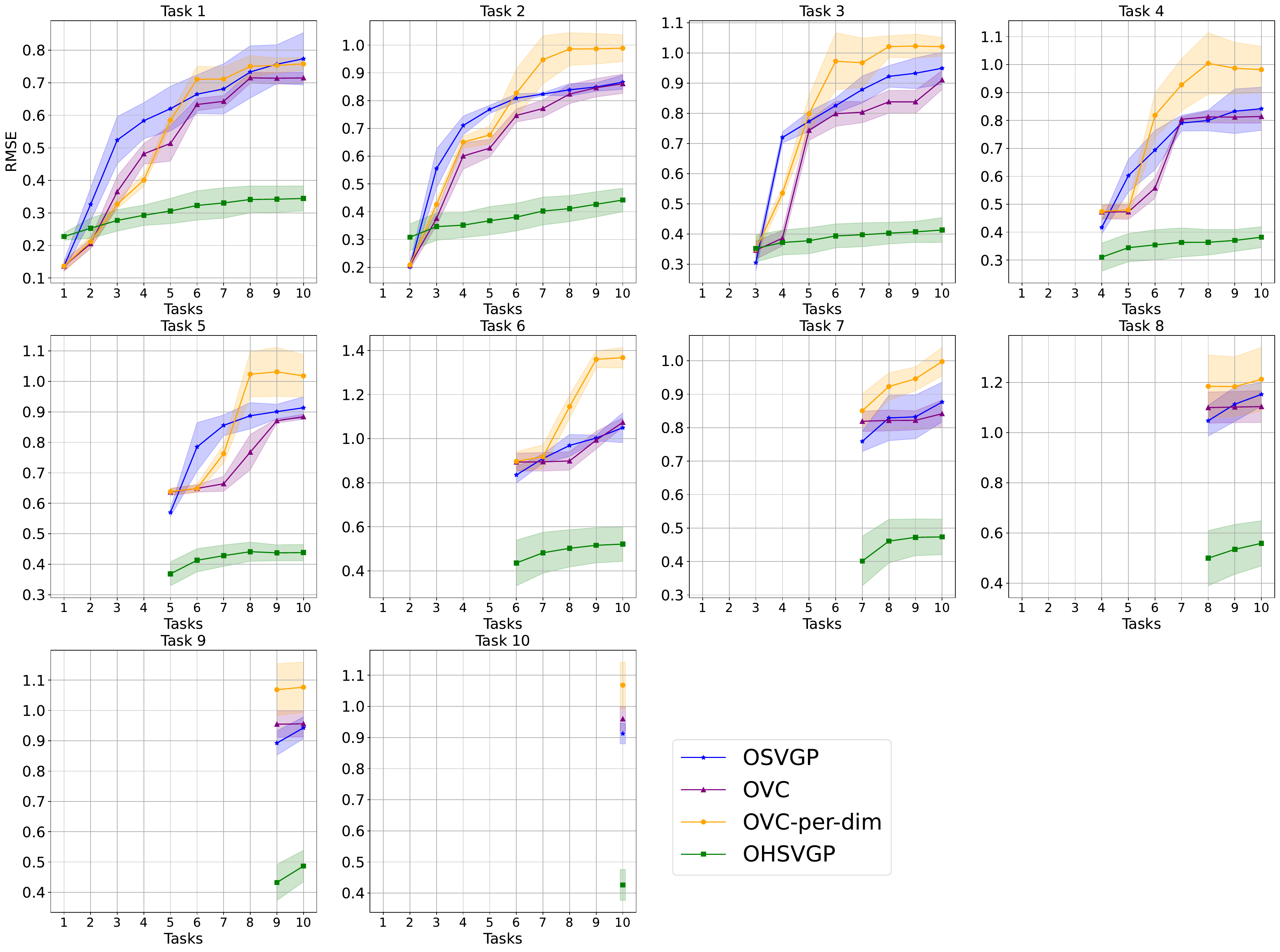}
      \caption{M=50}
  \end{subfigure}
  \hfill
  \begin{subfigure}[t]{0.95\linewidth}
      \centering
      \includegraphics[width=0.99\linewidth]{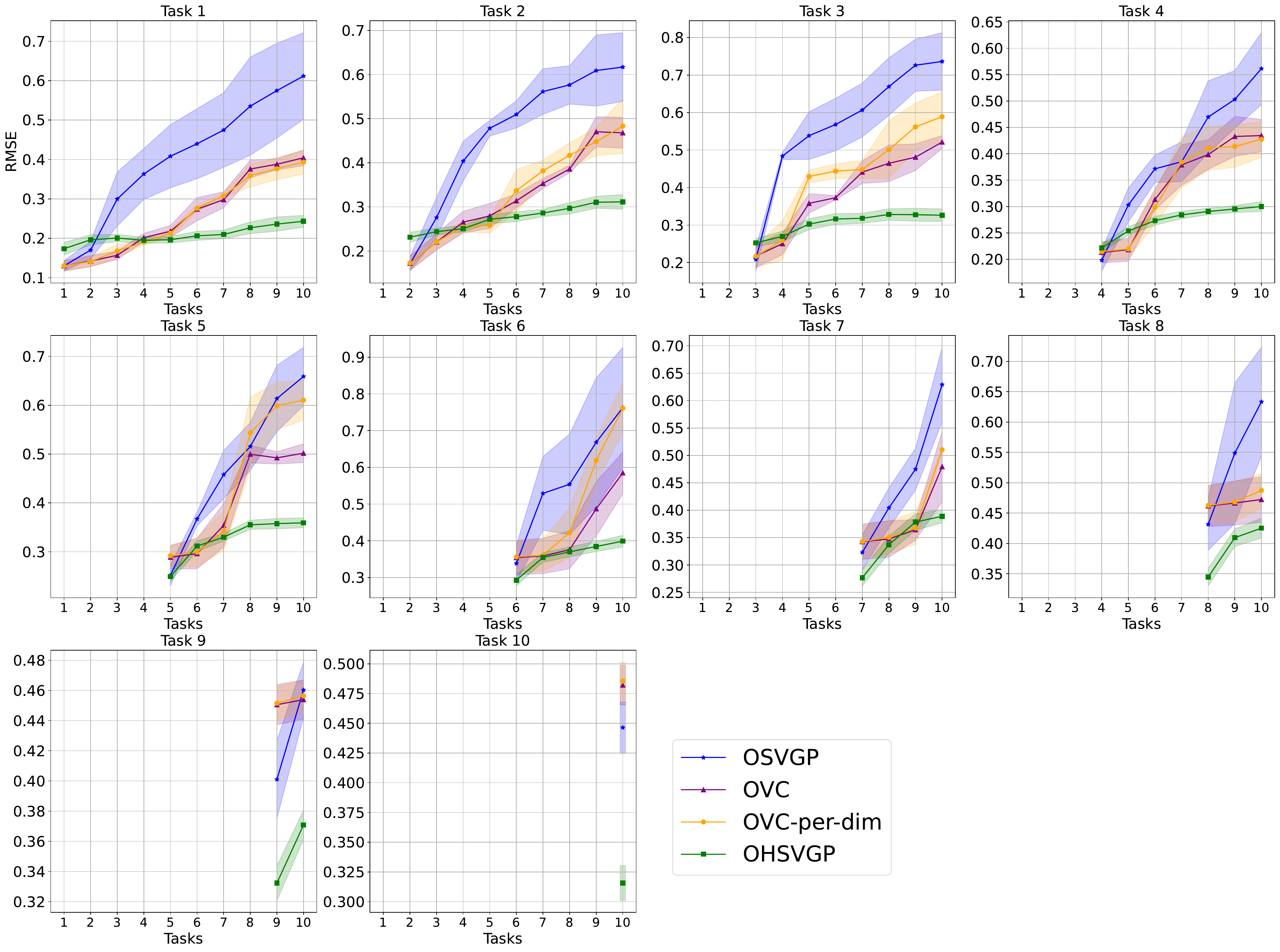}
      \caption{M=100}
  \end{subfigure}
  \caption{Test set RMSE per task after continually learning each task for all the 10 tasks on ERA5 dataset.}
  \label{fig:era_rmse_fullres}
\end{figure}

All models share a two‑layer MLP encoder–decoder, a 20‑dimensional latent space, and a multi‑output GP with independent components; we use \(M\in\{50,100\}\) and train each task for 20 epochs with learning rate 0.005 on a single NVIDIA A6000 GPU. The continual learning in SVGPVAE is achieved by further imposing Elastic Weight Consolidation (EWC; \citep{kirkpatrick2017overcoming}) loss on the encoder and decoder, which yields the vanilla baseline, Online SVGPVAE (OSVGP). Since EWC alone leaves inducing locations non-regularized, a principled online placement rule for the inducing locations will improve the model. Thus, we further consider OVC-SVGPVAE (OVC) which adjusts inducing locations online via Pivoted-Cholesky, and OVC-SVGPVAE per dimension (OVC-per-dim), which makes OVC more flexible by allocating a separate set of $M$ inducing points to every latent dimension. Our method, Online HiPPO SVGPVAE (OHSVGP), replaces standard inducing points in SVGPVAE with HiPPO inducing variables and updates them online via recurrence. Figure~\ref{fig:era_nlpd_fullres} and Figure~\ref{fig:era_rmse_fullres} plot the change of NLPD and RMSE during continual learning for all tasks, respectively. The performance of OHSVGP remains stable throughout, while the other methods all demonstrate obvious catastrophic forgetting shown by the large gaps between performances after learning current task $i$ and after learning final task 10. Two factors plausibly explain the gap: first, standard inducing points cannot adequately cover the long time axis, whereas OHSVGP ties its inducing variables to basis functions rather than time locations; second, the added encoder–decoder complexity makes optimization harder for models that must reuse a limited inducing set.  Increasing $M$ narrows the gap but scales at $\mathcal{O}(M^3)$ computational and $\mathcal{O}(M^2)$ memory cost respectively, underscoring OHSVGP’s superior efficiency.
\clearpage

\section{Conclusion}
We introduce OHSVGP, a novel online Gaussian process model that leverages the HiPPO framework for robust long-range memory in online/continual learning. By interpreting HiPPO's time-varying orthogonal projections as adaptive interdomain GP basis functions, we leverage SSM for improved online GP. This connection allows OHSVGP to harness HiPPO's efficient ODE-based recurrent updates while preserving GP-based uncertainty-aware prediction. Empirical results on a suite of online and continual learning tasks show that OHSVGP outperforms existing online sparse GP methods, especially in scenarios requiring long-term memory. Moreover, its recurrence-based covariance updates yield far lower computational overhead than OSVGP's sequential inducing point optimization. This efficient streaming capability and preservation of historical information make OHSVGP well-suited for real-world applications demanding both speed and accuracy.

\chapter{Your Image is Secretly the Last Frame of a Pseudo Video}
\label{cha:pseudovid}


\begin{tcolorbox}           [enhanced,colback=gray!5!white,colframe=gray!75!black,colbacktitle=red!80!black,fonttitle=\bfseries]
  This chapter is based on \citet{chen2025your}:
  \begin{itemize}
  \item  \underline{Wenlong Chen}$^\ast$, Wenlin Chen$^\ast$, Lapo Rastrelli, and Yingzhen Li. \textbf{Your Image is Secretly the Last Frame of a Pseudo Video}. In \textit{Deep Generative Model in Machine Learning: Theory, Principle and Efficacy (DeLTa) Workshop at ICLR}, 2025.
  \end{itemize}
  The main idea was developed by me while the last author provided useful suggestions. Both co-first authors ($\ast$) contributed to code implementation, experimentation, and paper writing under the supervision of the last author. The third author helped with the experiments related to video diffusion models. I also helped orchestrate other authors’ contributions.
\end{tcolorbox}
In the previous two chapters, we combined techniques from deep sequence modeling and probabilistic methods to build powerful predictive models. This chapter explores sequence modeling in the orthogonal domain of deep generative models. Our key observation is that one of the main differences between two classes of sequential generative models, diffusion models and hierarchical variational encoders, is that diffusion models incorporate inductive bias as direct supervision signals for their sequence of intermediate latent variables. Inspired by this observation and the empirical success of diffusion models, we propose to repeatedly apply data augmentation to images to create pseudo-video sequences and explore the possibilities of improving other image-generative models by jointly modeling the distribution of the sequence consisting of the original image and its corresponding pseudo video containing self-supervised information.

\section{Introduction}
Sequential models form a popular framework for generating images \citep{gulrajani2016pixelvae,sonderby2016ladder,ho2020denoising,liu2022flow,albergo2023stochastic,lipman2022flow,shi2023diffusion,wang2024rectified}. Instead of generating images from noise in one shot, which can be challenging, these models gradually transform noise into images using multiple intermediate steps. Among them, diffusion models \citep{sohl2015deep,ho2020denoising,song2020score} and their variants \citep{kingma2021variational,nichol2021improved,song2021denoising,rissanen2022generative, bansal2023cold, hoogeboom2023blurring} have shown impressive ability to generate high quality images in recent years. 

While diffusion models can be viewed as a special case of a traditional sequential generative model, i.e., hierarchical variational autoencoders (HVAEs) \citep{sonderby2016ladder,maaloe2019biva,vahdat2020nvae}, they tend to outperform standard HVAEs significantly in practice.
The major differences between diffusion models and standard HVAEs are two-fold. First, diffusion models tend to have much more intermediate states, which may help improve the generation quality \citep{huang2021variational}. Second, diffusion models incorporate exact self-supervised information for their intermediate states: they are supposed to match the corrupted (e.g., noisy or blurred) versions of the original target image at different corruption levels. These additional information helps regularize training and guide generation in diffusion models. On the other hand, the intermediate states in standard HVAEs are unobserved and one does not have explicit control of them. Consequently, there may be many different distributions over the intermediate states that are capable of generating images (i.e., the issue of unidentifiability \citep{locatello2019challenging,khemakhem2020variational}). The lack of identifiability of the intermediate states may pose challenges to the optimization during training since it suggests a huge hypothesis space with many sub-optimal solutions.

In this paper, we hypothesize that incorporating such self-supervised information into flexible generative models, as in diffusion models, may be one of the key reasons that they achieve good generation performance. Based on this assumption, we explore the possibility of improving other types of image generative models by extending them to video generative models and artificially injecting self-supervised information in the form of pseudo videos whose frames are created by applying data augmentation to the original images. These pseudo videos are then used to train our video generative models. After that, we compare the generation quality of the last frame of the pseudo video (corresponding to the original image) generated by the video generative model with that of the images generated by the original image generative model. Empirically, we observe improved image generation quality via pseudo video generation compared to the images directly generated by the original image generative model. Theoretically, we provide intuitions on why designing better pseudo videos with data augmentation beyond first-order Markov chains can be helpful.

The contributions of our paper are summarized below.
\begin{itemize}
    \item (Section~\ref{sec:motivation}) Our key insight is that pseudo videos created by corrupting the original target image may provide useful self-supervised information for training generative models. This is demonstrated by a comparison between diffusion models and standard HVAEs as a motivating example.
    \item (Sections~\ref{sec:phenaki} and~\ref{sec:videodiff}) We attempt to improve two popular generative model frameworks, VQVAE \citep{van2017neural} and Improved DDPM \citep{nichol2021improved}, by extending them to their video generative model counterparts and training them on pseudo videos. Empirically, we show that this procedure improves the image generation quality with pseudo videos of just a few frames. In general, our proposed framework provides a new way of scaling up image generative models with their video generative model counterparts for potential performance gain.
    \item (Section~\ref{sec:videodiff}) Theoretically, we analyze the potential issue of certain pseudo videos, including those in the form of first-order Markov chains, in autoregressive video generation frameworks. Based on our theoretical results, we propose a simple and effective approach which avoids the potential issues by constructing higher-order Markov pseudo videos.
\end{itemize}

\section{Motivation}\label{sec:motivation}
Let $\x\in\mathbb{R}^n$ (also denoted as $\z_0$) be an observed variable of interest. The task of generative modeling aims to fit a parametric model $p_{\theta}(\x)$ to estimate the data distribution $p(\x)$ using samples from $p(\x)$. We have already reviewed hierarchical variational autoencoders (HVAEs; \citep{sonderby2016ladder,maaloe2019biva,vahdat2020nvae}) and diffusion models  \citep{sohl2015deep,ho2020denoising,song2020score} in Section~\ref{sec:hvae_back} and Section~\ref{sec:diffusion_back}, respectively, and we refer the readers to these two sections if a refresh of the related background and notation is needed. Our key observation is obtained from the difference between the training objectives for these two classes of deep sequential generative models:

\begin{equation}
\begin{aligned}
 \mathcal{L}_{ELBO}^\text{HVAE}(\bm{x},\mparam, q_{\vparam})=\mathbb{E}_{q_{\vparam}(\z_{1:T}|\z_0)}\left[\log \frac{p(\z_T) \prod_{t=1}^T p_{\mparam}(\z_{t-1}|\z_t)}{q_{\vparam}(\z_{1:T}|\z_0)}\right],
\end{aligned}
\label{eq: have_loss}
\end{equation}

\begin{equation}
\begin{split}
    \mathcal{L}_{ELBO}^{\text{Diffusion}}(\x, \mparam)
    &=\mathbb{E}_{q(\z_{1:T}|\z_0)}\left[\log p_{\mparam}(\z_0|\z_1) - \sum_{t=1}^T  \text{KL}(q(\z_{t-1}|\z_t,\z_0)||p_{\mparam}(\z_t|\z_{t-1})) \right]\\
    &=\mathbb{E}_{q(\z_{1:T}|\z_0)}\left[\log p_{\mparam}(\z_0|\z_1) -\sum_{t=1}^T  \underbrace{\frac{\lVert \mu_{\mparam}(\z_t, t) - \tilde{\mu}(\z_t,\z_0) \rVert^2}{2\sigma_t^2}}_{\text{mean matching}} \right].
\end{split}
\label{eq: diffusion_loss}
\end{equation}

\subsection{Diffusion Model vs Hierarchical VAE}
\begin{figure}[t]
    \centering
    \includegraphics[width=0.6\linewidth]{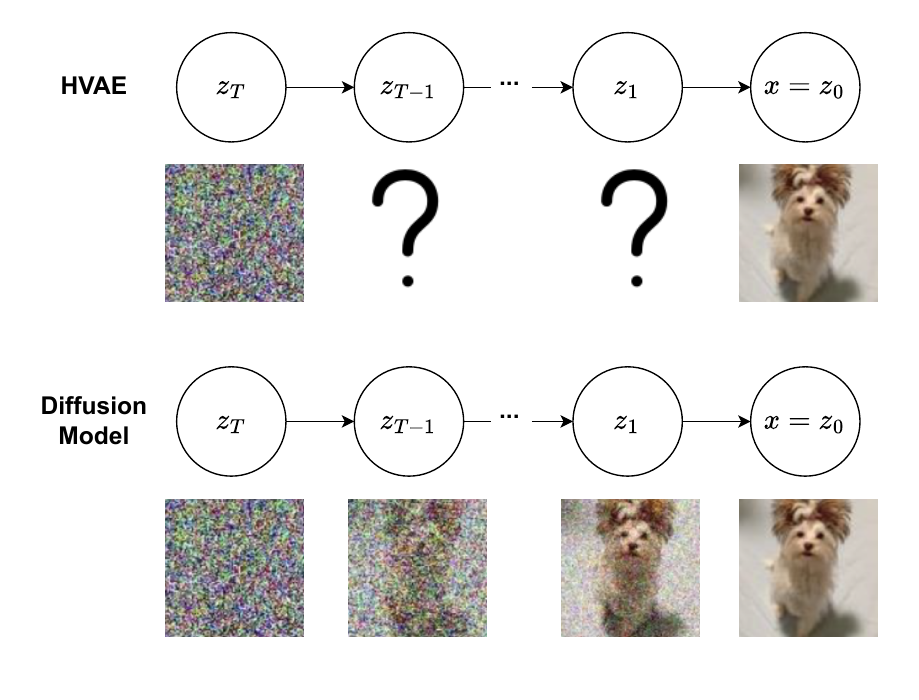}
    \caption{Diffusion model vs HVAE: compared to standard HVAEs, diffusion models incorporate inductive bias for its intermediate latent states with self-supervised signals.}
    \label{fig:diffusion_vs_hvae}
\end{figure}
Compared to the objective for training standard HVAEs (Eq.~\ref{eq: have_loss}), one can see that the objective for training diffusion models (Eq.~\ref{eq: diffusion_loss}) incorporates direct control for the intermediate states. Due to the fixed pre-defined inference model (Eq.~\ref{eq:diffusion_inference}), the objective in Eq.~\ref{eq: diffusion_loss} is simplified. In particular, its second term (mean matching) suggests that at each intermediate step $t$, the mean function $\mu_{\theta}(\z_t,t)$ in the generation model is trained by matching a noisy version $\tilde{\mu}(\z_t, \z_0)$ of the original target image $\z_0:=\x$. In contrast, the objective for standard HVAEs has no such information to impose any control over their intermediate states since their inference models are parameterized by flexible neural networks and keep being updated along with the generation model by end-to-end training. Figure~\ref{fig:diffusion_vs_hvae} illustrates this difference between diffusion models and standard HVAEs. Without the aid of the self-supervised information, the intermediate states in standard HVAEs are very flexible, which implies that they are unidentifiable in the sense that there could be many plausible distributions over them that can generate images (i.e., many sub-optimal solutions), which makes the optimization harder as $T$ becomes larger. In contrast, diffusion models may benefit from the self-supervised information (i.e., noisy images) for their intermediate states, for which optimization can be less challenging even with large $T$, since this inductive bias pins down one specific route of generation, which eliminates other solutions that are inconsistent with this inductive bias.

\begin{figure}[t]
    \centering
     \begin{subfigure}[t]{0.45\textwidth}
      \centering
      \includegraphics[width=\textwidth]{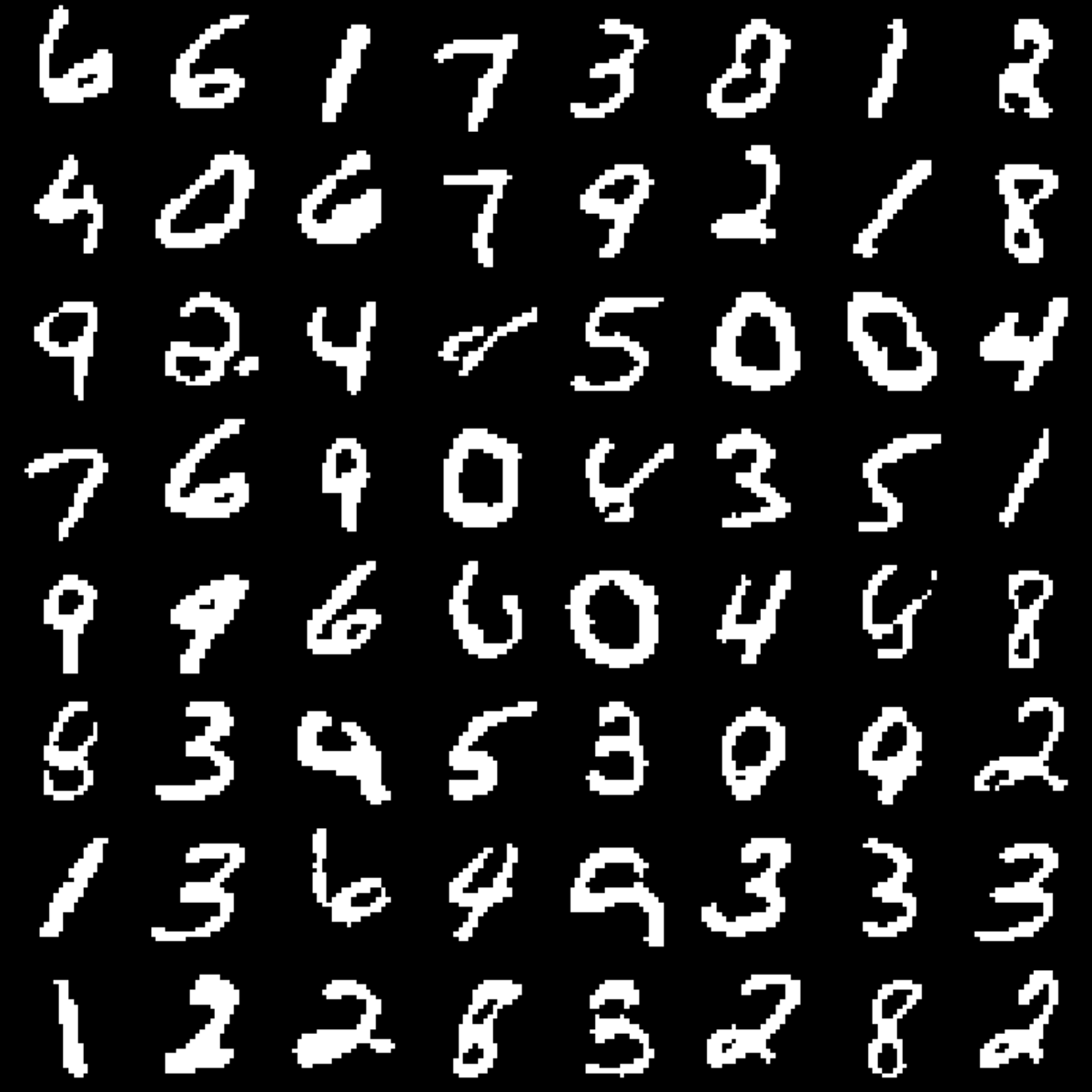}
      \caption{HVAE with heat equation encoder}
    \end{subfigure}
    \hfill
    \begin{subfigure}[t]{0.45\textwidth}
      \centering
      \includegraphics[width=\textwidth]{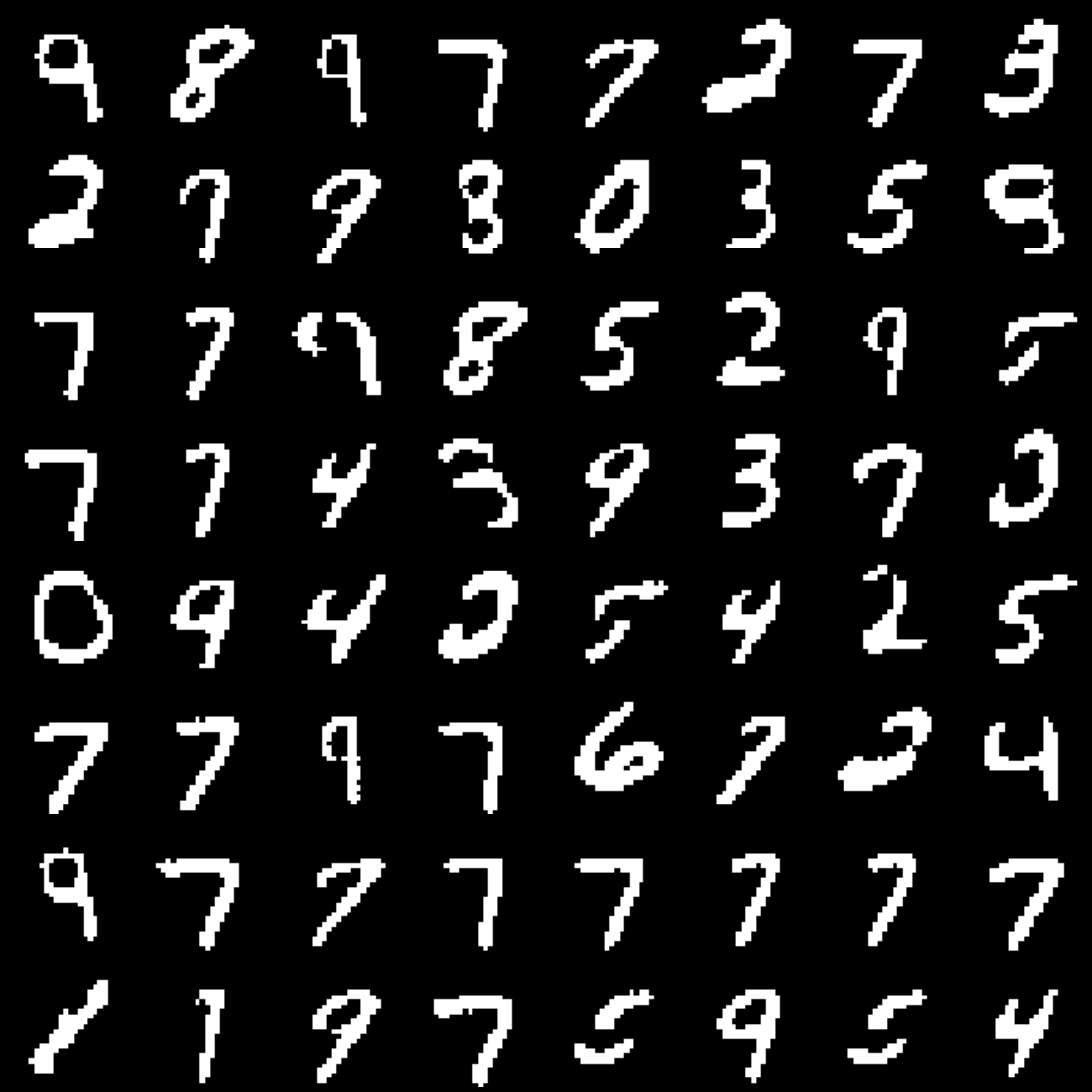}
      \caption{Standard HVAE with learnable encoder}
    \end{subfigure}
    \caption{Generated digits from HVAE with encoder fixed according to the heat equation and standard HVAE with learnable encoder. Both HVAEs use the same decoder architecture as in \citet{rissanen2022generative}}
    \label{fig:heat}
\end{figure}

\begin{table}[htbp]
    \centering
    \caption{Inception Score (IS) of generated digits from HVAE with encoder fixed according to the heat equation and standard
    HVAE with learnable encoder.}
    \begin{tabular}{c c c}
    \hline
           & HVAE with heat equation encoder & Standard HVAE with learnable encoder\\
    \hline
      IS    &  \textbf{9.32} & 7.27\\
    \hline
    \end{tabular}
    \label{table: inception}
\end{table}

To show the critical role of self-supervised information, we train an HVAE with a similar architecture as in \citet{rissanen2022generative} on the binarized MNIST dataset \citep{salimans2015markov} as a proof of concept, where the encoder is fixed according to the heat equation \citep{rissanen2022generative}:
\begin{equation}
    q(\z_{1:T}|\z_0)=\prod_{t=1}^T q(\z_t|\z_0)=\prod_{t=1}^T \mathcal{N}(\z_t|F_h(t)\z_0,\sigma_h^2),
\end{equation}
where $F_h$ is the matrix for simulating the heat equation until time $t$. This can be seen as an HVAE trained with explicit supervision signals from pseudo videos created by the heat equation. We create $T=18$ frames of pseudo videos for each training image. Figure~\ref{fig:heat} demonstrates that the HVAE trained with pseudo videos created by the heat equation can generate much sharper and diverse digits than a standard HVAE which uses the same architecture but with a learnable encoder. We also report in Table~\ref{table: inception} the Inception Score (IS) of the generated images for quantitative comparison and HVAE trained with pseudo videos created by the heat equation achieves noticeably higher IS than standard HVAE. Here, we deliberately use heat equation instead of Gaussian noise to create pseudo frames to show that there are plenty of choices to create pseudo videos that contain useful self-supervised information besides adding Gaussian noise as in standard diffusion models.

\subsection{Improving Image Generation via Pseudo Video Generation}
With the concept of pseudo videos and their effectiveness in diffusion models, we are interested in the following open generic question in this work: 
\begin{adjustwidth}{2em}{1.5em}
\textit{Is it possible to improve other types of image generative models by jointly modelling the distribution of the original image and its corresponding pseudo video which contains self-supervised information?}
\end{adjustwidth}
The answer is affirmative. In this work, we show empirical evidence of the advantages of pseudo videos on two types of generation models, namely improving VQVAE \citep{van2017neural}  (Section~\ref{sec: phenaki_exp}) and DDPM \citep{nichol2021improved} (Section~\ref{sec: video_diff_exp}) with Phenaki \citep{villegas2022phenaki} and Video Diffusion \citep{harvey2022flexible} trained on pseudo videos, respectively. Moreover, we provide theoretical arguments favouring the use of more expressive ways of creating pseudo videos in the autoregressive video generation framework (Section~\ref{sec: theory}), beyond the first-order Markov strategy as typically used in the forward process of standard diffusion models.

\section{Improved Reconstruction and Generation in VQ-VAE with Pseudo Videos}\label{sec:phenaki}
\label{sec: phenaki_exp} 
In this section, we utilize pseudo videos to improve image generation quality of Vector Quantized Variational Autoencoder (VQVAE) \citep{van2017neural}, where the latent variables $\z$ are discrete tokens. 

\subsection{Preliminaries of VQVAE}
Vector Quantized Variational Autoencoder (VQVAE), similar to VAE, is also a deep latent variable model assuming a data generative process of the form:
\begin{equation}
 \bm{x} \sim \int p_{\mparam}(\bm{x} | \latent)p_{\mparam}(\latent) d\latent.
\end{equation}
The differences of VQVAE from VAE are two folds. First, while VAE assumes $\latent$ to be continuous and $p_{\mparam}(\latent)$ to be a standard Gaussian apriori, VQVAE employs discrete latent variables $z \in \{1, \cdots, K\}$, and parameterizes $p_{\mparam}(z)$ with another deep neural network. Second, while VAE trains the generative model parameters $\mparam$ associated with the likelihood $p_{\mparam}(\bm{x}|\z)$ via approximate MLE, VQVAE follows a two-step training procedure by fitting the likelihood $p_{\mparam}(\bm{x}|z)$, and prior $p_{\mparam}(z)$, separately. Given a training set of observations $\{\x_n\}_{n=1}^N$, the first step is learning the conditional distribution $p_{\mparam}(\bm{x}|z)$, and it is achieved with an auxiliary encoder, $z_{\vparam}(\x):\mathcal{X}\rightarrow \{1, \cdots, K\}$, (similar to the concept of inference model in VAE):
\begin{equation}
    z_{\vparam}(\x) = \argmin_{i\in \{1, \cdots, K\}}||f_{\vparam}(\x)-\z^{\text{emb}}_i||,
    \label{eq:quantization}
\end{equation}
where $f_\phi(\cdot):\mathcal{X}\rightarrow\mathbb{R}^{d}$ is parameterized by a neural network and $\z^{\text{emb}}_i\in \mathbb{R}^{d}$ is a $d$-dimensional latent embedding associated with discrete code $i$. We then train the encoder and the likelihood jointly with the following loss for each observation:
\begin{equation}
\begin{split}
    \mathcal{L}_{VQVAE}(\x; \mparam, \vparam, \{\z_i\}_{i=1}^K) &= -\log p_{\mparam}(\x|k) + ||sg(f_{\vparam}(\x)) - \z_{k}^{\text{emb}}||_2^2 \\& \quad\quad + \beta||f_{\vparam}(\x) - sg(\z_{k}^{\text{emb}})||_2^2, \quad k = z_{\vparam}(\x)
\end{split}
\label{eq:vqvae_loss}
\end{equation}
where $sg(\cdot)$ denotes the ``stop gradient" operator. The first term in the objective is the negative log-likelihood which encourages accurate reconstruction, and the second term is the quantization loss, ensuring that the latent embedding is close to the encoder network outputs $f_{\phi}(\x)$, and the third term is the commitment loss which makes sure the update for encoder parameters is not too fast when $\beta<1$, so that the encoder can commit to an embedding. 

Note that the latent space can be represented by multi-dimensional discrete tokens. For example, for an RGB image $\x$ with size $H \times W \times 3$, $f_{\phi}(\cdot)$ is typically parameterized by a convolutional neural network which maps $\x$ to an embedding with size $h\times w \times d$, and each dimension of $f_{\vparam}(\x)$ is quantized into one of $K$ discrete tokens as in Eq.~\ref{eq:quantization}, so that $\z_{\vparam}(\x)$ is a tensor of size $h\times w$ with discrete elements. 

In the second step of learning $p_{\mparam}(\latent)$, we map each $\x$ in the training set into latent discrete tokens, which forms an empirical distribution of the discrete latent variables, $p(\latent) \approx p_{\text{empirical}}(\latent)=\frac{1}{N}\sum_{n=1}^N \delta_{\z_{\vparam}(\x_n)}(\z)$. Next, we can choose a particular factorization order for $p_{\text{empirical}}(\latent)$, and fit an autoregressive model \citep{oord2016conditional, oord2016pixel}, $p_{\mparam}(\latent_t|\latent_1, \cdots, \latent_{t-1})$, to predict the next discrete code given the sequence of previously generated tokens by maximum likelihood, with training set obtained from $p_{\text{empirical}}(\latent)$. The final prior model is then $p_{\mparam}(\latent) = p_{\mparam}(\latent_1) p_{\mparam}(\latent_2|\latent_1)\cdots (\latent_{T}|\latent_{t< T_z})$, where $T_z$ is the total number of discrete tokens in $\z$ (e.g., $T_z=h \times w$ in the above image-generative example).

\subsection{Experiments}
To incorporate the pseudo video sequences into the model, we employ a video generative model counterpart of VQVAE, C-ViViT \citep{villegas2022phenaki}, to compress the pseudo videos of size $T \times H \times W \times 3$ into latent discrete tokens of size $t \times h \times w$. Instead of convolutional neural network, C-ViViT includes temporal transformer layers to handle the temporal dimension and in its hidden layers, the height and width dimension are flatten into a single ``spatial dimension" of size $h \times w$, and C-ViViT also considers spatial transformer layers to compute attention along the spatial dimension. In addition to standard VQVAE loss defined in Eq.~\ref{eq:vqvae_loss}, C-ViViT also adds a GAN style adversarial loss \citep{karras2020analyzing} and an image perceptual loss \citep{johnson2016perceptual,zhang2018unreasonable} to the final objective. For the latent space created by a C-ViViT, we consider two generative models to fit a prior $p(\z)$ for the latent tokens:
\begin{itemize}
    \item VideoGPT \citep{yan2021videogpt} uses an autoregressive (AR) Transformer \citep{brown2020language} to factorize $p(\z) = \prod_{i=1}^{T_z} p(\z_i|\z_{<i})$ in an autoregressive manner with masked self-attention, where $T_z$ is the total number of the tokens, and is trained with maximum likelihood. VideoGPT is the video generative model counterpart extended from ImageGPT \citep{chen2020generative}.
    \item Phenaki \citep{villegas2022phenaki} uses a bidirectional Transformer \citep{vaswani2017attention} to predict all tokens in one shot rather than in an autoregressive manner. At each training step, one samples a masking ratio $\gamma\in(0,1)$, and the model is trained by predicting the masked tokens given the unmasked ones. During generation, all tokens are masked initially, and the model predicts all tokens simultaneously. The generation will then be refined following a few steps of re-masking and re-prediction, with a decreasing masking ratio as we proceed. Phenaki is the video generative model counterpart extended from MaskGit \citep{chang2022maskgit}.
\end{itemize}

We compare the generation quality of the last frames (corresponding to the original images) in the generated videos from the video generative model trained on pseudo videos to the images generated by the original image generative model trained on the original target images. We include experimental setups below and we also include detailed hyperparameters in Appendix~\ref{appendix:hyper_pseudo_vid}.

\textbf{Datasets.} 
We create 8-frame and 18-frame pseudo videos using images from two benchmark datasets, CIFAR10 (32 ${\times}$ 32) \citep{cifar} and CelebA (64 ${\times}$ 64) \citep{liu2015deep}, with the blurring technique from \citet{bansal2023cold}. To create 8-frame pseudo videos, we blur the images recursively 7 times with a Gaussian kernel of size 11$\times$11 and standard deviation growing exponentially at the rate of 0.05. For 18-frame pseudo videos, we blur the images 17 times with same Gaussian kernel but with a standard deviation growing exponentially at the rate of 0.01. The pseudo videos are organized such that the last frames are the original target images and the first frames are the blurriest images. Figure~\ref{fig:pseudo_video} shows an example of such a pseudo video. We use 1-frame to denote images generated by the original image VQVAE model trained on the original target images. Initially, we also tried using Gaussian noise as data augmentation but we did not manage to obtain decent results. We believe it might be because VQVAE and C-ViViT compress the data to discrete latent tokens, and if using Gaussian noise as data augmentation, it will increase the variability of the pixel values in each patch. Therefore it would be hard to compress the pseudo videos with so much variability into discrete tokens defined with finite codebook size. In contrast, blurring will reduce the variability in the image and make the pixels close to each other have similar values, which makes it easier for a VQVAE or C-ViViT with finite codebook size to compress the pseudo videos.
\begin{figure}[t]
    \centering
    \includegraphics[width=0.9\linewidth]{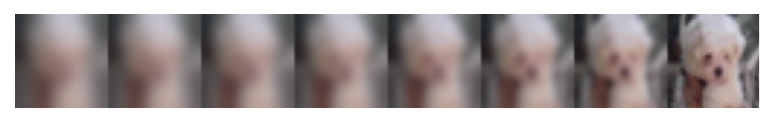}
    \caption{An example of pseudo video constructed by transforming an image of a dog using blurring.}
    \label{fig:pseudo_video}
\end{figure}

\textbf{Network Architectures.}\footnote{For 1-frame models, we have tried using deeper architectures but observed no improvement in performance, which suggests that our pseudo video framework can help further improve the performance of generative models while simply increasing the model size becomes ineffective.}
We train VQVAE based generative models with a codebook size ($K$) of 1024 for the discrete latent tokens. Specifically, we use C-ViViT as compression model (autoencoder) for pseudo videos. For a pseudo video with shape $(T, H, W, 3)$, we compress it to discrete latent tokens with shape $(\frac{T}{2}, \frac{H}{4}, \frac{W}{4})$ by extracting video patches of size $2\times4\times4$. For images with shape $(1,H,W,3)$, we use VQVAE to compress them to tokens with shape $(1,\frac{H}{4}, \frac{W}{4})$ by extracting image patches of size $4\times4$. We then consider two video generative models, VideoGPT and Phenaki (and their image generative model counterparts, ImageGPT and MaskGit), to fit the prior over the latent tokens.
\begin{itemize}
    \item \textbf{VQVAE/C-ViViT (reconstruction).} We use a similar architecture as in \citet{villegas2022phenaki}, which has a 4-layer spatial Transformer, a 4-layer temporal Transformer, with a hidden dimension of 512. For 1-frame VQVAE model, we consider a 8-layer spatial Transformer since there is no temporal dimension.
    \item \textbf{ImageGPT/VideoGPT (AR generation).} We use a similar architecture as in \citet{yan2021videogpt}, which has a 8-layer autoregressive (AR) Transformer with 4 attention heads and a hidden dimension of 144.
    \item \textbf{MaskGit/Phenaki (latent masked generation).} We use a similar architecture as in \citet{villegas2022phenaki}, which has a 6-layer bidirectional Transformer with a hidden dimension of 512.
\end{itemize}

\begin{figure}[t]
\centering
 \begin{subfigure}[t]{0.24\textwidth}
      \centering
      \includegraphics[width=\textwidth]{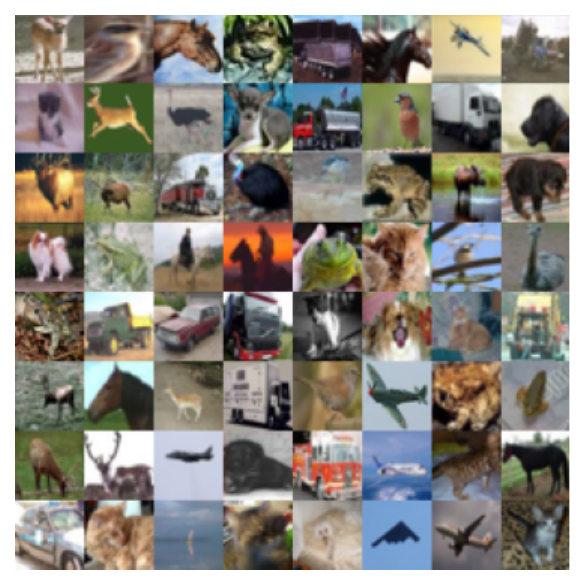}
      \caption{Ground-truth}
  \end{subfigure}
\hfill
\begin{subfigure}[t]{0.24\textwidth}
      \centering
      \includegraphics[width=\textwidth]{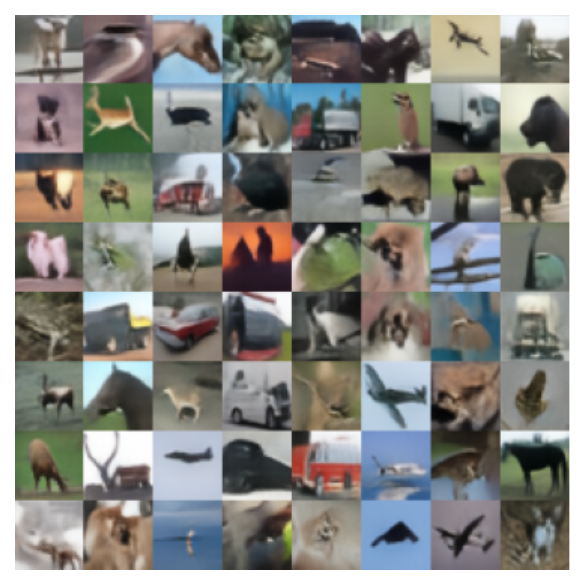}
      \caption{1-frame reconstruction}
  \end{subfigure}
\hfill
\begin{subfigure}[t]{0.24\textwidth}
      \centering
      \includegraphics[width=\textwidth]{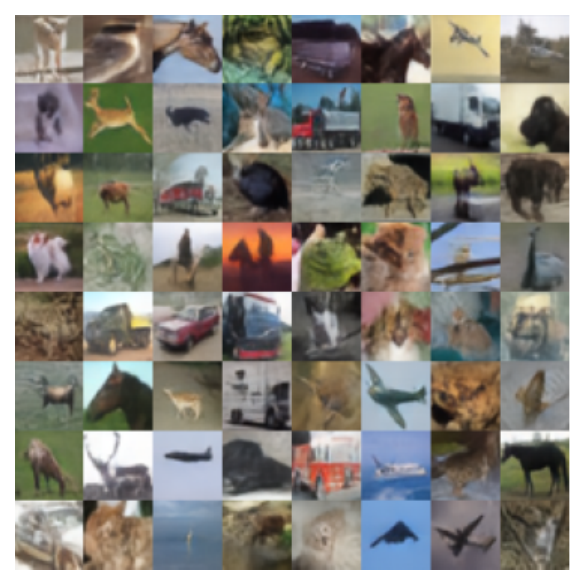}
      \caption{8-frame reconstruction}
  \end{subfigure}
\hfill
\begin{subfigure}[t]{0.24\textwidth}
      \centering
      \includegraphics[width=\textwidth]{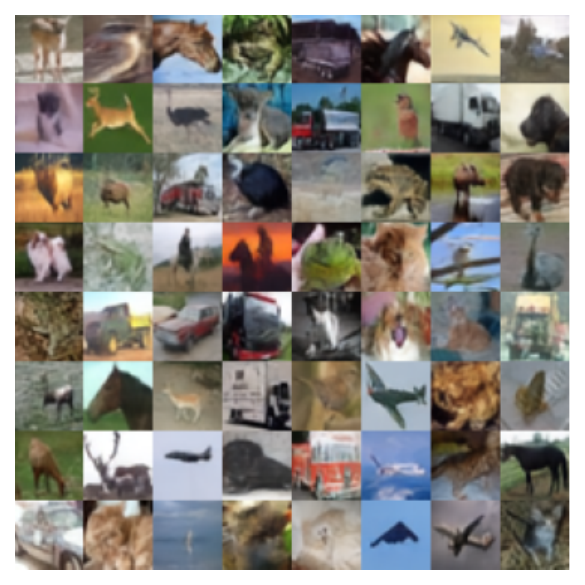}
      \caption{18-frame reconstruction}
  \end{subfigure}

\caption{Ground-truth images and reconstructed images from VQVAE/CViViT trained on CIFAR10.}
\label{fig:recon_cifar10}
\end{figure}
 
\begin{figure}[t]
\centering
 \begin{subfigure}[t]{0.24\textwidth}
      \centering
      \includegraphics[width=\textwidth]{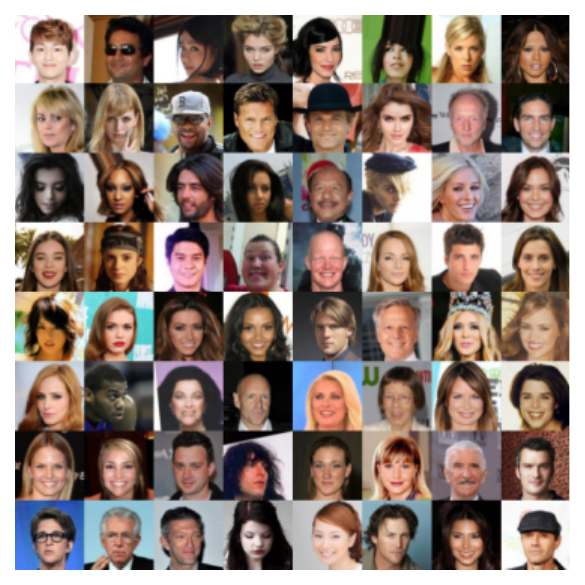}
      \caption{Ground-truth}
  \end{subfigure}
\hfill
\begin{subfigure}[t]{0.24\textwidth}
      \centering
      \includegraphics[width=\textwidth]{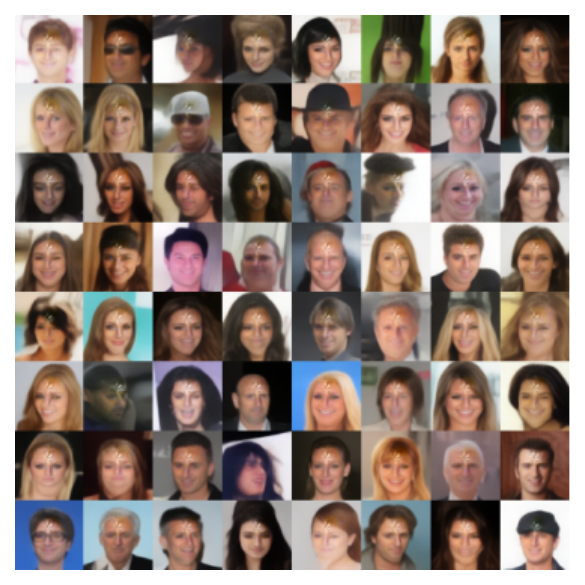}
      \caption{1-frame reconstruction}
  \end{subfigure}
\hfill
\begin{subfigure}[t]{0.24\textwidth}
      \centering
      \includegraphics[width=\textwidth]{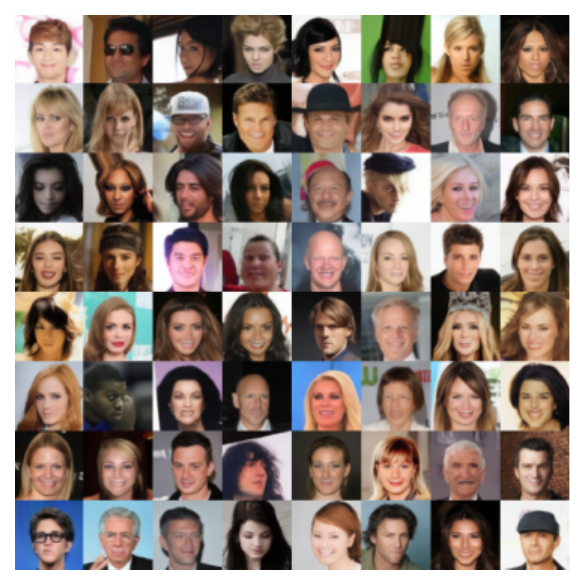}
      \caption{8-frame reconstruction}
  \end{subfigure}
\hfill
\begin{subfigure}[t]{0.24\textwidth}
      \centering
      \includegraphics[width=\textwidth]{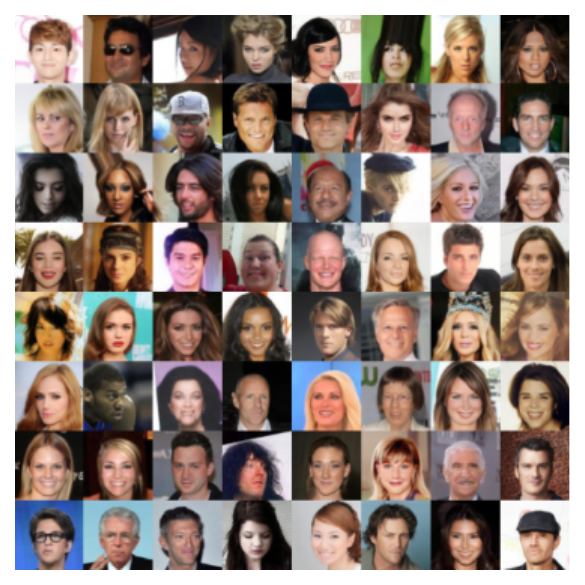}
      \caption{18-frame reconstruction}
  \end{subfigure}
\caption{Ground-truth images and reconstructed images from VQVAE/CViViT trained on CelebA.}
\label{fig:recon_celeba}
\end{figure}

\textbf{Evaluation Metric.} 
We compute Frechet Inception Distance (FID $\downarrow$) \citep{heusel2017gans} with 50k samples to evaluate the quality of images, either from the last-frames of videos generated by video models trained on pseudo videos, or images generated by image models trained on the original target images. We also report Peak Signal-to-noise Ratio (PSNR $\uparrow$) as an additional metric for reconstruction result.

\begin{table}[t]
\centering
\caption{Last-frame FID/PSNR of images produced by C-ViViT (reconstruction), VideoGPT (AR generation) and Phenaki (latent masked generation) trained on pseudo videos constructed from CIFAR10 and CelebA images. 1-frame results are obtained from their image counterparts VQVAE (reconstruction), ImageGPT (AR generation) and MaskGit (latent masked generation) trained on original CIFAR10 and CelebA images.}
\label{table:phenaki-results}
\begin{adjustbox}{width=\textwidth}
\begin{tabular}{@{}ccccccc@{}}
\toprule
\multirow{2}{*}{}        & \multicolumn{3}{c}{CIFAR10}                                                        & \multicolumn{3}{c}{CelebA}                                                  \\ \cmidrule(l){2-7} 
                         & \multicolumn{1}{c}{1-frame} & \multicolumn{1}{c}{8-frame}        & 18-frame       & \multicolumn{1}{c}{1-frame} & \multicolumn{1}{c}{8-frame} & 18-frame       \\ \midrule
Reconstruction (FID)           & \multicolumn{1}{c}{24.53}   & \multicolumn{1}{c}{13.81}          & \textbf{11.26} & \multicolumn{1}{c}{24.62}   & \multicolumn{1}{c}{5.72}    & \textbf{2.27}  \\ 
Reconstruction (PSNR)           & \multicolumn{1}{c}{19.97}   & \multicolumn{1}{c}{21.86}          & \textbf{25.89} & \multicolumn{1}{c}{20.52}   & \multicolumn{1}{c}{24.66}    & \textbf{29.68}  \\ \midrule
AR Generation (FID)            & \multicolumn{1}{c}{72.06}   & \multicolumn{1}{c}{\textbf{54.60}} & 69.23          & \multicolumn{1}{c}{32.98}   & \multicolumn{1}{c}{30.19}   & \textbf{28.08} \\
Latent Masked Generation (FID) & \multicolumn{1}{c}{49.27}   & \multicolumn{1}{c}{\textbf{35.50}} & 47.65          & \multicolumn{1}{c}{27.34}   & \multicolumn{1}{c}{16.87}   & \textbf{16.66} \\ \bottomrule
\end{tabular}
\end{adjustbox}
\end{table}

\textbf{Results.} Table \ref{table:phenaki-results} shows the last-frame FID and PSNR of C-ViViT for reconstruction and that of VideoGPT and Phenaki for AR and masked generation on CIFAR10 and CelebA, respectively. The 1-frame results correspond to the performance of their image model counterparts (i.e., VQVAE for image reconstruction, and ImageGPT and MaskGit for AR and masked image generation, respectively). We observe that pseudo videos indeed help improve the training of the C-ViViT as the reconstruction quality of the last frame is significantly improved with a few more frames. We show reconstructed CIFAR10 and CelebA images from different C-ViViT models in Figures~\ref{fig:recon_cifar10} and~\ref{fig:recon_celeba}, respectively. It can be seen that for both datasets, the reconstructed images trained with 1-frame models are relatively smoothed and this issue is resolved by using pseudo videos which produce sharper images. Quantitatively, reconstruction FID improves as more frames are used. Pseudo videos also improve the image generation performance compared to the 1-frame results. Interestingly, we see a diminishing return as we include more frames. For CelebA images, 18-frame pseudo videos help achieve the best image generation performance for both AR and latent masked generation, but 8-frame models achieve a comparable performance. For CIFAR10 images, 8-frame pseudo videos result in better image generation performance than 18-frame pseudo videos, which suggests the latent tokens have a more complex prior distribution in order to reconstruct pseudo videos with 18 frames well and therefore this prior is more difficult for VideoGPT or Phenaki to capture. In summary. while pseudo videos help improve the performance, the optimal number of frames may depend on the dataset and the augmentation strategy. This diminishing return is not a severe issue in practice since practitioners may prefer to improve the generation with just a few more pseudo frames to avoid introducing a high computational cost.

\section{Improved Generation via Higher-order Markov Pseudo Videos}\label{sec:videodiff}
Since pseudo video contains extra information on the target image, we would like to better understand what type of additional information can be leveraged to achieve better image generation quality. In practice, since there are infinitely many data augmentation strategies to create pseudo videos, we would like to study which types of data augmentation are more favorable to shed light on the practical design of pseudo videos.
\subsection{Is First-order Markov Chain the Optimal Choice for Creating Pseudo Videos?}
\label{sec: theory}
Consider pseudo video $\z_{1:T}$, where $\z_T:=\x$ is the target image\footnote{Unlike diffusion models where $\z_0$ denotes the original image, from here onwards we will denote the original image by $\z_T$ since it is the last frame of the pseudo video.}, and $\z_t$'s ($t < T$) are some noisy measurements of $\z_T$ created with some data augmentation. We show that in principle, generative models that utilize more pseudo frames to generate $\z_T$ are more likely to achieve better performance, and passing information of the target image to the pseudo frames with a first-order Markov chain as in standard diffusion models may not be the optimal choice. We demonstrate it with the following autoregressive video generation example.

Consider building a generative model $g$ that predicts $\z_T$ by taking advantage of the information in $\z_{T-1}$ alone. We train the model by minimizing the reconstruction error. The minimum of this loss is
\begin{equation}
\mathcal{L}^*_1 = \min_g \mathbb{E}_{p(\z_T, \z_{T-1})}[||\z_T-g(\z_{T-1})||_2^2] = \mathbb{E}_{p(\z_{T-1})}[\text{Var}_{p(\z_T|\z_{T-1})}(\z_T)],
\end{equation}
which is achieved at the non-parametric optimum $g^*(\z_{T-1})=\mathbb{E}_{p(\z_T|\z_{T-1})}[\z_T]$, where $\text{Var}_{p(\z_T|\z_{T-1})}(\z_T)=\mathbb{E}_{p(\z_T|\z_{T-1})}\{(\z_T-\mathbb{E}_{p(\z_T|\z_{T-1})}[\z_T])^\top(\z_T-\mathbb{E}_{p(\z_T|\z_{T-1})}[\z_T])\}$. Now consider another model $h$ that predicts $\z_T$ using both $\z_{T-1}$ and $\z_{T-2}$ by minimizing the reconstruction error again. The minimum reconstruction error this time is
\begin{equation}
    \mathcal{L}_2^* = \min_h \mathbb{E}_{p(\z_T, \z_{T-1}, \z_{T-2})}[||\z_T-h(\z_{T-1}, \z_{T-2})||_2^2] = \mathbb{E}_{p(\z_{T-1}, \z_{T-2})}[\text{Var}_{p(\z_T|\z_{T-1}, \z_{T-2})}(\z_T)],
\end{equation}
which is achieved at the non-parametric optimum $h^\ast(\z_{T-1}, \z_{T-2}) = \mathbb{E}_{p(\z_T|\z_{T-1}, \z_{T-2})}[\z_T]$. The benefit of using more pseudo frames to generate the target image can be seen from the fact that the minimum reconstruction error will never increase by using more pseudo frames since by the law of total variance,
\begin{equation}
    \mathcal{L}^*_2 -  \mathcal{L}_1^* = -\mathbb{E}_{p(\z_{T-1})}\{\text{Var}_{p(\z_{T-2}|\z_{T-1})}(\mathbb{E}_{p(\z_T|\z_{T-1}, \z_{T-2})}[\z_T])\} \leq 0.
    \label{eq: recon_ineq}
\end{equation}
Moreover, the non-optimality of creating pseudo video via first-order Markov chain becomes clear: the first-order Markov data augmentation implies that $p(\z_T|\z_{T-1})=p(\z_T|\z_{T-1}, \z_{T-2})$ and consequently $\mathcal{L}_2^* = \mathcal{L}_1^*$. More specifically, for strict inequality in Eq~\ref{eq: recon_ineq}, we need to avoid $p(\z_T|\z_{T-1}) = p(\z_T|\z_{T-1}, \z_{T-2})$, which is equivalent to avoiding the use of either first-order Markov chain $\z_T \rightarrow \z_{T-1} \rightarrow \z_{T-2}$ or $\z_T \leftarrow \z_{T-1} \rightarrow \z_{T-2}$. This analysis is informative for us to design better pseudo videos, for example through data augmentation with higher-order Markov chains. We formalize the above informal reasoning into Theorem~\ref{theorem} and provide the formal proof in Appendix~\ref{appendix:proof}.

\begin{theorem} \label{theorem}
Consider two video generative models that predict the last-frame $\z_T$ based on some previous frames. Suppose that they take the form of $\hat{\z}_T^{(g)}=g(\z_{s_1}, \z_{s_2}, \cdot\cdot\cdot, \z_{s_k})$ and $\hat{\z}_T^{(h)}=h(\z_{s_1}, \z_{s_2}, \cdot\cdot\cdot, \z_{s_l})$, respectively, where $T>s_1>\cdot\cdot\cdot>s_k>\cdot\cdot\cdot>s_l$. Then, we have
\begin{equation}
    \min_{\hat{\z}_T^{(h)}} \mathbb{E}_{p(\z_T, \z_{s_1}, \cdot\cdot\cdot, \z_{s_l})}[||\z_T-\hat{\z}_T^{(h)}||_2^2] \leq \min_{\hat{\z}_T^{(g)}} \mathbb{E}_{p(\z_T, \z_{s_1}, \cdot\cdot\cdot, \z_{s_k})}[||\z_T-\hat{\z}_T^{(g)}||_2^2],
\end{equation}
where the equality attains if $\z_T|\z_{s_1},\cdot\cdot\cdot, \z_{s_k} \,{\buildrel d \over =}\,  \z_T|\z_{s_1},\cdot\cdot\cdot, \z_{s_l}$.
\end{theorem}

\begin{remark} 
The minimum reconstruction errors above are obtained with non-parametric optima. In practice, this corresponds to the assumption that our neural networks $g_{\theta}$ and $h_{\phi}$ are flexible enough to accurately approximate the non-parametric optima for the theorem to hold. Besides, the analysis is based on the assumption that $\{\z_{s_i}\}_{i=1}^l$ are drawn from the ground-truth distribution, while in practice they also need to be generated with their associated previous frames, which means when the generated  $\{\hat{\z}_{s_i}\}_{i=1}^l$ are far away from their
ground-truth distribution, the theorem would not hold. Nevertheless, the analysis provides intuitions of the benefit of conditional generation with longer past contexts and the potential improvement in performance by using more expressive pseudo videos rather than the ones created with first-order Markov transition as in standard diffusion models, which we empirically verify with experiments in the following section.
\label{remark_of_theorem}
\end{remark}

\subsection{Experiments}
\label{sec: video_diff_exp}
We consider generating each frame of the pseudo videos using the information provided in the previously generated frames. In particular, we use a video diffusion model \citep{harvey2022flexible} trained by predicting frames autoregressively conditioning on the most recent previous frames in a context window of size $C$. In other words, the conditionals $p(\z_t|\z_{t-C:t-1})$ are estimated using a diffusion model parameterized by neural network taking $\z_{t-C:t-1}$ as an additional input. 
We compare the performance of video diffusion models trained on pseudo videos created by both standard first-order Markov transformation and higher-order Markov transformation to empirically verify our argument in Section~\ref{sec: theory}. We describe the experimental setup below and include detailed hyperparameters in Appendix~\ref{appendix:hyper_pseudo_vid}.

\textbf{Datasets.} 
We create 4-frame and 8-frame pseudo videos using images from CIFAR10 (32 ${\times}$ 32) and CelebA (64 ${\times}$ 64). We use Gaussian noise as data augmentation and we consider two strategies:
\begin{itemize}
    \item \textbf{First-order Markov.} We add Gaussian noise recursively 3 or 7 times to create first-order Markov pseudo videos with a linear schedule \citep{ho2020denoising} with $\beta$ ranging from 0.0001 to 0.05\footnote{During the hyperparameter selection stage, we experiment with maximal $\beta$ selected from \{0.02, 0.05, 0.1\}, and 0.05 is the one that leads to the best performance with first-order Markov data augmentation.}: $\z_{T-t} = \sqrt{1-\beta_t} \z_{T-t+1} + \sqrt{\beta_t} \epsilon$, $\epsilon \sim \mathcal{N}(0, I)$.
    \item \textbf{High-order Markov.} While using the same noise schedule to create $\z_{T-t}$, instead of adding Gaussian noise to $\z_{T-t+1}$, we use a simple strategy to create high-order Markov pseudo videos by adding Gaussian noise to the mean of $\{\z_{T-t+s}\}_{s=1}^{t}$: $\z_{T-t} = \sqrt{1-\beta_t} [\frac{1}{t}\sum_{s=1}^{t}\z_{T-t+s}] + \sqrt{\beta_t}\epsilon$, $\epsilon \sim \mathcal{N}(0, I)$.
\end{itemize}
We plot examples of pseudo videos created with the above two strategies in Figure~\ref{fig:high_order_example}. We again use 1-frame to denote the results of the image generative model counterparts, improved DDPM \citep{nichol2021improved}, trained on the original target images. We also consider blurring as the data augmentation, however, its performance is worse than the performance of using Gaussian noise (see the \textbf{Results} paragraph below). 

\begin{figure}[t]
    \centering
     \begin{subfigure}[t]{0.9\textwidth}
      \centering
      \includegraphics[width=\textwidth]{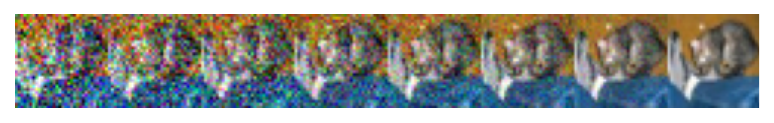}
      \caption{First-order}
    \end{subfigure}
    \hfill
    \begin{subfigure}[t]{0.9\textwidth}
      \centering
      \includegraphics[width=\textwidth]{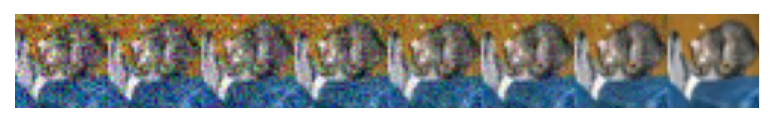}
      \caption{High-order}
    \end{subfigure}
    \hfill
    \caption{Examples of a pseudo video constructed by adding Gaussian noise to a CIFAR10 image using first-order Markov chain (top) and high-order Markov chain (bottom).}
    \label{fig:high_order_example}
    \hfill
\end{figure}

\textbf{Network Architectures.}
We use a similar UNet architecture as in \citet{harvey2022flexible}, with 2 residual blocks in each downsampling and upsampling layer and a base channel size of 128 across all models. Notice that \citet{harvey2022flexible} is built based on the same architecture as the 1-frame image diffusion model \citep{nichol2021improved}, and these hyperparameters are kept the same for the 1-frame image diffusion model. During generation, we use the ``Autoreg'' sampling scheme from \citet{harvey2022flexible} so that each frame $\z_t$ is generated by conditioning on the most recently generated frames in a context window, $\{\z_{t-c}\}_{c=1}^C$. The sizes of the context window $C$ (i.e., the time lag) are 2 and 4 for 4-frame and 8-frame models, respectively. We consider 1,000 diffusion steps every time we generate a new frame. Since 4-frame and 8-frame models jointly generate the first 2 and the first 4 frames (the initial context window) at the beginning, respectively, they use overall 3,000 and 5,000 diffusion steps to generate the whole pseudo videos, respectively. We also consider increasing the number of diffusion steps from 1,000 to 4,000 when training the 1-frame image diffusion model and compare it with the 4-frame video diffusion models with 3,000 diffusion steps in total to ensure the performance gain in video diffusion models is not simply because we have more diffusion steps overall.



\begin{figure}[t]
    \centering
    \begin{subfigure}[t]{0.45\textwidth}
      \centering
      \includegraphics[width=\textwidth]{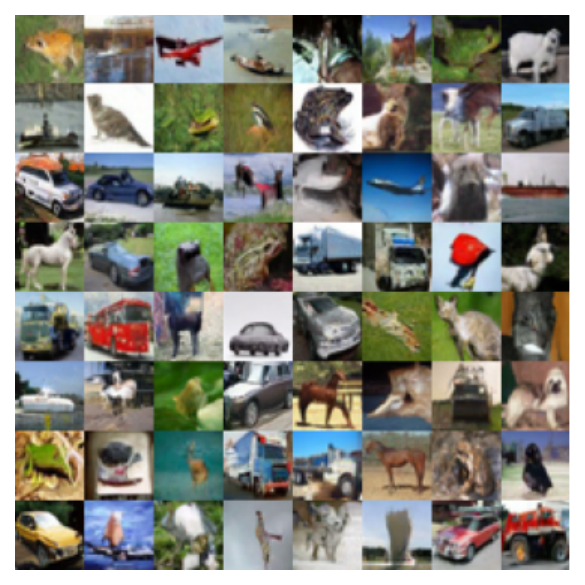}
      \caption{CIFAR10}
    \end{subfigure}
   \hfill
    \begin{subfigure}[t]{0.45\textwidth}
      \centering
      \includegraphics[width=\textwidth]{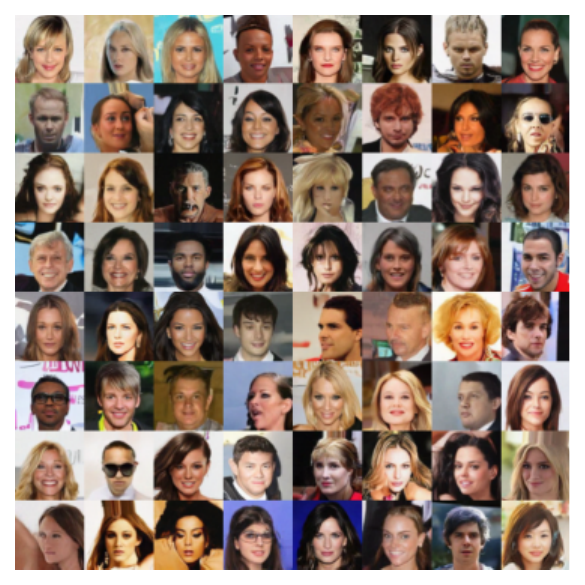}
    \caption{CelebA}
    \end{subfigure}
    \caption{Generated images from the video diffusion models trained on 4-frame high-order Markov pseudo videos of CIFAR10 and CelebA, respectively.}
    \label{fig:high_order}
    \hfill
\end{figure}

\begin{table}[t]
\centering
\caption{Last-frame FID of images generated by video diffusion models trained on pseudo videos constructed from CIFAR10 and CelebA images (with both first-order Markov or high-order Markov Gaussian noise data augmentation). 1-frame results are obtained from an image diffusion model trained on the original CIFAR10 and CelebA images with equivalent UNet architecture.}
\label{table:results-diffusion}
\begin{tabular}{@{}ccccccc@{}}
\toprule
\multirow{2}{*}{}  & \multicolumn{3}{c}{CIFAR10}                                                 & \multicolumn{3}{c}{CelebA}                                                 \\ \cmidrule(l){2-7} 
                   & \multicolumn{1}{c}{1-frame} & \multicolumn{1}{c}{4-frame}        & 8-frame & \multicolumn{1}{c}{1-frame} & \multicolumn{1}{c}{4-frame}       & 8-frame \\ \midrule
First-order Markov & \multicolumn{1}{c}{12.90}   & \multicolumn{1}{c}{17.30}          & 15.90   & \multicolumn{1}{c}{7.76}    & \multicolumn{1}{c}{13.61}         & 12.64   \\
High-order Markov  & \multicolumn{1}{c}{12.90}   & \multicolumn{1}{c}{\textbf{12.58}} & 12.80   & \multicolumn{1}{c}{7.76}    & \multicolumn{1}{c}{\textbf{6.88}} & 7.55    \\ \bottomrule
\end{tabular}
\end{table}

\textbf{Results.} 
We again compute FID (based on 10k samples) to evaluate the models. Table~\ref{table:results-diffusion} shows the last-frame FID of pseudo videos generated by video diffusion models for CIFAR10 and CelebA images, respectively. The 1-frame results correspond to the performance of their image counterparts (i.e., improved DDPM). While video diffusion models trained on first-order Markov pseudo videos do not outperform the 1-frame image diffusion model, both 4-frame and 8-frame video diffusion models trained on high-order Markov pseudo videos can achieve better results on both datasets, which empirically justify the non-optimality of first-order Markov chains in terms of passing information from the target images to the pseudo frames as shown in Section~\ref{sec: theory}, and our proposal of using more expressive pseudo videos rather than the ones created with first-order Markov chains. Notice that the 4-frame models outperform the 8-frame models, which may be due to the complex nature of the ground-truth distribution of longer pseudo videos, and thus more expressive architecture may be required to achieve optimal results (see Remark~\ref{remark_of_theorem}), while here we use the same UNet architecture across all models for fair practical comparison. Again, this U-turn should not be a severe issue in practice since practitioners may prefer to improve the generation with as few pseudo frames as possible to reduce additional computational cost. We visualize some generated images from the 4-frame models trained on CIFAR10 and CelebA in Figure~\ref{fig:high_order}.

Table~\ref{table:results-diffusion_4ksteps} compares the 4-frame model with an 1-frame model but with 4,000 diffusion steps. While the performance of the 1-frame model with more overall diffusion steps improves for CIFAR10 and outperforms the 4-frame model, its performance on CelebA is worse than the 4-frame model. Moreover, on CelebA, it even becomes worse than the baseline 1-frame model with only 1,000 diffusion steps while the 4-frame video diffusion model consistently improves the performance on both datasets, which suggests simply increasing the number of diffusion steps in an image diffusion model may not always be effective.

Instead of using Gaussian noise, we also tried using Gaussian blur to create pseudo videos as in Section~\ref{sec: phenaki_exp}. However, our experiments on CIFAT10 with Gaussian blur suggest worse results than adding Gaussian noise (see Table~\ref{table:celeba-results-diffusion_blur}), and we decided not to consider it for further experiments. This suggests that in practice the well-performed data augmentation strategies may vary across different classes of video generative models.

\begin{table}[t]
\centering
\caption{Last-frame FID of images generated by video diffusion models trained on pseudo videos constructed from CIFAR10 and CelebA images with high order Markov Gaussian noise data augmentation. 1-frame results are obtained from an image diffusion model trained on the original CIFAR10 and CelebA images. Here, the 1-frame models use 4,000 diffusion steps, while the 4-frame models use 3,000 diffusion steps overall.}
\begin{tabular}{c c c c}
\toprule
& 1-frame (1k steps) & 1-frame (4k steps) & 4-frame (3k steps overall)\\
\midrule
CIFAR10 & 12.90 & \textbf{11.95} & 12.58\\
CelebA & 7.76 & 7.87 & \textbf{6.88}\\
\bottomrule
\label{table:results-diffusion_4ksteps}
\end{tabular}
\end{table}

\begin{table}[h]
\centering
\caption{Last-frame FID of images generated by video diffusion models trained on pseudo videos constructed from CIFAR10 images with high-order Markov data augmentation (either Gaussian noise or Gaussian blur).}
\begin{tabular}{c c c }
\toprule
& 4-frame & 8-frame\\
\midrule
Gaussian noise  & \textbf{12.58} & 12.80\\
Gaussian blur & 15.33 & 22.63\\
\bottomrule
\label{table:celeba-results-diffusion_blur}
\end{tabular}
\end{table}

\section{Related Work}
\subsection{Sequential Generative Models}
Hierarchical variational autoencoders (HVAEs) \citep{sonderby2016ladder,maaloe2019biva,vahdat2020nvae,child2020very,xiao2023trading} are a class of sequential generative models constructed by stacking standard VAEs \citep{welling2014auto}. Although HVAEs represent a rich class of expressive generative models, they are hard to train in practice due to optimization difficulty, as discussed in Section~\ref{sec:motivation}.
Diffusion models \citep{sohl2015deep,ho2020denoising,song2020score,kingma2021variational,nichol2021improved,song2021denoising,rissanen2022generative, bansal2023cold, hoogeboom2023blurring} can be seen as a special case of HVAEs where the encoders are fixed, pre-defined Gaussian convolution kernels. Specifically, they essentially regress a sequence of noisy images created from the target image with self-supervision, as described in Section~\ref{sec:motivation}. Despite its similarity to HVAEs, diffusion models, and latent diffusion models \citep{rombach2022high} which apply diffusion models in the lower dimensional latent space of another latent variable model (e.g., VQVAE \citep{van2017neural}), have achieved state-of-the-art performance partially due to the additional self-supervision signal provided by the noise-corrupted images. Flow matching \citep{lipman2022flow,liu2022flow,albergo2023stochastic,gat2024discrete,wang2024rectified} is another state-of-the-art sequential generative modelling technique that trains continuous normalizing flows \citep{chen2018neural} by regressing a sequence of vector fields inducing a probability path that connects the data distribution and prior distribution with direct self-supervision. It has been show that flow matching can learn more straight trajectories than diffusion models, which requires less number of discretization steps at generation time. Furthermore, flow matching allows us to relax the Gaussian assumption for the prior distribution and thus enables coupling between arbitrary distributions \citep{albergo2023stochastic}.
In contrast, our proposed framework introduces a new family of approaches that leverage video generative models and pseudo videos with self-supervised frames to improve any given image generative models.

\subsection{Self-supervised Learning} 
Self-supervised learning \citep{liu2021self,shwartz2024compress} turns an unsupervised learning problem into a supervised learning problem by handcrafting pseudo labels for unlabeled data. There are two common approaches to self-supervised learning. 1) Contrastive learning \citep{chen2020simple,tian2020makes,wu2020mutual}, predicts whether two inputs are different augmentations of the same original data. 2) Masked learning \citep{devlin2018bert, he2022masked,fang2023eva} predicts randomly masked parts of an input given the unmasked parts. While our approach of fitting a video model to pseudo video sequences created by augmenting the original images does not belong to either of these families, it is essentially a new form of self-supervised learning since the pseudo video sequences can be seen as handcrafted pseudo labels for our model to predict, which provides the model with extra information (e.g., different fidelity of the original image).

\section{Conclusion and Discussions}\label{sec:discussion}
\label{sec:pseudo_conclusion}
\textbf{Summary.}
We drew our key insight from comparing standard HVAEs and diffusion models: the additional self-supervised information on the intermediate states provided by the noise corrupted pseudo frames in diffusion models may contribute to their success. Based on this insight, we proposed to leverage the self-supervised information from the pseudo videos constructed by applying data augmentation to the target images to improve the performance of image generative models. This was done by extending image generative models to their video generative models counterparts and training video generative models on pseudo videos. We show in our experiments that for two popular image generative models, VQVAE and Improved DDPM, their video generative model counterparts trained on pseudo videos of just a few frames can improve image generation performance, which empirically verified the benefit of the additional self-supervised information in the pseudo videos. 

\textbf{Discussions and Future Work.}
Our proposed framework provides an alternative approach of scaling up any given generative models: instead of making generative models larger by stacking more layers, we demonstrated that it was possible to improve the generation quality by turning an image generative model trained on images into its video generative model counterpart trained on pseudo videos, which is usually straightforward since many video generative models are built upon image generative models. On the other hand, this raises challenges on how to design informative pseudo videos. In autoregressive video generation frameworks, we show the potential issue of first-order Markov pseudo videos theoretically and propose to use higher-order Markov pseudo videos instead to address this issue. However, it is in general unclear what the optimal pseudo videos are within such a large design space, which we leave as a future research question. Another interesting future direction is to explore whether the same principle can be applied to other data modalities. While self-supervised signals can be easily obtained using data augmentation for images, it remains unclear whether there are proper ways to inject self-supervised information for other data modalities, such as text or molecules.
\chapter{Conclusion and Future Work}
\label{cha:conclusion}

This thesis focused on bridging deep sequence models and probabilistic methods to develop new machine learning models that enjoy the advantages from both fields. We have strived to exploit the natural connections between inductive biases in deep sequence models and existing probabilistic methods to guide our model design, enabling the proposed methods to be compatible with existing progresses in the field of deep sequence models. Chapter~\ref{cha:sgpa} and Chapter~\ref{cha:hsgp} leverage natural connections between popular deep sequence model architectures and Gaussian processes to build GP models tailored to Transformer architectures and online learning problems, respectively. Chapter~\ref{cha:pseudovid}, inspired by the comparison between diffusion models and HVAE, proposes to improve image-generative models by jointly modeling the sequence of the original image and its pseudo video constructed with self-supervised information. Below, we provide a detailed recap of our contributions made in each chapter in Section~\ref{sec:summary_final}. In Section~\ref{sec:limit_and_future}, we give a critical overview of the limits of the proposed methods and potential future works to fix them.

\section{Thesis Summary}
\label{sec:summary_final}
In Chapter~\ref{cha:sgpa}, we identified the equivalence between kernel attention and the posterior mean of SVGP: queries, keys, and values in kernel attention module are equivalent to queried input locations, inducing point locations, and variational parameters associated with the posterior mean in SVGP, respectively. Based on this key observation, we proposed sparse Gaussian process attention (SGPA) which places a GP prior for the attention output and quantifies the uncertainty of the attention output with the posterior covariance (constructed based on an additional variational covariance parameters) of the corresponding SVGP. We stack multiple SGPA layers together to build probabilistic Transformers which are essentially deep GPs, and train SGPA-based Transformers with variational inference. SGPA is tailored to Transformer architectures and quantifies uncertainty directly in the space of attention output rather than through posterior of weights as in weight-space Bayesian neural network methods. For self-attention based Transformers where standard SGPA requires input-dependent variational covariances obtained via amortization similar to the input-dependent keys and values, we further introduce decoupled SGPA which reduces the computational complexity using techniques from decoupled SVGP, where the variational covariance is constructed purely based on another set of global keys shared across all input sequences. SGPA was empirically evaluated on a suite of predictive tasks and it effectively calibrated the uncertainty in Transformer models while keeping the accuracy to be competitively high.

In Chapter~\ref{cha:hsgp}, we found that by replacing the deterministic signals with a GP prior in the HiPPO (High-order Polynomial Projection Operators) framework, we can naturally interpret the HiPPO time-varying orthogonal projections as inducing variables of an interdomain SVGP, where the basis functions are time-dependent orthogonal polynomials. These interdomain inducing variables adaptively compress the history of a GP while preserving long-term memory. Based on this, we developed online HiPPO-SVGP (OHSVGP) which leverages the long-range memory preservation capability of HiPPO for online learning tasks. OHSVGP bypassed the cumbersome optimization-based online inducing locations updates of standard online SVGP by updating the prior covariance matrices online via the efficient HiPPO-ODE recurrence, bringing an extra layer of computational efficiency to OHSVGP. Empirically, we evaluated OHSVHP on a suite of online and continual learning tasks and we showed that OHSVGP outperforms existing online sparse GP methods, especially in scenarios requiring long-term memory. The efficient streaming capability and preservation of historical information make OHSVGP well-suited for real-world online learning applications demanding both speed and accuracy.

In Chapter~\ref{cha:pseudovid}, we shifted our focus from predictive tasks to generative modeling. In particular, we explored the benefits of sequence or sequential generative modeling. We began with a comparison between two classes of sequential generative models, diffusion models and HVAEs. Our key observation is that diffusion models can be viewed as HVAEs but they impose direct supervision signals, in the form of pseudo videos created by corrupting the original target image, for the sequence of intermediate variables while standard HVAEs leave these intermediate states fully flexible, which might be one of the reasons why diffusion models beat HVAEs. Inspired by this difference, we explored the possibility of improving other types of image generative models by jointly
modeling the distribution of the original image and its corresponding pseudo video containing self-supervised information. Specifically, we extended two popular image generative model frameworks, VQVAE and Improved DDPM, to their video generative model counterparts and trained them on pseudo videos. We empirically showed that this procedure improves the image generation quality with pseudo videos of just a few frames. Moreover, we theoretically analyzed what probabilistic structure over the pseudo videos may be desirable in autoregressive video generation frameworks. This analysis allows us to identify potential issues of certain pseudo videos including those in the form of first-order Markov chains and we proposed a simple and effective approach to bypass the potential issues by constructing higher-order Markov pseudo videos.

\section{Limitations and Future Research Directions}
\label{sec:limit_and_future}
We have shown the promise of building probabilistic methods tailored to deep sequence models by keeping inductive biases in deep sequence models in mind, and we believe that this vision can inspire more future works that combine deep sequence models and probabilistic modeling, leading to mutually reinforced improvement. We list a few future research directions that may address the limitations of the methods proposed in this thesis for further improvements.

\paragraph{Extending SGPA to decoder-based Transformers.} In Chapter~\ref{cha:sgpa}, we only consider applying SGPA to encoder-based Transformers where the context in the sequence is fixed. However, many sequence modeling tasks require autoregressive prediction based on decoder-based Transformer \citep{devlin2018bert,openai2024gpt4technicalreport}. For example, in natural language applications, decoder-based Transformer are often used to predict the next token based on existing context consisting of previously generated tokens. Since the context is kept updated, the static GP prior as in Chapter~\ref{cha:sgpa} becomes unideal to represent the dynamical uncertainty associated with the changing unknown information beyond the adaptive context. Sequential Bayesian inference which keeps updating the current prior to be the (approximate) posterior at the previous time step is more suitable to model this adaptive context knowledge. To do so with variational inference \citep{nguyen2018variational}, we may replace the static ELBO training objective for encoder-based SGPA with the following online ELBO:
\begin{equation}
\begin{split}
     \mathcal{L}_{ELBO}^{(t)}&=\mathbb{E}_{q_t(\bm{F}^{L,t}|\bm{F}^{0,t},\{\bm{k}^{l,h.t}\}_{l=1, h=1}^{L,H})}[\log p(\bm{Y}^{t}|\bm{F}^{L,t})]\\
     -\text{KL}&(q_t(\bm{F}^{L,t}|\bm{F}^{0,t},\{\bm{k}^{l,h,t}\}_{l=1, h=1}^{L,H})||q_{t-1}(\bm{F}^{L,t-1}|\bm{F}^{0,t-1},\{\bm{k}^{l,h,t-1}\}_{l=1, h=1}^{L,H})),
\end{split}
\end{equation}
where $q_{t-1}$ obtained from the previous time step becomes the prior appearing in the KL regularizer term for the current step inference. Notice that both $q_t$ and $q_{t-1}$ are deep sparse GPs, but with different set of keys, so the KL regularizer in the above online ELBO can be simplified to KL terms evaluated purely based on inducing keys and values again due to the prior conditional matching condition, similar to the online ELBO for OHSVGP (Eq.~\ref{eq:online_elbo}). For autoregressive prediction tasks, the inducing keys and values at time $t$ ($\bm{k}^{l,h,t}$ and $\bm{v}^{l,h,t}$) are obtained by augmenting the previous inducing keys and values ($\bm{k}^{l,h,t-1}$ and $\bm{v}^{l,h,t-1}$) with an additional key and value computed based on the token generated most recently (a procedure known as Key-Value caching).

\paragraph{More efficient SGPA.}
Compared with standard Transformers, SGPA suffers from higher computational cost due to the additional computation of the posterior covariance. In Chapter~\ref{cha:sgpa}, we already used decoupled SVGP techniques to construct posterior covariance purely based on a fixed size of global inducing keys shared across all sequence data points. However, in order to scale SGPA up to large Transformers deployed in modern applications such as large language models, further computational cost reduction is indispensable. For example, random Fourier features (RFFs; \citep{rahimi_random_2007}) can be used to approximate the kernel to accelerate the computation of kernel matrices \citep{reid2023simplex}. Additionally, the architecture of SGPA may be tweaked to allow better inference schemes. There are equivalent yet cheaper ways to compute posterior for SGPA based on some special class of kernels. For instance, instead of computing the full covariance, posterior sampling from GP based on Markovian kernels can be computed via smoothing in a linear stochastic differential equation (SDE) framework \citep{Hartikainen2010kalman}.

\paragraph{HiPPO-SGPA.}
For decoder-based Transformers where the number of key-value pairs is increasing as we generate new tokens, it becomes tricky to construct suitable posterior covariance that balances between the flexibility of modeling the adaptive knowledge in the growing context and computational efficiency. Standard SGPA constructs variational covariance based on the growing number ($T$) of amortized inducing keys so that its size also keeps growing, which makes it computationally expensive. Decoupled SGPA constructs variational covariance based on the fixed number ($M$) of static inducing keys globally shared across all sequences, which may not be flexible enough to summarize the global knowledge in the entire dataset of sequences, especially when the number of sequences is huge (e.g., the large corpus to train large language models). Ideally, for each sequence, we want to maintain a fixed number of adaptive inducing key-value pairs that are dynamically updated, and construct fixed-size adaptive variational covariance based on them. HiPPO-SVGP, introduced in Chapter~\ref{cha:hsgp}, provides a potential option to achieve this goal since HiPPO maintains a fixed size adaptive memory state to capture the dynamical history. Given a set of queries, we can use the HiPPO recurrence as introduced in Section~\ref{sec:multi_dim} to obtain interdomain attention matrices, $\bm{K}_{\bm{f_qu}}$, and with a new query during the autoregressive prediction, the interdomain attention matrices will be updated via another recurrence step to incorporate the information from the new query (Eq.~\ref{eq:multi_dim_hippo_kfu}). The final components for SGPA computation is the fixed-size adaptive variational values and covariance based on this fixed-size interdomain inducing keys. Instead of obtaining them with projection matrices applied over the sequence inputs of the attention module, which is of adaptive size $T$, as in standard Transformer, one potential solution is to obtain them by applying another RNN over the sequence inputs to dynamically output the required fixed-size variational parameters.

\paragraph{Building probabilistic SSM based on HiPPO-SVGP.}
HiPPO-SVGP may also be extended to build probabilistic SSM. Standard SSM layers leverage structured recurrence to obtain memory state summarizing the information in the sequence and then project the memory state to obtain the output \citep{gu_s4_2022, gu_mamba_2023}. In HiPPO-SVGP, the concept of memory state (or coefficient) is equivalent to the interdomain inducing points, and they are not deterministic but are distributed according to a variational Gaussian distribution $q(\bm{u})$. Based on this, we may build probabilistic SSM layers enabling uncertainty quantification in SSM models. Still, we need to figure out sensible ways to stack these probabilistic SSM layers together to build deep models for enhanced flexibility. Moreover, although the structure of the recurrence parameters are designed based on fixed HiPPO recurrence matrices, they are trainable parameters in state-of-the-art deep SSMs  \citep{gu_s4_2022, gu_mamba_2023}. Hence, we may consider tuning the recurrence parameters in HiPPO-SVGP as well. In this case, the trained recurrence parameters will be associated with some implicit function basis and measures whose analytic forms are unknown to us.

\paragraph{Better pseudo video design.}
As we showed in Chapter~\ref{cha:pseudovid}, better pseudo video design in pseudo video generation framework may improve the image generation quality. In general, since the pseudo-video framework contains image-generative models as special cases (i.e., $T=1$), intuitively there will be a configuration of hyperparameters (type of data augmentation, architecture etc.) that makes the performance better than image (1-frame) generation. However, as we mentioned in Section~\ref{sec:pseudo_conclusion}, pseudo-video framework makes the space of hyperparameters much larger, especially the design space of data augmentation for creating pseudo-videos is huge. For different types of models and architectures, the optimal data augmentation is in general not the same (e.g., Gaussian blur data augmentation empirically works better for C-ViVIT but achieves worse results than Gaussian noise for video diffusion in Chapter~\ref{cha:pseudovid}). While we have theoretically shown some intuition that high-order Markov data augmentation is preferable in autoregressive video generation models, finding optimal data augmentation or practical suggestions of what data augmentation is more suitable for other types of models remains an open research question. Nevertheless, we believe our analysis and empirical experiments have demonstrated a proof of concept that this is a promising framework with a lot of potential to improve generation performance.

\clearpage
\renewcommand\bibname{{References}}
\bibliography{references}

\begin{thebibliography}{}

\bibitem[Albergo et~al., 2023]{albergo2023stochastic}
Albergo, M.~S., Boffi, N.~M., and Vanden-Eijnden, E. (2023).
\newblock Stochastic interpolants: A unifying framework for flows and diffusions.
\newblock {\em arXiv preprint arXiv:2303.08797}.

\bibitem[Arbel et~al., 2023]{arbel2023primer}
Arbel, J., Pitas, K., Vladimirova, M., and Fortuin, V. (2023).
\newblock A primer on {Bayesian} neural networks: Review and debates.
\newblock {\em arXiv preprint arXiv:2309.16314}.

\bibitem[Ashman et~al., 2020]{ashman_sparse_2020}
Ashman, M., So, J., Tebbutt, W., Fortuin, V., Pearce, M., and Turner, R.~E. (2020).
\newblock Sparse {Gaussian} process variational autoencoders.
\newblock {\em arXiv preprint arXiv:2010.10177}.

\bibitem[Ba et~al., 2016]{ba2016layer}
Ba, J.~L., Kiros, J.~R., and Hinton, G.~E. (2016).
\newblock Layer normalization.
\newblock In {\em arXiv preprint arXiv:1607.06450}.

\bibitem[Bansal et~al., 2023]{bansal2023cold}
Bansal, A., Borgnia, E., Chu, H.-M., Li, J., Kazemi, H., Huang, F., Goldblum, M., Geiping, J., and Goldstein, T. (2023).
\newblock Cold diffusion: Inverting arbitrary image transforms without noise.
\newblock {\em Advances in Neural Information Processing Systems}, 36.

\bibitem[Bayes, 1763]{bayes1763essay}
Bayes, T. (1763).
\newblock Lii. an essay towards solving a problem in the doctrine of chances. {By} the late rev. {Mr. Bayes}, {FRS} communicated by {Mr. Price}, in a letter to {John Canton}, {AMFR S}.
\newblock {\em Philosophical transactions of the Royal Society of London}, 53:370--418.

\bibitem[Beal, 2003]{beal:vi2003}
Beal, M.~J. (2003).
\newblock {\em Variational algorithms for approximate {B}ayesian inference}.
\newblock PhD thesis, University of London.

\bibitem[Benita et~al., 2023]{benita2023diffar}
Benita, R., Elad, M., and Keshet, J. (2023).
\newblock Diffar: Denoising diffusion autoregressive model for raw speech waveform generation.
\newblock In {\em arXiv preprint arXiv:2310.01381}.

\bibitem[Blair et~al., 2013]{skillcraft1_master_table_dataset_272}
Blair, M., Thompson, J., Henrey, A., and Chen, B. (2013).
\newblock {SkillCraft1 Master Table Dataset}.
\newblock UCI Machine Learning Repository.
\newblock {DOI}: https://doi.org/10.24432/C5161N.

\bibitem[Blundell et~al., 2015]{blundell2015bbp}
Blundell, C., Cornebise, J., Kavukcuoglu, K., and Wierstra, D. (2015).
\newblock Weight uncertainty in neural networks.
\newblock In {\em International Conference on Machine Learning}.

\bibitem[Bradshaw et~al., 2017]{bradshaw2017adversarial}
Bradshaw, J., Matthews, A. G. d.~G., and Ghahramani, Z. (2017).
\newblock Adversarial examples, uncertainty, and transfer testing robustness in {Gaussian} process hybrid deep networks.
\newblock {\em arXiv preprint arXiv:1707.02476}.

\bibitem[Brooks et~al., 2011]{brooks2011mcmc}
Brooks, S., Gelman, A., Jones, G., and Meng, X.~L. (2011).
\newblock {\em Handbook of {Markov} chain {Monte Carlo}}.
\newblock CRC press.

\bibitem[Brown et~al., 2020]{brown2020language}
Brown, T.~B., Mann, B., Ryder, N., Subbiah, M., Kaplan, J.~D., Dhariwal, P., Neelakantan, A., Shyam, P., Sastry, G., Askell, A., Agarwal, S., Herbert-Voss, A., Krueger, G., Henighan, T., Child, R., Ramesh, A., Ziegler, D.~M., Wu, J., Winter, C., Hesse, C., Chen, M., Sigler, E., Litwin, M., Gray, S., Chess, B., Clark, J., Berner, C., McCandlish, S., Radford, A., Sutskever, I., and Amodei, D. (2020).
\newblock Language models are few-shot learners.
\newblock {\em Advances in neural information processing systems}, 33:1877--1901.

\bibitem[Bui et~al., 2017]{bui_streaming_2017}
Bui, T.~D., Nguyen, C.~V., and Turner, R.~E. (2017).
\newblock Streaming sparse {Gaussian} process approximations.
\newblock In {\em Advances in {Neural} {Information} {Processing} {Systems}}.

\bibitem[Bui and Turner, 2014]{bui_tree_2014}
Bui, T.~D. and Turner, R.~E. (2014).
\newblock Tree-structured {Gaussian} process approximations.
\newblock In {\em Advances in {Neural} {Information} {Processing} {Systems}}.

\bibitem[Burda et~al., 2015]{burda2015importance}
Burda, Y., Grosse, R., and Salakhutdinov, R. (2015).
\newblock Importance weighted autoencoders.
\newblock {\em arXiv preprint arXiv:1509.00519}.

\bibitem[Burt et~al., 2019]{burt_rates_2019}
Burt, D.~R., Rasmussen, C.~E., and van~der Wilk, M. (2019).
\newblock Rates of convergence for sparse variational {Gaussian} process regression.
\newblock In {\em International Conference on Machine Learning (ICML)}.

\bibitem[Campbell et~al., 2023]{campbell2023trans}
Campbell, A., Harvey, W., Weilbach, C., De~Bortoli, V., Rainforth, T., and Doucet, A. (2023).
\newblock Trans-dimensional generative modeling via jump diffusion models.
\newblock In {\em Advances in Neural Information Processing Systems}.

\bibitem[Casale et~al., 2018]{casale_gaussian_2018}
Casale, F.~P., Dalca, A., Saglietti, L., Listgarten, J., and Fusi, N. (2018).
\newblock Gaussian process prior variational autoencoders.
\newblock {\em Advances in neural information processing systems}, 31.

\bibitem[Cesa-Bianchi and Lugosi, 2006]{cesabianchi2006prediction}
Cesa-Bianchi, N. and Lugosi, G. (2006).
\newblock {\em Prediction, learning, and games}.
\newblock Cambridge University Press.

\bibitem[Chang et~al., 2022]{chang2022maskgit}
Chang, H., Zhang, H., Jiang, L., Liu, C., and Freeman, W.~T. (2022).
\newblock Maskgit: Masked generative image transformer.
\newblock In {\em Proceedings of the IEEE/CVF Conference on Computer Vision and Pattern Recognition}, pages 11315--11325.

\bibitem[Chang et~al., 2023]{chang_memory_2023}
Chang, P.~E., Verma, P., John, S., Solin, A., and Khan, M.~E. (2023).
\newblock Memory-based dual {Gaussian} processes for sequential learning.
\newblock In {\em International {Conference} on {Machine} {Learning}}.

\bibitem[Chen et~al., 2020a]{chen2020generative}
Chen, M., Radford, A., Child, R., Wu, J., Jun, H., Luan, D., and Sutskever, I. (2020a).
\newblock Generative pretraining from pixels.
\newblock In {\em International conference on machine learning}, pages 1691--1703. PMLR.

\bibitem[Chen et~al., 2018]{chen2018neural}
Chen, R.~T., Rubanova, Y., Bettencourt, J., and Duvenaud, D.~K. (2018).
\newblock Neural ordinary differential equations.
\newblock {\em Advances in neural information processing systems}, 31.

\bibitem[Chen et~al., 2020b]{chen2020simple}
Chen, T., Kornblith, S., Norouzi, M., and Hinton, G. (2020b).
\newblock A simple framework for contrastive learning of visual representations.
\newblock In {\em International conference on machine learning}, pages 1597--1607. PMLR.

\bibitem[Chen et~al., 2025a]{chen2025your}
Chen, W., Chen, W., Rastrelli, L., and Li, Y. (2025a).
\newblock Your image is secretly the last frame of a pseudo video.
\newblock In {\em Deep Generative Model in Machine Learning: Theory, Principle and Efficacy (DeLTa) Workshop at ICLR}.

\bibitem[Chen et~al., 2025b]{chen2025recurrent}
Chen, W., Kiyohara, N., Zhu, H. B.~H., Curran-Sebastian, J., Bhatt, S., and Li, Y. (2025b).
\newblock Recurrent memory for online interdomain {Gaussian} processes.
\newblock In {\em Advances in Neural Information Processing Systems (NeurIPS)}.

\bibitem[Chen et~al., 2024]{chen2024post}
Chen, W., Klochkov, Y., and Liu, Y. (2024).
\newblock Post-hoc bias scoring is optimal for fair classification.
\newblock In {\em International Conference on Learning Representations (ICLR)}.

\bibitem[Chen et~al., 2025c]{chen2025bayesian}
Chen, W., Li, B., Zhang, R., and Li, Y. (2025c).
\newblock Bayesian computation in deep learning.
\newblock In {\em arXiv preprint arXiv:2502.18300}.

\bibitem[Chen and Li, 2023]{chen2023calibrating}
Chen, W. and Li, Y. (2023).
\newblock Calibrating transformers via sparse {Gaussian} processes.
\newblock In {\em International Conference on Learning Representations (ICLR)}.

\bibitem[Cheng and Boots, 2017]{cheng2017variational}
Cheng, C.-A. and Boots, B. (2017).
\newblock Variational inference for {Gaussian} process models with linear complexity.
\newblock In {\em Advances in Neural Processing Information Systems}.

\bibitem[Child, 2021]{child2020very}
Child, R. (2021).
\newblock Very deep {VAE}s generalize autoregressive models and can outperform them on images.
\newblock In {\em International Conference on Learning Representations}.

\bibitem[Cinquin et~al., 2021]{cinquin2021bayestrans}
Cinquin, T., Immer, A., Horn, M., and Fortuin, V. (2021).
\newblock Pathologies in priors and inference for {Bayesian} transformers.
\newblock In {\em NeurIPS 2021 "I can't believe it's not better" workshop}.

\bibitem[Coker et~al., 2022]{coker_aistats2022}
Coker, B., Bruinsma, W.~P., Burt, D.~R., Pan, W., and Doshi-Velez, F. (2022).
\newblock Wide mean-field variational {Bayesian} neural networks ignore the data.
\newblock In {\em International Conference on Artificial Intelligence and Statistics}.

\bibitem[Csató and Opper, 2002]{csato2002sparse}
Csató, L. and Opper, M. (2002).
\newblock Sparse on-line {Gaussian} processes.
\newblock {\em Neural Computation}, 14(3):641--668.

\bibitem[Damianou and Lawrence, 2013]{damia2013dgp}
Damianou, A.~C. and Lawrence, N.~D. (2013).
\newblock Deep {Gaussian} processes.
\newblock In {\em International Conference on Artificial Intelligence and Statistics}.

\bibitem[Dao and Gu, 2024]{dao_mamba2_2024}
Dao, T. and Gu, A. (2024).
\newblock Transformers are {SSM}s: Generalized models and efficient algorithms through structured state space duality.
\newblock In {\em International Conference on Machine Learning (ICML)}.

\bibitem[Devlin et~al., 2019]{devlin2018bert}
Devlin, J., Chang, M.-W., Lee, K., and Toutanova, K. (2019).
\newblock {BERT}: Pre-training of deep bidirectional transformers for language understanding.
\newblock In {\em Conference of the North American Chapter of the Association for Computational Linguistics}.

\bibitem[Dosovitskiy et~al., 2021]{dosovitskiy2020image}
Dosovitskiy, A., Beyer, L., Kolesnikov, A., Weissenborn, D., Zhai, X., Unterthiner, T., Dehghani, M., Minderer, M., Heigold, G., and Gelly, S. (2021).
\newblock An image is worth 16x16 words: Transformers for image recognition at scale.
\newblock In {\em International Conference on Learning Representations}.

\bibitem[Dutordoir et~al., 2020]{dutordoir2020sparse}
Dutordoir, V., Durrande, N., and Hensman, J. (2020).
\newblock Sparse {Gaussian} processes with spherical harmonic features.
\newblock In {\em International Conference on Machine Learning (ICML)}.

\bibitem[Duvenaud et~al., 2011]{add2011david}
Duvenaud, D.~K., Nickisch, H., and Rasmussen, C. (2011).
\newblock Additive {Gaussian} processes.
\newblock In {\em Advances in Neural Information Processing Systems}.

\bibitem[Dwivedi et~al., 2020]{dwivedi2020benchmarkgnns}
Dwivedi, V.~P., Joshi, C.~K., Luu, A.~T., Laurent, T., Bengio, Y., and Bresson, X. (2020).
\newblock Benchmarking graph neural networks.
\newblock {\em arXiv preprint arXiv:2003.00982}.

\bibitem[Fan et~al., 2020]{bam}
Fan, X.~F., Zhang, S., Chen, B., and Zhou, M. (2020).
\newblock Bayesian attention modules.
\newblock In {\em Advances in Neural Information Processing Systems}.

\bibitem[Fang et~al., 2023]{fang2023eva}
Fang, Y., Wang, W., Xie, B., Sun, Q., Wu, L., Wang, X., Huang, T., Wang, X., and Cao, Y. (2023).
\newblock Eva: Exploring the limits of masked visual representation learning at scale.
\newblock In {\em Proceedings of the IEEE/CVF Conference on Computer Vision and Pattern Recognition}, pages 19358--19369.

\bibitem[Foong et~al., 2020]{foong2020expressive}
Foong, A. Y.~K., Burt, D., Li, Y., and Turner, R. (2020).
\newblock On the expressiveness of approximate inference in {Bayesian} neural networks.
\newblock In {\em Advances in Neural Information Processing Systems}.

\bibitem[Fortuin et~al., 2020]{fortuin_gpvae_2020}
Fortuin, V., Baranchuk, D., Rätsch, G., and Mandt, S. (2020).
\newblock {GP-VAE}: {Deep} probabilistic time series imputation.
\newblock In {\em International {Conference} on {Artificial} {Intelligence} and {Statistics}}, pages 1651--1661. PMLR.

\bibitem[Gal, 2016]{gal2016thesis}
Gal, Y. (2016).
\newblock Uncertainty in deep learning.
\newblock {\em PhD dissertation, University of Cambridge}.

\bibitem[Gal and Ghahramani, 2016]{gal2016mcdrop}
Gal, Y. and Ghahramani, Z. (2016).
\newblock Dropout as a {Bayesian} approximation: Representing model uncertainty in deep learning.
\newblock In {\em International Conference on Machine Learning}.

\bibitem[Gal and Turner, 2015]{gal_improving_2015}
Gal, Y. and Turner, R.~E. (2015).
\newblock Improving the {Gaussian} process sparse spectrum approximation by representing uncertainty in frequency inputs.
\newblock In {\em International Conference on Machine Learning (ICML)}.

\bibitem[Ganin et~al., 2016]{ganin2016domain}
Ganin, Y., Ustinova, E., Ajakan, H., Germain, P., Larochelle, H., Laviolette, F., March, M., and Lempitsky, V. (2016).
\newblock Domain-adversarial training of neural networks.
\newblock {\em Journal of Machine Learning Research}, 17(59):1--35.

\bibitem[Garifolo et~al., 1993]{Garofolo1993timit}
Garifolo, J., Lamel, L., Fisher, W., Fiscus, J., Pallett, D., Dahlgren, N., and Zue, V. (1993).
\newblock {TIMIT} acoustic-phonetic continuous speech corpus {LDC93S1}.
\newblock In {\em Philadelphia: Linguistic Data Consortium}.

\bibitem[Gat et~al., 2024]{gat2024discrete}
Gat, I., Remez, T., Shaul, N., Kreuk, F., Chen, R.~T., Synnaeve, G., Adi, Y., and Lipman, Y. (2024).
\newblock Discrete flow matching.
\newblock {\em arXiv preprint arXiv:2407.15595}.

\bibitem[Gelfand, 2000]{gelfand2000gibbs}
Gelfand, A.~E. (2000).
\newblock Gibbs sampling.
\newblock {\em Journal of the American statistical Association}, 95(452):1300--1304.

\bibitem[Gershman and Goodman, 2014]{gershman2014amortized}
Gershman, S.~J. and Goodman, N.~D. (2014).
\newblock Amortized inference in probabilistic reasoning.
\newblock {\em Proceedings of the Annual Meeting of the Cognitive Science Society}, 36.

\bibitem[Gneiting et~al., 2007]{gneit2007proper}
Gneiting, T., Balabdaoui, F., and Raftery, A.~E. (2007).
\newblock Probabilistic forecasts, calibration and sharpness.
\newblock {\em Journal of the Royal Statistical Society: Series B (Statistical Methodology)}, 69(2):243--268.

\bibitem[Graves et~al., 2013a]{graves2013hybrid}
Graves, A., Jaitly, N., and Mohamed, A.-r. (2013a).
\newblock Hybrid speech recognition with deep bidirectional {LSTM}.
\newblock In {\em 2013 IEEE workshop on automatic speech recognition and understanding}, pages 273--278. IEEE.

\bibitem[Graves et~al., 2013b]{graves2013speech}
Graves, A., Mohamed, A.-r., and Hinton, G.~E. (2013b).
\newblock Speech recognition with deep recurrent neural networks.
\newblock In {\em 2013 IEEE international conference on acoustics, speech and signal processing}.

\bibitem[Gu and Dao, 2023]{gu_mamba_2023}
Gu, A. and Dao, T. (2023).
\newblock Mamba: {Linear}-time sequence modeling with selective state spaces.
\newblock {\em arXiv preprint arXiv:2312.00752}.

\bibitem[Gu et~al., 2020]{gu_hippo_2020}
Gu, A., Dao, T., Ermon, S., Rudra, A., and Ré, C. (2020).
\newblock {HiPPO}: {Recurrent} memory with optimal polynomial projections.
\newblock In {\em Advances in {Neural} {Information} {Processing} {Systems}}.

\bibitem[Gu et~al., 2022]{gu_s4_2022}
Gu, A., Goel, K., and R\'e, C. (2022).
\newblock Efficiently modeling long sequences with structured state spaces.
\newblock In {\em The International Conference on Learning Representations}.

\bibitem[Gu et~al., 2023]{gu_httyh_2023}
Gu, A., Johnson, I., Timalsina, A., Rudra, A., and Re, C. (2023).
\newblock How to train your {HIPPO}: State space models with generalized orthogonal basis projections.
\newblock In {\em International Conference on Learning Representations}.

\bibitem[Gulrajani et~al., 2017]{gulrajani2016pixelvae}
Gulrajani, I., Kumar, K., Ahmed, F., Taiga, A.~A., Visin, F., Vazquez, D., and Courville, A. (2017).
\newblock Pixel{VAE}: A latent variable model for natural images.
\newblock In {\em International Conference on Learning Representations}.

\bibitem[Guo et~al., 2017]{guo2017calibration}
Guo, C., Pleiss, G., Sun, Y., and Weinberger, K.~Q. (2017).
\newblock On calibration of modern neural networks.
\newblock In {\em International Conference on Machine Learning}.

\bibitem[Hartikainen and Särkkä, 2010]{Hartikainen2010kalman}
Hartikainen, J. and Särkkä, S. (2010).
\newblock Kalman filtering and smoothing solutions to temporal {Gaussian} process regression models.
\newblock In {\em IEEE International Workshop on Machine Learning for signal processing}.

\bibitem[Harvey et~al., 2022]{harvey2022flexible}
Harvey, W., Naderiparizi, S., Masrani, V., Weilbach, C., and Wood, F. (2022).
\newblock Flexible diffusion modeling of long video.
\newblock In {\em Advances in neural information processing systems}.

\bibitem[Hawryluk et~al., 2021]{hawryluk2021gaussian}
Hawryluk, I., Hoeltgebaum, H., Mishra, S., Miscouridou, X., Schnekenberg, R.~P., Whittaker, C., Vollmer, M., Flaxman, S., Bhatt, S., and Mellan, T.~A. (2021).
\newblock Gaussian process nowcasting: application to covid-19 mortality reporting.
\newblock In {\em Uncertainty in Artificial Intelligence}, pages 1258--1268. PMLR.

\bibitem[He et~al., 2022]{he2022masked}
He, K., Chen, X., Xie, S., Li, Y., Doll{\'a}r, P., and Girshick, R. (2022).
\newblock Masked autoencoders are scalable vision learners.
\newblock In {\em Proceedings of the IEEE/CVF conference on computer vision and pattern recognition}, pages 16000--16009.

\bibitem[He et~al., 2015]{https://doi.org/10.48550/arxiv.1512.03385}
He, K., Zhang, X., Ren, S., and Sun, J. (2015).
\newblock Deep residual learning for image recognition.
\newblock {\em arXiv:1512.03385}.

\bibitem[Hendrycks and Dietterich, 2019]{cifarc}
Hendrycks, D. and Dietterich, T. (2019).
\newblock Benchmarking neural network robustness to common corruptions and perturbations.
\newblock In {\em International Conference on Learning Representations}.

\bibitem[Hensman et~al., 2018]{hensman2018variational}
Hensman, J., Durrande, N., and Solin, A. (2018).
\newblock Variational {Fourier} features for {Gaussian} processes.
\newblock {\em Journal of Machine Learning Research}, 18(151):1--52.

\bibitem[Hensman et~al., 2013]{hensman2013gpbig}
Hensman, J., Fusi, N., and Lawrence, N.~D. (2013).
\newblock Gaussian processes for big data.
\newblock In {\em The Conference on Uncertainty in Artificial Intelligence}.

\bibitem[Hensman et~al., 2015]{hensman2015sgpclass}
Hensman, J., Matthews, A. G. d.~G., and Ghahramani, Z. (2015).
\newblock Scalable variational {Gaussian} process classification.
\newblock In {\em International Conference on Artificial Intelligence and Statistics}.

\bibitem[Hersbach et~al., 2023]{cds_era5_single_levels_2023}
Hersbach, H., Bell, B., Berrisford, P., Biavati, G., Hor\'{a}nyi, A., Mu\~{n}oz Sabater, J., Nicolas, J., Peubey, C., Radu, R., Rozum, I., Schepers, D., Simmons, A., Soci, C., Dee, D., and Th\'{e}paut, J.~N. (2023).
\newblock {ERA5} hourly data on single levels from 1940 to present.
\newblock Copernicus Climate Change Service (C3S) Climate Data Store (CDS).
\newblock \url{https://doi.org/10.24381/cds.adbb2d47}.

\bibitem[Heusel et~al., 2017]{heusel2017gans}
Heusel, M., Ramsauer, H., Unterthiner, T., Nessler, B., and Hochreiter, S. (2017).
\newblock Gans trained by a two time-scale update rule converge to a local nash equilibrium.
\newblock In {\em Advances in neural information processing systems}.

\bibitem[Ho et~al., 2020]{ho2020denoising}
Ho, J., Jain, A., and Abbeel, P. (2020).
\newblock Denoising diffusion probabilistic models.
\newblock {\em Advances in neural information processing systems}.

\bibitem[Hoogeboom and Salimans, 2023]{hoogeboom2023blurring}
Hoogeboom, E. and Salimans, T. (2023).
\newblock Blurring diffusion models.
\newblock In {\em International Conference on Learning Representations}.

\bibitem[Horn and Johnson, 1991]{horn1991topics}
Horn, R.~A. and Johnson, C.~R. (1991).
\newblock {\em Topics in Matrix Analysis}.
\newblock Cambridge University Press.

\bibitem[Huang et~al., 2021]{huang2021variational}
Huang, C.-W., Lim, J.~H., and Courville, A.~C. (2021).
\newblock A variational perspective on diffusion-based generative models and score matching.
\newblock {\em Advances in Neural Information Processing Systems}, 34:22863--22876.

\bibitem[Jafrasteh et~al., 2022]{daniel2022inputde}
Jafrasteh, B., Villacampa-Calvo, C., and Hernandez-Lobato, D. (2022).
\newblock Input dependent sparse {Gaussian} processes.
\newblock In {\em International Conference on Machine Learning}.

\bibitem[Jayasekera et~al., 2025]{jayasekera2025variational}
Jayasekera, I.~S., Si, J., Valdettaro, F., Chen, W., Faisal, A.~A., and Li, Y. (2025).
\newblock Variational uncertainty decomposition for in-context learning.
\newblock In {\em Advances in Neural Information Processing Systems (NeurIPS)}.

\bibitem[Jazbec et~al., 2021]{jazbec_scalable_2021}
Jazbec, M., Ashman, M., Fortuin, V., Pearce, M., Mandt, S., and Rätsch, G. (2021).
\newblock Scalable {Gaussian} process variational autoencoders.
\newblock In {\em International {Conference} on {Artificial} {Intelligence} and {Statistics}}, pages 3511--3519. PMLR.

\bibitem[Johnson et~al., 2016]{johnson2016perceptual}
Johnson, J., Alahi, A., and Fei-Fei, L. (2016).
\newblock Perceptual losses for real-time style transfer and super-resolution.
\newblock In {\em Computer Vision--ECCV 2016: 14th European Conference, Amsterdam, The Netherlands, October 11-14, 2016, Proceedings, Part II 14}, pages 694--711. Springer.

\bibitem[Jordan et~al., 1999]{jordan1999vi}
Jordan, M.~I., Ghahramani, Z., Jaakkola, T.~S., and Saul, L.~K. (1999).
\newblock An introduction to variational methods for graphical models.
\newblock {\em Machine Learning}, 37(2):183--233.

\bibitem[Jumper et~al., 2021]{jumper2021highly}
Jumper, J., Evans, R., Pritzel, A., Green, T., Figurnov, M., Ronneberger, O., Tunyasuvunakool, K., Bates, R., Žídek, A., Potapenko, A., Bridgland, A., et~al. (2021).
\newblock Highly accurate protein structure prediction with alphafold.
\newblock {\em Nature}, 596:583--589.

\bibitem[Jung et~al., 2025a]{jung2025compact2}
Jung, Y., Lee, H., Chen, W., Möllenhoff, T., Li, Y., Lee, J., and Khan, M.~E. (2025a).
\newblock Compact memory for continual logistic regression.
\newblock In {\em Advances in Neural Information Processing Systems (NeurIPS)}.

\bibitem[Jung et~al., 2025b]{jung2025compact}
Jung, Y., Lee, H., Chen, W., Möllenhoff, T., Li, Y., Lee, J., and Khan, M.~E. (2025b).
\newblock Compact memory for k-prior based continual learning.
\newblock In {\em Symposium on Advances in Approximate {Bayesian} Inference (AABI)}.

\bibitem[Kapoor et~al., 2021]{kapoor2021variational}
Kapoor, S., Karaletsos, T., and Bui, T.~D. (2021).
\newblock Variational auto-regressive {Gaussian} processes for continual learning.
\newblock In {\em International {Conference} on {Machine} {Learning}}.

\bibitem[Karras et~al., 2020]{karras2020analyzing}
Karras, T., Laine, S., Aittala, M., Hellsten, J., Lehtinen, J., and Aila, T. (2020).
\newblock Analyzing and improving the image quality of stylegan.
\newblock In {\em Proceedings of the IEEE/CVF conference on computer vision and pattern recognition}, pages 8110--8119.

\bibitem[Khemakhem et~al., 2020]{khemakhem2020variational}
Khemakhem, I., Kingma, D., Monti, R., and Hyvarinen, A. (2020).
\newblock Variational autoencoders and nonlinear {ICA}: A unifying framework.
\newblock In {\em International conference on artificial intelligence and statistics}, pages 2207--2217. PMLR.

\bibitem[Kim et~al., 2022]{kim2021transformers}
Kim, J., Nguyen, T.~D., Min, S., Cho, S., Lee, M., Lee, H., and Hong, S. (2022).
\newblock Pure transformers are powerful graph learners.
\newblock {\em arXiv}, abs/2207.02505.

\bibitem[Kingma et~al., 2021]{kingma2021variational}
Kingma, D., Salimans, T., Poole, B., and Ho, J. (2021).
\newblock Variational diffusion models.
\newblock {\em Advances in neural information processing systems}, 34:21696--21707.

\bibitem[Kingma and Ba, 2015]{kingma2015adam}
Kingma, D.~P. and Ba, J. (2015).
\newblock Adam: A method for stochastic optimization.
\newblock In {\em International Conference on Learning Representations}.

\bibitem[Kingma and Welling, 2014]{welling2014auto}
Kingma, D.~P. and Welling, M. (2014).
\newblock Auto-encoding variational {Bayes}.
\newblock In {\em International Conference on Learning Representations}.

\bibitem[Kirkpatrick et~al., 2017]{kirkpatrick2017overcoming}
Kirkpatrick, J., Pascanu, R., Rabinowitz, N., Veness, J., Desjardins, G., Rusu, A.~A., Milan, K., Quan, J., Ramalho, T., Grabska-Barwinska, A., et~al. (2017).
\newblock Overcoming catastrophic forgetting in neural networks.
\newblock {\em Proceedings of the national academy of sciences}, 114(13):3521--3526.

\bibitem[Klambauer et~al., 2017]{klambauer2017self}
Klambauer, G., Unterthiner, T., Mayr, A., and Hochreiter, S. (2017).
\newblock Self-normalizing neural networks.
\newblock In {\em Advances in Neural Information Processing Systems}.

\bibitem[Kristiadi et~al., 2020]{kris2020bayes}
Kristiadi, A., Hein, M., and Hennig, P. (2020).
\newblock Being {Bayesian}, even just a bit, fixes overconfidence in relu networks.
\newblock In {\em International Conference on Machine Learning}.

\bibitem[Krizhevsky et~al., 2009]{cifar}
Krizhevsky, A., Nair, V., and Hinton, G. (2009).
\newblock Cifar-10 and {CIFAR}-100 datasets.

\bibitem[Krizhevsky et~al., 2012]{krizhevsky2012imagenet}
Krizhevsky, A., Sutskever, I., and Hinton, G.~E. (2012).
\newblock Imagenet classification with deep convolutional neural networks.
\newblock In {\em Advances in Neural Information Processing Systems}.

\bibitem[Lakshminarayanan et~al., 2017]{bala2017ensemble}
Lakshminarayanan, B., Pritzel, A., and Blundell, C. (2017).
\newblock Simple and scalable predictive uncertainty estimation using deep ensembles.
\newblock In {\em Advances in neural information processing systems}.

\bibitem[Lean, 2004]{lean2004solar}
Lean, J. (2004).
\newblock Solar irradiance reconstruction.
\newblock In {\em Data contribution series \# 2004-035, IGBP PAGES/World Data Center for Paleoclimatology NOAA/NGDC Paleoclimatology Program, Boulder, CO, USA}.

\bibitem[LeCun and Hinton, 2015]{lecun2015dl}
LeCun, Y.and~Bengio, Y. and Hinton, G. (2015).
\newblock Deep learning.
\newblock {\em Nature}, 521:436--444.

\bibitem[Leibfried et~al., 2020]{leibfried_tutorial_2020}
Leibfried, F., Dutordoir, V., John, S., and Durrande, N. (2020).
\newblock A tutorial on sparse {Gaussian} processes and variational inference.
\newblock {\em arXiv preprint arXiv:2012.13962}.

\bibitem[Li et~al., 2024]{li2024videomamba}
Li, K., Li, X., Wang, Y., He, Y., Wang, Y., Wang, L., and Qiao, Y. (2024).
\newblock Videomamba: State space model for efficient video understanding.
\newblock In {\em arXiv preprint arXiv:2403.06977}.

\bibitem[Li, 2018]{li2018approx}
Li, Y. (2018).
\newblock Approximate inference: New visions.
\newblock {\em PhD dissertation, University of Cambridge}.

\bibitem[Lipman et~al., 2023]{lipman2022flow}
Lipman, Y., Chen, R. T.~Q., Ben-Hamu, H., Nickel, M., and Le, M. (2023).
\newblock Flow matching for generative modeling.
\newblock In {\em The Eleventh International Conference on Learning Representations}.

\bibitem[Liu et~al., 2020]{liu2020sngp}
Liu, J.~Z., Lin, Z., Padhy, S., Tran, D., Bedrax-Weiss, T., and Lakshminarayanan, B. (2020).
\newblock Simple and principled uncertainty estimation with deterministic deep learning via distance awareness.
\newblock In {\em Advances in Neural Information Processing Systems}.

\bibitem[Liu et~al., 2022]{liu2022flow}
Liu, X., Gong, C., and Liu, Q. (2022).
\newblock Flow straight and fast: Learning to generate and transfer data with rectified flow.
\newblock {\em arXiv preprint arXiv:2209.03003}.

\bibitem[Liu et~al., 2021]{liu2021self}
Liu, X., Zhang, F., Hou, Z., Mian, L., Wang, Z., Zhang, J., and Tang, J. (2021).
\newblock Self-supervised learning: Generative or contrastive.
\newblock {\em IEEE transactions on knowledge and data engineering}, 35(1):857--876.

\bibitem[Liu et~al., 2015]{liu2015deep}
Liu, Z., Luo, P., Wang, X., and Tang, X. (2015).
\newblock Deep learning face attributes in the wild.
\newblock In {\em International Conference on Computer Vision (ICCV)}.

\bibitem[Locatello et~al., 2019]{locatello2019challenging}
Locatello, F., Bauer, S., Lucic, M., Rätsch, G., Gelly, S., Schölkopf, B., and Bachem, O. (2019).
\newblock Challenging common assumptions in the unsupervised learning of disentangled representations.
\newblock In {\em International Conference on Machine Learning}.

\bibitem[Lou et~al., 2024]{lou2024discrete}
Lou, A., Meng, C., and Ermon, S. (2024).
\newblock Discrete diffusion modeling by estimating the ratios of the data distribution.
\newblock In {\em International Conference on Machine Learning (ICML)}.

\bibitem[Lázaro-Gredilla and Figueiras-Vidal, 2009]{lazaro-gredilla_inter-domain_2009}
Lázaro-Gredilla, M. and Figueiras-Vidal, A. (2009).
\newblock Inter-domain {Gaussian} processes for sparse inference using inducing features.
\newblock In {\em Advances in {Neural} {Information} {Processing} {Systems}}.

\bibitem[Ma and Hern{\'a}ndez-Lobato, 2021]{ma2021functional}
Ma, C. and Hern{\'a}ndez-Lobato, J.~M. (2021).
\newblock Functional variational inference based on stochastic process generators.
\newblock In {\em Advances in Neural Information Processing Systems}.

\bibitem[Ma et~al., 2015]{ma2015recipe}
Ma, Y., Chen, T., and Fox, E. (2015).
\newblock A complete recipe for stochastic gradient {MCMC}.
\newblock In {\em Advances in Neural Information Processing Systems}.

\bibitem[Maal{\o}e et~al., 2019]{maaloe2019biva}
Maal{\o}e, L., Fraccaro, M., Li{\'e}vin, V., and Winther, O. (2019).
\newblock Biva: A very deep hierarchy of latent variables for generative modeling.
\newblock {\em Advances in neural information processing systems}, 32.

\bibitem[Maas et~al., 2011]{imdb}
Maas, A.~L., Daly, R.~E., Pham, P.~T., Huang, D.;~Ng, A.~Y., and Potts, C. (2011).
\newblock Learning word vectors for sentiment analysis.
\newblock In {\em NAACL-HLT}.

\bibitem[Maddox et~al., 2021]{maddox_conditioning_2021}
Maddox, W.~J., Stanton, S., and Wilson, A.~G. (2021).
\newblock Conditioning sparse variational {Gaussian} processes for online decision-making.
\newblock In {\em Advances in {Neural} {Information} {Processing} {Systems}}.

\bibitem[Matthews, 1975]{mcc}
Matthews, B.~W. (1975).
\newblock Comparison of the predicted and observed secondary structure of t4 phage lysozyme.
\newblock {\em Biochimica et Biophysica Acta (BBA)-Protein Structure, 405(2): 442–451}.

\bibitem[McCulloch and Pitts, 1943]{mcCulloch1943logical}
McCulloch, W.~S. and Pitts, W. (1943).
\newblock A logical calculus of the ideas immanent in nervous activity.
\newblock {\em The bulletin of mathematical biophysics}, 5:115--133.

\bibitem[Mercer, 1909]{mercer1909functions}
Mercer, J. (1909).
\newblock Functions of positive and negative type and their connection with the theory of integral equations.
\newblock {\em Philosophical Transactions of the Royal Society}, A(209):415--446.

\bibitem[Metropolis et~al., 1953]{metropolis1953equation}
Metropolis, N., Rosenbluth, A.~W., Rosenbluth, M.~N., and Teller, A.~H. (1953).
\newblock Equation of state calculations by fast computing machines.
\newblock {\em The journal of chemical physics 21.6}.
\newblock 1087--1092.

\bibitem[Mukhoti et~al., 2020]{mukhoti2020calibrating}
Mukhoti, J., Kulharia, V., Sanyal, A., Golodetz, S., Torr, P.~H., and Dokania, P.~K. (2020).
\newblock Calibrating deep neural networks using focal loss.
\newblock In {\em Advances in Neural Information Processing Systems}.

\bibitem[Naeini et~al., 2015]{Pakdaman_Naeini_Cooper_Hauskrecht_2015}
Naeini, M.~P., Cooper, G.~F., and Hauskrecht, M. (2015).
\newblock Obtaining well calibrated probabilities using {Bayesian} binning.
\newblock {\em Proceedings of the AAAI Conference on Artificial Intelligence}, 29.

\bibitem[Nair and Hinton, 2010]{hinton2010relu}
Nair, V. and Hinton, G.~E. (2010).
\newblock Rectified linear units improve restricted {Boltzmann} machines.
\newblock In {\em International Conference on Machine Learning}.

\bibitem[Nguyen et~al., 2018]{nguyen2018variational}
Nguyen, C.~V., Li, Y., Bui, T.~D., and Turner, R.~E. (2018).
\newblock Variational continual learning.
\newblock In {\em International Conference on Learning Representations}.

\bibitem[Nichol and Dhariwal, 2021]{nichol2021improved}
Nichol, A. and Dhariwal, P. (2021).
\newblock Improved denoising diffusion probabilistic models.
\newblock {\em arXiv preprint arXiv:2102.09672}.

\bibitem[OpenAI et~al., 2024a]{openai2024gpt4ocard}
OpenAI, :, Hurst, A., Lerer, A., Goucher, A.~P., Perelman, A., Ramesh, A., Clark, A., Ostrow, A., Welihinda, A., Hayes, A., Radford, A., et~al. (2024a).
\newblock Gpt-4o system card.
\newblock In {\em arXiv preprint arXiv:2410.21276}.

\bibitem[OpenAI et~al., 2024b]{openai2024gpt4technicalreport}
OpenAI, Achiam, J., Adler, S., Agarwal, S., Ahmad, L., Akkaya, I., Aleman, F.~L., Almeida, D., et~al. (2024b).
\newblock Gpt-4 technical report.
\newblock In {\em arXiv preprint arXiv:2303.08774}.

\bibitem[Papamarkou et~al., 2024]{papamarkou2024positionbayesiandeeplearning}
Papamarkou, T., Skoularidou, M., Palla, K., Aitchison, L., Arbel, J., Dunson, D., Filippone, M., Fortuin, V., et~al. (2024).
\newblock Position: {Bayesian} deep learning is needed in the age of large scale ai.
\newblock {\em arXiv preprint arXiv:2402.00809}.

\bibitem[Paszke et~al., 2019]{NEURIPS2019_9015}
Paszke, A., Gross, S., Massa, F., Lerer, A., Bradbury, J., Chanan, G., Killeen, T., Lin, Z., Gimelshein, N., Antiga, L., Desmaison, A., Kopf, A., Yang, E., DeVito, Z., Raison, M., Tejani, A., Chilamkurthy, S., Steiner, B., Fang, L., Bai, J., and Chintala, S. (2019).
\newblock Pytorch: An imperative style, high-performance deep learning library.
\newblock In {\em Advances in Neural Information Processing Systems 32}, pages 8024--8035. Curran Associates, Inc.

\bibitem[Peters et~al., 2018]{elmo2018}
Peters, M., Neumann, M., Iyyer, M., Gardner, M., Clark, C., Lee, K., and Zettlemoyer, L. (2018).
\newblock Deep contextualized word representations.
\newblock In {\em Conference of the North American Chapter of the Association for Computational Linguistics}.

\bibitem[Quan and Li, 2024]{quan2024multichannel}
Quan, C. and Li, X. (2024).
\newblock Multichannel long-term streaming neural speech enhancement for static and moving speakers.
\newblock In {\em arXiv preprint arXiv:2403.07675}.

\bibitem[Rahimi and Recht, 2007]{rahimi_random_2007}
Rahimi, A. and Recht, B. (2007).
\newblock Random features for large-scale kernel machines.
\newblock In {\em Advances in {Neural} {Information} {Processing} {Systems}}.

\bibitem[Rasmussen and Williams, 2006]{rasmussen2006gp}
Rasmussen, C.~E. and Williams, C. K.~I. (2006).
\newblock {\em Gaussian processes for machine learning}.
\newblock The MIT Press.
\newblock {ISBN:} 0-262-18253-X.

\bibitem[Reid et~al., 2023]{reid2023simplex}
Reid, I., Choromanski, K., Likhosherstov, V., and Weller, A. (2023).
\newblock Simplex random features.
\newblock In {\em International Conference on Machine Learning}.

\bibitem[Rezende et~al., 2014]{rezende:vae2014}
Rezende, D.~J., Mohamed, S., and Wierstra, D. (2014).
\newblock Stochastic backpropagation and approximate inference in deep generative models.
\newblock In {\em Proceedings of the 31st International Conference on Machine Learning}, pages 1278--1286.

\bibitem[Rissanen et~al., 2023]{rissanen2022generative}
Rissanen, S., Heinonen, M., and Solin, A. (2023).
\newblock Generative modelling with inverse heat dissipation.
\newblock {\em International Conference on Learning Representations}.

\bibitem[Ritter et~al., 2018]{ritter2018online}
Ritter, H., Botev, A., and Barber, D. (2018).
\newblock Online structured laplace approximations for overcoming catastrophic forgetting.
\newblock In {\em Advances in Neural Information Processing Systems}, volume~31.

\bibitem[Ritter et~al., 2021]{ritter2021sparse}
Ritter, H., Kukla, M., Zhang, C., and Li, Y. (2021).
\newblock Sparse uncertainty representation in deep learning with inducing weights.
\newblock {\em Advances in Neural Information Processing Systems}, 34:6515--6528.

\bibitem[Roberts et~al., 2013]{roberts_gaussian_2013}
Roberts, S., Osborne, M., Ebden, M., Reece, S., Gibson, N., and Aigrain, S. (2013).
\newblock Gaussian processes for time-series modelling.
\newblock {\em Philosophical Transactions of the Royal Society A: Mathematical, Physical and Engineering Sciences, 371(1984):20110550}.

\bibitem[Rombach et~al., 2022]{rombach2022high}
Rombach, R., Blattmann, A., Lorenz, D., Esser, P., and Ommer, B. (2022).
\newblock High-resolution image synthesis with latent diffusion models.
\newblock In {\em Proceedings of the IEEE/CVF conference on computer vision and pattern recognition}, pages 10684--10695.

\bibitem[Rosenblatt, 1958]{rosenblatt1958perceptron}
Rosenblatt, F. (1958).
\newblock The perceptron: a probabilistic model for information storage and organization in the brain.
\newblock {\em Psychological review}, 65:386.

\bibitem[Rudin, 1994]{rudin_fourier_1994}
Rudin, W. (1994).
\newblock Fourier analysis on groups.
\newblock {\em Wiley Classics Library. Wiley-Interscience New York, reprint edition}.

\bibitem[Salimans, 2016]{salimans2016structured}
Salimans, T. (2016).
\newblock A structured variational auto-encoder for learning deep hierarchies of sparse features.
\newblock {\em arXiv preprint arXiv:1602.08734}.

\bibitem[Salimans et~al., 2015]{salimans2015markov}
Salimans, T., Kingma, D., and Welling, M. (2015).
\newblock {Markov} chain {Monte Carlo} and variational inference: Bridging the gap.
\newblock In {\em Proceedings of the International Conference on Machine Learning (ICML)}.

\bibitem[Salimbeni et~al., 2018]{sali2018ortho}
Salimbeni, H., Cheng, C.-A., Boots, B., and Deisenroth, M. (2018).
\newblock Orthogonally decoupled variational {Gaussian} processes.
\newblock In {\em Advances in Neural Information Processing Systems}.

\bibitem[Salimbeni and Deisenroth, 2017]{salimbeni2017doubly}
Salimbeni, H. and Deisenroth, M. (2017).
\newblock Doubly stochastic variational inference for deep {Gaussian} processes.
\newblock In {\em Advances in Neural Information Processing Systems}.

\bibitem[S{\"a}rkk{\"a} and Solin, 2019]{sarkka2019applied}
S{\"a}rkk{\"a}, S. and Solin, A. (2019).
\newblock {\em Applied stochastic differential equations}, volume~10.
\newblock Cambridge University Press.

\bibitem[Shi et~al., 2023]{shi2023diffusion}
Shi, Y., De~Bortoli, V., Campbell, A., and Doucet, A. (2023).
\newblock Diffusion schr{\"o}dinger bridge matching.
\newblock {\em Advances in Neural Information Processing Systems}, 36.

\bibitem[Shwartz~Ziv and LeCun, 2024]{shwartz2024compress}
Shwartz~Ziv, R. and LeCun, Y. (2024).
\newblock To compress or not to compress—self-supervised learning and information theory: A review.
\newblock {\em Entropy}, 26(3):252.

\bibitem[Snelson and Ghahramani, 2005]{snelson2005sgp}
Snelson, E. and Ghahramani, Z. (2005).
\newblock Sparse {Gaussian} processes using pseudo-inputs.
\newblock In {\em Advances in Neural Information Processing Systems}.

\bibitem[Sohl-Dickstein et~al., 2015]{sohl2015deep}
Sohl-Dickstein, J., Weiss, E., Maheswaranathan, N., and Ganguli, S. (2015).
\newblock Deep unsupervised learning using nonequilibrium thermodynamics.
\newblock In {\em International conference on machine learning}, pages 2256--2265. PMLR.

\bibitem[S{\o}nderby et~al., 2016]{sonderby2016ladder}
S{\o}nderby, C.~K., Raiko, T., Maal{\o}e, L., S{\o}nderby, S.~K., and Winther, O. (2016).
\newblock Ladder variational autoencoders.
\newblock {\em Advances in neural information processing systems}, 29.

\bibitem[Song et~al., 2021a]{song2021denoising}
Song, J., Meng, C., and Ermon, S. (2021a).
\newblock Denoising diffusion implicit models.
\newblock In {\em International Conference on Learning Representations}.

\bibitem[Song et~al., 2021b]{song2020score}
Song, Y., Sohl-Dickstein, J., Kingma, D.~P., Kumar, A., Ermon, S., and Poole, B. (2021b).
\newblock Score-based generative modeling through stochastic differential equations.
\newblock In {\em International Conference on Learning Representations}.

\bibitem[Stanton et~al., 2021]{stanton_kernel_2021}
Stanton, S., Maddox, W.~J., Delbridge, I., and Wilson, A.~G. (2021).
\newblock Kernel interpolation for scalable online {Gaussian} processes.
\newblock In {\em International {Conference} on {Artificial} {Intelligence} and {Statistics}}. PMLR.

\bibitem[Sun et~al., 2021]{hkd2021sun}
Sun, S., Shi, J., Wilson, A.~G., and Grosse, R. (2021).
\newblock Scalable variational {Gaussian} processes via harmonic kernel decomposition.
\newblock In {\em International Conference on Machine Learning}.

\bibitem[Sun et~al., 2019]{sun2019functional}
Sun, S., Zhang, G., Shi, J., and Grosse, R. (2019).
\newblock Functional variational {Bayesian} neural networks.
\newblock In {\em International Conference on Learning Representation}.

\bibitem[Tfekci and Kaya, 2014]{combined_cycle_power_plant_294}
Tfekci, P. and Kaya, H. (2014).
\newblock {Combined Cycle Power Plant}.
\newblock UCI Machine Learning Repository.
\newblock {DOI}: https://doi.org/10.24432/C5002N.

\bibitem[Tian et~al., 2020]{tian2020makes}
Tian, Y., Sun, C., Poole, B., Krishnan, D., Schmid, C., and Isola, P. (2020).
\newblock What makes for good views for contrastive learning?
\newblock {\em Advances in neural information processing systems}, 33:6827--6839.

\bibitem[Titsias, 2009]{svgp2009titsias}
Titsias, M. (2009).
\newblock Variational learning of inducing variables in sparse {Gaussian} processes.
\newblock In {\em Artificial Intelligence and Statistics}.

\bibitem[Ton et~al., 2018]{ton2018spatial}
Ton, J.-F., Flaxman, S., Sejdinovic, D., and Bhatt, S. (2018).
\newblock Spatial mapping with {Gaussian} processes and nonstationary fourier features.
\newblock {\em Journal of Spatial Statistics}, 28:59--78.

\bibitem[Touvron et~al., 2022]{Touvron2022DeiTIR}
Touvron, H., Cord, M., and Jegou, H. (2022).
\newblock Deit iii: Revenge of the vit.
\newblock {\em arXiv preprint arXiv:2204.07118}.

\bibitem[Tran et~al., 2019]{tran2019bl}
Tran, D., Dusenberry, M.~W., van~der Wilk, M., and Hafner, D. (2019).
\newblock Bayesian layers: A module for neural network uncertainty.
\newblock {\em arXiv:1812.03973}.

\bibitem[Tran et~al., 2021]{tran2021swsgp}
Tran, G.-L., Milios, D., Michiardi, P., and Filippone, M. (2021).
\newblock Sparse within sparse {Gaussian} processes using neighbor information.
\newblock In {\em International Conference on Machine Learning}.

\bibitem[Tsai et~al., 2019]{tsai2019TransformerDissection}
Tsai, Y.-H.~H., Bai, S., Yamada, M., Morency, L.-P., and Salakhutdinov, R. (2019).
\newblock Transformer dissection: An unified understanding for transformer's attention via the lens of kernel.
\newblock In {\em EMNLP}.

\bibitem[Vahdat and Kautz, 2020]{vahdat2020nvae}
Vahdat, A. and Kautz, J. (2020).
\newblock Nvae: A deep hierarchical variational autoencoder.
\newblock {\em Advances in neural information processing systems}, 33:19667--19679.

\bibitem[van~den Oord et~al., 2016a]{oord2016conditional}
van~den Oord, A., Kalchbrenner, N., Espeholt, L., kavukcuoglu, K., Vinyals, O., and Graves, A. (2016a).
\newblock Conditional image generation with pixelcnn decoders.
\newblock In {\em Advances in Neural Information Processing Systems}.

\bibitem[van~den Oord et~al., 2016b]{oord2016pixel}
van~den Oord, A., Kalchbrenner, N., and kavukcuoglu, K. (2016b).
\newblock Pixel recurrent neural networks.
\newblock In {\em International Conference on Machine Learning}.

\bibitem[Van Den~Oord et~al., 2017]{van2017neural}
Van Den~Oord, A., Vinyals, O., et~al. (2017).
\newblock Neural discrete representation learning.
\newblock {\em Advances in neural information processing systems}, 30.

\bibitem[Van~der Wilk et~al., 2020]{van_der_wilk_framework_2020}
Van~der Wilk, M., Dutordoir, V., John, S., Artemev, A., Adam, V., and Hensman, J. (2020).
\newblock A framework for interdomain and multioutput {Gaussian} processes.
\newblock {\em arXiv preprint arXiv:2003.01115}.

\bibitem[Vaswani et~al., 2017]{vaswani2017attention}
Vaswani, A., Shazeer, N., Parmar, N., Uszkoreit, J., Jones, L., Gomez, A.~N., Kaiser, {\L}., and Polosukhin, I. (2017).
\newblock Attention is all you need.
\newblock {\em Advances in neural information processing systems}, 30.

\bibitem[Villegas et~al., 2022]{villegas2022phenaki}
Villegas, R., Babaeizadeh, M., Kindermans, P.-J., Moraldo, H., Zhang, H., Saffar, M.~T., Castro, S., Kunze, J., and Erhan, D. (2022).
\newblock Phenaki: Variable length video generation from open domain textual descriptions.
\newblock In {\em International Conference on Learning Representations}.

\bibitem[Wang et~al., 2024]{wang2024rectified}
Wang, F.-Y., Yang, L., Huang, Z., Wang, M., and Li, H. (2024).
\newblock Rectified diffusion: Straightness is not your need in rectified flow.
\newblock {\em arXiv preprint arXiv:2410.07303}.

\bibitem[Warstadt et~al., 2019]{cola}
Warstadt, A., Singh, A., and Bowman, S.~R. (2019).
\newblock Neural network acceptability judgments.
\newblock In {\em Transactions of the Association for Computational Linguistics}.

\bibitem[Wen et~al., 2024]{wen2024mitigating}
Wen, B., Xu, C., Wolfe, R., Wang, L.~L., and Howe, B. (2024).
\newblock Mitigating overconfidence in large language models: A behavioral lens on confidence estimation and calibration.
\newblock In {\em NeurIPS 2024 Workshop on Behavioral Machine Learning}.

\bibitem[Wilkinson et~al., 2021]{wilkinson_sparse_2021}
Wilkinson, W., Solin, A., and Adam, V. (2021).
\newblock Sparse algorithms for {Markovian} {Gaussian} processes.
\newblock In {\em International {Conference} on {Artificial} {Intelligence} and {Statistics}}, pages 1747--1755. PMLR.

\bibitem[Wilson et~al., 2016]{agw2015dkl}
Wilson, A.~G., Hu, Z., Salakhutdinov, R., and Xing, E.~P. (2016).
\newblock Deep kernel learning.
\newblock In {\em International Conference on Artificial Intelligence and Statistics}.

\bibitem[Wilson and Izmailov, 2022]{wilson2022bayesian}
Wilson, A.~G. and Izmailov, P. (2022).
\newblock Bayesian deep learning and a probabilistic perspective of generalization.
\newblock {\em arXiv preprint arXiv:2002.08791}.

\bibitem[Wilson et~al., 2020]{wilsonl2020eff}
Wilson, J.~T., Borovitskiy, V., Terenin, A., Mostowsky, P., and Deisenroth, M.~P. (2020).
\newblock Efficiently sampling functions from {Gaussian} process posteriors.
\newblock In {\em International Conference on Machine Learning}.

\bibitem[Wu et~al., 2020]{wu2020mutual}
Wu, M., Zhuang, C., Mosse, M., Yamins, D., and Goodman, N. (2020).
\newblock On mutual information in contrastive learning for visual representations.
\newblock {\em arXiv preprint arXiv:2005.13149}.

\bibitem[Xiao and Bamler, 2023]{xiao2023trading}
Xiao, T.~Z. and Bamler, R. (2023).
\newblock Trading information between latents in hierarchical variational autoencoders.
\newblock In {\em The Eleventh International Conference on Learning Representations}.

\bibitem[Xiong et~al., 2024]{xiong2024can}
Xiong, M., Hu, Z., Lu, X., Li, Y., Fu, J., He, J., and Hooi, B. (2024).
\newblock Can {LLMs} express their uncertainty? an empirical evaluation of confidence elicitation in {LLMs}.
\newblock In {\em International Conference on Learning Representations (ICLR)}.

\bibitem[Xue et~al., 2021]{xue2021bayestran}
Xue, B., Yu, J., Xu, J., Liu, S., Hu, S., Ye, Z., Geng, M., Liu, X., and Meng, H. (2021).
\newblock Bayesian transformer language models for speech recognition.
\newblock {\em arXiv:2102.04754}.

\bibitem[Yan et~al., 2021]{yan2021videogpt}
Yan, W., Zhang, Y., Abbeel, P., and Srinivas, A. (2021).
\newblock Videogpt: Video generation using vq-vae and transformers.
\newblock {\em arXiv preprint arXiv:2104.10157}.

\bibitem[Zeni et~al., 2024]{zeni2024mattergengenerativemodelinorganic}
Zeni, C., Pinsler, R., Zügner, D., Fowler, A., Horton, M., Fu, X., Shysheya, S., Crabbé, J., Sun, L., Smith, J., Nguyen, B., Schulz, H., Lewis, S., Huang, C.-W., Lu, Z., Zhou, Y., Yang, H., Hao, H., Li, J., Tomioka, R., and Xie, T. (2024).
\newblock Mattergen: a generative model for inorganic materials design.
\newblock In {\em arXiv preprint arXiv:2312.03687}.

\bibitem[Zhang et~al., 2018a]{zhang2018advances}
Zhang, C., B{\"u}tepage, J., Kjellstr{\"o}m, H., and Mandt, S. (2018a).
\newblock Advances in variational inference.
\newblock {\em IEEE transactions on pattern analysis and machine intelligence}, 41(8):2008--2026.

\bibitem[Zhang et~al., 2018b]{zhang2018unreasonable}
Zhang, R., Isola, P., Efros, A.~A., Shechtman, E., and Wang, O. (2018b).
\newblock The unreasonable effectiveness of deep features as a perceptual metric.
\newblock In {\em Proceedings of the IEEE conference on computer vision and pattern recognition}, pages 586--595.

\bibitem[Zhang et~al., 2019]{zhang2019cyclical}
Zhang, R., Li, C., Zhang, J., Chen, C., and Wilson, A.~G. (2019).
\newblock Cyclical stochastic gradient {MCMC} for {Bayesian} deep learning.
\newblock In {\em International Conference on Learning Representations}.

\bibitem[Zhu et~al., 2023]{zhu_markovian_2023}
Zhu, H., Rodas, C.~B., and Li, Y. (2023).
\newblock Markovian {Gaussian} process variational autoencoders.
\newblock In {\em International {Conference} on {Machine} {Learning}}.

\bibitem[Zhu et~al., 2024]{zhu_vision_2024}
Zhu, L., Liao, B., Zhang, Q., Wang, X., Liu, W., and Wang, X. (2024).
\newblock Vision {Mamba}: Efficient visual representation learning with bidirectional state space model.
\newblock In {\em International Conference on Machine Learning (ICML)}.

\end{thebibliography}
\bibliographystyle{apalike}

\clearpage
\appendix
\chapter{Supplementary Material for Chapter~\ref{cha:back}}

\section{HiPPO Variants in Addition to HiPPO-LegS}
\label{appendix:hippo_measure_basis}

In addition to HiPPO-LegS, we also provide here the formulations of a few other HiPPO variants \citep{gu_hippo_2020} based on non-uniform measures.

\paragraph{HiPPO-LegT.} 
HiPPO-LegT uses a uniform measure over the most recent context window with length $\theta$, $\omega^{(t)}(x) = \frac{1}{\theta}\bm{1}_{[t-{\theta},t]}(x)$ and scaled Legendre polynomials with input domain adapted to $[t-\theta,t]$, \(g_m^{(t)}(x) = (2 m+1)^{1 / 2} P_m\left(\frac{2 (x-t)}{\theta}+1\right)\), as basis functions, where $P_m(\cdot)$ is the standard $m$-th Legendre polynomial with input domain $[-1, 1]$. The formulas of matrices $\bm{A}(t) \in \mathbb{R}^{M \times M}$ and $\bm{B}(t) \in \mathbb{R}^{M \times 1}$ for the corresponding linear ODE are:
\begin{equation}
\bm{A}(t) = -\frac{1}{\theta}\bm{A}, \quad
[\bm{A}]_{mk} = \begin{cases}
    \sqrt{(2m+1)(2k+1)} & \text{if } m \geq k \\
    (-1)^{m-k}\sqrt{(2m+1)(2k+1)} & \text{if } m < k
\end{cases}
\end{equation}
and
\begin{equation}
[\bm{B}(t)]_m =\frac{[\bm{B}]_m}{\theta} =\frac{\sqrt{2m+1}}{\theta},
\end{equation}

\paragraph{HiPPO-LagT.}
HiPPO-LagT uses an exponential decay measure, which assigns more importance to more recent history, $\omega^{(t)}(x) = \exp \left( -(t-x) \right)\bm{1}_{(-\infty,t]}(x)$ and scaled Laguerre polynomials with input domain adapted to $(-\infty,t]$, \(g_m^{(t)}(x) = L_m\left(t-x\right)\), as basis functions, where $L_m(\cdot)$ is the standard $m$-th Laguerre polynomial with input domain $[0, \infty)$. The formulas of matrices $\bm{A}(t) \in \mathbb{R}^{M \times M}$ and $\bm{B}(t) \in \mathbb{R}^{M \times 1}$ for the corresponding linear ODE are:
\begin{equation}
[\bm{A}(t)]_{mk} = 
[\bm{A}]_{mk} = \begin{cases}
    -1 & \text{if } m \geq k \\
    0 & \text{if } m < k
\end{cases}
\end{equation}
and
\begin{equation}
[\bm{B}(t)]_m =[\bm{B}]_m =1,
\end{equation}

\paragraph{HiPPO-FouT.}
Similar to HiPPO-LegT, HiPPO-FouT also uses the uniform measure over a sliding window window, $\omega^{(t)}(x) = \frac{1}{\theta}\bm{1}_{[t-{\theta},t]}(x)$, but is based on scaled Fourier basis functions with input domain adapted to $[t-\theta,t]$, $g_m^{(t)}(x) = \exp \left( 2\pi i m \frac{t-x}{\theta}\right)$, as basis functions. The formulas of matrices $\bm{A}(t) \in \mathbb{R}^{M \times M}$ and $\bm{B}(t) \in \mathbb{R}^{M \times 1}$ for the corresponding linear ODE are:
\begin{equation}
[\bm{A}(t)]_{mk} = 
[\bm{A}]_{mk} = \begin{cases}
    -\frac{1}{\theta} & \text{if } m \neq k \\
    \frac{2\pi i m-1}{\theta} & \text{if } m = k
\end{cases}
\end{equation}
and
\begin{equation}
[\bm{B}(t)]_m =[\bm{B}]_m =\frac{1}{\theta},
\end{equation}
\chapter{Supplementary Material for Chapter~\ref{cha:sgpa}}

\section{Derivations}

\subsection{ELBO Derivations for SVGP}
\label{a4:svgp_elbo}
For completeness, we include the full derivation of ELBO as shown in Eq.~\ref{eq:elbo_svgp} for standard SVGP \citep{svgp2009titsias, hensman2013gpbig}. With $M$ inducing points pairs $\{(\bm{z}_m, u_m)\}_{m=1}^M$, the prior distribution of $[\bm{f}, \bm{u}]^\top$ is:
\begin{equation}
    p(\bm{f}, \bm{u}|\bm{X},\bm{Z})= \mathcal{N}(\bm{0}, \begin{pmatrix}
     \bm{K}_{\bm{XX}} & \bm{K}_{\bm{XZ}} \\
     \bm{K}_{\bm{Z}\bm{X}} & \bm{K}_{\bm{ZZ}}
    \end{pmatrix}).
\end{equation}
With prior conditional matching assumption (see Eq.~\ref{eq:prior_conditional_matching}), the approximate posterior conditional distribution of function values $\bm{f}$ for inputs $\bm{X}$ given inducing points $\bm{u}$ is the same as the prior conditional distribution: 
\begin{equation}
    q(\bm{f}|\bm{u,Z,X})=p(\bm{f}|\bm{u,Z,X}).
\end{equation}
Under prior conditional matching assumption, $q(\bm{f},\bm{u}|\bm{Z},\bm{X})=p(\bm{f}|\bm{u},\bm{Z},\bm{X})q(\bm{u})$, where $q(\bm{u})=\mathcal{N}(\bm{m_u},\bm{S_u})$. Suppose the observation likelihood is $p(\bm{y}|\bm{f})$, ELBO can be simplified as follows:
\begin{equation}
    \begin{split}
        \mathcal{L}_{ELBO}&=\mathbb{E}_{q(\bm{f},\bm{u}|\bm{Z}, \bm{X})}[\log\frac{p(\bm{y},\bm{f},\bm{u}|\bm{Z},\bm{X})}{q(\bm{f},\bm{u}|\bm{Z},\bm{X})}]\\
        &=\mathbb{E}_{q(\bm{f},\bm{u}|\bm{Z}, \bm{X})}[\log\frac{p(\bm{y}|\bm{f})\cancel{p(\bm{f}|\bm{u,Z,X})}p(\bm{u}|\bm{Z})}{\cancel{p(\bm{f}|\bm{u},\bm{Z},\bm{X})}q(\bm{u})}]\\
        &=\int (\int p(\bm{f}|\bm{u,Z,X})q(\bm{u}) d\bm{u})  \log p(\bm{y}|\bm{f}) d\bm{f}  \\& \qquad\qquad\qquad +\int q(\bm{u})\log\frac{p(\bm{u}|\bm{Z})}{q(\bm{u})}\cancel{\int p(\bm{f}|\bm{u,Z,X}) d\bm{f}} d\bm{u} \\
        &=\underbrace{\mathbb{E}_{q(\bm{f}|\bm{X},\bm{Z})}[\log p(\bm{y|f})]}_{\text{ELL}} -\underbrace{\text{KL}(q(\bm{u})||p(\bm{u}|\bm{Z}))}_{\text{KL regularizer}}.
    \end{split}
\label{eqa:svgp_elbo_derivation}
\end{equation} 
Here $q(\bm{f}|\bm{X}, \bm{Z})=\int p(\bm{f}|\bm{u,Z,X})q(\bm{u}) d \bm{u}$ is a Gaussian and is given as:
\begin{equation}
q(\bm{f}|\bm{X}, \bm{Z})=\mathcal{N}(\bm{K_{X Z}K_{ZZ}^{-1}m_u},\bm{K_{X X}+K_{X Z}K_{ZZ}^{-1}(S_u-K_{ZZ})K_{ZZ}^{-1}K_{ZX}}).
\label{eqa:svgp}
\end{equation}
With Gaussian likelihood, the first term in ELBO can be evaluated analytically. Otherwise we estimate the first term using Monte-Carlo samples $\bm{f} \sim q(\bm{f}| \bm{X}, \bm{Z})$. The second term is a KL-divergence between two Gaussian distributions. Thus, it admits a closed form:
\begin{equation}
\text{KL}(q(\bm{u})||p(\bm{u|Z}))=\frac{1}{2}\left[Tr(\bm{K_{ZZ}}^{-1}\bm{S_u})+\bm{m_u}^\top \bm{K_{ZZ}}^{-1}\bm{m_u}+\log\frac{|\bm{K_{ZZ}}|}{|\bm{S_u}|}\right] +\text{const}.
\end{equation}
In standard SGPA the ELBO objective remains almost the same, except that as the variational mean is reparameterized to $\bm{v} := \bm{K}_{\bm{ZZ}}^{-1} \bm{m}_{\bm{u}}$, the mean of $q(\bm{f} | \bm{X}, \bm{Z})$ becomes $\bm{K_{X Z}}\bm{v}$, and the quadratic term in $\text{KL}(q(\bm{u})||p(\bm{u|Z}))$ becomes $\bm{v}^\top \bm{K_{ZZ}}\bm{v}$.

\subsection{Derivation of ELBO for Transformers Based on SGPA}
\label{a4:sgpa_obj}
An $L$-layer Transformer based on SGPA is a deep GP \citep{damia2013dgp}, and we train it using the doubly stochastic variational inference framework \citep{salimbeni2017doubly}. For each input sequence $\bm{F}^0 := \bm{X}$, the joint distribution for $\bm{Y}, \{\bm{F}^l \}_{l=1}^L, \{\bm{u}_{a\cup g}^{l,h}\}_{l=1,h=1}^{L,H}$ is:
\begin{equation}
\begin{split}
    p(\bm{Y}, &\{\bm{F}^l\}_{l=1}^L, \{\bm{u}_{a\cup g}^{l,h}\}_{l=1,h=1}^{L,H}|\bm{F}^0) =\\ &p(\bm{Y}|\bm{F}^L)
    [\prod_{l=1}^L p(\bm{F}^l| \{\bm{u}_{a\cup g}^{l,h}\}_{h=1}^{H}, \bm{F}^{l-1})
    p(\{\bm{u}_{a\cup g}^{l,h}\}_{h=1}^{H}| \{\bm{k}_{g}^{l,h}\}_{h=1}^{H}, \bm{F}^{l-1})],
\end{split}
\end{equation}
where $p(\{\bm{u}_{a\cup g}^{l,h}\}_{h=1}^{H}| \{\bm{k}_{g}^{l,h}\}_{h=1}^{H}, \bm{F}^{l-1}) = \prod_{h=1}^H p(\bm{u}_{a\cup g}^{l,h}| \bm{k}_{g}^{l,h}, \bm{F}^{l-1})$ since we assume the prior for inducing points factorizes across heads in each layer. Note here the amortized keys $\bm{k}_a^{l, h}$ depend on $\bm{F}^{l-1}$ in a deterministic manner, therefore we drop the amortized key terms in the conditioning. Assuming prior conditional matching, i.e, 
\begin{equation}
q(\bm{F}^l| \{\bm{u}_{a\cup g}^{l,h}\}_{h=1}^{H}, \bm{F}^{l-1})=p(\bm{F}^l| \{\bm{u}_{a\cup g}^{l,h}\}_{h=1}^{H}, \bm{F}^{l-1}),
\end{equation}
the joint approximate posterior for $ \{\bm{F}^l\}_{l=1}^L, \{\bm{u}_{a\cup g}^{l,h}\}_{l=1,h=1}^{L,H}$ is:
\begin{equation}
\begin{split}
    q(\{\bm{F}^l\}_{l=1}^L, \{\bm{u}_{a\cup g}^{l,h}\}_{l=1,h=1}^{L,H}|\bm{F}^0) =
    \prod_{l=1}^L p(\bm{F}^l| \{\bm{u}_{a\cup g}^{l,h}\}_{h=1}^{H}, \bm{F}^{l-1})
    q(\{\bm{u}_{a\cup g}^{l,h}\}_{h=1}^{H}| \{\bm{k}_{g}^{l,h}\}_{h=1}^{H}, \bm{F}^{l-1}),
\end{split}
\end{equation}
where $q(\{\bm{u}_{a\cup g}^{l,h}\}_{h=1}^{H}| \{\bm{k}_{g}^{l,h}\}_{h=1}^{H}, \bm{F}^{l-1}) = \prod_{h=1}^H q(\bm{u}_{a\cup g}^{l,h}| \bm{k}_{g}^{l,h}, \bm{F}^{l-1})$ since we let the approximate distribution for $\{\bm{u}_{a\cup g}^{l,h}\}_{h=1}^{H}$ also factorizes across heads. The ELBO is derived in a similar manner as the single-layer GP case (Eq.~\ref{eqa:svgp_elbo_derivation}) and again the conditional distribution terms in $q$ and $p$ cancel with each other. This simplifies the ELBO to:
\begin{equation}
    \begin{split}
        \mathcal{L}_{ELBO}&=\mathbb{E}_{q(\{\bm{F}^l\}_{l=1}^L, \{\bm{u}_{a\cup g}^{l,h}\}_{l=1,h=1}^{L,H}|\bm{F}^0)} [\log \frac{p(\bm{Y}, \{\bm{F}^l\}_{l=1}^L, \{\bm{u}_{a\cup g}^{l,h}\}_{l=1,h=1}^{L,H}|\bm{F}^0)}{q(\{\bm{F}^l\}_{l=1}^L, \{\bm{u}_{a\cup g}^{l,h}\}_{l=1,h=1}^{L,H}|\bm{F}^0)}]\\
        &=\mathbb{E}_{q(\bm{F}^L|\bm{F}^0,\{\bm{k}^{l,h}_{g}\}_{l=1, h=1}^{L,H})}[\log p(\bm{Y}|\bm{F}^L)]\\ 
        &\quad-\sum_{l=1}^L \sum_{h=1}^H \mathbb{E}_{q(\bm{F}^{l}|\bm{F}^0,\{\bm{k}^{j,h}_{g}\}_{j=1, h=1}^{l,H}))}\{\text{KL}(q(\bm{u}_{a \cup g}^{l,h}|\bm{k}^{l,h}_{g},\bm{F}^{l-1})||p(\bm{u}_{a \cup g}^{l,h}|\bm{k}^{l,h}_{g},\bm{F}^{l-1}))\},
    \end{split}
\end{equation}
where 
\begin{equation}
\begin{split}
q(\bm{F}^l|\bm{F}^0,&\{\bm{k}^{j,h}_{g}\}_{j=1, h=1}^{l,H})=\\&\int \prod_{j=1}^l p(\bm{F}^j| \{\bm{u}_{a\cup g}^{j,h}\}_{h=1}^{H}, \bm{F}^{j-1})
    q(\{\bm{u}_{a\cup g}^{j,h}\}_{h=1}^{H}| \{\bm{k}_{g}^{j,h}\}_{h=1}^{H}, \bm{F}^{j-1}) d\bm{u}_{a\cup g}^{1:l,1:H} d\bm{F}^{1:l-1}.
\end{split}
\end{equation}
Both terms in the ELBO can be estimated using samples generated iteratively through each layer using the reparameterization trick. For the second ``regularization'' term, the KL-divergence within the expectation admits a simplified form as we assume for each attention output dimension ($d$), an independent decoupled SVGP is fitted: 
\begin{equation}
\begin{split}
&\quad \text{KL}(q(\bm{u}_{a \cup g}^{l,h}|\bm{k}^{l,h}_{g},\bm{F}^{l-1})||p(\bm{u}_{a \cup g}^{l,h}|\bm{k}^{l,h}_{g},\bm{F}^{l-1}))\\
&=\frac{1}{2}\sum_{d=1}^D {\Large\{} [\bm{v}_{a}^{l,h}]_{:,d}^\top(\bm{K}_{\bm{k}_{a}^{l,h} \bm{k}_{a}^{l,h}}- \bm{K}_{\bm{k}_{a}^{l,h} \bm{k}_{g}^{l,h}} \bm{K}^{{-1}}_{\bm{k}_{g}^{l,h} \bm{k}_{g}^{l,h}} \bm{K}_{\bm{k}_{g}^{l,h} \bm{k}_{a}^{l,h}})[\bm{v}_{a}^{l,h}]_{:,d} + [\bm{v}_{g}^{l,h}]_{:,d}^\top \bm{K}_{\bm{k}_{g}^{l,h} \bm{k}_{g}^{l,h}} [\bm{v}_{g}^{l,h}]_{:,d}\\
&\qquad+\text{Tr}([\bm{S}_{g}^{l,h}]_{:,:,d} \bm{K}^{-1}_{\bm{k}_{g}^{l,h} \bm{k}_{g}^{l,h}}) -\log\det[\bm{S}_{g}^{l,h}]_{:,:,d} + \log\det\bm{K}_{\bm{k}_{g}^{l,h} \bm{k}_{g}^{l,h}} 
 {\Large\}} +\text{const},
\end{split}
\end{equation}
where $D$ is the total number of attention output dimensions and $M_g$ is the number of global inducing points for each head.

\section{Uncertainty Calibration Metrics}
\label{appendix:uq_metric}
We use two types of widely used metrics to assess the quality of uncertainty estimates for in-distribution data, proper scoring rule and calibration error. Specifically, negative log-likelihood is a proper scoring rule based metric, and expected calibration error (ECE) and maximum calibration error (MCE) are two calibration error based metrics.

\subsection{Proper Scoring Rule}
A proper scoring rule \citep{gneit2007proper} is a function taking a distribution and a data point as inputs. Suppose the data generating distribution is $p_{d}$, and the predictive distribution provided by our model is $\hat{p}$, then for each observation $\bm{x}_i$ we can obtain a score, $S(\hat{p},\bm{x}_i)$, based on the predictive distribution. To measure the deviation of $\hat{p}$ from $p_d$, we take the expectation of the aforementioned scores over $p_d$: $S(\hat{p},p_d) = \mathbb{E}_{p_d}[S(\hat{p},\bm{x})]$. For a proper scoring rule, $S(p_d,p_d)\leq S(\hat{p},p_d)\text{,}\quad \forall \hat{p}$, so the minimum of $S(\hat{p},p_d)$ can be achieved when $\hat{p}=p_d$. For a scoring rule which is strictly proper, $S(p_d,p_d)$ is the only global minimum. This justifies why it is proper for assessing the quality of predictive distribution. In practice, $S(\hat{p},p_d)$ is analytically intractable since the ground-truth $p_d$ is unknown, and we typically resort to Monte-Carlo to estimation. NLL is one popular metric computed based on a proper scoring rule, specifically the logarithm score. Given a test set $\{(\mathbf{x}_i, \mathbf{y}_i)\}_{n=1}^N$, NLL is computed via:
\[\mathrm{NLL} = -\mathbb{E}_{p_d}[\log \hat{p}(\bm{y}|\bm{x})]\approx -\frac{1}{N}\sum_{n=1}^N \log \hat{p}(\bm{y}_n \mid \bm{x}_n).\]

\subsection{Calibration Error}
Calibration error \citep{Pakdaman_Naeini_Cooper_Hauskrecht_2015,guo2017calibration} measures the deviation between the model confidence in classification correctness and the ground truth probability of correctness. The predicted probability, $\hat{p}$, for the class the data point is classified to reflects the model's confidence in the classification correctness. For a well-calibrated model, for each confidence level $\hat{p}$, the proportion $p$ of the data points belonging to this ``$\hat{p}$-confidence group" that are classified correctly should be equal to $\hat{p}$. Calibration error is then computed based on the absolute difference between $\hat{p}$ and $p$. The ECE is the expectation of this difference w.r.t. the distribution of confidence $\hat{p}$:
\begin{equation}
    ECE = \mathbb{E}_{\hat{p}}\left[|\hat{p}-p|\right]
\end{equation}
The MCE is the maximum of these differences:
\begin{equation}
    MCE = \argmax_{\hat{p}}|\hat{p}-p|
\end{equation}
In practice, since we do not have sufficient amount of data for each confidence level, a biased estimation of calibration errors is obtained using histogram binning \citep{Pakdaman_Naeini_Cooper_Hauskrecht_2015,guo2017calibration}. We partition the domain of $\hat{p}$, $[0,1]$, into 15 bins of equal size as in \citet{guo2017calibration}.
%
%
\section{Additional Experiments}
\label{a2}
\subsection{Empirical Comparison Between Standard SGPA and Decoupled SGPA}
In our preliminary experiments, we compare performance of ViT based on standard SGPA versus decoupled SGPA for image classification on CIFAR10 without data augmentation. We also consider decoupled SGPA based on DSVGP \citep{cheng2017variational}. 
According to Eq.~\ref{eq:dsvgp_predcitive}, the posterior mean and covariance formula for decoupled SGPA based on DSVGP \citep{cheng2017variational} is given as follows:
\begin{equation}
\begin{split}
\bm{m}_{d}&=\bm{K}_{\bm{q} \bm{k}_a}[\bm{v}_{a}]_{:,d}, \\
\bm{\Sigma}_{d}&= \bm{K}_{\bm{q} \bm{q}} + \bm{K}_{\bm{q} \bm{k}_{g}}\bm{K}_{\bm{k}_{g}\bm{k}_{g}}^{-1}([\bm{S}_{g}]_{:,:,d}-\bm{K}_{\bm{k}_{g}\bm{k}_{g}})\bm{K}^{-1}_{\bm{k}_{g}\bm{k}_{g}}\bm{K}_{\bm{k}_{g}\bm{q}},
\end{split}
\label{eq:naivedsvgp}
\end{equation}
Table~\ref{table:ssgpa} shows ViT based on standard SGPA and decoupled SGPA based on DSVGP \citep{cheng2017variational} achieve worse performance than decouple SGPA based on ODSVGP \citep{sali2018ortho}. In particular, standard SGPA considerably underfits the data. Therefore we only consider decoupled SGPA based on ODSVGP \citep{sali2018ortho} for the rest of the experiments.
\begin{table}[htbp]
\centering
\caption{Test set accuracy and NLL of ViTs based on standard and two variants of decoupled SGPA, trained on CIFAR10 without data augmentation.}
 \begin{tabular}{  c  c c    } 
\hline
 Model & Accuracy & NLL \\ 
 \hline
   Standard SGPA & 0.6435$\pm$0.0039 & 1.0159$\pm$0.0065\\
   Decoupled SGPA \citep{cheng2017variational} & 0.7513$\pm$0.0014 &	0.7877$\pm$0.0045\\
  Decoupled SGPA \citep{sali2018ortho} & \textbf{0.7787$\pm$0.0024}&	\textbf{0.6968$\pm$0.0032}\\
\hline
\end{tabular}
\label{table:ssgpa}
\end{table}

\label{a2:ssgpa}

\subsection{Graph Property Regression with ZINC Dataset}
\begin{figure}[t]
\centering
\includegraphics[width=0.99\linewidth]{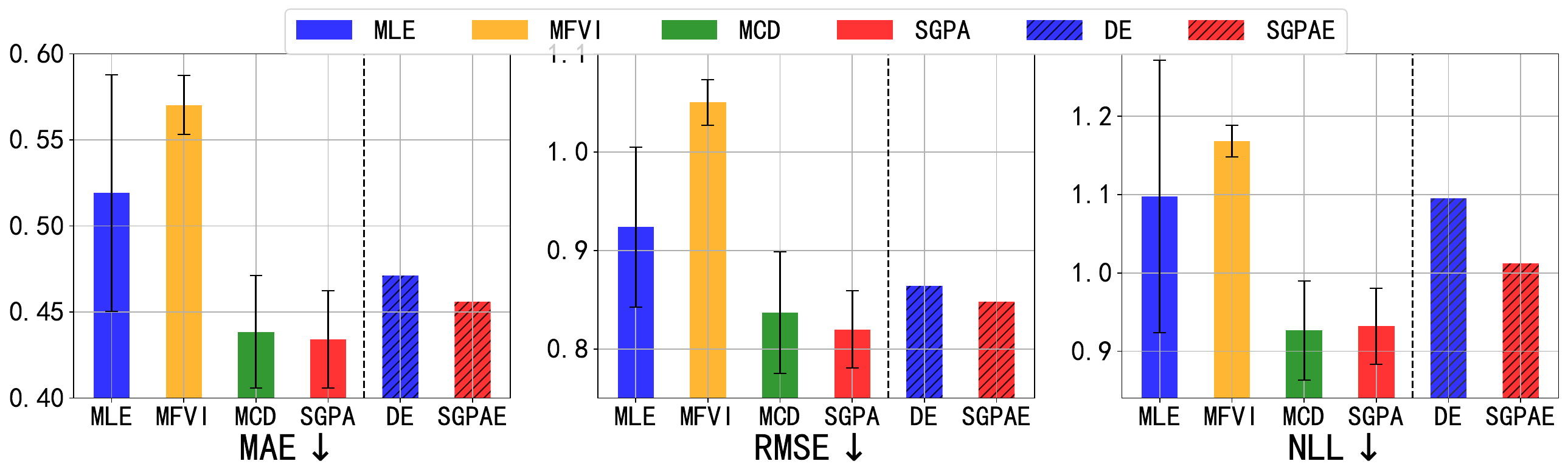}
\caption{Test set regression error and NLL metrics for Transformers trained on ZINC.}
\label{fig:zinc_in}
\end{figure}

\begin{figure}[b]
\centering
\includegraphics[width=0.99\linewidth]{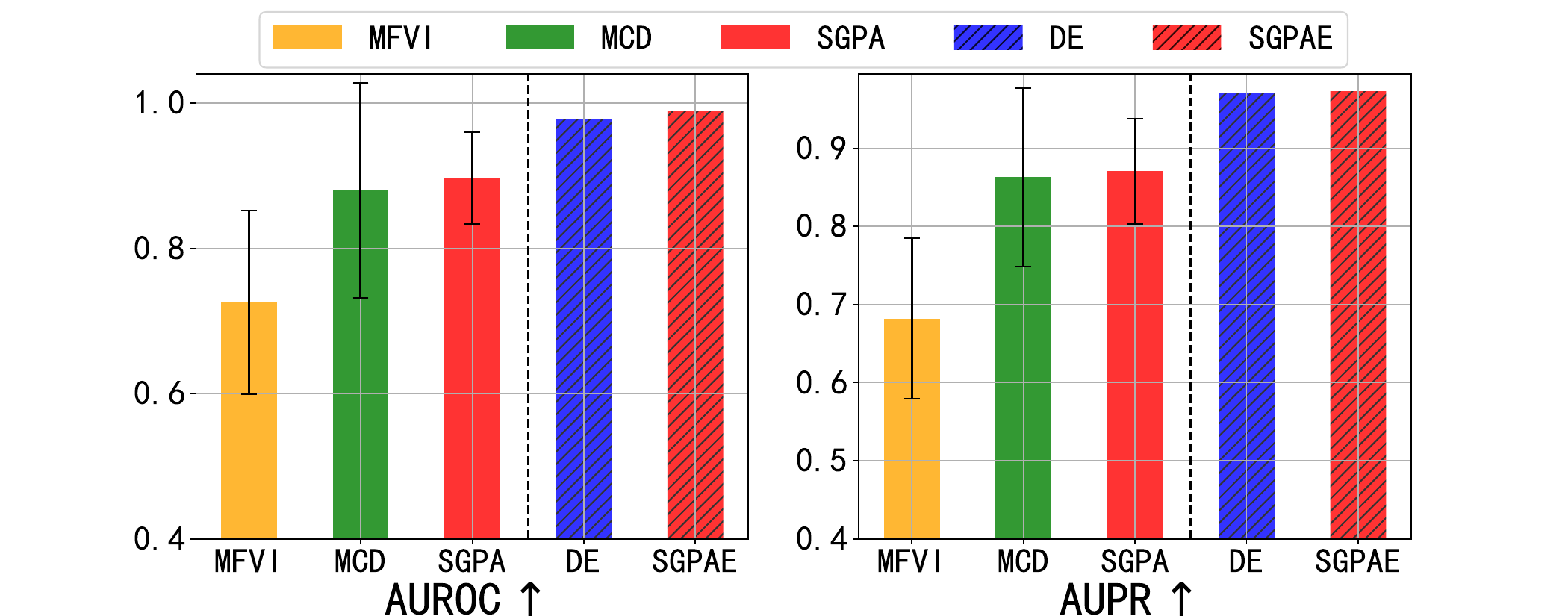}
\caption{AUROC and AUPR metrics for OOD detection using Transformers trained on ZINC.}
\label{fig:zinc_ood}
\end{figure}
\label{a2:zinc}
For graph property regression, we assume a Laplace likelihood with a trainable scale parameter $b$ (i.e., the density of the observation likelihood is $g(y|f)=\frac{1}{2b}\exp(-\frac{|y-f|}{b})$, where $f$ is the scalar function value output by the Transformer). We compute mean-absolute-error (MAE), root-mean-square error (RMSE) and negative-log-likelihood (NLL) to evaluate the models, with results presented in Figure~\ref{fig:zinc_in}. Moreover, we use predictive variances as scores and evaluate OOD detection performance in Figure~\ref{fig:zinc_ood}. Note that MLE is useless for OOD detection in this case since it produces homogeneous predictive variances for all instances. We use a synthetic OOD dataset generated from test set: for each test instance, we remove the existing edges from the adjacency matrix and add edges between nodes that are not originally connected. Within ``single-model'' methods,  SGPA and MCD achieve much better results than MLE and MFVI. For this task, the difference in performance between SGPA and MCD is negligible. However, when compared to MCD, SGPA performs more robustly as it returns smaller error bars. Ensemble methods outperform ``single-model" methods in OOD detection with SGPAE achieving the best result. Interestingly, for in-distribution calibration, they achieve worse performance than SGPA and MCE.
\section{Detailed Experimental Setups and Hyperparameters}
\textbf{Training settings shared across experiments.}
For MLE and MCD, we initially considered both dropout rates 0.1 and 0.2, but we found models trained with dropout rate 0.1 consistently outperformed models trained with dropout rate 0.2 in terms of accuracy in our preliminary experiments for image classification. Therefore, we decided to use dropout rate 0.1 for all methods except MFVI. For each layer, we use mean-pooling strategy. The non-linear mapping $G_{\phi^l}$ at each attention layer is parameterized by a 2-layer MLP as in \citet{vaswani2017attention}. For models with kernel-based attentions, we use exponential kernel for sentiment analysis and linguistic acceptability, \citep{tsai2019TransformerDissection}: $k(\bm{x},\bm{x}')=\sigma_f^2\exp(\sum_{j=1}^D \frac{x_j x_j'}{\sigma_j^2})$, and we use ARD-RBF kernel \citep{rasmussen2006gp} for image classification and graph property regression: $k(\bm{x},\bm{x}')=\sigma_f^2\exp(-\frac{1}{2}\sum_{j=1}^D \frac{(x_j-x_j')^2}{\sigma_j^2})$, where $D$ is the dimension of $\bm{x}$ and $\bm{x}'$, $\sigma_f^2$ is the output variance, and $\sigma_j$ is the length-scale for the $j$-th dimension. For MFVI, MCD, SNGP and SGPA, predictive uncertainty is estimated using 10 Monte Carlo samples. 
\begin{itemize}[leftmargin=15pt, topsep=-1pt]
\item \textbf{Initialization:} we train all our models from scratch (i.e. all parameters are randomly initialized without any pretraining). Apart from the global inducing points parameters, the rest of the parameters are the same as in standard transformers, and are initialized via the default method of the deep learning platform (we use Pytorch \citep{NEURIPS2019_9015}). Each dimension of global inducing locations and the global variational mean are randomly initialized with standard Gaussian. For the Cholesky factor used to parameterize the global variational covariance, we randomly initialize each element in the lower triangular part and each element in the log-diagonal with a standard Gaussian. 
    \item \textbf{Optimization:} all the models are trained using ADAM optimizer \citep{kingma2015adam}, and for each input sequence in a batch, we only draw one sample to estimate the ELBO (eq.~\ref{eq:elbo}). In our experiments, we observe no optimization issue: the ELBO is consistently minimized during the training without significant spikes.
\end{itemize}

\textbf{Sentiment analysis with IMDB \citep{imdb}.} 
We consider 5 different splits, each includes 35,000 training, 5,000 validation, and 10,000 test instances. The maximum number of tokens in each input sequence is 512. Architecture-wise, we use a Transformer with 1 MHSA layer and 8 attention heads, same embedding dimension and hidden dimension of 128. For SGPA we use 50 global inducing points for each head. We train all the models (except post-hoc methods, TS and KFLLLA) for 20 epochs with batch size 32 and with a initial learning rate 0.001 which decays linearly to 0.0001. The best model is selected based on the validation accuracy computed after every training epoch. 

\textbf{Linguistic acceptability with CoLA \citep{cola}.}
The 516 OOD samples provided by the original dataset are used to assess the models' OOD robustness. Within each of the 5 independent runs, the remaining 9,078 in-distribution samples are randomly split into 7,262 training and 1,816 in-distribution test instances. We use a Transformer with 2 MHSA layers, each with 4 attention heads, an embedding dimension of 128 and a hidden dimension of 256. For the input embeddings, we use ELMO-style representation \citep{elmo2018}. For SGPA we use 5 global inducing points for each head. We train all models (except post-hoc method KFLLLA) for 50 epochs with batch size 32 and with a initial learning rate 0.0005 which decays linearly to 0.00001 and we use the model from the final epoch for evaluation. 

\textbf{Image classification with CIFAR10 and CIFAR100 \citep{cifar}.}
For both CIFAR10 and CIFAR100 datasets, we randomly split the original training set into 4,5000 training and 5,000 validation instances, and test on the original 10,000 test instances. The input images are tokenized with a patch size of $4 \times 4$. For CIFAR10 without data augmentation, we use a ViT \citep{dosovitskiy2020image} with 5 MHSA layers, each with 4 attention heads and a hidden dimension of 128. For all the other experiments, we use a ViT with 6 layers 4 attention heads, a hidden dimension of 256. We train all models (except post-hoc methods, TS and KFLLLA) except SGPA for 600 epochs with batch size 100 and with initial learning rate of 0.0005 which decays linearly to 0.00001. For SGPA, we use 32 global inducing points for each head, and we use the parameters from the 100th epoch of MLE to initialize the deep kernel hyperparameters, and continue training for 500 epochs. The best model is selected based on the validation accuracy computed every 10 epochs. For experiments with data augmentation, we consider the same data augmentation strategy (3-Augment) as in \cite{Touvron2022DeiTIR}.

\textbf{Graph property regression with ZINC \citep{dwivedi2020benchmarkgnns}.} The results are presented in~\ref{a2:zinc} which are averaged from 3 independent runs. We use the same split as in \citep{dwivedi2020benchmarkgnns}, resulting in 10,000 training, 1,000 validation, and 1,000 test instances. Instead of applying graph-specific modifications to the network architecture, we use the feature engineering technique proposed in \citet{kim2021transformers} to transform each graph into a sequence of input embeddings. We consider Transformers with 8 layer and 8 attention heads, same embedding dimension and hidden dimension of 80. For SGPA we use 10 global inducing points for each head. We train all models for 500 epochs with batch size 64 and with an initial learning rate 0.0004 which decays linearly to 0.000002. The best model is selected based on the validation accuracy computed at the end of every 10 epochs. 
\label{a1:hyperparam}
\section{Running Time}
\label{appendix:sgpa_time}
We analyze the wall-clock training and inference time of SGPA here.
In Table~\ref{table:inf_time}, we present the computational time for a single batch at the inference stage with 10 Monte Carlo samples for CoLA (batch size = 227) and CIFAR10 (batch size = 200) (results obtained using a single Nvidia GTX 2080 Ti GPU card). For SGPA, we first pay an one-off cost of inverting the kernel matrices related to global inducing points ($\bm{K}_{\bm{k}_g \bm{k}_g}^{-1}$). Once this is done, we treat them as constant matrices and plug them in Eq.~\ref{eq:odsvgp}. When generating samples that are passed to the next layer, we diagonalize the covariance ($\bm{\Sigma}_d$) in Eq.~\ref{eq:odsvgp} to avoid the costly computation of the Cholesky factor of $\bm{\Sigma}_d$.

The computational cost depends on the number of global inducing points used. For CoLA, we only used 5 global inducing points for each head, and the relative difference between inference times for MCD and SGPA is less than that for CIFAR10, where we use 32 global inducing points for each head. 
It is noteworthy that we haven’t done extensive hyperparameter tuning for the number of global inducing points. It is likely that SGPA can still work well with a smaller number of global inducing points. For example, for CIFAR10, we later on also trained an SGPA model with 16 global inducing points for each head, and we did not see a considerable performance drop (Accuracy: 0.7790, NLL: 0.7259, ECE: 0.0625, MCE: 0.0819). In this case, the inference time can be further reduced from 0.986 to 0.807.

\begin{table}[!htb]
\centering
\caption{The computational time (in $s$) for a single batch at the inference stage with 10 Monte Carlo samples for CoLA (batch size = 227) and CIFAR10 (batch size = 200) (results obtained using a single Nvidia GTX 2080 Ti GPU card).}
 \begin{tabular}{c  c  c c  } 
\hline
& MCD & MFVI & SGPA \\ 
 \hline
 CoLA & 1.839 & 1.882 & 2.358\\
  CIFAR10 (32 $\bm{k}_g$) & 0.234 & 0.395 & 0.986\\
  CIFAR10 (16 $\bm{k}_g$) & 0.234 & 0.395 & 0.807\\
\hline
\end{tabular}
\label{table:inf_time}
\end{table}

In Table~\ref{table:train_time} we present the training time (in $s$) of one epoch for SGPA and MLE on CoLA (batch size = 32) and CIFAR10 (batch size = 100) (results obtained using a single Nvidia GTX 2080 Ti GPU card):
\begin{table}[!htb]
\centering
\caption{The training time (in $s$) of one epoch for SGPA and MLE on CoLA (batch size = 32) and CIFAR10 (batch size = 100) (results obtained using a single Nvidia GTX 2080 Ti GPU card).}
 \begin{tabular}{c  c  c   } 
\hline
& MLE & SGPA \\ 
 \hline
 CoLA & 17.834 & 20.089\\
  CIFAR10 (32 $\bm{k}_g$) & 30.593 & 138.609 \\
  CIFAR10 (16 $\bm{k}_g$) & 30.593 & 109.298\\
\hline
\end{tabular}
\label{table:train_time}
\end{table}

\section{Connection with Sparse within Sparse GP}
\label{a3:swsgp}
Unlike standard sparse GPs where the inducing points are shared across all inputs, SGPA consists of a set of input-dependent (or ``amortized'') inducing locations $\{\bm{k}_{a}^{l,h}\}$ and the corresponding variational parameters, which means we are using a different mean function for each input sequence. Consequently, SGPA based Transformers can not perfectly model the correlation between input sequences. Instead, they can only provide marginal uncertainty for each input sequence. Nevertheless, empirically we found that correlation might not be critical in applications such as text or image classification. 

Interestingly, SGPA can be considered as an instantiation of sparse-within-sparse Gaussian process (SWSGP) \citep{tran2021swsgp,daniel2022inputde}, which allows adaptive inducing points for each input. Suppose the index set (in our case the embedding space) is $\bm{\chi}$, and $p(\bm{Z})$ is a distribution over $M$-element subset of $\bm{\chi}$ (ie. each random draw will give us $M$ inducing locations from $\bm{\chi}$). The joint prior and approximate posterior become $p(\bm{f,u,Z})=p(\bm{f}|\bm{u})p(\bm{u}|\bm{Z})p(\bm{Z})$, and $q(\bm{u},\bm{Z})=q(\bm{u}|\bm{Z})p(\bm{Z})$, respectively. In \cite{tran2021swsgp}, for each input $\bm{x}$, they propose to select its $M$-nearest neighbors taken from the training inputs as inducing locations, so that $q(\bm{Z})$ is a delta distribution conditioned on $\bm{x}$, and $q(\bm{u}|\bm{Z})$ is the marginal variational distribution over function values evaluated at the selected inducing locations. In contrast, for each input sequence $\bm{x}$, in layer $l$, the inducing locations used by Transformer based on SGPA consist both input-dependent ones ($\{\bm{k}_{a}^{l,h}\}_{h=1}^H$), which are obtained from $\bm{x}$ using neural network as in \citet{daniel2022inputde}, and global inducing locations $\{\bm{k}_g^{l,h}\}_{h=1}^H$, which are shared across all input sequences.

Note that the input-dependent inducing points used during test time may not be encountered in training. Instead, we rely on the learned neural network to amortize them \citep{daniel2022inputde}. Therefore the fitted mean function may not be consistent when test sequence includes tokens far away from the tokens in training sequences (as $\bm{v}_{a}$ may not be consistent). 
Empirically, we found this inconsistency issue to be minor for in-distribution test sequences. Furthermore, we argue that for OOD inputs, although the fitted posterior mean might be unreliable, the uncertainty still increases since the posterior covariance is fully determined by the global inducing points which exhibits no inconsistency issue. Intuitively, the global keys $\{\bm{k}_g^{l,h}\}_{l=1,h=1}^{L,H}$, shared across all input sequences, play a similar role as the inducing locations in standard SVGP: they summarize the training set but focus on the uncertainty behavior only. As a result, the posterior variance in Eq.~\ref{eq:odsvgp} still increases for queries that are less similar to the global keys as measured by the kernel. By propagating the uncertainty through each layer, we can still obtain increased uncertainty for input sequences that are very different from the training data, so that users can still be notified ``when the model does not know'' \citep{gal2016thesis}.

\section{Results in Tables}
We present numerical results (mean$\pm$standard error) for all experiments in tables.

We show in-distribution results in Tables~\ref{table:imdb} to~\ref{table:zinc}.

We show OOD robustness results in Tables~\ref{table:cola_ood} to~\ref{table:ood_rob_end}.

We show OOD detection results in Tables~\ref{table:ood_detect_start} to~\ref{table:ood_detect_end}.

\begin{table}[!htb]
\centering
\caption{In-distribution performance: sentiment analysis with IMDB}
 \begin{tabular}{c  c  c c c c   } 
\hline
 & Model & Accuracy & NLL &ECE&MCE \\ 
 \hline
  \multirow{ 1}{*}{SDP} & MLE & 0.8806$\pm$0.0016 	& 0.3267$\pm$0.0068 & 0.0548$\pm$0.0029 &	0.1432$\pm$0.0093\\
  \hline
  \multirow{ 4}{*}{Kernel}& MLE & 0.8822$\pm$0.0019 	& 0.3202$\pm$0.0136 	& 0.0519$\pm$0.0055 	& 0.1568$\pm$0.0273\\
  &TS &  0.8822$\pm$0.0019&	0.3202$\pm$0.0136	& 0.0519$\pm$0.0055 & 0.1568$\pm$0.0004\\
  &MFVI & 0.8799$\pm$0.0022 & 0.3476$\pm$0.0230 & 0.0585$\pm$0.0106 & 0.1493$\pm$0.0271\\
  &MCD &0.8817$\pm$0.0014&	0.2986$\pm$0.0070&	0.0285$\pm$0.0058&	0.0910$\pm$0.0218\\
  &KFLLLA &  0.8822$\pm$0.0019&	0.4366$\pm$0.0015	& 0.1905$\pm$0.0034 & 0.2263$\pm$0.0048\\
  &SNGP &  0.8783$\pm$0.0013&	0.3499$\pm$0.0080	& 0.0646$\pm$0.0037 & 0.1741$\pm$0.0106\\
  &SGPA &  0.8875$\pm$0.0004&	0.2823$\pm$0.0027	& 0.0215$\pm$0.0034 & 0.0647$\pm$0.0105\\
\hline
\end{tabular}
\label{table:imdb}
\end{table}

\begin{table}[!htb]
\centering
\caption{In-distribution performance: linguistic acceptability with CoLA}
 \begin{tabular}{c    c c c c c  } 
\hline
 & Model & MCC & NLL &ECE&MCE \\ 
 \hline
  \multirow{ 1}{*}{SDP} & MLE & 25.0408$\pm$0.8223 & 1.8976$\pm$0.0223 &	0.2654$\pm$0.0039  & 0.3887$\pm$0.0230\\
  \hline
  \multirow{ 4}{*}{Kernel}& MLE 	& 26.0722$\pm$0.4300 	& 2.0083$\pm$0.0432 & 0.2643$\pm$0.0077 & 0.3702$\pm$0.0217\\
  &MFVI  & 21.0864$\pm$1.8124 & 1.0809$\pm$0.0092 & 0.2018$\pm$0.0080 & 0.266$\pm$0.0113\\
  &MCD  &26.4399$\pm$0.4939  &	0.9121$\pm$0.0154 & 0.2113$\pm$0.0049& 0.3233$\pm$0.0066\\
  &KFLLLA  &26.0295$\pm$0.4050  &	0.6023$\pm$0.0077 & 0.0241$\pm$0.0044& 0.1063$\pm$0.0204\\
  &SNGP  &27.8715$\pm$1.2584  &	1.2075$\pm$0.0418 & 0.2325$\pm$0.0067& 0.3230$\pm$0.0146\\
  &SGPA &  	27.2380$\pm$1.1188	& 0.9070$\pm$0.0207 & 0.2076$\pm$0.0082 & 0.3062$\pm$0.0097\\
\hline
\end{tabular}
\label{table:cola_in}
\end{table}

\begin{table}[!htb]
\centering
\caption{In-distribution performance: image classification for CIFAR10 without data augmentation}
 \begin{tabular}{c  c  c c c c   } 
\hline
 & Model & Accuracy & NLL &ECE&MCE \\ 
 \hline
  \multirow{ 1}{*}{SDP} & MLE & 0.7439$\pm$0.0015 & 2.2940$\pm$0.0277 & 0.2703$\pm$0.0021 &  0.4509$\pm$0.0079\\
  \hline
  \multirow{ 4}{*}{Kernel}& MLE & 0.7811$\pm$0.0019 & 1.3395$\pm$0.0135 & 0.2082$\pm$0.0010	& 0.3724$\pm$0.0103\\
  &TS &0.7811$\pm$0.0019 &0.8254$\pm$0.0077 & 0.1357$\pm$0.0012 & 0.2499$\pm$0.0081\\
  &MFVI &0.7202$\pm$0.0024 &0.7995$\pm$0.0073	& 0.0331$\pm$0.0023 & 0.0664$\pm$0.0064\\
  &MCD & 0.7910$\pm$0.0015	&0.7789$\pm$0.0063	& 0.0799$\pm$0.0012 &  0.1501$\pm$0.0055\\
  &KFLLLA & 0.7752$\pm$0.0019	&0.7201$\pm$0.0242	& 0.0710$\pm$0.0067&0.1497$\pm$0.0217\\
  &SNGP & 0.7837$\pm$0.0025	&0.7534$\pm$0.0031	& 0.0745$\pm$0.0003 & 0.1443$\pm$0.0049\\
  &SGPA & 0.7787$\pm$0.0024	&0.6968$\pm$0.0032	& 0.0589$\pm$0.0012 & 0.0825$\pm$0.0066\\
  &DE & 0.8235 & 0.6934 & 0.0504 &  0.0699\\
  &SGPAE &0.8172 & 0.5657 & 0.0184 &  0.0579\\
\hline
\end{tabular}
\label{table:cifar10_woaug}
\end{table}

\begin{table}[!htb]
\centering
\caption{In-distribution performance: image classification for CIFAR10 with data augmentation}
 \begin{tabular}{c  c  c c c c   } 
\hline
 & Model & Accuracy & NLL &ECE&MCE \\ 
 \hline
  \multirow{ 1}{*}{SDP} & MLE & 0.8898$\pm$0.0009&	0.6073$\pm$0.0330&	0.1526$\pm$0.0043&	0.3250$\pm$0.0118\\
  \hline
  \multirow{ 4}{*}{Kernel}& MLE & 0.8442$\pm$0.0057&	0.4700$\pm$0.0127	&0.0408$\pm$0.0031&	0.0951$\pm$0.0116\\
  & TS & 0.8442$\pm$0.0057& 0.4552$\pm$0.0176 & 0.0326$\pm$0.0028 & 0.0532$\pm$0.0077	\\
  &MFVI &0.6967$\pm$0.0022	&0.8643$\pm$0.0078&	0.0428$\pm$0.0072&	0.0635$\pm$0.0083\\
  &MCD  &0.8437$\pm$0.0074&	0.4541$\pm$0.0190	&0.0122$\pm$0.0013	&0.0428$\pm$0.0049\\
  &KFLLLA & 0.8445$\pm$0.0057&	0.4554$\pm$0.0158&	0.0279$\pm$0.0022&	0.0570$\pm$0.0090\\
  &SNGP & 0.8286$\pm$0.0118&	0.4912$\pm$0.0351&	0.0076$\pm$0.0007&	0.0398$\pm$0.0051\\
  &SGPA & 0.8475$\pm$0.0048&	0.4489$\pm$0.0121&	0.0086$\pm$0.0013&	0.0395$\pm$0.0037\\
  &DE & 0.8654 & 0.4062 & 0.0079 & 0.0724\\
  &SGPAE & 0.8691 & 0.3885 & 0.0068 & 0.0421\\
\hline
\end{tabular}
\label{table:cifar10_aug}
\end{table}

\begin{table}[!htb]
\centering
\caption{In-distribution performance: image classification for CIFAR100 without data augmentation}
 \begin{tabular}{c  c  c c c c   } 
\hline
 & Model & Accuracy & NLL &ECE&MCE \\ 
 \hline
  \multirow{ 1}{*}{SDP} & MLE & 0.4893$\pm$0.0022&	5.7677$\pm$0.0492&	0.5276$\pm$0.0021&	0.6216$\pm$0.0076\\
  \hline
  \multirow{ 4}{*}{Kernel}& MLE & 0.5216$\pm$0.0100	&4.3898$\pm$0.0829&	0.4061$\pm$0.0023 & 0.5696$\pm$0.0093\\
  & TS & 0.5216$\pm$0.0100	&  2.5863$\pm$0.0527 & 0.2769$\pm$0.0013 &  0.4067$\pm$0.0032\\
  &MFVI & 0.4117$\pm$0.0015&	2.4195$\pm$0.0113	& 0.0753$\pm$0.0002& 0.1021$\pm$0.0057\\
  &MCD & 0.5341$\pm$0.0096&	2.7122$\pm$0.0329	& 0.1791$\pm$0.0013 & 0.2651$\pm$0.0116\\
  &KFLLLA & 0.5084$\pm$0.0103&	3.2984$\pm$0.0210&	0.1682$\pm$0.0054 & 0.7690$\pm$0.0159\\
  &SNGP & 0.4715$\pm$0.0039&	2.8296$\pm$0.0338&	0.2451$\pm$0.0005 & 0.3028$\pm$0.0082\\
  &SGPA & 0.5302$\pm$0.0071&	2.5643$\pm$0.0813&	0.1730$\pm$0.0013 &0.2463$\pm$0.0071\\
  &DE & 0.5815 & 2.4887 & 0.1026 &  0.1381\\
  &SGPAE & 0.5700 & 1.9434 & 0.0467& 0.0704\\
\hline
\end{tabular}
\label{table:cifar100_woaug}
\end{table}

\begin{table}[!htb]
\centering
\caption{In-distribution performance: image classification for CIFAR100 with data augmentation}
 \begin{tabular}{c  c  c c c c   } 
\hline
 & Model & Accuracy & NLL &ECE&MCE \\ 
 \hline
  \multirow{ 1}{*}{SDP} & MLE & 0.5984$\pm$0.0024	&1.8743$\pm$0.2032	&0.1879$\pm$0.0006&	0.2792$\pm$0.0512\\
  \hline
  \multirow{ 4}{*}{Kernel}& MLE & 0.6053$\pm$0.0048	&1.5254$\pm$0.0195	&0.1066$\pm$0.0036	&0.1578$\pm$0.0159\\
  & TS & 0.6053$\pm$0.0048	&1.4512$\pm$0.0251	&0.0303$\pm$0.0027	&0.0825$\pm$0.0109\\
  &MFVI &0.4496$\pm$0.0045&	2.1132$\pm$0.0105&	0.0527$\pm$0.0063&	0.1339$\pm$0.0156\\
  &MCD & 0.6149$\pm$0.0074&	1.4255$\pm$0.0275&	0.0231$\pm$0.0032&	0.0468$\pm$0.0056\\
  &KFLLLA & 0.5994$\pm$0.0038&	1.4661$\pm$0.0206&	0.0198$\pm$0.0022&	0.0474$\pm$0.0071\\
  &SNGP & 0.5645$\pm$0.0106&	1.5839$\pm$0.0416&	0.0254$\pm$0.0018&	0.0424$\pm$0.0060\\
  &SGPA &0.6154$\pm$0.0010&	1.4486$\pm$0.0055	&0.0364$\pm$0.0055&	0.0612$\pm$0.0067\\
  &DE &0.6471 &	1.3023	&0.0565&	0.0986\\
  &SGPEA &0.65&	1.2761	&0.0191&	0.0593\\
\hline
\end{tabular}
\label{table:cifar100_aug}
\end{table}

\begin{table}[!htb]
\centering
\caption{In-distribution performance: graph property regression with ZINC dataset}
 \begin{tabular}{c  c  c c  c   } 
\hline
 & Model & MAE & RMSE  & NLL \\ 
 \hline
  \multirow{ 1}{*}{SDP} & MLE & 0.5733$\pm$0.0084 & 1.1459$\pm$0.0139 &	1.1459$\pm$0.0139\\
  \hline
  \multirow{ 3}{*}{Kernel}& MLE & 0.5191$\pm$0.0344 	& 1.0977$\pm$0.0869 	&  1.0977$\pm$0.0869\\
  &MCD &0.4385$\pm$0.0163&	0.8369$\pm$0.0308& 0.9264$\pm$0.0317\\
  &MFVI &  0.5702$\pm$0.0086 & 1.1050$\pm$0.0115 & 1.1683$\pm$0.0101 \\
  &SGPA &  0.4341$\pm$0.0142&	0.8199$\pm$0.0195	 & 0.9319$\pm$0.0242\\
  &DE &0.4711	&0.8635	&	1.0949\\
  &SGPAE &0.4557	&0.8479	&1.0118\\
\hline
\end{tabular}
\label{table:zinc}
\end{table}

\begin{table}[!htb]
\centering
\caption{OOD robustness: linguistic acceptability with CoLA dataset}
 \begin{tabular}{c   c c c c c  } 
\hline
 & Model &  MCC & NLL &ECE&MCE \\ 
 \hline
  \multirow{ 1}{*}{SDP} & MLE  	& 19.4835$\pm$2.0234 & 2.0574$\pm$0.0516 &	0.2914$\pm$0.0098  & 0.4913$\pm$0.0418\\
  \hline
  \multirow{ 4}{*}{Kernel}& MLE &  18.6467$\pm$1.6116 	& 2.3565$\pm$0.0974 & 0.2909$\pm$0.0067 & 0.3864$\pm$0.0303\\
  &MFVI &  16.6363$\pm$1.5793 & 1.1349$\pm$0.0445 & 0.2184$\pm$0.0035 & 0.3289$\pm$0.0151\\
  &MCD &	20.4774$\pm$1.4727 &	1.0114$\pm$0.0234 & 0.2370$\pm$0.0075 & 0.3439$\pm$0.0177\\
  &KFLLLA &	18.6250$\pm$1.6321 &	0.6299$\pm$0.0067 & 0.0361$\pm$0.0065 & 0.0876$\pm$0.0065\\
  &SNGP &	 23.6410$\pm$1.5986	& 1.3070$\pm$0.0412 & 0.2537$\pm$0.0049 & 0.3828$\pm$0.0189\\
  &SGPA &  22.9190$\pm$1.5986	& 0.9514$\pm$0.0261 & 0.2257$\pm$0.0069 & 0.3399$\pm$0.0294\\
\hline
\end{tabular}
\label{table:cola_ood}
\end{table}

\begin{table}[!htb]
\centering
\caption{Accuracy of CIFAR10-C for ViTs trained on clean data without data augmentation.}
\begin{adjustbox}{width=\textwidth}
\begin{tabular}{c c c c c c c}
\hline
& \multicolumn{6}{c}{Skew Intensity}\\
 & Model &                1 &                2 &              3 &               4 &               5 \\
 \hline
   \multirow{ 1}{*}{SDP} & MLE &  0.7120$\pm$0.0011 &  0.6704$\pm$0.0011 &  0.6359$\pm$0.0014 &  0.5883$\pm$0.0017 &  0.5354$\pm$0.0015 \\
  \hline
   \multirow{ 4}{*}{Kernel}&   MLE &  0.7254$\pm$0.0019 &  0.6632$\pm$0.0028 &  0.6139$\pm$0.0023 &  0.5494$\pm$0.0028 &  0.4863$\pm$0.0027 \\
 &  MFVI &  0.6606$\pm$0.0093 &  0.6109$\pm$0.0083 &  0.5730$\pm$0.0077 &  0.5266$\pm$0.0073 &  0.4778$\pm$0.0065 \\
 &   MCD &  0.7327$\pm$0.0013 &  0.6717$\pm$0.0015 &  0.6229$\pm$0.0016 &  0.5603$\pm$0.0019 &  0.4982$\pm$0.0020 \\
  &  KFLLLA &  0.7186$\pm$0.0014 &  0.6580$\pm$0.0035 &  0.6092$\pm$0.0038 &  0.5462$\pm$0.0042 &  0.4825$\pm$0.0038 \\
 &       SNGP &  0.7281$\pm$0.0024 &  0.6676$\pm$0.0025 &  0.6208$\pm$0.0020 &  0.5622$\pm$0.0020 &  0.5006$\pm$0.0019 \\
 &  SGPA &  0.7175$\pm$0.0038 &  0.6559$\pm$0.0042 &  0.6067$\pm$0.0040 &  0.5466$\pm$0.0033 &  0.4832$\pm$0.0046 \\
  &         DE &             0.7700 &             0.7089 &             0.6585 &             0.5908 &             0.5246 \\
 &      SGPAE &             0.7516 &             0.6911 &             0.6401 &             0.5779 &             0.5173 \\
   \hline
\end{tabular}
\end{adjustbox}
\end{table}

\begin{table}[!htb]
\centering
\caption{Accuracy of CIFAR10-C for ViTs trained on clean data with data augmentation.}
\begin{adjustbox}{width=\textwidth}
\begin{tabular}{c c c c c c c}
\hline
& \multicolumn{6}{c}{Skew Intensity}\\
 & Model &                1 &                2 &              3 &               4 &               5 \\
 \hline
   \multirow{ 1}{*}{SDP} 
& MLE &  0.8516$\pm$0.0010 &  0.8139$\pm$0.0011 &  0.7752$\pm$0.0010 &  0.7165$\pm$0.0013 &  0.6500$\pm$0.0012 \\
   \hline
   \multirow{ 4}{*}{Kernel}&   MLE &  0.7944$\pm$0.0061 &  0.7499$\pm$0.0074 &  0.7090$\pm$0.0082 &  0.6465$\pm$0.0085 &  0.5807$\pm$0.0094 \\
 &  MFVI &  0.6605$\pm$0.0016 &  0.6222$\pm$0.0024 &  0.5875$\pm$0.0028 &  0.5377$\pm$0.0032 &  0.4913$\pm$0.0040 \\
 &   MCD &  0.7939$\pm$0.0068 &  0.7480$\pm$0.0082 &  0.7052$\pm$0.0089 &  0.6415$\pm$0.0092 &  0.5751$\pm$0.0097 \\
  &  KFLLLA &  0.7949$\pm$0.0030 &  0.7508$\pm$0.0037 &  0.7100$\pm$0.0041 &  0.6480$\pm$0.0043 &  0.5821$\pm$0.0047 \\
 &       SNGP &  0.7783$\pm$0.0062 &  0.7336$\pm$0.0069 &  0.6932$\pm$0.0069 &  0.6314$\pm$0.0062 &  0.5710$\pm$0.0057 \\
 &  SGPA &  0.7910$\pm$0.0027 &  0.7426$\pm$0.0021 &  0.6946$\pm$0.0075 &  0.6356$\pm$0.0033 &  0.5712$\pm$0.0020 \\
 &         DE &             0.8237 &             0.7819 &             0.7428 &             0.6795 &             0.6130 \\
 &      SGPAE &             0.8105 &             0.7646 &             0.7220 &             0.6581 &             0.5914 \\
   \hline
\end{tabular}
\end{adjustbox}
\end{table}

\begin{table}[!htb]
\centering
\caption{Accuracy of CIFAR100-C for ViTs trained on clean data without data augmentation.}
\begin{adjustbox}{width=\textwidth}
\begin{tabular}{c c c c c c c}
\hline
& \multicolumn{6}{c}{Skew Intensity}\\
 & Model &                1 &                2 &              3 &               4 &               5 \\
 \hline
   \multirow{ 1}{*}{SDP} 
&   MLE &  0.3947$\pm$0.0022 &  0.3484$\pm$0.0018 &  0.3223$\pm$0.0014 &  0.2805$\pm$0.0013 &  0.2402$\pm$0.0009 \\
\hline
   \multirow{ 4}{*}{Kernel} &   MLE &  0.4456$\pm$0.0075 &  0.3769$\pm$0.0058 &  0.3373$\pm$0.0049 &  0.2835$\pm$0.0034 &  0.2339$\pm$0.0029 \\
 &  MFVI &  0.3380$\pm$0.0020 &  0.2832$\pm$0.0022 &  0.2540$\pm$0.0018 &  0.2155$\pm$0.0014 &  0.1803$\pm$0.0015 \\
 &   MCD &  0.4526$\pm$0.0074 &  0.3840$\pm$0.0061 &  0.3446$\pm$0.0052 &  0.2913$\pm$0.0037 &  0.2403$\pm$0.0031 \\
  &  KFLLLA &  0.4350$\pm$0.0070 &  0.3656$\pm$0.0053 &  0.3251$\pm$0.0045 &  0.2717$\pm$0.0030 &  0.2229$\pm$0.0025 \\
 &       SNGP &  0.4053$\pm$0.0020 &  0.3455$\pm$0.0014 &  0.3093$\pm$0.0015 &  0.2624$\pm$0.0015 &  0.2203$\pm$0.0011 \\
 &  SGPA &  0.4472$\pm$0.0050 &  0.3764$\pm$0.0035 &  0.3371$\pm$0.0027 &  0.2828$\pm$0.0021 &  0.2317$\pm$0.0022 \\
  &         DE &             0.5054 &             0.4289 &             0.3858 &             0.3234 &             0.2690 \\
 &      SGPAE &             0.4848 &             0.4104 &             0.3675 &             0.3087 &             0.2539 \\
 \hline
\end{tabular}
\end{adjustbox}
\end{table}

\begin{table}[!htb]
\centering
\caption{Accuracy of CIFAR100-C for ViTs trained on clean data with data augmentation.}
\begin{adjustbox}{width=\textwidth}
\begin{tabular}{c c c c c c c}
\hline
& \multicolumn{6}{c}{Skew Intensity}\\
 & Model &                1 &                2 &              3 &               4 &               5 \\
 \hline
   \multirow{ 1}{*}{SDP} &
 MLE &  0.5384$\pm$0.0030 &  0.4854$\pm$0.0038 &  0.4464$\pm$0.0041 &  0.3930$\pm$0.0037 &  0.3410$\pm$0.0029 \\
 \hline
   \multirow{ 4}{*}{Kernel} &   MLE &  0.5327$\pm$0.0018 &  0.4700$\pm$0.0018 &  0.4264$\pm$0.0017 &  0.3690$\pm$0.0017 &  0.3134$\pm$0.0013 \\
 &  MFVI &  0.3991$\pm$0.0019 &  0.3535$\pm$0.0017 &  0.3224$\pm$0.0019 &  0.2836$\pm$0.0017 &  0.2444$\pm$0.0014 \\
 &   MCD &  0.5397$\pm$0.0020 &  0.4775$\pm$0.0018 &  0.4349$\pm$0.0015 &  0.3777$\pm$0.0011 &  0.3213$\pm$0.0010 \\
 &  KFLLLA &  0.5196$\pm$0.0019 &  0.4577$\pm$0.0018 &  0.4148$\pm$0.0017 &  0.3581$\pm$0.0014 &  0.3035$\pm$0.0014 \\
 &       SNGP &  0.4979$\pm$0.0038 &  0.4440$\pm$0.0035 &  0.4042$\pm$0.0034 &  0.3534$\pm$0.0029 &  0.3066$\pm$0.0023 \\
 &  SGPA &  0.5189$\pm$0.0054 &  0.4603$\pm$0.0010 &  0.4181$\pm$0.0013 &  0.3625$\pm$0.0014 &  0.3061$\pm$0.0014 \\
 &         DE &             0.5742 &             0.5104 &             0.4649 &             0.4039 &             0.3449 \\
 &      SGPAE &             0.5592 &             0.4936 &             0.4476 &             0.3908 &             0.3316 \\
 \hline
\end{tabular}
\end{adjustbox}
\end{table}

\begin{table}[!htb]
\centering
\caption{NLL of CIFAR10-C for ViTs trained on clean data without data augmentation.}
\begin{adjustbox}{width=\textwidth}
\begin{tabular}{c c c c c c c}
\hline
& \multicolumn{6}{c}{Skew Intensity}\\
 & Model &                1 &                2 &              3 &               4 &               5 \\
 \hline
   \multirow{ 1}{*}{SDP}&MLE &  2.6411$\pm$0.0202 &  3.1620$\pm$0.0202 &  3.6449$\pm$0.0231 &  4.3843$\pm$0.0283 &  5.2626$\pm$0.0287 \\
 \hline
   \multirow{ 4}{*}{Kernel}&   MLE &  1.7906$\pm$0.0179 &  2.3460$\pm$0.0321 &  2.8899$\pm$0.0385 &  3.6899$\pm$0.0544 &  4.5003$\pm$0.0811 \\
 &  MFVI &  0.9761$\pm$0.0274 &  1.1345$\pm$0.0267 &  1.2737$\pm$0.0259 &  1.4705$\pm$0.0237 &  1.7038$\pm$0.0222 \\
 &   MCD &  1.1466$\pm$0.0099 &  1.5153$\pm$0.0168 &  1.8926$\pm$0.0234 &  2.4692$\pm$0.0315 &  3.1086$\pm$0.0431 \\
  &  KFLLLA &  0.9573$\pm$0.0498 &  1.2092$\pm$0.0577 &  1.4435$\pm$0.0683 &  1.7889$\pm$0.0937 &  2.1537$\pm$0.1198 \\
 &       SNGP &  1.0579$\pm$0.0066 &  1.3681$\pm$0.0115 &  1.6570$\pm$0.0129 &  2.0807$\pm$0.0136 &  2.5271$\pm$0.0146 \\
 &  SGPA &  0.9603$\pm$0.0128 &  1.2340$\pm$0.0213 &  1.4972$\pm$0.0283 &  1.8951$\pm$0.0383 &  2.3812$\pm$0.0586 \\
  &         DE &             0.9317 &             1.2424 &             1.5716 &             2.0911 &             2.6373 \\
 &      SGPAE &             0.7816 &             0.9985 &             1.2125 &             1.5365 &             1.9349 \\
\hline
\end{tabular}
\end{adjustbox}
\end{table}

\begin{table}[!htb]
\centering
\caption{NLL of CIFAR10-C for ViTs trained on clean data with data augmentation.}
\begin{adjustbox}{width=\textwidth}
\begin{tabular}{c c c c c c c}
\hline
& \multicolumn{6}{c}{Skew Intensity}\\
 & Model &                1 &                2 &              3 &               4 &               5 \\
 \hline
   \multirow{ 1}{*}{SDP}&MLE &  0.8306$\pm$0.0422 &  1.0712$\pm$0.0559 &  1.3665$\pm$0.0753 &  1.8795$\pm$0.1085 &  2.5810$\pm$0.1571 \\
 \hline
   \multirow{ 4}{*}{Kernel}&   MLE &  0.6364$\pm$0.0140 &  0.7960$\pm$0.0207 &  0.9614$\pm$0.0296 &  1.2404$\pm$0.0416 &  1.5818$\pm$0.0538 \\
 &  MFVI &  0.9616$\pm$0.0044 &  1.0706$\pm$0.0061 &  1.1802$\pm$0.0098 &  1.3612$\pm$0.0160 &  1.5629$\pm$0.0248 \\
 &   MCD &  0.6067$\pm$0.0180 &  0.7579$\pm$0.0236 &  0.9144$\pm$0.0304 &  1.1707$\pm$0.0363 &  1.4977$\pm$0.0420 \\
 &  KFLLLA &  0.6104$\pm$0.0080 &  0.7565$\pm$0.0108 &  0.9044$\pm$0.0144 &  1.1496$\pm$0.0183 &  1.4487$\pm$0.0230 \\
 &       SNGP &  0.6439$\pm$0.0187 &  0.7874$\pm$0.0218 &  0.9304$\pm$0.0225 &  1.1773$\pm$0.0223 &  1.4671$\pm$0.0225 \\
 &  SGPA &  0.6204$\pm$0.0054 &  0.7739$\pm$0.0054 &  0.9298$\pm$0.0047 &  1.1873$\pm$0.0118 &  1.5058$\pm$0.0226 \\
 &         DE &             0.5641 &             0.7028 &             0.8441 &             1.0737 &             1.3595 \\
 &      SGPAE &             0.5230 &             0.6503 &             0.7814 &             1.0027 &             1.2835 \\
 \hline
\end{tabular}
\end{adjustbox}
\end{table}

\begin{table}[!htb]
\centering
\caption{NLL of CIFAR100-C for ViTs trained on clean data without data augmentation.}
\begin{adjustbox}{width=\textwidth}
\begin{tabular}{c c c c c c c}
\hline
& \multicolumn{6}{c}{Skew Intensity}\\
 & Model &                1 &                2 &              3 &               4 &               5 \\
 \hline
   \multirow{ 1}{*}{SDP}&MLE &  7.4649$\pm$0.0475 &  8.5233$\pm$0.0524 &  9.1964$\pm$0.0527 &  10.4061$\pm$0.0555 &  11.8540$\pm$0.0649 \\
 \hline\multirow{ 4}{*}{Kernel}&   MLE &  5.5126$\pm$0.0578 &  6.7077$\pm$0.0416 &  7.5503$\pm$0.0237 &   9.0093$\pm$0.0353 &  10.5647$\pm$0.0699 \\
 &  MFVI &  3.0138$\pm$0.0132 &  3.4762$\pm$0.0136 &  3.7921$\pm$0.0135 &   4.3295$\pm$0.0153 &   4.9042$\pm$0.0200 \\
 &   MCD &  3.8996$\pm$0.0266 &  4.8412$\pm$0.0178 &  5.5397$\pm$0.0269 &   6.7984$\pm$0.0591 &   8.2454$\pm$0.0959 \\
  &  KFLLLA &  3.4366$\pm$0.0172 &  3.5660$\pm$0.0160 &  3.6457$\pm$0.0152 &  3.7645$\pm$0.0135 &   3.8862$\pm$0.0115 \\
 &       SNGP &  3.6404$\pm$0.0297 &  4.3067$\pm$0.0360 &  4.7994$\pm$0.0446 &  5.6242$\pm$0.0502 &   6.5029$\pm$0.0504 \\
 &  SGPA &  3.6565$\pm$0.1290 &  4.5188$\pm$0.1401 &  5.1522$\pm$0.1510 &   6.3096$\pm$0.1784 &   7.6306$\pm$0.2161 \\
  &         DE &             3.1930 &             4.0004 &             4.6072 &             5.7357 &              6.9765 \\
 &      SGPAE &             2.8407 &             3.5393 &             4.0466 &             4.9890 &              6.0803 \\
 \hline
\end{tabular}
\end{adjustbox}
\end{table}

 \begin{table}[!htb]
\centering
\caption{NLL of CIFAR100-C for ViTs trained on clean data with data augmentation.}
\begin{adjustbox}{width=\textwidth}
\begin{tabular}{c c c c c c c}
\hline
& \multicolumn{6}{c}{Skew Intensity}\\
 & Model &                1 &                2 &              3 &               4 &               5 \\
 \hline
   \multirow{ 1}{*}{SDP}&MLE &  2.2595$\pm$0.2276 &  2.6313$\pm$0.2735 &  2.9786$\pm$0.3148 &  3.5393$\pm$0.3720 &  4.1865$\pm$0.4426 \\
 \hline\multirow{ 4}{*}{Kernel}&   MLE &  1.9695$\pm$0.0167 &  2.3721$\pm$0.0260 &  2.7423$\pm$0.0346 &  3.3355$\pm$0.0500 &  3.9813$\pm$0.0602 \\
 &  MFVI &  2.3501$\pm$0.0073 &  2.5702$\pm$0.0084 &  2.7571$\pm$0.0112 &  3.0582$\pm$0.0123 &  3.3805$\pm$0.0147 \\
 &   MCD &  1.8337$\pm$0.0048 &  2.1874$\pm$0.0073 &  2.5076$\pm$0.0112 &  3.0263$\pm$0.0205 &  3.6196$\pm$0.0273 \\
  &  KFLLLA &  1.8645$\pm$0.0092 &  2.2008$\pm$0.0095 &  2.4935$\pm$0.0095 &  2.9465$\pm$0.0112 &  3.4418$\pm$0.0135 \\
 &       SNGP &  1.9167$\pm$0.0161 &  2.2067$\pm$0.0162 &  2.4753$\pm$0.0163 &  2.9166$\pm$0.0133 &  3.3598$\pm$0.0116 \\
 &  SGPA &  1.9307$\pm$0.0075 &  2.3154$\pm$0.0104 &  2.6598$\pm$0.0130 &  3.2105$\pm$0.0151 &  3.8575$\pm$0.0180 \\
  &         DE &             1.7354 &             2.0787 &             2.3857 &             2.8802 &             3.4604 \\
 &      SGPAE &             1.6583 &             1.9839 &             2.2802 &             2.7592 &             3.2976 \\
 \hline
\end{tabular}
\end{adjustbox}
\end{table}

\begin{table}[!htb]
\centering
\caption{ECE of CIFAR10-C for ViTs trained on clean data without data augmentation.}
\begin{adjustbox}{width=\textwidth}
\begin{tabular}{c c c c c c c}
\hline
& \multicolumn{6}{c}{Skew Intensity}\\
 & Model &                1 &                2 &              3 &               4 &               5 \\
 \hline
   \multirow{ 1}{*}{SDP}&
MLE &  0.2919$\pm$0.0021 &  0.3297$\pm$0.0021 &  0.3623$\pm$0.0023 &  0.3964$\pm$0.0026 &  0.4458$\pm$0.0030 \\
\hline\multirow{ 4}{*}{Kernel}& MLE & 0.2514$\pm$0.0019 & 0.3003$\pm$0.0020 & 0.3396$\pm$0.0026 & 0.3898$\pm$0.0028 & 0.4355$\pm$0.0032 \\
 & MFVI & 0.0324$\pm$0.0013 & 0.0767$\pm$0.0019 & 0.1031$\pm$0.0021 & 0.1273$\pm$0.0025 & 0.1684$\pm$0.0025 \\
 & MCD & 0.1191$\pm$0.0011 & 0.1589$\pm$0.0013 & 0.1928$\pm$0.0013 & 0.2390$\pm$0.0015 & 0.2831$\pm$0.0017 \\
 & KFLLLA & 0.1044$\pm$0.0012 & 0.1437$\pm$0.0017 & 0.1827$\pm$0.0017 & 0.2306$\pm$0.0029 & 0.2674$\pm$0.0035 \\
 & SNGP & 0.1095$\pm$0.0011 & 0.1477$\pm$0.0013 & 0.1860$\pm$0.0015 & 0.2337$\pm$0.0015 & 0.2724$\pm$0.0018 \\
 & SGPA & 0.0894$\pm$0.0013 & 0.1305$\pm$0.0018 & 0.1653$\pm$0.0020 & 0.2137$\pm$0.0024 & 0.2562$\pm$0.0024 \\
 & DE & 0.0771 & 0.1095 & 0.1406 & 0.1826 & 0.2168 \\
 & SGPAE & 0.0323 & 0.0644 & 0.0946 & 0.1381 & 0.1736 \\
\hline
\end{tabular}
\end{adjustbox}
\end{table}

\begin{table}[!htb]
\centering
\caption{ECE of CIFAR10-C for ViTs trained on clean data with data augmentation.}
\begin{adjustbox}{width=\textwidth}
\begin{tabular}{c c c c c c c}
\hline
& \multicolumn{6}{c}{Skew Intensity}\\
 & Model &                1 &                2 &              3 &               4 &               5 \\
 \hline
   \multirow{ 1}{*}{SDP}&
MLE &  0.1731$\pm$0.0015 &  0.1962$\pm$0.0015 &  0.2094$\pm$0.0017 &  0.2582$\pm$0.0021 &  0.3046$\pm$0.0031 \\
\hline\multirow{ 4}{*}{Kernel} & MLE & 0.0666$\pm$0.0014 & 0.0920$\pm$0.0016 & 0.1170$\pm$0.0018 & 0.1605$\pm$0.0027 & 0.2090$\pm$0.0036 \\
 & MFVI & 0.0295$\pm$0.0006 & 0.0519$\pm$0.0008 & 0.0715$\pm$0.0015 & 0.1097$\pm$0.0019 & 0.1534$\pm$0.0026 \\
 & MCD & 0.0245$\pm$0.0006 & 0.0433$\pm$0.0008 & 0.0657$\pm$0.0013 & 0.1084$\pm$0.0017 & 0.1545$\pm$0.0023 \\
 & KFAC-LLLA & 0.0352$\pm$0.0006 & 0.0587$\pm$0.0008 & 0.0797$\pm$0.0014 & 0.1175$\pm$0.0018 & 0.1634$\pm$0.0024 \\
 & SNGP & 0.0260$\pm$0.0004 & 0.0368$\pm$0.0009 & 0.0555$\pm$0.0016 & 0.0929$\pm$0.0025 & 0.1370$\pm$0.0035 \\
 & SGPA & 0.0217$\pm$0.0004 & 0.0368$\pm$0.0009 & 0.0565$\pm$0.0017 & 0.0971$\pm$0.0027 & 0.1421$\pm$0.0038 \\
 & DE & 0.0291 & 0.0369 & 0.0549 & 0.0904 & 0.1338 \\
 & SGPAE & 0.0224 & 0.0351 & 0.0498 & 0.0821 & 0.1196 \\
 \hline
\end{tabular}
\end{adjustbox}
\end{table}

\begin{table}[!htb]
\centering
\caption{ECE of CIFAR100-C for ViTs trained on clean data without data augmentation.}
\begin{adjustbox}{width=\textwidth}
\begin{tabular}{c c c c c c c}
\hline
& \multicolumn{6}{c}{Skew Intensity}\\
 & Model &                1 &                2 &              3 &               4 &               5 \\
 \hline
   \multirow{ 1}{*}{SDP}&
MLE &  0.5735$\pm$0.0032 &  0.6016$\pm$0.0028 &  0.6272$\pm$0.0031 &  0.6450$\pm$0.0027 &  0.6799$\pm$0.0024 \\
\hline\multirow{ 4}{*}{Kernel} & MLE & 0.4588$\pm$0.0028 & 0.5095$\pm$0.0030 & 0.5393$\pm$0.0028 & 0.5826$\pm$0.0019 & 0.6209$\pm$0.0020 \\
 & MFVI & 0.1014$\pm$0.0012 & 0.1458$\pm$0.0011 & 0.1592$\pm$0.0012 & 0.2092$\pm$0.0013 & 0.2575$\pm$0.0023 \\
 & MCD & 0.2376$\pm$0.0010 & 0.2799$\pm$0.0009 & 0.3086$\pm$0.0017 & 0.3562$\pm$0.0023 & 0.4012$\pm$0.0031 \\
 & KFAC-LLLA & 0.1753$\pm$0.0029 & 0.1740$\pm$0.0031 & 0.1553$\pm$0.0029 & 0.1478$\pm$0.0020 & 0.1427$\pm$0.0020 \\
 & SNGP & 0.2427$\pm$0.0010 & 0.2839$\pm$0.0010 & 0.3236$\pm$0.0017 & 0.3621$\pm$0.0023 & 0.4056$\pm$0.0031 \\
 & SGPA & 0.2337$\pm$0.0012 & 0.2763$\pm$0.0011 & 0.3045$\pm$0.0012 & 0.3520$\pm$0.0012 & 0.3963$\pm$0.0023 \\
 & DE & 0.1696 & 0.2117 & 0.2400 & 0.2877 & 0.3283 \\
 & SGPAE & 0.0972 & 0.1390 & 0.1665 & 0.2149 & 0.2605 \\
\hline
\end{tabular}
\end{adjustbox}
\end{table}

\begin{table}[!htb]
\centering
\caption{ECE of CIFAR100-C for ViTs trained on clean data with data augmentation.}
\begin{adjustbox}{width=\textwidth}
\begin{tabular}{c c c c c c c}
\hline
& \multicolumn{6}{c}{Skew Intensity}\\
 & Model &                1 &                2 &              3 &               4 &               5 \\
 \hline
   \multirow{ 1}{*}{SDP}&
MLE &  0.2179$\pm$0.0022 &  0.2485$\pm$0.0020 &  0.2901$\pm$0.0023 &  0.3327$\pm$0.0025 &  0.3816$\pm$0.0025 \\
 \hline\multirow{ 4}{*}{Kernel}&MLE & 0.1459$\pm$0.0020 & 0.1793$\pm$0.0021 & 0.2067$\pm$0.0022 & 0.2454$\pm$0.0023 & 0.2797$\pm$0.0023 \\
 & MFVI & 0.0478$\pm$0.0032 & 0.0409$\pm$0.0038 & 0.0510$\pm$0.0040 & 0.0681$\pm$0.0042 & 0.0824$\pm$0.0044 \\
 & MCD & 0.0614$\pm$0.0021 & 0.0919$\pm$0.0025 & 0.1197$\pm$0.0028 & 0.1596$\pm$0.0028 & 0.1979$\pm$0.0032 \\
 & KFAC-LLLA & 0.0493$\pm$0.0045 & 0.0635$\pm$0.0057 & 0.0862$\pm$0.0053 & 0.1323$\pm$0.0045 & 0.1764$\pm$0.0047 \\
 & SNGP & 0.0368$\pm$0.0019 & 0.0543$\pm$0.0025 & 0.0779$\pm$0.0027 & 0.1192$\pm$0.0025 & 0.1429$\pm$0.0031 \\
 & SGPA & 0.0742$\pm$0.0040 & 0.1060$\pm$0.0045 & 0.1341$\pm$0.0046 & 0.1733$\pm$0.0044 & 0.2124$\pm$0.0045 \\
 & DE & 0.0355 & 0.0562 & 0.0780 & 0.1144 & 0.1440 \\
 & SGPAE & 0.0434 & 0.0511 & 0.0680 & 0.1002 & 0.1308 \\
\hline
\end{tabular}
\end{adjustbox}
\end{table}

\begin{table}[!htb]
\centering
\caption{MCE of CIFAR10-C for ViTs trained on clean data without data augmentation.}
\begin{adjustbox}{width=\textwidth}
\begin{tabular}{c c c c c c c}
\hline
& \multicolumn{6}{c}{Skew Intensity}\\
 & Model &                1 &                2 &              3 &               4 &               5 \\
 \hline
   \multirow{ 1}{*}{SDP}&
MLE &  0.4666$\pm$0.0027 &  0.4840$\pm$0.0026 &  0.5041$\pm$0.0018 &  0.5280$\pm$0.0019 &  0.5534$\pm$0.0020 \\
\hline\multirow{ 4}{*}{Kernel} &   MLE &  0.4038$\pm$0.0028 &  0.4362$\pm$0.0020 &  0.4655$\pm$0.0020 &  0.5051$\pm$0.0033 &  0.5442$\pm$0.0036 \\
 &  MFVI &  0.0598$\pm$0.0112 &  0.0852$\pm$0.0150 &  0.1125$\pm$0.0147 &  0.1586$\pm$0.0147 &  0.2011$\pm$0.0149 \\
 &   MCD &  0.1902$\pm$0.0021 &  0.2402$\pm$0.0025 &  0.2807$\pm$0.0026 &  0.3413$\pm$0.0036 &  0.3949$\pm$0.0041 \\
 & KFLLLA &  0.1987$\pm$0.0251 &  0.2362$\pm$0.0241 &  0.2729$\pm$0.0238 &  0.3248$\pm$0.0220 &  0.3733$\pm$0.0210 \\
 &  SNGP &  0.1854$\pm$0.0018 &  0.2337$\pm$0.0021 &  0.2710$\pm$0.0011 &  0.3229$\pm$0.0021 &  0.3798$\pm$0.0008 \\
 &  SGPA &  0.1273$\pm$0.0042 &  0.1801$\pm$0.0046 &  0.2212$\pm$0.0061 &  0.2807$\pm$0.0057 &  0.3340$\pm$0.0047 \\
 &         DE &             0.1051 &             0.1402 &             0.1813 &             0.2476 &             0.3091 \\
&      SGPAE &             0.0570 &             0.0816 &             0.1173 &             0.1757 &             0.2208 \\
\hline
\end{tabular}
\end{adjustbox}
\end{table}

\begin{table}[!htb]
\centering
\caption{MCE of CIFAR10-C for ViTs trained on clean data with data augmentation.}
\begin{adjustbox}{width=\textwidth}
\begin{tabular}{c c c c c c c}
\hline
& \multicolumn{6}{c}{Skew Intensity}\\
 & Model &                1 &                2 &              3 &               4 &               5 \\
 \hline
   \multirow{ 1}{*}{SDP}&
MLE &  0.3500$\pm$0.0125 &  0.3649$\pm$0.0126 &  0.3870$\pm$0.0123 &  0.4186$\pm$0.0122 &  0.4561$\pm$0.0129 \\
 \hline\multirow{ 4}{*}{Kernel}&   MLE &  0.1188$\pm$0.0093 &  0.1451$\pm$0.0091 &  0.1776$\pm$0.0090 &  0.2341$\pm$0.0092 &  0.2886$\pm$0.0090 \\
 &  MFVI &  0.0772$\pm$0.0048 &  0.0965$\pm$0.0025 &  0.1283$\pm$0.0052 &  0.1903$\pm$0.0067 &  0.2261$\pm$0.0090 \\
 &   MCD &  0.0566$\pm$0.0032 &  0.0788$\pm$0.0031 &  0.1110$\pm$0.0048 &  0.1720$\pm$0.0062 &  0.2298$\pm$0.0057 \\
  &  KFLLLA &  0.0808$\pm$0.0031 &  0.1035$\pm$0.0028 &  0.1328$\pm$0.0030 &  0.1845$\pm$0.0036 &  0.2396$\pm$0.0036 \\
 &       SNGP &  0.0557$\pm$0.0013 &  0.0797$\pm$0.0029 &  0.1093$\pm$0.0039 &  0.1764$\pm$0.0040 &  0.2318$\pm$0.0044 \\
 &  SGPA &  0.0569$\pm$0.0022 &  0.0781$\pm$0.0026 &  0.1059$\pm$0.0030 &  0.1601$\pm$0.0054 &  0.2167$\pm$0.0050 \\
  &         DE &             0.0868 &             0.0855 &             0.1003 &             0.1334 &             0.1705 \\
 &      SGPAE &             0.0627 &             0.0659 &             0.0882 &             0.1262 &             0.1776 \\
\hline
\end{tabular}
\end{adjustbox}
\end{table}

\begin{table}[!htb]
\centering
\caption{MCE of CIFAR100-C for ViTs trained on clean data without data augmentation.}
\begin{adjustbox}{width=\textwidth}
\begin{tabular}{c c c c c c c}
\hline
& \multicolumn{6}{c}{Skew Intensity}\\
 & Model &                1 &                2 &              3 &               4 &               5 \\
 \hline
   \multirow{ 1}{*}{SDP}&
MLE &  0.6608$\pm$0.0014 &  0.6852$\pm$0.0014 &  0.7009$\pm$0.0022 &  0.7251$\pm$0.0020 &  0.7542$\pm$0.0011 \\
\hline\multirow{ 4}{*}{Kernel} &   MLE &  0.6037$\pm$0.0032 &  0.6428$\pm$0.0031 &  0.6649$\pm$0.0027 &  0.7037$\pm$0.0007 &  0.7385$\pm$0.0010 \\
 &  MFVI &  0.1925$\pm$0.0031 &  0.2594$\pm$0.0039 &  0.3069$\pm$0.0055 &  0.3809$\pm$0.0044 &  0.4562$\pm$0.0027 \\
 &   MCD &  0.3424$\pm$0.0042 &  0.4103$\pm$0.0047 &  0.4554$\pm$0.0045 &  0.5255$\pm$0.0052 &  0.5902$\pm$0.0038 \\
  &  KFLLLA &  0.7691$\pm$0.0165 &  0.7632$\pm$0.0181 &  0.7512$\pm$0.0194 &  0.7248$\pm$0.0146 &  0.6699$\pm$0.0228 \\
 &       SNGP &  0.3631$\pm$0.0042 &  0.4201$\pm$0.0044 &  0.4624$\pm$0.0039 &  0.5177$\pm$0.0043 &  0.5749$\pm$0.0037 \\
 &  SGPA &  0.3217$\pm$0.0057 &  0.3886$\pm$0.0083 &  0.4353$\pm$0.0093 &  0.5047$\pm$0.0083 &  0.5726$\pm$0.0072 \\
  &         DE &             0.1767 &             0.2257 &             0.2655 &             0.3522 &             0.4202 \\
 &      SGPAE &             0.1247 &             0.1766 &             0.2209 &             0.3142 &             0.3965 \\
\hline
\end{tabular}
\end{adjustbox}
\end{table}

\begin{table}[!htb]
\centering
\caption{MCE of CIFAR100-C for ViTs trained on clean data with data augmentation.}
\begin{adjustbox}{width=\textwidth}
\begin{tabular}{c c c c c c c}
\hline
& \multicolumn{6}{c}{Skew Intensity}\\
 & Model &                1 &                2 &              3 &               4 &               5 \\
 \hline
   \multirow{ 1}{*}{SDP}&
MLE &  0.3142$\pm$0.0470 &  0.3487$\pm$0.0458 &  0.3816$\pm$0.0451 &  0.4308$\pm$0.0412 &  0.4839$\pm$0.0387 \\
\hline\multirow{ 4}{*}{Kernel}&   MLE &  0.2237$\pm$0.0136 &  0.2650$\pm$0.0122 &  0.3098$\pm$0.0131 &  0.3733$\pm$0.0117 &  0.4381$\pm$0.0140 \\
 &  MFVI &  0.1330$\pm$0.0083 &  0.1173$\pm$0.0064 &  0.1275$\pm$0.0051 &  0.1751$\pm$0.0056 &  0.2208$\pm$0.0069 \\
 &   MCD &  0.0995$\pm$0.0066 &  0.1418$\pm$0.0067 &  0.1845$\pm$0.0067 &  0.2568$\pm$0.0081 &  0.3240$\pm$0.0082 \\
  &  KFLLLA &  0.0893$\pm$0.0040 &  0.1278$\pm$0.0048 &  0.1689$\pm$0.0057 &  0.2355$\pm$0.0067 &  0.2999$\pm$0.0076 \\
 &       SNGP &  0.0636$\pm$0.0012 &  0.0972$\pm$0.0013 &  0.1377$\pm$0.0015 &  0.2007$\pm$0.0017 &  0.2620$\pm$0.0012 \\
 &  SGPA &  0.1120$\pm$0.0022 &  0.1582$\pm$0.0034 &  0.2062$\pm$0.0040 &  0.2776$\pm$0.0028 &  0.3455$\pm$0.0030 \\
  &         DE &             0.0920 &             0.0971 &             0.1220 &             0.1906 &             0.2444 \\
 &      SGPAE &             0.0722 &             0.0888 &             0.1088 &             0.1766 &             0.2291 \\
\hline
\end{tabular}
\end{adjustbox}
\label{table:ood_rob_end}
\end{table}

\begin{table}[!htb]
\centering
\caption{AUROC and AUPR metrics for OOD detection using ViTs trained on CIFAR10 without data augmentation.}
\begin{adjustbox}{width=\textwidth}
\begin{tabular}{c c c c c c c c}
\hline
& &  \multicolumn{2}{c}{OOD: CIFAR100}&  \multicolumn{2}{c}{OOD: SVHN}&  \multicolumn{2}{c}{OOD: MINIIMAGENET}\\
 & Model &   AUROC&AUPR  &AUROC&AUPR&  AUROC&AUPR\\
 \hline
   \multirow{ 1}{*}{SDP}&
MLE &  0.6893$\pm$0.0018 &  0.6433$\pm$0.0023 &  0.6983$\pm$0.0109 &  0.8206$\pm$0.0074 &  0.7115$\pm$0.0018 &  0.9219$\pm$0.0006 \\
 \hline\multirow{ 4}{*}{Kernel}&   MLE &  0.7208$\pm$0.0029 &  0.6774$\pm$0.0029 &  0.7477$\pm$0.0118 &  0.8518$\pm$0.0068 &  0.7325$\pm$0.0016 &  0.9300$\pm$0.0004 \\
 &  MFVI &  0.7177$\pm$0.0055 &  0.6722$\pm$0.0050 &  0.7059$\pm$0.0195 &  0.7944$\pm$0.0134 &  0.7531$\pm$0.0036 &  0.9407$\pm$0.0008 \\
 &   MCD &  0.7426$\pm$0.0016 &  0.7046$\pm$0.0016 &  0.6860$\pm$0.0138 &  0.8045$\pm$0.0083 &  0.7443$\pm$0.0019 &  0.9356$\pm$0.0006 \\
  &  KFLLLA &  0.7228$\pm$0.0056 &  0.6889$\pm$0.0055 &  0.7599$\pm$0.0105 &  0.8345$\pm$0.0069 &  0.7855$\pm$0.0046 &  0.9487$\pm$0.0015 \\
 &       SNGP &  0.7391$\pm$0.0016 &  0.7023$\pm$0.0029 &  0.7135$\pm$0.0125 &  0.8258$\pm$0.0083 &  0.7391$\pm$0.0050 &  0.9341$\pm$0.0015 \\
 &  SGPA &  0.7498$\pm$0.0010 &  0.7111$\pm$0.0023 &  0.7198$\pm$0.0085 &  0.8125$\pm$0.0067 &  0.7501$\pm$0.0013 &  0.9382$\pm$0.0004 \\
  &         DE &  0.7757$\pm$0.0000 &  0.7423$\pm$0.0000 &  0.7819$\pm$0.0000 &  0.8512$\pm$0.0000 &  0.8007$\pm$0.0000 &  0.9520$\pm$0.0000 \\
 &      SGPAE &  0.7872$\pm$0.0000 &  0.7482$\pm$0.0000 &  0.7600$\pm$0.0000 &  0.8619$\pm$0.0000 &  0.8031$\pm$0.0000 &  0.9529$\pm$0.0000 \\
\hline
\end{tabular}
\end{adjustbox}
\label{table:ood_detect_start}
\end{table}

\begin{table}[!htb]
\centering
\caption{AUROC and AUPR metrics for OOD detection using ViTs trained on CIFAR10 with data augmentation.}
\begin{adjustbox}{width=\textwidth}
\begin{tabular}{c c c c c c c c}
\hline
& &  \multicolumn{2}{c}{OOD: CIFAR100}&  \multicolumn{2}{c}{OOD: SVHN}&  \multicolumn{2}{c}{OOD: MINIIMAGENET}\\
 & Model &   AUROC&AUPR  &AUROC&AUPR&  AUROC&AUPR\\
 \hline
   \multirow{ 1}{*}{SDP}&
MLE &  0.8245$\pm$0.0007 &  0.7780$\pm$0.0013 &  0.8738$\pm$0.0020 &  0.9308$\pm$0.0015 &  0.8375$\pm$0.0009 &  0.9585$\pm$0.0005 \\
\hline\multirow{ 4}{*}{Kernel} &   MLE &  0.7908$\pm$0.0034 &  0.7523$\pm$0.0028 &  0.8748$\pm$0.0023 &  0.9338$\pm$0.0015 &  0.8148$\pm$0.0027 &  0.9555$\pm$0.0005 \\
 &  MFVI &  0.7009$\pm$0.0043 &  0.6530$\pm$0.0045 &  0.7652$\pm$0.0179 &  0.8468$\pm$0.0109 &  0.7506$\pm$0.0073 &  0.9379$\pm$0.0027 \\
 &   MCD &  0.7995$\pm$0.0038 &  0.7614$\pm$0.0030 &  0.8754$\pm$0.0040 &  0.9304$\pm$0.0029 &  0.8282$\pm$0.0029 &  0.9595$\pm$0.0007 \\
 &  KFLLLA &  0.7989$\pm$0.0041 &  0.7617$\pm$0.0036 &  0.8884$\pm$0.0036 &  0.9415$\pm$0.0022 &  0.8273$\pm$0.0034 &  0.9591$\pm$0.0007 \\
 &       SNGP &  0.7889$\pm$0.0092 &  0.7543$\pm$0.0104 &  0.8776$\pm$0.0044 &  0.9326$\pm$0.0032 &  0.8125$\pm$0.0072 &  0.9551$\pm$0.0020 \\
 &  SGPA &  0.8007$\pm$0.0014 &  0.7630$\pm$0.0013 &  0.8746$\pm$0.0061 &  0.9281$\pm$0.0043 &  0.8319$\pm$0.0014 &  0.9600$\pm$0.0004 \\
  &         DE &  0.8230$\pm$0.0000 &  0.7869$\pm$0.0000 &  0.9195$\pm$0.0000 &  0.9575$\pm$0.0000 &  0.8555$\pm$0.0000 &  0.9666$\pm$0.0000 \\
 &      SGPAE &  0.8264$\pm$0.0000 &  0.7917$\pm$0.0000 &  0.9004$\pm$0.0000 &  0.9452$\pm$0.0000 &  0.8606$\pm$0.0000 &  0.9677$\pm$0.0000 \\
\hline
\end{tabular}
\end{adjustbox}
\end{table}

\begin{table}[!htb]
\centering
\caption{AUROC and AUPR metrics for OOD detection using ViTs trained on CIFAR100 without data augmentation.}
\begin{adjustbox}{width=\textwidth}
\begin{tabular}{c c c c c c c c}
\hline
& &  \multicolumn{2}{c}{OOD: CIFAR10}&  \multicolumn{2}{c}{OOD: SVHN}&  \multicolumn{2}{c}{OOD: MINIIMAGENET}\\
 & Model &   AUROC&AUPR  &AUROC&AUPR&  AUROC&AUPR\\
 \hline
   \multirow{ 1}{*}{SDP}&
MLE &  0.6091$\pm$0.0023 &  0.5756$\pm$0.0030 &  0.6242$\pm$0.0014 &  0.7856$\pm$0.0015 &  0.6512$\pm$0.0018 &  0.9042$\pm$0.0007 \\
 \hline\multirow{ 4}{*}{Kernel}&   MLE &  0.6355$\pm$0.0052 &  0.5983$\pm$0.0045 &  0.6725$\pm$0.0111 &  0.8171$\pm$0.0062 &  0.6582$\pm$0.0035 &  0.9071$\pm$0.0010 \\
 &  MFVI &  0.6277$\pm$0.0036 &  0.5900$\pm$0.0037 &  0.6242$\pm$0.0108 &  0.7729$\pm$0.0079 &  0.6496$\pm$0.0054 &  0.9051$\pm$0.0024 \\
 &   MCD &  0.6729$\pm$0.0069 &  0.6310$\pm$0.0076 &  0.6054$\pm$0.0057 &  0.7597$\pm$0.0038 &  0.6859$\pm$0.0014 &  0.9155$\pm$0.0005 \\
  &  KFLLLA &  0.6297$\pm$0.0059 &  0.5920$\pm$0.0071 &  0.6733$\pm$0.0159 &  0.8101$\pm$0.0086 &  0.6915$\pm$0.0057 &  0.9194$\pm$0.0017 \\
 &       SNGP &  0.6448$\pm$0.0016 &  0.6080$\pm$0.0020 &  0.6764$\pm$0.0096 &  0.8093$\pm$0.0063 &  0.6823$\pm$0.0023 &  0.9161$\pm$0.0009 \\
 &  SGPA &  0.6712$\pm$0.0052 &  0.6297$\pm$0.0057 &  0.6454$\pm$0.0147 &  0.7854$\pm$0.0099 &  0.6911$\pm$0.0028 &  0.9170$\pm$0.0009 \\
  &         DE &  0.6848$\pm$0.0000 &  0.6377$\pm$0.0000 &  0.7388$\pm$0.0000 &  0.8425$\pm$0.0000 &  0.7394$\pm$0.0000 &  0.9300$\pm$0.0000 \\
 &      SGPAE &  0.6925$\pm$0.0000 &  0.6461$\pm$0.0000 &  0.6961$\pm$0.0000 &  0.8213$\pm$0.0000 &  0.7398$\pm$0.0000 &  0.9304$\pm$0.0000 \\
\hline
\end{tabular}
\end{adjustbox}
\end{table}

\begin{table}[!htb]
\centering
\caption{AUROC and AUPR metrics for OOD detection using ViTs trained on CIFAR100 with data augmentation.}
\begin{adjustbox}{width=\textwidth}
\begin{tabular}{c c c c c c c c}
\hline
& &  \multicolumn{2}{c}{OOD: CIFAR10}&  \multicolumn{2}{c}{OOD: SVHN}&  \multicolumn{2}{c}{OOD: MINIIMAGENET}\\
 & Model &   AUROC&AUPR  &AUROC&AUPR&  AUROC&AUPR\\
 \hline
   \multirow{ 1}{*}{SDP}&
MLE &  0.6778$\pm$0.0022 &  0.6312$\pm$0.0001 &  0.7652$\pm$0.0093 &  0.8769$\pm$0.0055 &  0.7430$\pm$0.0065 &  0.9346$\pm$0.0020 \\
 \hline\multirow{ 4}{*}{Kernel}&   MLE &  0.6899$\pm$0.0026 &  0.6486$\pm$0.0023 &  0.7570$\pm$0.0056 &  0.8670$\pm$0.0025 &  0.7536$\pm$0.0038 &  0.9374$\pm$0.0012 \\
 &  MFVI &  0.6403$\pm$0.0042 &  0.5924$\pm$0.0035 &  0.7485$\pm$0.0106 &  0.8640$\pm$0.0068 &  0.7269$\pm$0.0099 &  0.9317$\pm$0.0030 \\
 &   MCD &  0.7047$\pm$0.0025 &  0.6568$\pm$0.0029 &  0.7972$\pm$0.0040 &  0.8885$\pm$0.0023 &  0.7780$\pm$0.0023 &  0.9440$\pm$0.0008 \\
  &  KFLLLA &  0.6924$\pm$0.0011 &  0.6510$\pm$0.0022 &  0.8164$\pm$0.0043 &  0.9104$\pm$0.0053 &  0.7726$\pm$0.0084 &  0.9435$\pm$0.0024 \\
 &       SNGP &  0.6813$\pm$0.0034 &  0.6388$\pm$0.0034 &  0.7616$\pm$0.0090 &  0.8701$\pm$0.0052 &  0.7486$\pm$0.0048 &  0.9367$\pm$0.0014 \\
 &  SGPA &  0.7056$\pm$0.0012 &  0.6586$\pm$0.0011 &  0.8120$\pm$0.0035 &  0.8971$\pm$0.0027 &  0.7755$\pm$0.0012 &  0.9430$\pm$0.0004 \\
  &         DE &  0.7171$\pm$0.0000 &  0.6718$\pm$0.0000 &  0.8459$\pm$0.0000 &  0.9173$\pm$0.0000 &  0.8076$\pm$0.0000 &  0.9517$\pm$0.0000 \\
 &      SGPAE &  0.7198$\pm$0.0000 &  0.6719$\pm$0.0000 &  0.8686$\pm$0.0000  & 0.9298$\pm$0.0000 &  0.8152$\pm$0.0000 &  0.9536$\pm$0.0000 \\
\hline
\end{tabular}
\end{adjustbox}
\end{table}

\begin{table}[!htb]
\centering
\caption{AUROC and AUPR metrics for OOD detection using Transformers trained on ZINC.}
\begin{tabular}{c c c c }
\hline
 & Model &  AUROC & AUPR\\
 \hline
   \multirow{ 1}{*}{SDP}
 &   MCD &  0.4566$\pm$0.0072 &  0.4889$\pm$0.0029 \\
 \hline\multirow{ 3}{*}{Kernel}&  MFVI &  0.7254$\pm$0.0632 &  0.6821$\pm$0.0512 \\
 &   MCD &  0.8797$\pm$0.0739 &  0.8629$\pm$0.0571 \\
 &  SGPA &  0.8968$\pm$0.0317 &  0.8705$\pm$0.0335 \\
  &     DE &  0.9783$\pm$0.0000 &  0.9697$\pm$0.0000 \\
 &  SGPAE &  0.9884$\pm$0.0000 &  0.9727$\pm$0.0000 \\
\hline
\end{tabular}
\label{table:ood_detect_end}
\end{table}

\label{a5:table}
\chapter{Supplementary Material for Chapter~\ref{cha:hsgp}}

\section{Computing the Prior Inducing Covariance $\mathbf{K}_{\mathbf{uu}}^{(t)}$ as Direct ODE Evolution}
\label{appendix:direct-ode-hippo-legs}
Although in the experiments in the main text, we compute $\bm{K}_{\bm{uu}}^{(t)}$ based on RFF approximation, here for completeness, we provide the detailed derivation for an alternative approach of computing it, which directly evolves $\bm{K}_{\bm{uu}}^{(t)}$ as a matrix ODE, when the inducing functions are defined via HiPPO-LegS. Recall that
\begin{equation}
[\bm{K}_{\bm{uu}}^{(t)}]_{\ell,m}
=
\iint
k(x, x^{\prime}) \phi_{\ell}^{(t)}(x)\phi_{m}^{(t)}(x^{\prime}) \mathrm{d}x \mathrm{d}x^{\prime},
\label{eq:kuu_definition}
\end{equation}
where $\phi_{\ell}^{(t)}(x) = g_{\ell}^{(t)}(x)\,\omega^{(t)}(x)$ are the time-varying basis functions under the HiPPO-LegS framework.

Differentiating $[\bm{K}_{\bm{uu}}^{(t)}]_{\ell,m}$ with respect to $t$ gives
\begin{equation}
\frac{d}{dt}\,[\bm{K}_{\bm{uu}}^{(t)}]_{\ell,m}
=
\iint
k(x,x^{\prime})
\frac{\partial}{\partial t}
\left[\phi_{\ell}^{(t)}(x)\phi_{m}^{(t)}(x^{\prime})\right]
\mathrm{d}x\mathrm{d}x^{\prime}.
\end{equation}
Applying the product rule:
\begin{equation}
\frac{\mathrm{d}}{\mathrm{d}t}[\bm{K}_{\bm{uu}}^{(t)}]_{\ell,m}
=
\iint
k(x,x^{\prime})\frac{\partial}{\partial t}\phi_{\ell}^{(t)}(x)\phi_{m}^{(t)}(x^{\prime})\mathrm{d}x\mathrm{d}x^{\prime}
+
\iint
k(x,x^{\prime})\phi_{\ell}^{(t)}(x)\frac{\partial}{\partial t}\phi_{m}^{(t)}(x^{\prime})\mathrm{d}x\mathrm{d}x^{\prime}.
\end{equation}
In HiPPO-LegS, each $\phi_{\ell}^{(t)}(x)$ obeys an ODE governed by lower order scaled Legendre polynomials on $[0,t]$ and a Dirac delta boundary term at $x=t$, as we show in Section~\ref{sec:hippo}. Concretely,
\begin{equation}
\begin{split}
\frac{\partial}{\partial t}\phi_{\ell}^{(t)}(x)
&=
-\frac{\sqrt{2\ell+1}}{t}
{\Large[}
\frac{\ell+1}{\sqrt{2\ell+1}}\phi_{\ell}^{(t)}(x)
+
\sqrt{2\ell-1}\phi_{\ell-1}^{(t)}(x)\\ & \qquad+ \sqrt{2\ell-3}\phi_{\ell-2}^{(t)}(x)\cdots
{\Large]}
+
\frac{1}{t}\,\delta_{t}(x),
\end{split}
\end{equation}
where $\delta_{t}(x)$ is the Dirac delta at $x=t$.

Substituting this expression into the integrals yields the boundary terms of the form $\int k(t, x^{\prime})\,\phi_{m}^{(t)}(x^{\prime})\,\mathrm{d}x^{\prime}$, along with lower-order terms involving $\{[\bm{K}_{\bm{uu}}^{(t)}]_{\ell,m},[\bm{K}_{\bm{uu}}^{(t)}]_{\ell-1,m},\cdots\}$, etc. Summarizing in matrix form leads to
\begin{equation}
\label{eq:kuu_ode}
\frac{\mathrm{d}}{\mathrm{d}t}\,\bm{K}_{\bm{uu}}^{(t)}
=
\left[
\bm{A}(t) \,\bm{K}_{\bm{uu}}^{(t)}
+
\bm{K}_{\bm{uu}}^{(t)}\bm{A}(t)^\top
\right]
+
\frac{1}{t}
\left[
\tilde{\bm{B}}(t)+\tilde{\bm{B}}(t)^\top
\right],
\end{equation}
where the $lm$-th entry of $\bm{K}_{\bm{uu}}^{(t)} \in \mathbb{R}^{M \times M}$ is $[\bm{K}_{\bm{uu}}^{(t)}]_{\ell,m}$, $\bm{A}(t) \in \mathbb{R}^{M\times M}$ is the same lower-triangular matrix from the HiPPO-LegS framework defined in Eq.~\ref{eq:hippo_matrix_a}, and $\tilde{\bm{B}}(t) \in \mathbb{R}^{M\times M}$ is built from the boundary contributions as
\begin{equation}
\label{eq:boundary_matrix}
\tilde{\bm{B}}(t) = \bm{c}(t) \bm{1}_M,
\end{equation}
where $\bm{1}_M \in \mathbb{R}^{1\times M}$ is a row vector of ones of size $M$ and $\bm{c}(t) \in \mathbb{R}^{M \times 1}$ is the coefficient vector with each element being
\begin{equation}
    [\bm{c}(t)]_{\ell}
    =
    \int\!
    k\left(t, x\right)\,\phi_{\ell}^{(t)}(x)\,\mathrm{d}x.
\end{equation}
After discretizing in $t$ (e.g.\ an Euler scheme), one can recurrently update $\bm{K}_{\bm{uu}}^{(t)}$ and the boundary vector $\bm{c}(t)$ over time.

\paragraph{Unstability of directly evolving $\mathbf{K}_{\mathbf{uu}}^{(t)}$ as ODE.}
\label{appendix:direct-ode-hippo-legs-unstability}
Empirically, we find that the direct ODE approach is less stable compared with RFF approach. Intuitively, it can be seen from the difference in the forms of their evolutions, especially in the first term. In RFF approach, the first term of the evolution of Fourier feature is of the form $\bm{A}(t)\,\bm{Z}_{w}^{(t)}$, which includes evolving vectors with the operator $\mathcal{L}_1: \bm{X} \rightarrow \bm{A}(t)\bm{X}$. In direct ODE approach, the first term in the direct evolution of $\bm{K}_{\bm{uu}}^{(t)}$ is of the form $\bm{A}(t) \,\bm{K}_{\bm{uu}}^{(t)}
+
\bm{K}_{\bm{uu}}^{(t)}\bm{A}(t)^\top$, which requires the Lyapunov operator $\mathcal{L}_2:\bm{X}\rightarrow \bm{A}(t)\bm{X} + \bm{XA}(t)^\top$. The critical difference is that $\mathcal{L}_2$ has eigenvalues $\{\lambda_i + \lambda_j\}$ (where $\lambda_i$ and $\lambda_j$ are eigenvalues of $\bm{A}(t)$) \citep{horn1991topics}, while $\mathcal{L}_1$ has eigenvalues $\{\lambda_i\}$. Since HiPPO-LegS uses a lower-triangular $\bm{A}(t)$ with negative diagonal entries, the eigenvalues are all negative $\lambda_i<0$. Hence, the eigenvalues of the Lyapunov operator $\mathcal{L}_2$ are approximately as twice negative as the eigenvalues of $\mathcal{L}_1$, leading to a stiff ODE system with poorer numerical conditioning.

\section{Finite Basis Approximation of Posterior OHSVGP}
\label{appendix:recon}

Here, we show that $q(\bm{u}^{(t)})$ is the distribution of HiPPO coefficients of the posterior OHSVGP $q_t(f)$. From Eq.~\ref{eq:q_predictive_svgp}, the posterior of the function values evaluated at arbitrary indices $\bm{X}$ is 
\(
    q_t(\bm{f_X}) = \mathcal{N}(\bm{f_X}; \bK^{(t)}_{\bm{f_X}\bu}\bK_{\bu\bu}^{(t)-1}\bm{m}^{(t)}_{\bu}, \bK_{\bm{f_X} \bm{f_X}} - \bK^{(t)}_{\bm{f_X}\bu}\bK_{\bu\bu}^{(t)-1}[\bK^{(t)}_{\bu\bu} - \bS^{(t)}_{\bu}]\bK_{\bu\bu}^{(t)-1}\bK^{(t)}_{\bu \bm{f_X}}).
\)

Based on this, we compute the mean of the $m$-th HiPPO coefficient for $q_t(f)$ as follows,
\begin{equation}
    \begin{split}
        & \quad E_{q_t(f)}\left[\int f(x)\phi_m^{(t)}(x) dx\right]\\
        &= \int E_{q_t(f)}\left[f(x)\right]\phi_m^{(t)}(x) dx\\
        &=\left(\int\bK^{(t)}_{f_x\bu}\bK_{\bu\bu}^{(t)-1} \phi_m^{(t)}(x) dx\right) \bm{m}^{(t)}_{\bu}\\
        &= \left[\int \int k(x, x')
        \begin{pmatrix}
        \phi_1^{(t)}(x')\\
        \cdot\cdot\cdot\\
        \phi_M^{(t)}(x')\\
        \end{pmatrix}
        dx'\bK_{\bu\bu}^{(t)-1}\phi_m^{(t)}(x) dx\right]  \bm{m}^{(t)}_{\bu}\\
        &= \left[\int \int k(x, x')
        \begin{pmatrix}
        \phi_1^{(t)}(x')\phi_m^{(t)}(x)\\
        \cdot\cdot\cdot\\
        \phi_M^{(t)}(x')\phi_m^{(t)}(x)\\
        \end{pmatrix}
         dx'dx \right] \bK_{\bu\bu}^{(t)-1}\bm{m}^{(t)}_{\bu}\\
         & = \left[ \bK_{\bu\bu}^{(t)} \right]_{m,:}\bK_{\bu\bu}^{(t)-1}\bm{m}^{(t)}_{\bu} \\
         &=\left[\bm{m}^{(t)}_{\bu}\right]_m,\\
    \end{split}
\end{equation}
which is exactly the variational mean of $q(u_m^{(t)})$. Similarly, the covariance between the $l$-th and the $m$-th HiPPO coefficient for $q_t(f)$ can be computed as
\begin{equation}
    \begin{split}
    & \quad E_{q_t(f)}\left[ \left(\int f(x)\phi_l^{(t)}(x)dx \right)  \left( \int f(x')\phi_m^{(t)}(x')dx' \right ) \right] -  \left[\bm{m}^{(t)}_{\bu}\right]_l \left[\bm{m}^{(t)}_{\bu}\right]_m \\
    &=\int\int E_{q_t(f)}\left[ f(x)  f(x') \right]\phi_l^{(t)}(x)\phi_m^{(t)}(x')dxdx' -  \left[\bm{m}^{(t)}_{\bu}\right]_l \left[\bm{m}^{(t)}_{\bu}\right]_m \\
   &= \int\int \left( k(x, x') - \bK^{(t)}_{f_x \bu}\bK_{\bu\bu}^{(t)-1}[\bK^{(t)}_{\bu\bu} - \bS^{(t)}_{\bu}]\bK_{\bu\bu}^{(t)-1}\bK^{(t)}_{\bu f_{x'}} \right) \phi_l^{(t)}(x)\phi_m^{(t)}(x')dx dx' \\
   & \quad \quad \quad+ \int\cancel{\int  E_{q_t(f)}\left[f(x)\right] E_{q_t(f)}\left[f(x')\right] \phi_l^{(t)}(x)\phi_m^{(t)}(x')dx dx'} -  \cancel{\left[\bm{m}^{(t)}_{\bu}\right]_l \left[\bm{m}^{(t)}_{\bu}\right]_m}\\  
   &=\left[\bK_{\bu\bu}^{(t)}\right]_{lm} - \left( \int \bK^{(t)}_{f_x\bu}\bK_{\bu\bu}^{(t)-1} \phi_l^{(t)}(x) dx\right)  [\bK^{(t)}_{\bu\bu} - \bS^{(t)}_{\bu}] \left(\int  \phi_m^{(t)}(x') \bK_{\bu\bu}^{(t)-1} \bK^{(t)}_{\bu f_{x'}}dx'\right)\\
   &=\cancel{\left[\bK_{\bu\bu}^{(t)}\right]_{lm}} - \cancel{\left[ \bK_{\bu\bu}^{(t)} \right]_{l,:}\bK_{\bu\bu}^{(t)-1}\bK_{\bu\bu}^{(t)} \bK_{\bu\bu}^{(t)-1}\left[\bK_{\bu\bu}^{(t)} \right]_{:,m} }+ \left[ \bK_{\bu\bu}^{(t)} \right]_{l,:} \bK_{\bu\bu}^{(t)-1} \bS^{(t)}_{\bu} \bK_{\bu\bu}^{(t)-1} \left[ \bK_{\bu\bu}^{(t)} \right]_{:,m}\\
   & = \left[\bS^{(t)}_{\bu}\right]_{lm},
    \end{split}
\end{equation}
which is exactly the variational covariance between $u_l^{(t)}$ and $u_m^{(t)}$ in $q(\bm{u}^{(t)})$.

Hence, if $f \sim q_t(f)$, then $ \int f(x) \phi_{m}^{(t)}(x) \mathrm{d}x \sim q(u_m^{(t)})$, which implies that we can approximate the posterior OHSVGP with finite basis: $f = \sum_{m=1}^M u_m^{(t)}g_{m}^{(t)}(x)$, $u_m^{(t)}\sim q(u_m^{(t)})$.
\section{SVGPVAE Model Details}
\label{appendix:gpvae}
With SVGPVAE, we utilize \citet{jazbec_scalable_2021} and notation from \citet{zhu_markovian_2023}, and we have the following encoder-decoder model:
\begin{align}
    p(\by_{1:T})=p(\mathbf{f}_{1:T}) \prod_{t=1}^T p(\by_{t}|\mathbf{f}_t),
\end{align}
with likelihood $p(\by_t|\mathbf{f}_t)=\mathcal{N}(\by_t|\varphi(\mathbf{f}_t), \sigma^2 I)$ and decoder network $\varphi:\mathbb{R}^{L}\rightarrow\mathbb{R}^{d_y}$. The encoder $\phi:\mathbb{R}^{d_y}\rightarrow\mathbb{R}^{2L}$ yields $(\tilde{\mathbf{y}}_{t}^{1:L}, \tilde{\mathbf{v}_t}^{1:L})=\phi(\by_{t})$. $\mathbf{f}_t$ follows an $L$-dimensional multi-output GP and its approximate posterior is given by:
\begin{align}
q(\mathbf{f}_{1:T}) &= \prod_{l=1}^L p(\mathbf{f}^l_{1:T}|\bu^l_m)q(\bu^l),\quad q(\bu^l) = \mathcal{N}(\bu^l | \mathbf{m}^l, \mathbf{A}^l),\\
\mathbf{S}^l&=\mathbf{K}_{\bu\bu}^l + \mathbf{K}_{\bu\mathbf{f}}^l \text{diag}(\tilde{\bv}^l_{1:T})^{-1}\mathbf{K}_{\mathbf{f}\bu}^l,
\quad \mathbf{m}^l = \mathbf{K}_{\bu\bu}^l (\mathbf{S}^l)^{-1}\mathbf{K}_{\bu\mathbf{f}}^l\text{diag}(\tilde{\bv}^l_{1:T})^{-1}\tilde{\by}^l_{1:T},\\
\mathbf{A}^l&=\mathbf{K}_{\bu\bu}^l(\mathbf{S}^l)^{-1}\mathbf{K}_{\bu\bu}^l,
\end{align}
where $p(\mathbf{f}^l_{1:T}|\mathbf{u}_m^l)$ is the 
prior conditional distribution. 

Following \citet{jazbec_scalable_2021}, the objective function is defined as:
\begin{align}
    \mathcal{L}_{\text{SVGPVAE}}(\theta) = \sum_{t=1}^T\big[ \mathbb{E}_{q(\mathbf{f}_t)}\log p(\mathbf{y}_t|\mathbf{f}_t) - \log \mathcal{N}(\mathbf{f}_t| \tilde{\mathbf{y}}_{t}, \tilde{\mathbf{v}}_{t} )\big] + \sum_{l=1}^L \mathcal{L}_H^l,
\end{align}
where $\mathcal{L}_H^l$ is the "Hensman" ELBO described in Equation~7 of \citet{jazbec_scalable_2021}. Since the variational parameters $\mathbf{m}^l$ and $\mathbf{S}^l$, and the likelihood are all amortized by neural networks, we further add EWC (Elastic Weight Consolidation; \citep{kirkpatrick2017overcoming}) regularization for both encoder and decoder networks to the loss above for continual learning.
\section{Additional Results}

\subsection{Results for Time Series Regression with Trainable Kernel Hyperparameters}
\label{appendix:trainable_kernel}
\begin{figure}[t]
  \begin{subfigure}[t]{0.29\textwidth}
      \centering
      \includegraphics[width=\textwidth]{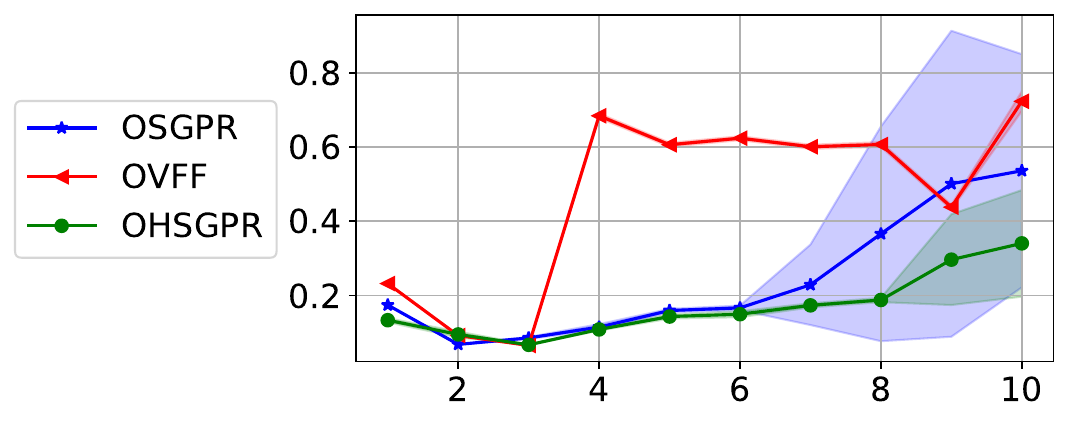}
      \caption{Solar, M=50}
  \end{subfigure}
  \hfill
  \begin{subfigure}[t]{0.224\textwidth}
      \centering
      \includegraphics[width=\textwidth]{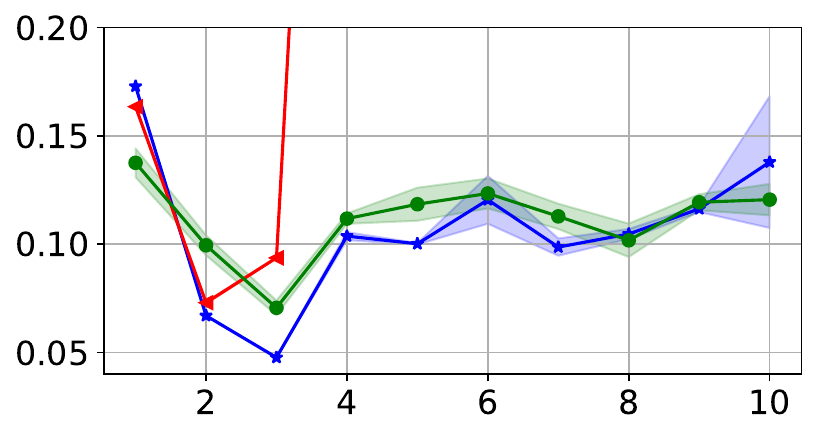}
      \caption{Solar, M=150}
  \end{subfigure}
  \hfill
  \begin{subfigure}[t]{0.224\textwidth}
      \centering
      \includegraphics[width=\textwidth]{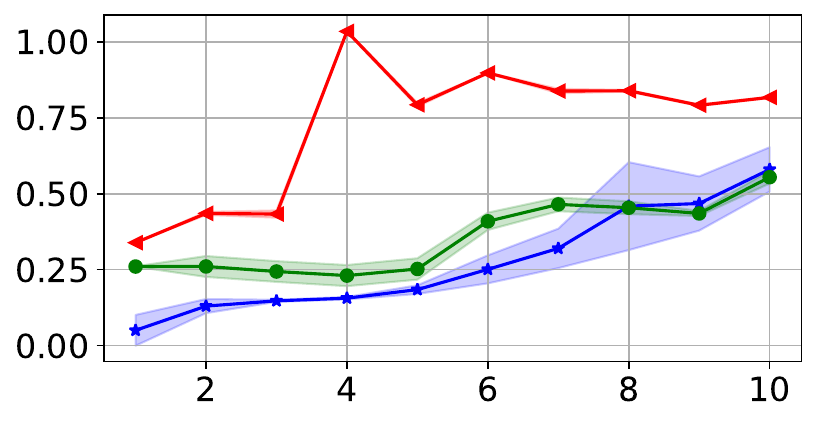}
    \caption{Audio, M=100}
  \end{subfigure}
  \begin{subfigure}[t]{0.224\textwidth}
      \centering
      \includegraphics[width=\textwidth]{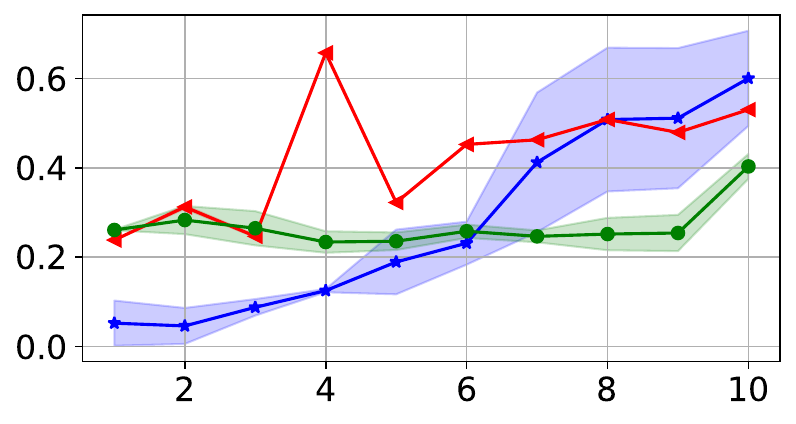}
    \caption{Audio, M=200}
  \end{subfigure}
  \caption{Test set RMSE over the learned tasks vs. number of learned tasks for Solar Irradiance and Audio signal prediction dataset (keep updating kernel hyperparameters).}
  \label{fig:time_series_rmse_traink}
\end{figure}

\begin{figure}[htbp]
  \begin{subfigure}[t]{0.29\textwidth}
      \centering
      \includegraphics[width=\textwidth]{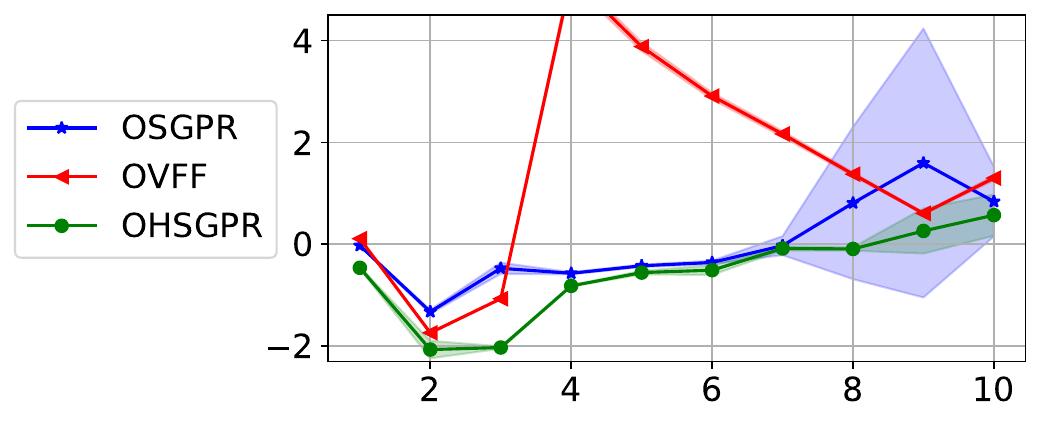}
      \caption{Solar, M=50}
  \end{subfigure}
  \hfill
  \begin{subfigure}[t]{0.224\textwidth}
      \centering
      \includegraphics[width=\textwidth]{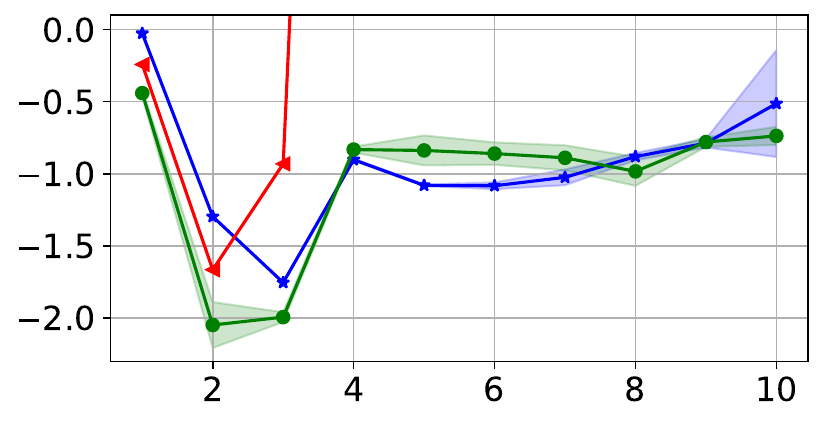}
      \caption{Solar, M=150}
  \end{subfigure}
  \hfill
  \begin{subfigure}[t]{0.224\textwidth}
      \centering
      \includegraphics[width=\textwidth]{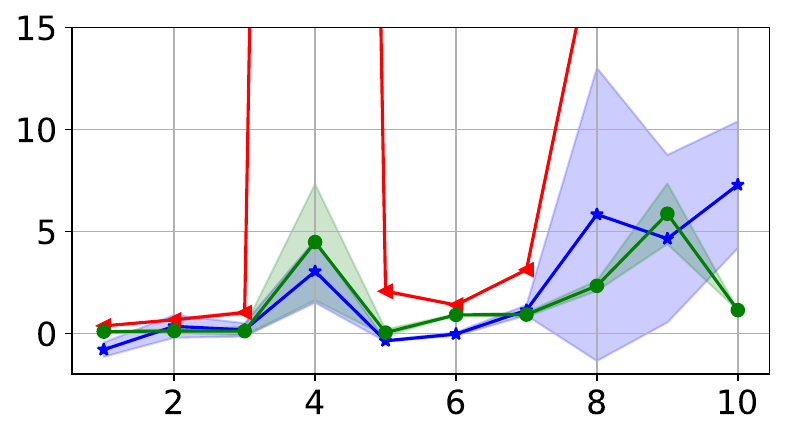}
    \caption{Audio, M=100}
  \end{subfigure}
  \begin{subfigure}[t]{0.224\textwidth}
      \centering
      \includegraphics[width=\textwidth]{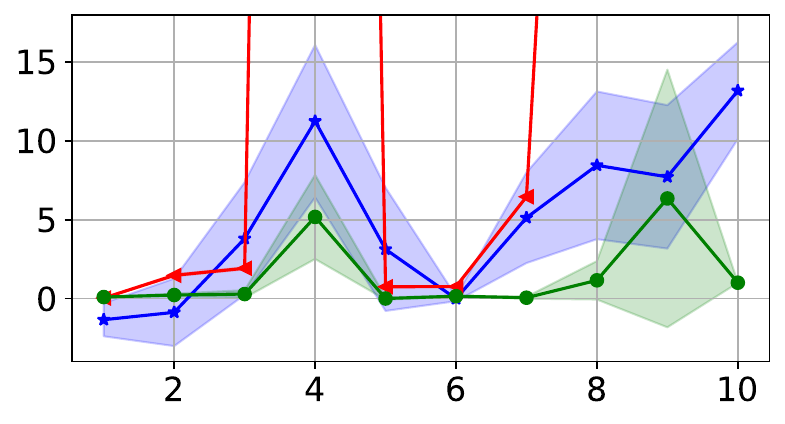}
    \caption{Audio, M=200}
  \end{subfigure}
  \caption{Test set NLPD over the learned tasks vs. number of learned tasks for Solar Irradiance and Audio signal prediction dataset (keep updating kernel hyperparameters).}
   \label{fig:time_series_nlpd_traink}
\end{figure}
Figure~\ref{fig:time_series_rmse_traink} and~\ref{fig:time_series_nlpd_traink} show RMSE and NLPD results for time series regression experiments based on trainable kernel hyperparameters (i.e., keep optimizing kernel hyperparameters online in all the tasks). Notice that OVC is only compatible with fixed kernel \citep{maddox_conditioning_2021}, so we don't consider it here. Compared with the results based on fixed kernel in the main text, here all methods show less stable performance. Previous works either find a well-performed fixed kernel \citep{maddox_conditioning_2021} or consider generalized VI by scaling the KL terms in the online ELBO objective with a positive factor requiring careful tuning to mitigate the unstable online optimization of kernel hyperparameters \citep{stanton_kernel_2021, kapoor2021variational}.

\subsection{Comparison of Basis–measure Variants}
\label{appendix:basis_measure_variant}

Figure~\ref{fig:hippo_variants} shows the results of OHSGPR applied to a toy time-series regression dataset, where the data is split chronologically into three equal segments, used as Tasks 1–3 in an online learning setup. The figure compares the effect of several variants of the HiPPO operators \citep{gu_hippo_2020, gu_httyh_2023} when used for OHSGPR. Subfigures (a–c) correspond to HiPPO-LegS as used in all of our main experiments. Subfigures (d–f) apply HiPPO-LegT, (g–i) apply HiPPO-LagT, based on the Laguerre polynomial basis, and (j–l) apply HiPPO-FouT, based on Fourier basis functions. The detailed formulation of these HiPPO variants can be found in Section~\ref{sec:hippo} and Appendix~\ref{appendix:hippo_measure_basis}. While OHSGPR-LegS successfully memorizes all the past tasks, OHSGPR-LegT, OHSGPR-LagT and OHSGPR-FouT all demonstrate catastrophic forgetting to certain degree since instead of using the uniform measure over the past (as is used in HiPPO-LegS), they are based on measures which place more mass over the recent history. LegT and FouT use a fixed-length sliding window measure, while LagT uses exponentially decaying measure, which assigns more importance to recent history.

\begin{figure}[htbp]
  \begin{subfigure}[t]{0.32\textwidth}
      \centering
      \includegraphics[width=\textwidth]{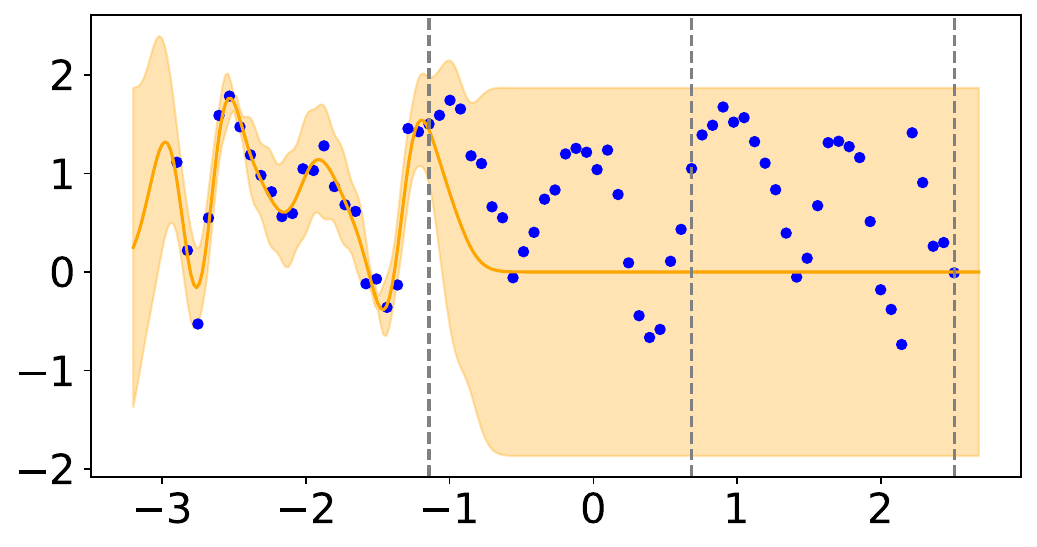}
      \caption{OHSGPR-LegS, after Task~1}
  \end{subfigure}
  \hfill
  \begin{subfigure}[t]{0.32\textwidth}
      \centering
      \includegraphics[width=\textwidth]{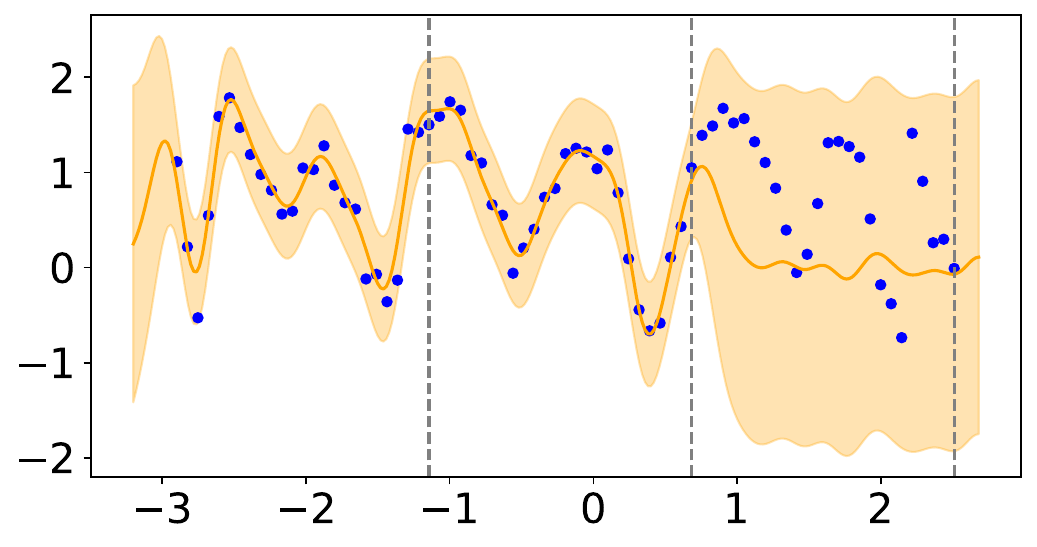}
      \caption{OHSGPR-LegS, after Task~2}
  \end{subfigure}
  \hfill
  \begin{subfigure}[t]{0.32\textwidth}
      \centering
      \includegraphics[width=\textwidth]{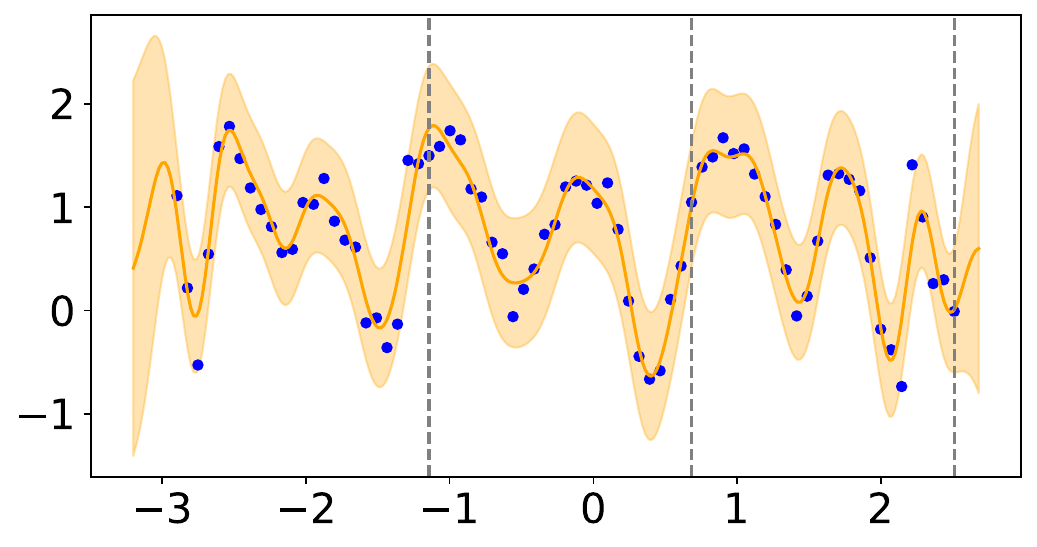}
      \caption{OHSGPR-LegS, after Task~3}
  \end{subfigure}

  \begin{subfigure}[t]{0.32\textwidth}
      \centering
      \includegraphics[width=\textwidth]{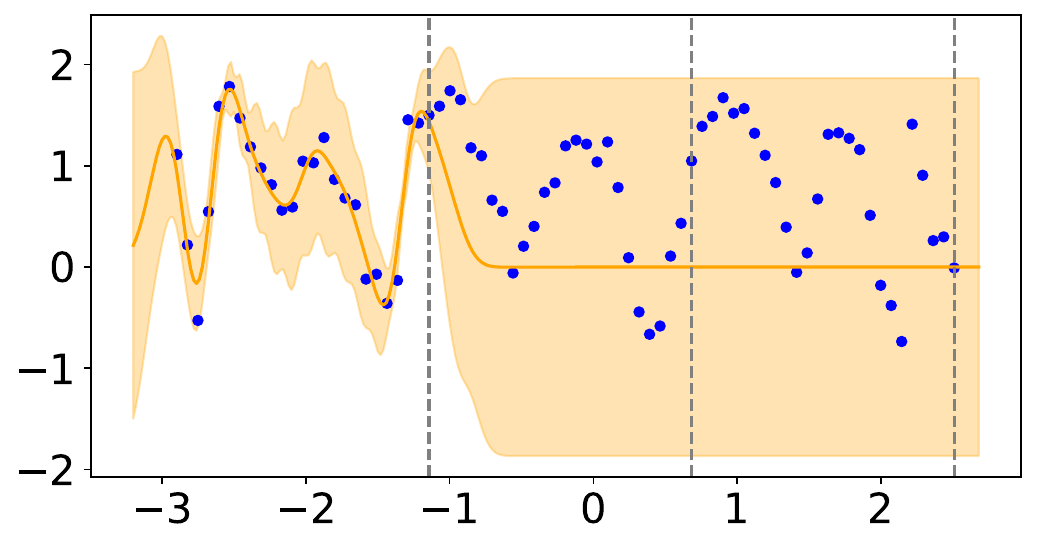}
      \caption{OHSGPR-LegT, after Task~1}
  \end{subfigure}
  \hfill
  \begin{subfigure}[t]{0.32\textwidth}
      \centering
      \includegraphics[width=\textwidth]{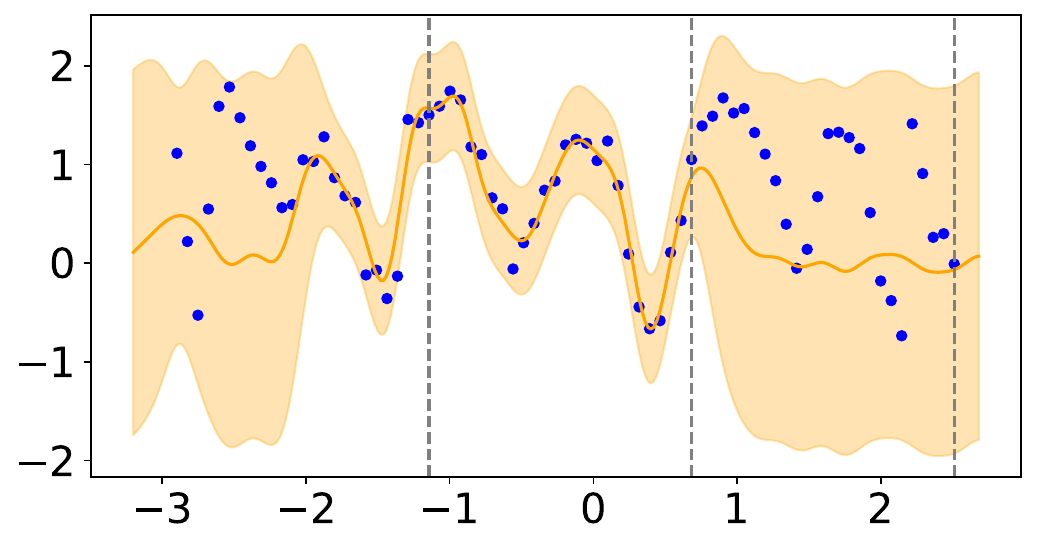}
      \caption{OHSGPR-LegT, after Task~2}
  \end{subfigure}
  \hfill
  \begin{subfigure}[t]{0.32\textwidth}
      \centering
      \includegraphics[width=\textwidth]{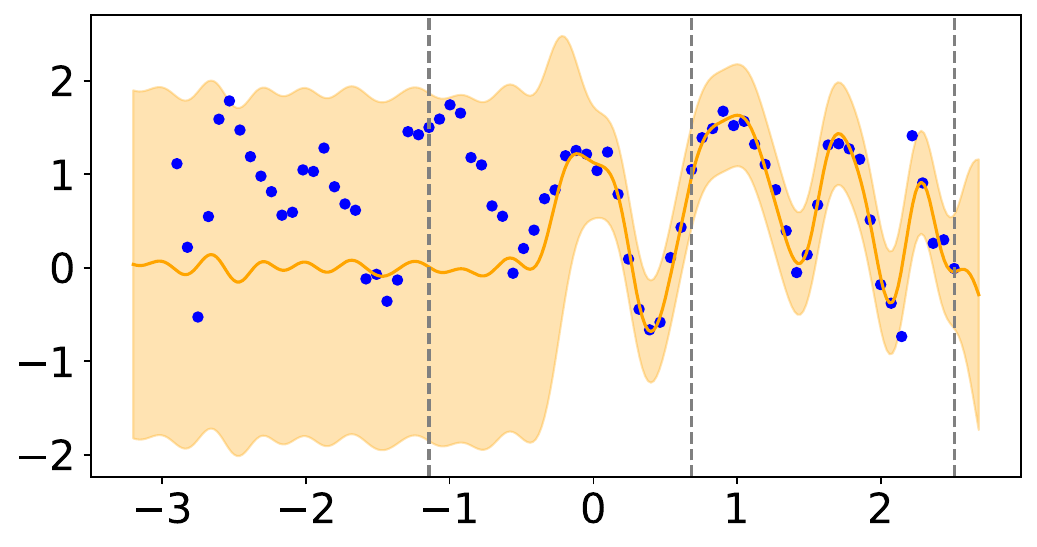}
      \caption{OHSGPR-LegT, after Task~3}
  \end{subfigure}

  \begin{subfigure}[t]{0.32\textwidth}
      \centering
      \includegraphics[width=\textwidth]{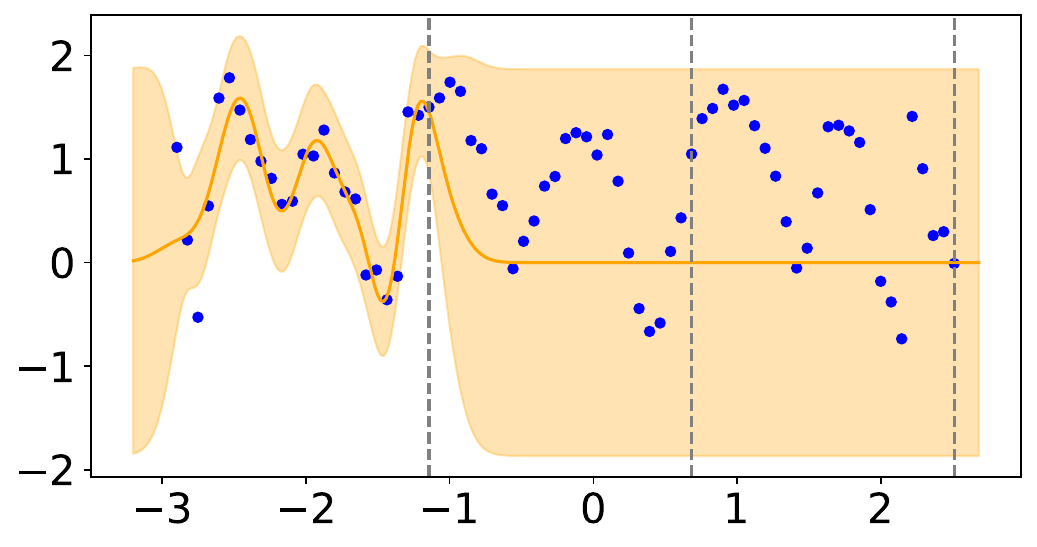}
      \caption{OHSGPR-LagT, after Task~1}
  \end{subfigure}
  \hfill
  \begin{subfigure}[t]{0.32\textwidth}
      \centering
      \includegraphics[width=\textwidth]{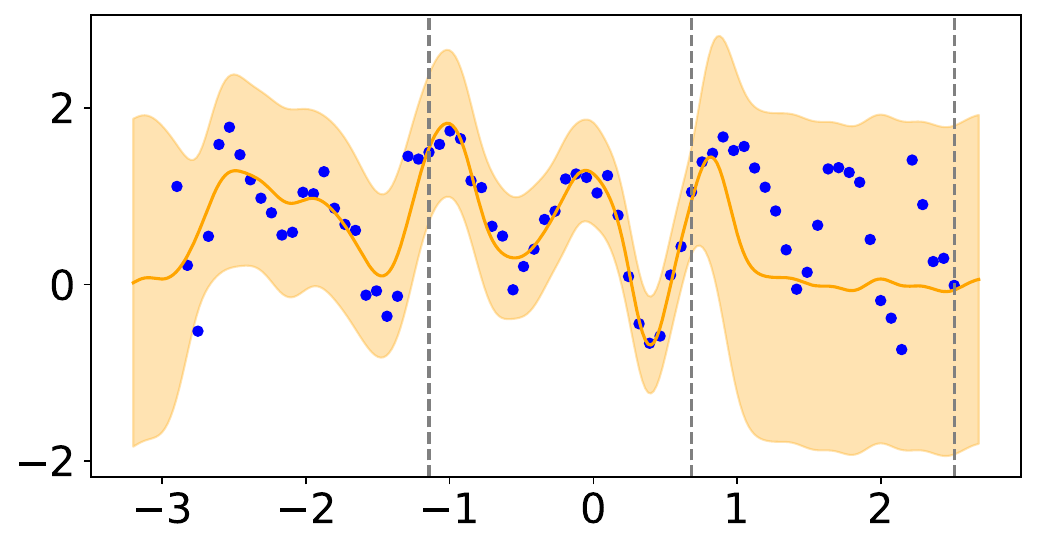}
      \caption{OHSGPR-LagT, after Task~2}
  \end{subfigure}
  \hfill
  \begin{subfigure}[t]{0.32\textwidth}
      \centering
      \includegraphics[width=\textwidth]{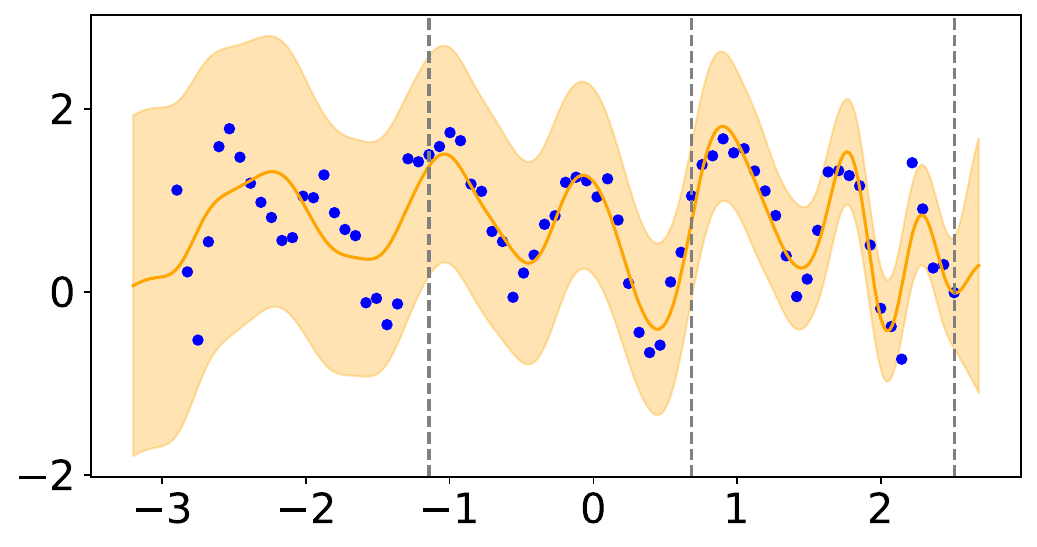}
      \caption{OHSGPR-LagT, after Task~3}
  \end{subfigure}
 
  \begin{subfigure}[t]{0.32\textwidth}
      \centering
      \includegraphics[width=\textwidth]{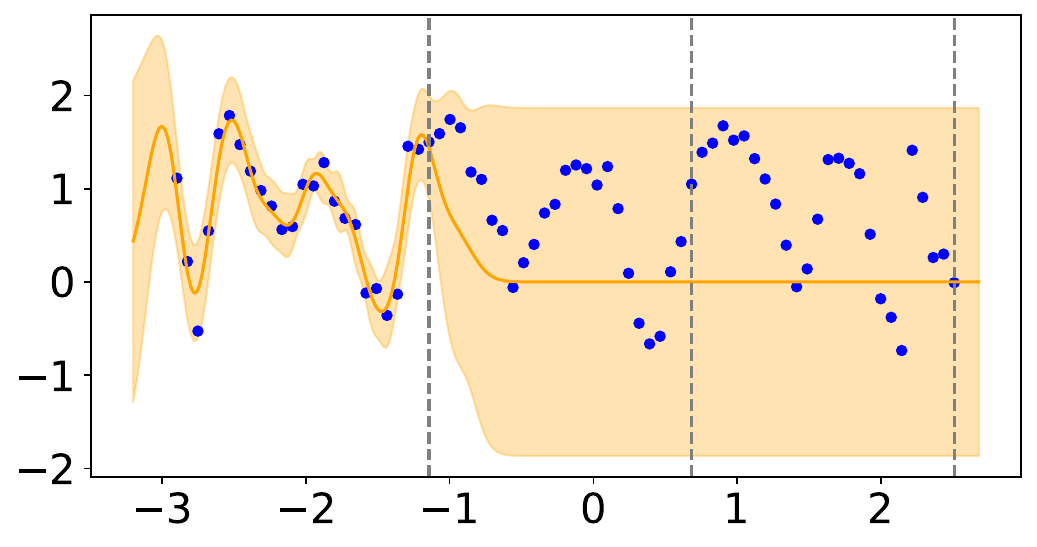}
      \caption{OHSGPR-FouT, after Task 1}
  \end{subfigure}
  \hfill
  \begin{subfigure}[t]{0.32\textwidth}
      \centering
      \includegraphics[width=\textwidth]{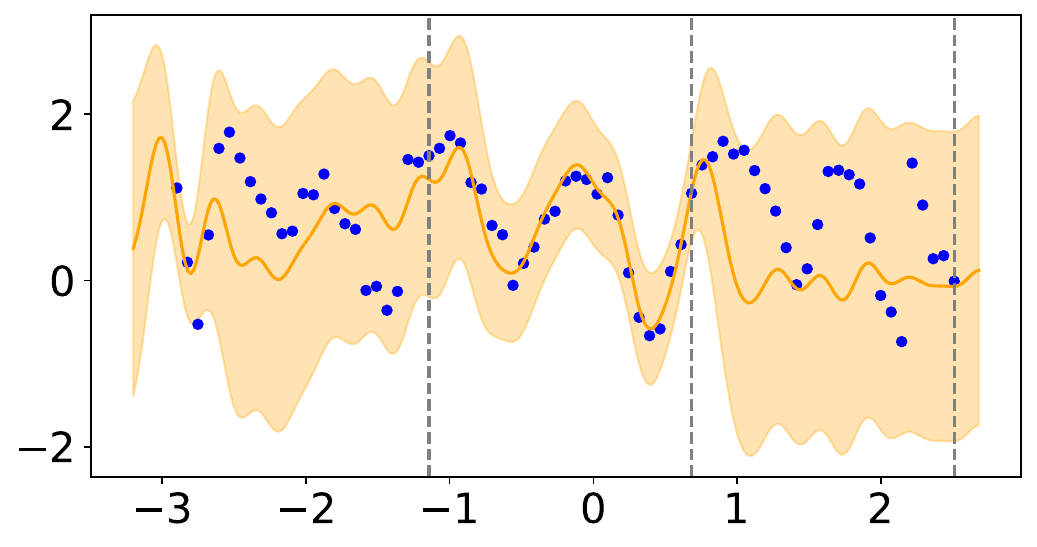}
      \caption{OHSGPR-FouT, after Task 2}
  \end{subfigure}
  \hfill
  \begin{subfigure}[t]{0.32\textwidth}
      \centering
      \includegraphics[width=\textwidth]{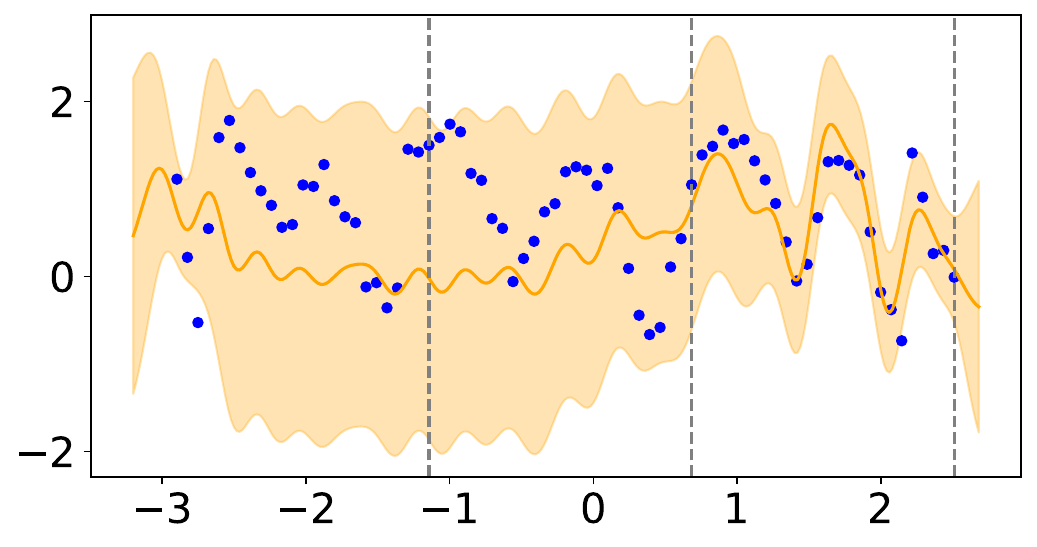}
      \caption{OHSGPR-FouT, after Task 3}
  \end{subfigure}
 
  \caption{Comparison of OHSGPR based on different HiPPO variants on a toy online regression dataset.}
  \label{fig:hippo_variants}
\end{figure}
\chapter{Supplementary Material for Chapter~\ref{cha:pseudovid}}

\section{Proof of Theorem~\ref{theorem}}\label{appendix:proof}
\begin{proof}
To derive $\hat{\z}_T^{(g)*}(\z_{s_1}, \cdot\cdot\cdot, \z_{s_k})=\argmin_{\hat{\z}_T^{(g)}} \mathbb{E}_{p(\z_T, \z_{s_1}, \cdot\cdot\cdot, \z_{s_k})}[||\z_T-\hat{\z}_T^{(g)}||^2_2]$, we first compute the gradient, 
\begin{equation}
\begin{split}
\nabla_{\hat{\z}_T^{(g)}}\mathbb{E}_{p(\z_T, \z_{s_1}, \cdot\cdot\cdot, \z_{s_k})}[||\z_T-\hat{\z}_T^{(g)}||^2_2]&=\mathbb{E}_{p(\z_{s_1}, \cdot\cdot\cdot, \z_{s_k})}\{\nabla_{\hat{\z}_T^{(g)}} \mathbb{E}_{p(\z_T| \z_{s_1}, \cdot\cdot\cdot, \z_{s_k})}[||\z_T-\hat{\z}_T^{(g)}||^2_2]\}\\
&=2\mathbb{E}_{p(\z_{s_1}, \cdot\cdot\cdot, \z_{s_k})}\{\mathbb{E}_{p(\z_T|\z_{s_1}, \cdot\cdot\cdot, \z_{s_k})}[\z_T -\hat{\z}_T^{(g)}]\}.\\
\end{split}
\end{equation}
Setting the above gradient to 0 gives us
\begin{equation}
\begin{split}
\mathbb{E}_{p(\z_T|\z_{s_1}, \cdot\cdot\cdot, \z_{s_k})}[\z_T -\hat{\z}_T^{(g)}] = 0 \implies \hat{\z}_T^{(g)*}(\z_{s_1}, \cdot\cdot\cdot, \z_{s_k}) = \mathbb{E}_{p(\z_T|\z_{s_1}, \cdot\cdot\cdot, \z_{s_k})}[\z_T].
\end{split}
\end{equation}
The minimum reconstruction error is obtained by plugging  $\hat{\z}_T^{(g)*}(\z_{s_1}, \cdot\cdot\cdot, \z_{s_k})$ in $\mathbb{E}_{p(\z_T, \z_{s_1}, \cdot\cdot\cdot, \z_{s_k})}[||\z_T-\hat{\z}_T^{(g)}||^2_2]$,
\begin{equation}
\begin{split}
    \min_{\hat{\z}_T^{(g)}} \mathbb{E}_{p(\z_T, \z_{s_1}, \cdot\cdot\cdot, \z_{s_k})}[||\z_T-\hat{\z}_T^{(g)}||^2_2]&=\mathbb{E}_{p(\z_T, \z_{s_1}, \cdot\cdot\cdot, \z_{s_k})}[||\z_T-\mathbb{E}_{p(\z_T|\z_{s_1}, \cdot\cdot\cdot, \z_{s_k})}[\z_T]||^2_2]\\
    &= \mathbb{E}_{p(\z_{s_1}, \cdot\cdot\cdot, \z_{s_k})}[\text{Var}_{p(\z_T| \z_{s_1}, \cdot\cdot\cdot, \z_{s_k})}(\z_T)].
\end{split}
\end{equation}

Similarly, 
\begin{equation}
    \begin{split}
         \min_{\hat{\z}_T^{(h)}} \mathbb{E}_{p(\z_T, \z_{s_1}, \cdot\cdot\cdot, \z_{s_l})}[||\z_T-\hat{\z}_T^{(h)}||^2_2] &= \mathbb{E}_{p(\z_{s_1}, \cdot\cdot\cdot, \z_{s_l})}[\text{Var}_{p(\z_T| \z_{s_1}, \cdot\cdot\cdot, \z_{s_l})}(\z_T)],\\
         \hat{\z}_T^{(h)*}(\z_{s_1}, \cdot\cdot\cdot, \z_{s_l}) &= \mathbb{E}_{p(\z_T|\z_{s_1}, \cdot\cdot\cdot, \z_{s_l})}[\z_T],
    \end{split}
\end{equation}

We now show that the reconstruction error can never increase with more previous frames as inputs by observing that with $T>s_1>\cdot\cdot\cdot>s_k>\cdot\cdot\cdot>s_l$,
\begin{equation}
    \begin{split}
        \min_{\hat{\z}_T^{(h)}} &\mathbb{E}_{p(\z_T, \z_{s_1}, \cdot\cdot\cdot, \z_{s_l})}[||\z_T-\hat{\z}_T^{(h)}||^2_2] = \mathbb{E}_{p(\z_{s_1}, \cdot\cdot\cdot, \z_{s_l})}[\text{Var}_{p(\z_T| \z_{s_1}, \cdot\cdot\cdot, \z_{s_l})}(\z_T)]\\
        &\leq 
        \mathbb{E}_{p(\z_{s_1}, \cdot\cdot\cdot, \z_{s_k})}[\text{Var}_{p(\z_T| \z_{s_1}, \cdot\cdot\cdot, \z_{s_k})}(\z_T)]=\min_{\hat{\z}_T^{(g)}} \mathbb{E}_{p(\z_T, \z_{s_1}, \cdot\cdot\cdot, \z_{s_k})}[||\z_T-\hat{\z}_T^{(g)}||^2_2]. 
    \end{split}
\end{equation}
Indeed,
\begin{equation}
    \begin{split}
        &\quad \mathbb{E}_{p(\z_{s_1}, \cdot\cdot\cdot, \z_{s_l})}[\text{Var}_{p(\z_T| \z_{s_1}, \cdot\cdot\cdot, \z_{s_l})}(\z_T)] - 
        \mathbb{E}_{p(\z_{s_1}, \cdot\cdot\cdot, \z_{s_k})}[\text{Var}_{p(\z_T| \z_{s_1}, \cdot\cdot\cdot, \z_{s_k})}(\z_T)]\\
        &= \mathbb{E}_{p(\z_{s_1}, \cdot\cdot\cdot, \z_{s_k})}\{\mathbb{E}_{p(\z_{s_{k+1}}, \cdot\cdot\cdot, \z_{s_l} |\z_{s_1}, \cdot\cdot\cdot, \z_{s_k})}[\text{Var}_{p(\z_T| \z_{s_1}, \cdot\cdot\cdot, \z_{s_l} )}(\z_T)] - \text{Var}_{p(\z_T| \z_{s_1}, \cdot\cdot\cdot, \z_{s_k})} (\z_T)\}\\
        &=  -\mathbb{E}_{p(\z_{s_1}, \cdot\cdot\cdot, \z_{s_k})}\{\text{Var}_{p(\z_{s_{k+1}}\cdot\cdot\cdot,\z_{s_l}|\z_{s_1}, \cdot\cdot\cdot, \z_{s_k})}(\mathbb{E}_{p(\z_T|\z_{s_1}, \cdot\cdot\cdot, \z_{s_l})}[\z_T])\} \quad\quad \text{(Law of total variance)}\\
        &\leq 0.
    \end{split}
\end{equation}

Notice that if $\z_T|\z_{s_1},\cdot\cdot\cdot, \z_{s_k} \,{\buildrel d \over =}\,  \z_T|\z_{s_1},\cdot\cdot\cdot, \z_{s_l}$, then the above difference will become 0:
\begin{equation}
    \begin{split}
        &\quad \text{Var}_{p(\z_{s_{k+1}}\cdot\cdot\cdot,\z_{s_l}|\z_{s_1}, \cdot\cdot\cdot, \z_{s_k})}(\mathbb{E}_{p(\z_T|\z_{s_1}, \cdot\cdot\cdot, \z_{s_l})}[\z_T]) \\
    &=\text{Var}_{p(\z_{s_{k+1}}\cdot\cdot\cdot,\z_{s_l}|\z_{s_1}, \cdot\cdot\cdot, \z_{s_k})}(\mathbb{E}_{p(\z_T|\z_{s_1}, \cdot\cdot\cdot, \z_{s_k})}[\z_T]) \\
        &=0,
    \end{split}
\end{equation}
since $\mathbb{E}_{p(\z_T|\z_{s_1}, \cdot\cdot\cdot, \z_{s_k})}[\z_T]$ is a function of $\z_{s_1}, \cdot\cdot\cdot, \z_{s_k}$ only.

Therefore, for strict inequality, it is necessary to avoid $\ConditionallyIndependent{\z_T} {\z_{s_{k+1}}, \cdot\cdot\cdot, \z_{s_l}}  {\z_{s_1}, \cdot\cdot\cdot, \z_{s_k}}$, which includes first-order Markov chain ($\z_T \rightarrow \cdot\cdot\cdot \z_{s_k} \rightarrow \z_{s_l}$) as a special case.

\end{proof}

\section{Hyperparameters}
\label{appendix:hyper_pseudo_vid}
During our experiments, we first choose the largest number of frames to be the largest one that we can train on our GPU (an NVIDIA RTX A6000) and they are 18 and 8 for CViViT and Video Diffusion, respectively. Then we consider including the results with the number of frames approximately being half of the maximal one (8 and 4 for CViViT and Video Diffusion, respectively). The hyperparameters for our experiments are shown in the tables below.

\begin{table}[htbp]
    \centering
    \caption{Hyperparamters used for C-ViViT architecture and optimizer.}
    \begin{tabular}{c c c c}
    \hline
         & 1-frame & 8-frame & 18-frame \\
    \hline
      Number of spatial layers & 8 & 4 & 4  \\
      Number of temporal layers & - & 4 & 4 \\
      Embedding dimension & 512 & 512 & 512\\
      Hidden dimension & 512 & 512 & 512\\
      Number of heads & 8 & 8 & 8\\ 
      Learning rate & 1e-4 & 1e-4 & 1e-4 \\
      Learning rate scheduler & Cosine decay& Cosine decay& Cosine decay\\
      Number of training steps & 100k&  100k& 100k \\
      Batch size & 64 & 64 & 64\\
    \hline
    \end{tabular}
\end{table}

\begin{table}[htbp]
    \centering
    \caption{Hyperparamters used for VideoGPT architecture and optimizer.}
    \begin{tabular}{c c c c}
    \hline
         & 1-frame & 8-frame & 18-frame \\
    \hline
      Number of layers & 8 & 8 & 8  \\
      Embedding dimension & 144 & 144 & 144\\
      Hidden dimension & 144 & 144 & 144\\
      Number of heads & 4 & 4 & 4\\ 
      Learning rate & 1e-4 & 1e-4 & 1e-4 \\
      Learning rate scheduler & Cosine decay& Cosine decay& Cosine decay\\
      Number of training steps & 100k&  100k& 100k \\
      Batch size & 64 & 64 & 64\\
    \hline
    \end{tabular}
\end{table}

\begin{table}[htbp]
    \centering
    \caption{Hyperparamters used for Phenaki architecture and optimizer.}
    \begin{tabular}{c c c c}
    \hline
         & 1-frame & 8-frame & 18-frame \\
    \hline
      Number of layers & 6 & 6 & 6  \\
      Embedding dimension & 512 & 512 & 512\\
      Hidden dimension & 512 & 512 & 512\\
      Number of heads & 8 & 8& 8\\ 
      Learning rate & 1e-4 & 1e-4 & 1e-4 \\
      Learning rate scheduler & Cosine decay& Cosine decay& Cosine decay\\
      Number of training steps & 200k&  200k& 200k \\
      Batch size & 64 & 64 & 64\\
    \hline
    \end{tabular}
\end{table}

\begin{table}[htbp]
    \centering
    \caption{Hyperparamters used for UNet architecture and optimizer for Video diffusion.}
    \begin{tabular}{c c c c}
    \hline
         & 1-frame & 4-frame & 8-frame \\
    \hline
      Number of downsampling/upsampling layers & 4 & 4 & 4\\
      Number of residual blocks & 2 & 2 & 2  \\
      Base channel size & 128 & 128 & 128\\
      Number of diffusion steps per generating a frame & 1000 & 1000 & 1000\\
      Learning rate & 1e-4 & 1e-4 & 1e-4 \\
      Number of training steps & 500k&  500k& 500k \\
      Batch size & 32 & 32 & 32\\
    \hline
    \end{tabular}
\end{table}
\end{document}